\documentclass{article}

\usepackage{microtype}
\usepackage{graphicx}
\usepackage{subcaption}
\usepackage{booktabs} 

\usepackage{hyperref}

\usepackage{bbm}
\usepackage{enumitem}
\usepackage{comment}

\usepackage[preprint]{icml2026}

\usepackage{amsmath}
\allowdisplaybreaks
\usepackage{amssymb}
\usepackage{mathtools}
\usepackage{amsthm}

\usepackage[capitalize,noabbrev]{cleveref}

\theoremstyle{plain}
\newtheorem{theorem}{Theorem}[section]

\newtheorem{condition}[theorem]{Condition}
\newtheorem{lemma}[theorem]{Lemma}

\theoremstyle{definition}
\newtheorem{definition}[theorem]{Definition}

\theoremstyle{remark}

\newtheorem{claim}{Claim}

\usepackage[textsize=tiny]{todonotes}

\icmltitlerunning{Benign Overfitting in Adversarial Training for Vision Transformers}

\begin{document}

\twocolumn[
  \icmltitle{Benign Overfitting in Adversarial Training for Vision Transformers}



  \icmlsetsymbol{equal}{*}

  \begin{icmlauthorlist}
    \icmlauthor{Jiaming Zhang}{yy2,yyy}
    \icmlauthor{Meng Ding}{yy3}
    \icmlauthor{Shaopeng Fu}{yy2}
    \icmlauthor{Jingfeng Zhang}{yy4}
    \icmlauthor{Di Wang}{yy2}
  \end{icmlauthorlist}

  \icmlaffiliation{yy2}{Division of CEMSE, King Abdullah University of Science and Technology}
  \icmlaffiliation{yyy}{School of Statistics, Renmin University of China}
  \icmlaffiliation{yy3}{Department of Computer Science and Engineering, State University of New York at Buffalo}
\icmlaffiliation{yy4}{School of Computer Science, University of Auckland} 
  \icmlcorrespondingauthor{Di Wang}{di.wang@kaust.edu.sa}

  \icmlkeywords{Machine Learning, ICML}

  \vskip 0.3in
]



\printAffiliationsAndNotice{}  

\begin{abstract}
Despite the remarkable success of Vision Transformers (ViTs) across a wide range of vision tasks, recent studies have revealed that they remain vulnerable to adversarial examples, much like Convolutional Neural Networks (CNNs). A common empirical defense strategy is adversarial training, yet the theoretical underpinnings of its robustness in ViTs remain largely unexplored. In this work, we present the first theoretical analysis of adversarial training under simplified ViT architectures. We show that, when trained under a signal-to-noise ratio that satisfies a certain condition and within a moderate perturbation budget, adversarial training enables ViTs to achieve nearly zero robust training loss and robust generalization error under certain regimes. Remarkably, this leads to strong generalization even in the presence of overfitting, a phenomenon known as \emph{benign overfitting}, previously only observed in CNNs (with adversarial training). Experiments on both synthetic and real-world datasets further validate our theoretical findings.
\end{abstract}

\section{Introduction}
Vision Transformers (ViTs) have emerged as a powerful alternative to convolutional neural networks (CNNs) \citep{krizhevsky2012imagenet} for a wide range of computer vision tasks, including image classification, object detection, semantic segmentation, and vision-language modeling~\citep{dosovitskiy2020image,liu2021swin,hu2025stable}. Despite their strong performance, recent studies \citep{hu2023seat,shao2021adversarial,hu2025towards,hu2024improving} have revealed that ViTs, similar to CNNs, can still be vulnerable to small, carefully crafted perturbations. These perturbations, referring to adversarial examples \citep{szegedy2013intriguing}, can often cause significant performance degradation. 
 
A widely studied defense mechanism against such vulnerabilities is \emph{adversarial training} \citep{goodfellow2014explaining,goodfellow2016deep}, which augments the training process with adversarially perturbed samples to improve model robustness. While adversarial training has proven effective in enhancing model robustness, it is frequently accompanied by a noticeable degradation in generalization performance on clean data \citep{raghunathan2019adversarial,xu2023llm}. Such a robustness-generalization trade-off \citep{fu2023theoretical,zhang2021towards,xiao2022stability} has been extensively studied, raising the question of whether it is possible to preserve robustness without sacrificing clean-data accuracy.  

\emph{Benign Overfitting} \citep{bartlett2020benign} refers to the phenomenon that overparameterized models can interpolate the training data (achieving near-zero empirical risk) yet still generalize well. In the standard (non-adversarial) setting, this behavior has been studied across various model architectures, including linear regression \citep{bartlett2020benign,zhou2023implicit}, logistic regression \cite{wang2021benign,caogu2021risk}, ridge regression \citep{tsigler2023ridge}, kernel methods \citep{mei2022generalization}, and  neural networks \citep{li2021towards,cao2022benign,kou2023benign,frei2023benign,jiang2024unveil,frei2024trained}. In the adversarial setting, \cite{chen2023benign} provides the first theoretical analysis that benign overfitting can arise in adversarial training for linear regression under sub-Gaussian mixture models. Subsequently, \cite{wang2024benign} shows that adversarial training can still generalize well in the presence of inference-time attacks for two-layer neural networks under appropriate distributional assumptions.  
However, to the best of our knowledge, it remains unclear whether analogous behavior occurs in more advanced architectures such as ViTs. 

Compared with linear and two-layer models, analyzing ViTs poses distinctive challenges for analyzing robust benign overfitting. Unlike CNNs or two-layer neural networks with activation functions followed by a linear model, ViTs incorporate attention heads with query, key, and value projection matrices, leading to substantially more complex training dynamics. The effect of perturbations on attention heads differs significantly from their effect on linear layers or activation-plus-linear structures, and varying perturbation magnitudes can markedly influence the training dynamics of ViTs (A more detailed discussion is provided in Section~\ref{sec:main_results}). Thus, the benign overfitting behavior or its conditions for ViTs may be significantly different from previous models. 

To fill the gap, in this work, we provide the first comprehensive theoretical analysis of robust benign overfitting for a simplified  ViT model. Specifically,
\begin{enumerate}
    \item We validate that \emph{benign overfitting} can also arise in adversarially trained Vision Transformers when the signal-to-noise ratio and the perturbation magnitude satisfy certain conditions, similar to linear and two-layer neural network models.
    {That is, the adversarially trained interpolator attains near-zero robust training loss while maintaining small robust test error.} 
    \item By analyzing the adversarial training dynamics of ViTs, we identify three key regimes: 1) Small perturbations yield trajectories close to clean training; 2) Moderate perturbations cause the attention mechanism to fail, such that the ViT collapses to a linear model; 3) Large perturbations lead to significant generalization error beyond benign overfitting. In all cases, we provide explicit upper bounds on the clean and robust test error. 
    \item We empirically verify our theoretical findings on both synthetic data and real-world data (MNIST, CIFAR-10, Tiny-ImageNet), demonstrating agreement between the derived bound and observed conditions for the occurrence of benign overfitting.
\end{enumerate}

\section{Related Work}
\textbf{Benign Overfitting.}
Benign overfitting refers to the phenomenon where a predictor perfectly fits (interpolates) noisy training data yet still achieves strong generalization performance on unseen data.
\cite{bartlett2020benign} analyzed benign overfitting in linear regression with Gaussian noise, showing that in high-dimensional (overparameterized) regimes, the excess risk of interpolation can be asymptotically optimal.
Foundational studies have further established benign overfitting in various linear settings, including regression, sparse regression, logistic regression, ridge regression, and kernel methods \citep{belkin2018understand,bartlett2020benign,hastie2022surprises,ding2024revisiting,ding2024understanding}.
In neural networks, \cite{frei2023benign} studied benign overfitting in two-layer networks without relying on the lazy training assumption in finite-width regimes. \cite{cao2022benign,kou2023benign} examined benign overfitting in convolutional networks from a feature learning perspective. Recently, \cite{magen2024benign,jiang2024unveil,sakamoto2024benign} investigated learning dynamics in Vision Transformers (ViTs), delineating the boundary between benign and harmful overfitting. \cite{frei2024trained} also analyzed benign overfitting in trained Transformer classifiers within in-context learning setups.

\textbf{Benign Overfitting for Adversarial Training.} 
\cite{chen2023benign} initiated the study of benign overfitting under adversarial training in the context of linear classifiers with sub-Gaussian mixture data, proving that under moderate perturbations, linear classifiers can achieve near-optimal standard and adversarial risks. Building on this line of work, \cite{wang2024benign} extended the analysis to two-layer networks and demonstrated that, regardless of whether the activation function is smooth or non-smooth, adversarial training can achieve near-optimal robust generalization error. In contrast, \cite{hao2024surprising} showed that under non-negligible noise, linear regression and NTK regression models tend to overfit the training data, yielding estimators with inflated Lipschitz norms and consequently elevated adversarial risk. However, none of them considered and their conclusions do not hold for transformer architectures. 

\textbf{Adversarial Robustness in Transformer.}
Adversarial training has been widely used to improve Transformer robustness. For Vision Transformers (ViTs),~\citet{herrmann2022pyramid} proposed PyramidAT, combining consistent dropout and stochastic depth to alleviate performance degradation, while ~\citet{wu2022towards} reduced computational cost via attention-guided dropping of patch embeddings. ~\citet{gopal2025safer} introduced SAFER, a layer-selective fine-tuning method with sharpness-aware minimization to mitigate adversarial overfitting, and ~\citet{islam2025mechanistic} analyzed layer-wise perturbation propagation, proposing a neuron-level suppression mechanism. However, these works are purely empirical and lack theoretical guarantees. Recent studies have begun to analyze robustness from a theoretical perspective in linear Transformer in-context learning, showing that robustness can be enhanced through adversarial training against hijacking attacks ~\citep{anwarunderstanding}, short-suffix training to defend long-suffix jailbreaks ~\citep{fu2025short}, and multi-task adversarial pretraining without downstream AT ~\citep{kumano2025adversarially}. Yet, these analyses are restricted to the in-context learning setting with linear Transformers, and theoretical guarantees for adversarial robustness in general Transformer architectures remain an open problem.

\section{Problem Setup}
In this section, we introduce the necessary definitions and formally describe the Gradient Descent-based Adversarial Training under the multi-patch data distribution and the two-layer Vision Transformer model. 

\textbf{Notations.} For two sequences $\{x_n\}$ and $\{y_n\}$, we write $x_n = O(y_n)$ if there exist absolute constants $C > 0$ and $N > 0$ such that $|x_n| \leq C|y_n|$ for all $n \geq N$. Similarly, we write $x_n = \Omega(y_n)$ if $y_n = O(x_n)$. We say $x_n = \Theta(y_n)$ if both $x_n = O(y_n)$ and $x_n = \Omega(y_n)$. Moreover, $x_n = o(y_n)$ if $\lim_{n \to \infty} |x_n / y_n| = 0$. Finally, we use $\widetilde{O}(\cdot)$, $\widetilde{\Omega}(\cdot)$, and $\widetilde{\Theta}(\cdot)$ to denote the corresponding notations with logarithmic factors suppressed.

\begin{definition}[Data Generation Model]\label{def:data_gen}
 Let $\boldsymbol{\boldsymbol{\mu}}_+, \boldsymbol{\boldsymbol{\mu}}_- \in \mathbb{R}^d$ be fixed vectors representing the signals contained in data points, where $\|\boldsymbol{\boldsymbol{\mu}}_+\|_2 = \|\boldsymbol{\boldsymbol{\mu}}_-\|_2 = \|\boldsymbol{\boldsymbol{\mu}}\|_2$ and $\langle \boldsymbol{\boldsymbol{\mu}}_+, \boldsymbol{\boldsymbol{\mu}}_- \rangle = 0$. Then each data point $(\mathbf{X}, y)$ with $\mathbf{X} = (\mathbf{x}_1, \mathbf{x}_2, \ldots, \mathbf{x}_M)^\top \in \mathbb{R}^{M \times d}$ and $y \in \{-1, 1\}$ is generated from the following distribution $D$:
\begin{enumerate}[leftmargin=0.3in, label = {(\arabic*)}, itemsep=-0.05in, topsep=-0.05in]
    \item The label $y$ is generated as a Rademacher random variable, i.e., $ \mathbb{P}(y = 1) = \mathbb{P}(y = -1) = \tfrac{1}{2}.$ 
    \item If $y = 1$ then $\mathbf{x}_1$ is given as $\boldsymbol{\boldsymbol{\mu}}_+$, if $y = -1$ then $\mathbf{x}_1$ is given as $\boldsymbol{\boldsymbol{\mu}}_-$, which represents signals.
    \item $\mathbf{x}_2, \ldots, \mathbf{x}_M$ are given by noise vectors $\boldsymbol{\boldsymbol{\xi}}_2, \ldots, \boldsymbol{\boldsymbol{\xi}}_M$, generated i.i.d from the Gaussian distribution $\mathcal{N}(0, \sigma_p^2 \cdot (\mathbf{I} - \boldsymbol{\boldsymbol{\mu}}_+ \boldsymbol{\boldsymbol{\mu}}_+^\top \cdot \|\boldsymbol{\boldsymbol{\mu}}\|_2^{-2} - \boldsymbol{\boldsymbol{\mu}}_- \boldsymbol{\boldsymbol{\mu}}_-^\top \cdot \|\boldsymbol{\boldsymbol{\mu}}\|_2^{-2})$, which represent noises.
\end{enumerate}
\end{definition}
Our data generation model is motivated by the patch-level structure of real image data, where some patches encode class-relevant signals (e.g. semantic information) while others capture irrelevant noise (e.g. background artifacts).
Similar constructions have been widely employed in the feature learning literature to analyze the generalization behavior of overparameterized classifiers \citep{allen2020towards,cao2022benign,jelassi2022towards,kou2023benign,zou2023benefits,jiang2024unveil,han2024feature,ding2025understanding}. In our setup, the noise component $\boldsymbol{\boldsymbol{\xi}}$ is modeled as a Gaussian variable, with its covariance structure designed to remain orthogonal to the signal component $\boldsymbol{\boldsymbol{\mu}}$, ensuring that the data noise is independent of and unrelated to the feature.

\textbf{Two-layer Transformer.} Following the architecture introduced by \citet{jiang2024unveil}, we consider a simplified two-layer Transformer consisting of a self-attention layer followed by a fixed linear layer, defined as the following, {where $\theta=(\mathbf{W}_Q, \mathbf{W}_K, \mathbf{W}_V)$}. 
\begin{equation}\label{eq:2trans}
f(\mathbf{X}, \theta) = \frac{1}{M} \sum_{l=1}^{M} \varphi(\mathbf{x}_l^\top \mathbf{W}_Q \mathbf{W}_K^\top \mathbf{X}^\top) \mathbf{X} \mathbf{W}_V \boldsymbol{w}_O.
\end{equation}
Here, $\varphi(\cdot) : \mathbb{R}^M \rightarrow \mathbb{R}^M$ denotes the softmax function; $\mathbf{W}_Q, \mathbf{W}_K \in \mathbb{R}^{d \times d_h}$ and $\mathbf{W}_V \in \mathbb{R}^{d \times d_v}$ represent the query, key, and value matrices, respectively; and $\boldsymbol{w}_O \in \mathbb{R}^{d_v}$ represents the weight vector of the linear layer. We use $\theta$ to denote the collection of all the model weights. This model is not reduced to a linear or single-layer attention architecture \citep{magen2024benign,sakamoto2024benign}; instead, it more closely resembles the structure of a real Transformer, with a correspondingly more complex parameter update process. Furthermore, our architecture readily extends to the multi-head attention mechanism, as detailed in Appendix~\ref{sec:mha}. Finally, the Feed-Forward Network (FFN) is simplified and subsumed into the value matrix $\mathbf{W}_V$.

\textbf{Loss Function.}
Let $S = \{(\mathbf{X}_n, y_n)\}_{n=1}^N$ denote the training dataset drawn from the distribution $D$ defined in Definition~\ref{def:data_gen}, where $n$ indexes the samples (so $(\mathbf{X}_n, y_n)$ represents the $n$-th sample). In this work, we adopt the empirical cross-entropy loss as a surrogate for the non-differentiable $0/1$ loss, and train the two-layer Transformer by minimizing this loss:
\begin{equation*}
L_S(\theta) = \frac{1}{N} \sum_{n=1}^{N} \ell(y_n f(\mathbf{X}_n, \theta)),
\end{equation*}
where $\ell(z) = \log(1 + \exp(-z))$ and $f(\mathbf{X}, \theta)$ is the two-layer Transformer. 
We measure the generalization ability of the two-layer Transformer using the \emph{test error}, defined as the expected $0/1$ loss over the data distribution $D$:
\begin{equation*}
    L_D(\theta) =\mathbb{E}_{(\mathbf{x},\mathbf{y})\sim D}\mathbbm{1} \left( y f(\mathbf{X}, \theta) \leq 0 \right).
\end{equation*}

\textbf{Robust Loss.}
We consider $\ell_p$-norm bounded $(p\geq2)$ adversarial perturbations\footnote{Our theory covers all norm attacks discussed in Appendix~\ref{sec:mul-norm}; for simplicity, we only report the $\ell_2$-norm setting.} applied to each component of the input sequence $\mathbf{X} = [\mathbf{x}_1, \dots, \mathbf{x}_M] \in \mathcal{X}^M$, where each $\mathbf{x}_m \in \mathbb{R}^d$ denotes a token (or patch) embedding. 
For a perturbation budget $\tau > 0$, the admissible perturbation set is
\(
B(\mathbf{X}, \tau) := \big\{\, \widetilde{\mathbf{X}} = [\widetilde{\mathbf{x}}_1, \dots, \widetilde{\mathbf{x}}_M] \ \big|\ \|\widetilde{\mathbf{x}}_m - \mathbf{x}_m\|_p \le \tau,\ \forall m \in [M] \,\big\}.
\)
Under this threat model, the robust $0/1$ loss is defined as
\(
\ell_{\mathrm{rob}}^{0/1}(y f(\mathbf{X}, \theta)) 
:= \max_{\widetilde{\mathbf{X}} \in B(\mathbf{X}, \tau)} \mathbbm{1}\big( y\, f(\widetilde{\mathbf{X}}, \theta) \le 0 \big)
\), and the robust loss is defined as \(
\ell_{\mathrm{rob}}(y_n f(\mathbf{X}_n, \theta)) 
:= \max_{\widetilde{\mathbf{X}}_n \in B(\mathbf{X}_n, \tau)} \ell\big( y_n f(\widetilde{\mathbf{X}}_n, \theta)\big).
\)
The robust test error and robust test loss are:
\[
L_D^{\mathrm{rob}}(\theta) 
:= \mathbb{E}_{(\mathbf{X},y) \sim \mathcal{D}} \ \ell_{\mathrm{rob}}^{0/1}(y f(\mathbf{X}, \theta))\]\[\quad L_S^{\mathrm{rob}}(\theta):=\frac{1}{N} \sum_{n=1}^{N} \ell_{\mathrm{rob}}(y_n f(\mathbf{X}_n, \theta))
\]
\textbf{Adversarial Training.} 
We adopt a Gradient Descent-based Adversarial Training algorithm to update the network parameters, as summarized in Algorithm~\ref{alg:gd-adversarial}.  Note that we initialize the network weights $\mathbf{W}_Q$, $\mathbf{W}_K$, and $\mathbf{W}_V$ with Gaussian distributions, where each entry of $\mathbf{W}_Q$ and $\mathbf{W}_K$ is drawn from $\mathcal{N}(0, \sigma_h^2)$, and each entry of $\mathbf{W}_V$ is drawn from $\mathcal{N}(0, \sigma_V^2)$. The algorithm iteratively constructs adversarial training examples by maximizing the training loss with respect to the input and updates the model parameters based on them. In our setting, only the attention projection matrices \(\mathbf{W}_Q, \mathbf{W}_K, \mathbf{W}_V\) are updated during training, while the output vector 
\(\boldsymbol{w}_O\) remains fixed.

\begin{algorithm}[t]
   \caption{Gradient Descent-based Adversarial Training}
   \label{alg:gd-adversarial}
\begin{algorithmic}[1]
   \STATE {\bfseries Input:} Learning rate $\eta$, perturbation budget per token $\tau$, iterations $T$, init variance $\sigma_V, \sigma_h$
   \STATE {\bfseries Initialize:} $(\mathbf{W}_Q^0, \mathbf{W}_K^0)_{ij} \sim \mathcal{N}(0, \sigma_h^2)$, $(\mathbf{W}_V^0)_{ij} \sim \mathcal{N}(0, \sigma_V^2)$ i.i.d. and $\theta^0 \gets (\mathbf{W}_Q^0, \mathbf{W}_K^0, \mathbf{W}_V^0)$
   
   \FOR{$t = 0$ {\bfseries to} $T-1$}
      \STATE {\small // Phase 1: Generate Adversarial Examples}
      \FOR{$n = 1$ {\bfseries to} $N$}
         \STATE $\widetilde{\mathbf{X}}_n^t \gets \arg\max_{\widetilde{\mathbf{X}} \in B(\mathbf{X}_n, \tau)} \ell\big(y_n f(\widetilde{\mathbf{X}}; \theta^t\big)$
      \ENDFOR
      
      \STATE {\small // Phase 2: Simultaneous Weight Update}
      \STATE $\theta^t \gets (\mathbf{W}_Q^t, \mathbf{W}_K^t, \mathbf{W}_V^t)$
      \STATE $\nabla \mathcal{L}(\theta^t) \gets \frac{1}{N} \sum_{n=1}^N \nabla_{\theta} \ell\big(y_n f(\widetilde{\mathbf{X}}_n^t; \theta^t)\big)$
      \STATE $\Theta^{t+1} \gets \Theta^t - \eta \nabla \mathcal{L}(\Theta^t)$
   \ENDFOR
   \STATE {\bfseries Output:} Final weights $\theta^T = (\mathbf{W}_Q^T, \mathbf{W}_K^T, \mathbf{W}_V^T)$
\end{algorithmic}
\end{algorithm}
\section{Main Results}
\label{sec:main_results}

In this section, we present our main theoretical results on the convergence and generalization of the ViT model, demonstrating how the signal-to-noise ratio $\operatorname{SNR} = \|\boldsymbol{\boldsymbol{\mu}}\|_2 / (\sigma_p \sqrt{d})$ and the sample size $N$ influence its adversarial training dynamics. We first introduce the following conditions.

\begin{condition}
\label{ass:1}
Given a sufficiently small failure probability $\delta > 0$ and a target training loss $\epsilon > 0$, suppose that:
\begin{enumerate}[leftmargin=0.3in, label = {(\arabic*)}, itemsep=-0.05in, topsep=-0.05in ]
    \item \label{con:dim} The dimension $d$ and $d_h$ are sufficiently large satisfying $d = \widetilde{\Omega} \left( \epsilon^{-2} N^2 d_h \right)$ and $d_h = \widetilde{\Omega} \left( \max\{\text{SNR}^4, \text{SNR}^{-4}\} \right) N^2 \epsilon^{-2}$. 
    \item \label{con:samplesize} The training sample size $N$ is large enough such that $N = \Omega(\text{poly}\log(d))$.
    \item \label{con:num_token} The number of input tokens is bounded as $M = \Theta(1)$, and the $\ell_2$-norm of linear layer weights satisfies $\|\boldsymbol{w}_O\|_2 = \Theta(1)$. 
    \item \label{con:lr} The learning rate $\eta$ is chosen sufficiently small so that $\eta \lesssim \widetilde{\mathcal{O}}(\min\{\|\boldsymbol{\boldsymbol{\mu}\|}_2^{-2}, (\sigma_p^2 d)^{-1}\} \cdot d_h^{-\frac{1}{2}})$. 
    \item \label{con:initialization} The Gaussian initialization  is appropriately chosen such that the standard deviation $\sigma_V$ satisfies {$\sigma_V \leq \widetilde{\mathcal{O}}(||\boldsymbol{w}_O||_2^{-1} \cdot \min\{\|\boldsymbol{\boldsymbol{\mu}}||_2^{-1}, (\sigma_p \sqrt{d})^{-1}\} \cdot d_h^{-\frac{1}{4}}d^{-\frac{1}{2}})$ ,and the variance $\sigma_h^2$ satisfies $\min\{\|\boldsymbol{\boldsymbol{\mu}}||_2^{-2}, (\sigma_p^2 d)^{-1}\} \cdot d_h^{-\frac{1}{2}} \cdot (\log(6N^2 M^2 / \delta))^{-2} \leq \sigma_h^2 \leq \min\{\|\boldsymbol{\boldsymbol{\mu}}||_2^{-2}, (\sigma_p^2 d)^{-1}\} \cdot d_h^{-\frac{1}{2}} \cdot (\log(6N^2 M^2 / \delta))^{-\frac{3}{2}}$}.
    \item \label{con:training_loss} The target training loss is satisfying $\epsilon \leq O(1 / \text{poly}\log(d))$.
\end{enumerate}
\end{condition}
Conditions~\ref{con:dim} and~\ref{con:samplesize} ensure that the learning problem is set in a sufficiently over-parameterized regime, allowing the model to fully capture the feature signal described in Definition~\ref{def:data_gen}. Condition~\ref{con:num_token} guarantees that each class contains enough samples 
with high probability. Conditions~\ref{con:lr} and~\ref{con:initialization} simplify the analysis, though they can be generalized to the settings $M = \Omega(1)$, $\|\boldsymbol{w}_O\|_2 = o(1)$, or $\|\boldsymbol{w}_O\|_2 = \omega(1)$. Together, Conditions~\ref{con:lr} and~\ref{con:initialization} further ensure that the Transformer can be effectively trained. Finally, Condition~\ref{con:training_loss} ensures that the Transformer sufficiently overfits the training data. Similar conditions are widely made in the theoretical analysis of benign overfitting in neural networks~\citep{allen2020towards,cao2022benign,frei2022benign,jelassi2022towards,chatterji2023deep,zou2023benefits,kou2023benign,frei2024trained,jiang2024unveil}. 

\begin{theorem}[Benign Overfitting under Adversarial Training]
\label{thm:main_thm}
Under Condition~\ref{ass:1}, we distinguish two cases:

\textbf{Case 1.} If we have $N \cdot \operatorname{SNR}^2 = \Omega(1)$ and $\tau \leq O(\tfrac{\|\boldsymbol{\mu}\|_2}{\log d_h})$, ViT’s attention head is effectively trainable
. In this case, let $T = \Theta(\eta^{-1} \epsilon^{-1} \|\boldsymbol{\mu}\|_2^{-2} \|\boldsymbol{w}_O\|_2^{-2})$.  

\textbf{Case 2.} If we have $N \cdot \operatorname{SNR}^2 = \Omega(\frac{1}{\epsilon})$ and $\omega(\tfrac{\|\boldsymbol{\mu}\|_2}{\log d_h})\leq \tau\leq O(\|\boldsymbol{\mu}\|_2)$ 
, the attention head parameters barely update, causing the attention weights to remain nearly uniform  and the ViT degenerates into a linear model 
. In this case, let $T =M\cdot \Theta(\eta^{-1} \epsilon^{-1} \|\boldsymbol{\mu}\|_2^{-2} \|\boldsymbol{w}_O\|_2^{-2})$.

In both cases, with probability at least $1 - d^{-1}$, the following holds:
\begin{enumerate}[leftmargin=0.2in, label = {\arabic*.}, itemsep=-0.05in, topsep=-0.05in ]
    \item The robust training loss converges to $\epsilon$: $$L_S^{\mathrm{rob}}(\theta(T)) \leq \epsilon.$$
    \item The clean test {error} satisfies: $$L_D(\theta(T))  \leq \exp\left(-C\cdot d \operatorname{SNR}^{2}\right).$$
    \item The robust test {error} satisfies: {$$L_D^{\mathrm{rob}}(\theta(T)) \leq  \exp\left(-C\cdot d \operatorname{SNR}^{2}(1-\frac{\tau}{\|\boldsymbol{\mu}\|_2})^2\right).$$}
\end{enumerate}
\end{theorem}

Theorem~\ref{thm:main_thm} demonstrates that , under the assumption on $\tau$, the model attains adversarial robustness, while the SNR condition guarantees that it prioritizes the signal over the noise, thereby leading to benign overfitting.
The adversarially-trained two-layer Transformer model exhibits three key observations. 1) The model fits the training data well, with training loss converging to $\epsilon$. 2) For generalization guarantee, the clean test error decays rapidly with increasing
$\operatorname{SNR}$, and obviously less than $\epsilon$ by Condition~\ref{con:samplesize}. This is consistent with classical benign overfitting phenomena established in prior work \citep{jiang2024unveil}. 3) For robustness guarantees, Theorem~\ref{thm:main_thm} provides an explicit upper bound on the robust test error as a function of the perturbation radius $\tau$ and the signal strength $\|\boldsymbol{\mu}\|_2$, showing that the bound increases with increasing $\tau$, which is consist with prior empirical observations\citep{madry2017towards,schmidt2018adversarially}. 

\textbf{The impact of perturbation $\tau$ radius.} Here, we give more discussion on the perturbation radius $\tau$ to illustrate how it affects the different dynamics of ViTs, leading to the various results stated in Theorem~\ref{thm:main_thm}. First, it can be observed that the softmax structure in attention is highly sensitive to perturbations according to Lemma~\ref{lemma:rsoftmax1}. When $\tau \leq O(\tfrac{\|\boldsymbol{\mu}\|_2}{\log d_h})$, adversarial perturbations do not dominate the learning dynamics, and the adversarial training trajectory of ViT remains close to standard clean training, corresponding to Case 1 in Theorem~\ref{thm:main_thm}. In contrast, when $\omega(\tfrac{\|\boldsymbol{\mu}\|_2}{\log d_h})\leq \tau\leq O(\|\boldsymbol{\mu}\|_2)$, the updates between signal and noise in attention are effectively canceled out by the perturbation, forcing attention weights to remain close to their initialized uniform distribution, under which ViT degenerates into a linear model, corresponding to Case 2 in Theorem~\ref{thm:main_thm}. The relevant lemmas will be provided in Section~\ref{sec:ef_soft} later.


\textbf{The difference between ViT and degenerated linear model.} Although both cases can achieve benign overfitting, the attention mechanism in ViTs in Case 1 enables the model to learn the signal more rapidly, leading to faster convergence, and allows it to extract useful information even from sparser signals. According to Theorem~\ref{thm:main_thm}, for the degenerated linear model, the convergence time $T$
is $M$ times slower 
than that of the ViT, and achieving benign overfitting requires a higher signal-to-noise ratio, $N \cdot \operatorname{SNR}^2 = \Omega(1/\epsilon)$, highlighting the advantages of the ViT architecture. 

\textbf{Comparison with prior work.} The most closely related works to ours are \cite{chen2023benign, wang2024benign}, which also investigated the phenomenon of benign overfitting under adversarial training. However, \citet{chen2023benign} focused on linear regression models with moderate perturbations, and \citet{wang2024benign} studied simplified neural networks, leaving more complex architectures unexplored. Our results on ViTs therefore complement this line of research. In addition, even in the more complex setting, our analysis does not rely on the implicit assumption of a large $\|\boldsymbol{\mu}\|_2$, i.e.,$\|\boldsymbol{\mu}\|_2=\Theta(d^r)$ for some $r\in(1/4,1/2]$ in \citep{chen2023benign}.
Moreover, our results require a minimum convergence time of $T \sim O(N\max\{\operatorname{SNR}^2,\operatorname{SNR}^{-4}\})$, which is significantly smaller than the convergence time $T \sim O(\frac{d^2}{\|\boldsymbol{\mu}\|_2^2\epsilon^2})$ reported in prior studies on CNN models by \citep{wang2024benign} in the over-parameterized
regime.

Next, we show that once the perturbation radius $\tau$ exceeds the signal strength $\|\boldsymbol{\mu}\|_2$, no classifier can achieve nontrivial robust accuracy. This implies that excessive adversarial training leads to poor model performance, consistent with the findings of \citet{wang2024benign}.  Combining with Theorem~\ref{thm:main_thm}, we can see that the assumption regarding the relationship between $\tau$  and $\|\boldsymbol{\mu}\|_2$ is essential to understanding benign overfitting. Additionally, our radius for benign overfitting is almost tight. 
\begin{theorem}
\label{coro:larger_pert}
For any given classifier $f(\cdot; \boldsymbol{\theta})$, when $\tau \geq \|\boldsymbol{\mu}\|_2$, the robust test error satisfies $L_D^{\text{rob}}(\theta) \geq 0.25$.
\end{theorem}

\textbf{Practical Guidelines for Adversarial Training on ViTs.} 
Based on our theoretical framework, we provide several ``take-away'' tips for improving adversarial training for ViTs:

\begin{enumerate}[leftmargin=*]
    \item \textbf{Optimal Selection of Perturbation Budget $\tau$:} 
    our theoretical analysis (Theorem~\ref{thm:main_thm}) suggests that to maintain effective adversarial learning, one should choose the perturbation budget $\tau \leq O(\tfrac{\|\boldsymbol{\mu}\|_2}{\log d_h})$. 

    \item \textbf{Balancing Sample Scale $N$ and Signal-to-Noise Ratio:}
    Another key condition for preventing benign overfitting during adversarial training is ensuring that both $N$ and the SNR are large enough. In practice, researchers sometimes employ data augmentation by injecting controlled noise into the dataset. From the perspective of our theoretical results, this approach decreases the $SNR$ while increasing $N$, since noise injection yields ``new'' data points. Because reducing the SNR may be harmful to generalization, it is important to ensure that a sufficiently large number of data points is used to train the model.
    
\end{enumerate}    
\textbf{Discussion on Multi-Norm Attacks.}
For $\ell_p$-norm ($p\geq2$), they can be mapped to $\ell_2$ through standard norm-equivalence $\|\mathbf{x}\|_2 \le \|\mathbf{x}\|_p \le d^{\frac{1}{2} - \frac{1}{p}} \|\mathbf{x}\|_2$. Thus, an $\ell_p$ perturbation budget $\tau$ corresponds to an $\ell_2$ budget scaled by at most $d^{1/2-1/p}$. This implies that all proofs still hold, with the only difference being the perturbation radius $\tau$.

For the $\ell_0$-norm perturbation model, our theoretical lower bounds implicitly show that it cannot provide benign overfitting guarantees. According to Theorem~\ref{coro:larger_pert}, once the $\ell_2$-norm perturbation radius becomes sufficiently large (i.e., $\tau \ge \|\boldsymbol{\mu}\|_2$), the model incurs a large robust test error. This implies that even an $\ell_0$-norm radius of 1 can still lead to substantial robust test error in the worst case.

\section{Proof Sketch}
In this section, we provide a proof sketch of the different adversarial training dynamics due to different perturbation radii. Based on the ViT formulation in Eq.(\ref{eq:2trans}), when the perturbations are relatively small, we show its impact on the learning of the $\mathbf{W}_Q$ and $\mathbf{W}_K$ matrices is limited. If more attention is allocated to signal tokens, the value matrix $\mathbf{W}_V$ tends to align more closely with a perturbed signal vector $\widetilde{\boldsymbol{\mu}}$, while its direction remains dominated by the true signal $\boldsymbol{\mu}$. Consequently, the gradients propagated to the attention heads associated with the signal are larger than those to the noise, forming a positive feedback loop that facilitates better generalization, similar to the learning dynamics observed in the clean training~\citep{jiang2024unveil}.  
When the perturbations are moderate, adversarial training suppresses the learning of token-to-signal attention, keeping the attention weights near their initialization. When $\mathbf{W}_V$ and $\mathbf{W}_K$ are initialized with small Gaussian variance, the attention distribution remains nearly uniform, and the ViT degenerates into a linear model. In this regime, the model requires a larger SNR to align $\mathbf{W}_V {\boldsymbol{w}_O}$ with the signal $\boldsymbol{\mu}$, thereby achieving benign overfitting.
\subsection{Vectorized Q \& K and scalarized V with Time-Independent Directions}\label{sec:5_1}
First, unlike CNNs that treat convolutional kernels as vectors for signal–noise decomposition, our setting requires handling the more complex interactions among the matrices $\mathbf{W}_V$, $\mathbf{W}_K$, and $\mathbf{W}_Q$. Therefore, we consider vectorizing $\mathbf{W}_K$, and $\mathbf{W}_Q$ and scalarizing  $\mathbf{W}_V$. 
To analyze the learning dynamics under iteration-dependent adversarial perturbations 
(as in Algorithm~\ref{alg:gd-adversarial}), we decompose the weight updates by projecting them onto a set 
of reference perturbed directions. Specifically, we fix a universal reference 
$\widetilde{\mathbf{X}} = \big[\widetilde{\boldsymbol{\mu}}, \widetilde{\boldsymbol{\xi}}_{n,2}, \ldots, \widetilde{\boldsymbol{\xi}}_{n,M}\big]$ within the perturbation ball 
$B(\mathbf{X}, \tau)$. Crucially, while the actual adversary $\widetilde{\mathbf{X}}^{(t)}$ 
in Algorithm 1 changes per iteration, the following scalarization/vectorization 
relative to these reference directions allows us to track the growth of signal 
and noise components uniformly across the entire perturbation manifold.

For example, for $\mathbf{W}_V$ we have the following. 

\begin{definition}[Scalarized V]\label{def:2}
Let $\mathbf{W}_V^{(t)}$ denote the V matrix of the ViT at the $t$-th iteration of adversarial training by Algorithm~\ref{alg:gd-adversarial}. 
For the fixed perturbed vectors above, there exist scalars $\gamma_{V,+}^{(t)}$, $\gamma_{V,-}^{(t)}$, and $\rho_{V,n,i}^{(t)}$ such that 
\begin{align*}
\widetilde{\boldsymbol{\mu}}_+^\top \boldsymbol{W}_V^{(t)} \boldsymbol{w}_O &= \widetilde{\boldsymbol{\mu}}_+^\top \boldsymbol{W}_V^{(0)} \boldsymbol{w}_O + \gamma_{V,+}^{(t)} \|\boldsymbol{w}_O\|_2^2, \\
\widetilde{\boldsymbol{\mu}}_-^\top \boldsymbol{W}_V^{(t)} \boldsymbol{w}_O &= \widetilde{\boldsymbol{\mu}}_-^\top \boldsymbol{W}_V^{(0)} \boldsymbol{w}_O + \gamma_{V,-}^{(t)} \|\boldsymbol{w}_O\|_2^2, \\
\widetilde{\boldsymbol{\xi}}_{n,i}^\top \boldsymbol{W}_V^{(t)} \boldsymbol{w}_O &= \widetilde{\boldsymbol{\xi}}_{n,i}^\top \boldsymbol{W}_V^{(0)} \boldsymbol{w}_O + \rho_{V,n,i}^{(t)} \|\boldsymbol{w}_O\|_2^2,
\end{align*}
for $i \in [M] \setminus \{1\}$ and $n \in [N]$.

We further denote the $V_+^{(t)} := \widetilde{\boldsymbol{\mu}}_+^{\top} \mathbf{W}_V^{(t)} \boldsymbol{w}_O$, $V_-^{(t)} := \widetilde{\boldsymbol{\mu}}_-^{\top} \mathbf{W}_V^{(t)} \boldsymbol{w}_O$ and $V_{n,i}^{(t)} := \widetilde{\boldsymbol{\xi}}_{n,i}^{\top} \mathbf{W}_V^{(t)}$.
\end{definition}
With scalarized $V$ in Defintion~\ref{def:2}, we can provide the dynamics of matrix $\mathbf{W}_V^{(t)}$ by analyzing the update of coefficients $\gamma^{(t)}$ as follows:
\begin{align*}
\gamma_{V,+}^{(t+1)} 
&= \gamma_{V,+}^{(t)} 
- \frac{\eta}{NM} \sum_{n \in S_+} \widetilde{\ell}_n'^{(t)}\cdot \\&
\Bigg[
\sum_{l=1}^M \langle \widetilde{\boldsymbol{\mu}}_+, \widetilde{\boldsymbol{\mu}}_+^{(t)} \rangle \, \varphi(\widetilde{\mathbf{x}}_{n,l} \mathbf{W}_Q \mathbf{W}_K^\top (\widetilde{\mathbf{X}}_n)^\top)_1 \\
&\qquad + \sum_{i=2}^M \sum_{l=1}^M \langle \widetilde{\boldsymbol{\mu}}_+, \widetilde{\boldsymbol{\xi}}_{n,i}^{(t)} \rangle \, \varphi(\widetilde{\mathbf{x}}_{n,l} \mathbf{W}_Q \mathbf{W}_K^\top (\widetilde{\mathbf{X}}_n)^\top)_i
\Bigg].
\end{align*}

First, the perturbed signal and noise components break the orthogonality between signal and noise, resulting in the appearance of terms of the form $\langle \widetilde{\boldsymbol{\mu}}_+, \widetilde{\boldsymbol{\xi}}_{n,i}^{(t)} \rangle$.
The upper bound of $\langle \widetilde{\boldsymbol{\mu}}_+, \widetilde{\boldsymbol{\mu}}_+^{(t)} \rangle$ is given by $(\|\boldsymbol{\mu}\|_2+\tau)^2$, while the upper bound of $\langle \widetilde{\boldsymbol{\mu}}_+, \widetilde{\boldsymbol{\xi}}_{n,i}^{(t)} \rangle$ is $(\|\boldsymbol{\mu}\|_2\tau+\sigma_p\tau\sqrt{2\log(4NM/\delta)}+\tau^2)$. 
Therefore, when the perturbation magnitude is small, $\langle \widetilde{\boldsymbol{\mu}}_+, \widetilde{\boldsymbol{\mu}}_+^{(t)} \rangle$ dominates, and the training dynamics under adversarial training closely resemble those under clean training.

When bounding $V_+^{(t)}$ (similar for $V_-^{(t)}$), we can consider a special case $\widetilde{\boldsymbol{\mu}} = \widetilde{\boldsymbol{\mu}}^{(t)}$ and derive the bound via cumulative summation, yielding $|V_+^{(t)}| \leq |V_+^{(0)}| + \sum_{s=0}^{t-1} |\gamma_{V,+}^{(s+1)} - \gamma_{V,+}^{(s)}| \cdot \|\boldsymbol{w}_O\|_2^2$ by Definition~\ref{def:2}. 
In fact, we establish that for any $\widetilde{\boldsymbol{\mu}} \in B(\boldsymbol{\mu},\tau)$, the quantity $\widetilde{\boldsymbol{\mu}}^\top \mathbf{W}_V^{(t)} \boldsymbol{w}_O$ admits a uniform upper bound.

Similarly, we define vectorized $\mathbf{W}_Q$ and $\mathbf{W}_K$ as follows.
\begin{definition}[Vectorized Q \& K.]\label{def:3}
Let $\mathbf{W}_Q^{(t)}$ and $\mathbf{W}_K^{(t)}$ be the QK matrices of the ViT at the $t$-th iteration of gradient descent. Then we define the vectorized Q and vectorized K as follows
\begin{align*}
    \boldsymbol{q}_+^{(t)} &= \widetilde{\boldsymbol{\mu}}_+^{\top} \mathbf{W}_Q^{(t)}, \quad \boldsymbol{q}_-^{(t)} = \widetilde{\boldsymbol{\mu}}_-^\top \mathbf{W}_Q^{(t)}, \quad \boldsymbol{q}_{n,i}^{(t)} = \widetilde{\boldsymbol{\xi}}_{n,i}^\top \mathbf{W}_Q^{(t)}, \\
    \boldsymbol{k}_+^{(t)} &= \widetilde{\boldsymbol{\mu}}_+^\top \mathbf{W}_K^{(t)}, \quad \boldsymbol{k}_-^{(t)} = \widetilde{\boldsymbol{\mu}}_-^\top \mathbf{W}_K^{(t)}, \quad \boldsymbol{k}_{n,i}^{(t)} = \widetilde{\boldsymbol{\xi}}_{n,i}^\top \mathbf{W}_K^{(t)}
\end{align*}
for $i \in [M] \backslash \{1\}, n \in [N]$.

\end{definition}

With Definition~\ref{def:2} and~\ref{def:3}, we can analyze the learning dynamics of the transformed coefficients rather than the original matrices.
\subsection{Effects of Perturbation Magnitude on Attention}\label{sec:ef_soft}
Our second key technique is to analyze how the attention mechanism behaves under perturbations of varying magnitudes. First, we present the following lemma, which shows that the attention mechanism is robust to small perturbations. 

\begin{lemma}\label{lemma:rsoftmax1}
Under Condition~\ref{ass:1}, suppose the perturbation satisfies $\tau \leq O(\tfrac{\|\boldsymbol{\boldsymbol{\mu}}\|_2}{\log d_h})$ and $t \geq \Omega \left( \frac{1}{\eta \|\boldsymbol{\mu}\|_2^2 \|\boldsymbol{w}_O\|_2^2 \log(6N^2M^2 / \delta)} \right)$. Then, there exists a universal constant $ C \leq e/2$ such that
\[
\frac{\max_{\widetilde{\mathbf{X}}\in B(\mathbf{X},\tau)} 
\mathrm{softmax}\!\left(\langle\mathbf{q}^{(t)} ,\mathbf{k}^{(t)}\rangle\right)}
{\min_{\widetilde{\mathbf{X}}\in B(\mathbf{X},\tau)} 
\mathrm{softmax}\!\left(\langle\mathbf{q}^{(t)} ,\mathbf{k}^{(t)}\rangle\right)}
\;\leq\; C.
\]
\end{lemma}
This lemma implies that the relative change in attention weights under perturbations is uniformly bounded; hence, attention computed on perturbed inputs closely matches that on the clean inputs.
Thus, when considering the attention from the perturbed patch to another, it can be well-approximated by the attention from the clean signal to the clean noise.

Next, we present the following lemma, which characterizes the behavior of the attention mechanism under moderate perturbations.
\begin{lemma}
Under Condition~\ref{ass:1}, supposing the perturbation satisfies $\omega(\tfrac{\|\boldsymbol{\mu}\|_2}{\log d_h})\leq \tau\leq O(\|\boldsymbol{\mu}\|_2)$, for any $t \geq 0$, we have: 
\begin{equation*}
    \frac{1}{M} - o(1) \leq \mathrm{softmax}(\langle \mathbf{q}^{(t)}, \mathbf{k}^{(t)} \rangle) \leq \frac{1}{M} + o(1)
\end{equation*}
\end{lemma}
This lemma demonstrates that, under moderate perturbations, the attention distribution remains invariant and stays close to its initial uniform form.
\subsection{Generalization Guarantee}

For a new data point $(\mathbf{X},y)$ generated from the distribution defined in Definition~\ref{def:data_gen}, we interpret the test error as the probability that the noise component dominates the output.

For adversarial test data with added perturbations, we need to bound the maximal distance between them and the corresponding clean test data. First, from the analysis in Section~\ref{sec:5_1}, we know that bounding 
\(
V^{(t)} = \widetilde{\mathbf{x}}^{\top} \mathbf{W}_V^{(t)} {\boldsymbol{w}_O}
\) 
implies that its deviation from the unperturbed counterpart ${\mathbf{x}}^{\top}\mathbf{W}_V {\boldsymbol{w}_O}$ is at most
\(
\big|\langle \widetilde{\mathbf{x}}-\mathbf{x},\, \mathbf{W}_V {\boldsymbol{w}_O} \rangle\big| \leq \tau \|\mathbf{W}_V {\boldsymbol{w}_O}\|_2.
\)
Second, from the analysis in Section~\ref{sec:ef_soft}, we know that the attention component ${softmax}(\langle \boldsymbol{q}, \boldsymbol{k}\rangle)$ exhibits robustness under perturbations of bounded magnitude. Combining these insights, we state the following lemma.

\begin{lemma}\label{lem:532}
Under Condition~\ref{ass:1}, if\\ \(
t \;\geq\; \Omega\!\left( 
   \eta^{-1} \epsilon^{-1} \|\boldsymbol{\mu}\|_2^{-2} \|\boldsymbol{w}_O\|_2^{-2} 
   \log^{-1}\!\left(\tfrac{6N^2 M^2}{\delta}\right)
\right)
\) and $\tau \leq \tfrac{\|\boldsymbol{\mu}\|_2}{\log(d_h)}$, we have:  
\begin{equation*}
    \begin{aligned}
        & y f(\mathbf{X},\theta(t)) 
         -\min_{\widetilde{\mathbf{X}}\in B(\mathbf{X},\tau)} y f(\widetilde{\mathbf{X}},\theta(t)) \lesssim M\big\| \mathbf{W}_{V}^{(t)} {\boldsymbol{w}_O} \big\|_2  \tau.
    \end{aligned}
\end{equation*}
\end{lemma}
As a result, the robust test error can be interpreted as the probability of misclassification when the output under clean test data is perturbed by its maximal adversarial deviation.

\begin{lemma}

Under the same conditions of Lemma~\ref{lem:532}, with high probability, we have for some constant $c>0$:
{\small
    \begin{equation*}
    \begin{aligned}
       &P(\exists \ \widetilde{\mathbf{X}} \in B(\mathbf{X}, \tau) : y f(\widetilde{\mathbf{X}}, \theta(t)) \le 0) \\&=P\big(yf(\mathbf{X},\theta(t))+(yf( \mathbf{X},\theta(t))-\min_{\widetilde{\mathbf{X}}\in B(\mathbf{X},\tau)}yf( \mathbf{\widetilde{X}},\theta(t))) \leq 0\big)\\
       &\leq  \exp\big(-c (\frac{  (V_{+}^{(t)} - V_{-}^{(t)})- \|\mathbf{W}_{V}^{(t)} {\boldsymbol{w}_O}\|_2\tau}{\sigma_p  \|\mathbf{W}_{V}^{(t)} {\boldsymbol{w}_O}\|_2})^2\big). 
    \end{aligned}
    \end{equation*}}
    
\end{lemma}
Consequently, the robust test error differs from the clean test error only by an additive factor that scales with both the perturbation radius and the model complexity. More details are in Appendix.
\begin{figure*}[t]
    \centering
    \begin{subfigure}{0.23\linewidth}
        \centering
        \includegraphics[width=\linewidth]{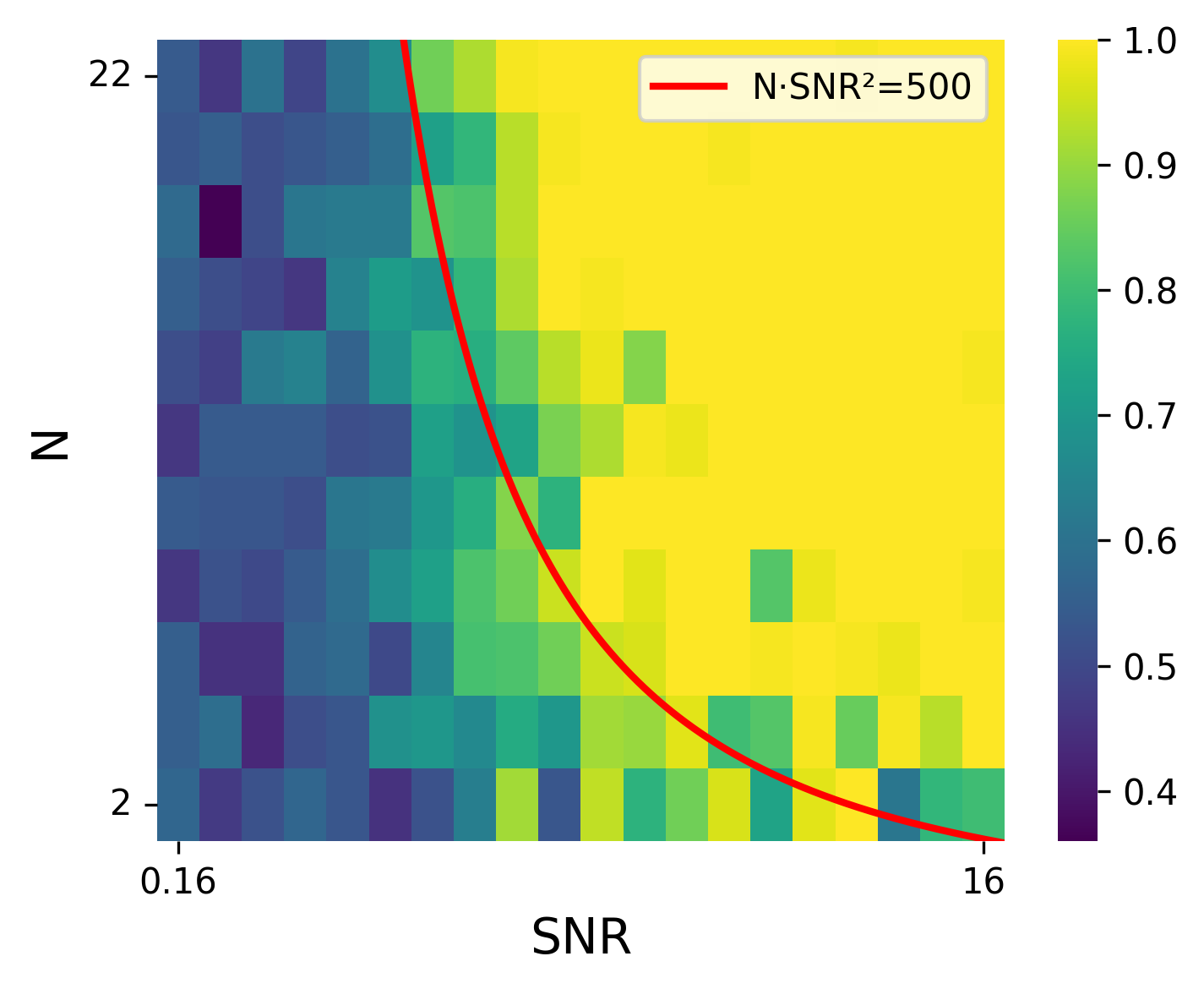}
        \caption{Clean Heatmap}
        \label{fig:clean_heatmap_hc}
    \end{subfigure}
    \hfill
    \begin{subfigure}{0.23\linewidth}
        \centering
        \includegraphics[width=\linewidth]{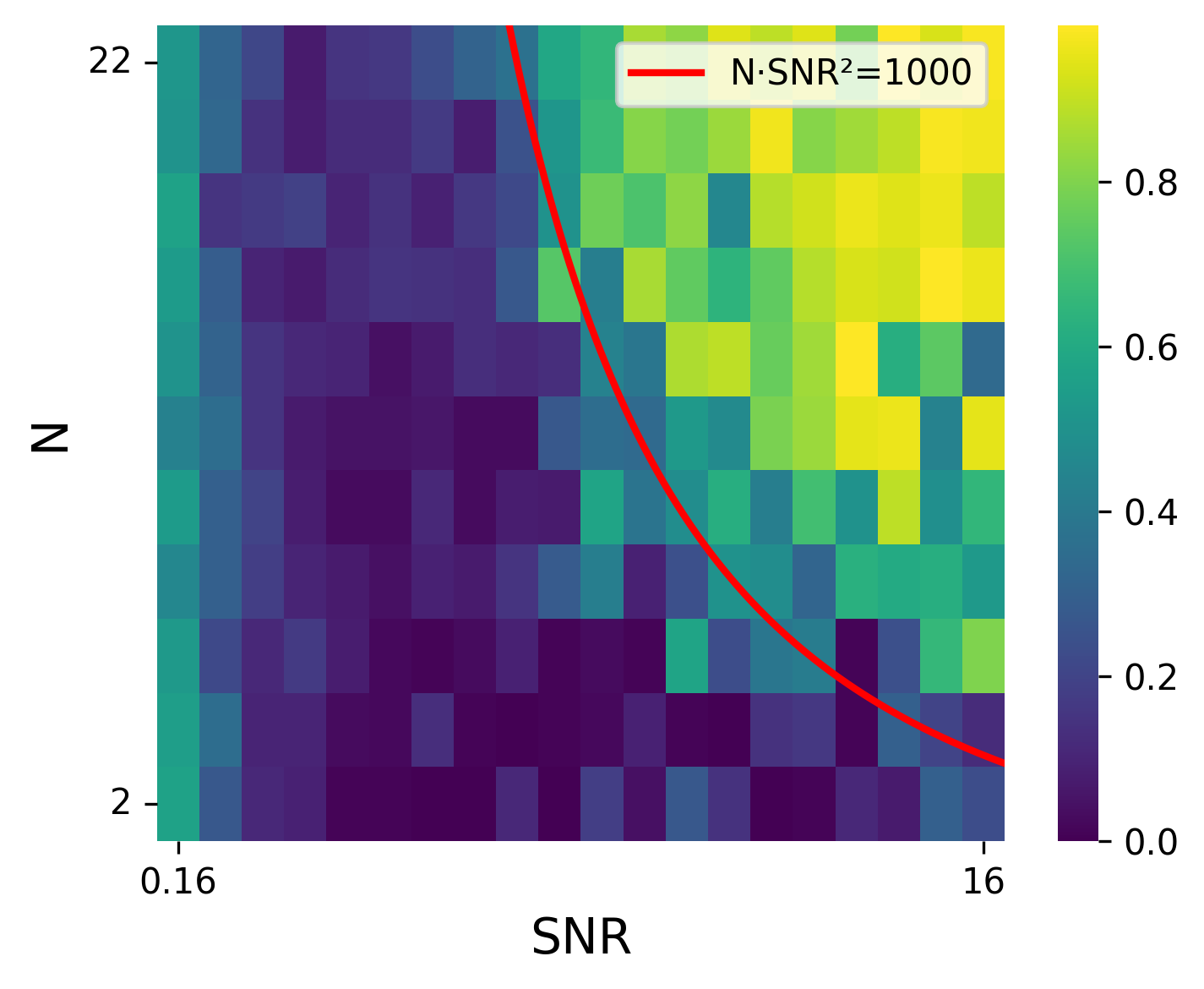}
        \caption{Robust Heatmap}
        \label{fig:robust_heatmap_hc}
    \end{subfigure}
    \hfill
    \begin{subfigure}{0.23\linewidth}
        \centering
        \includegraphics[width=\linewidth]{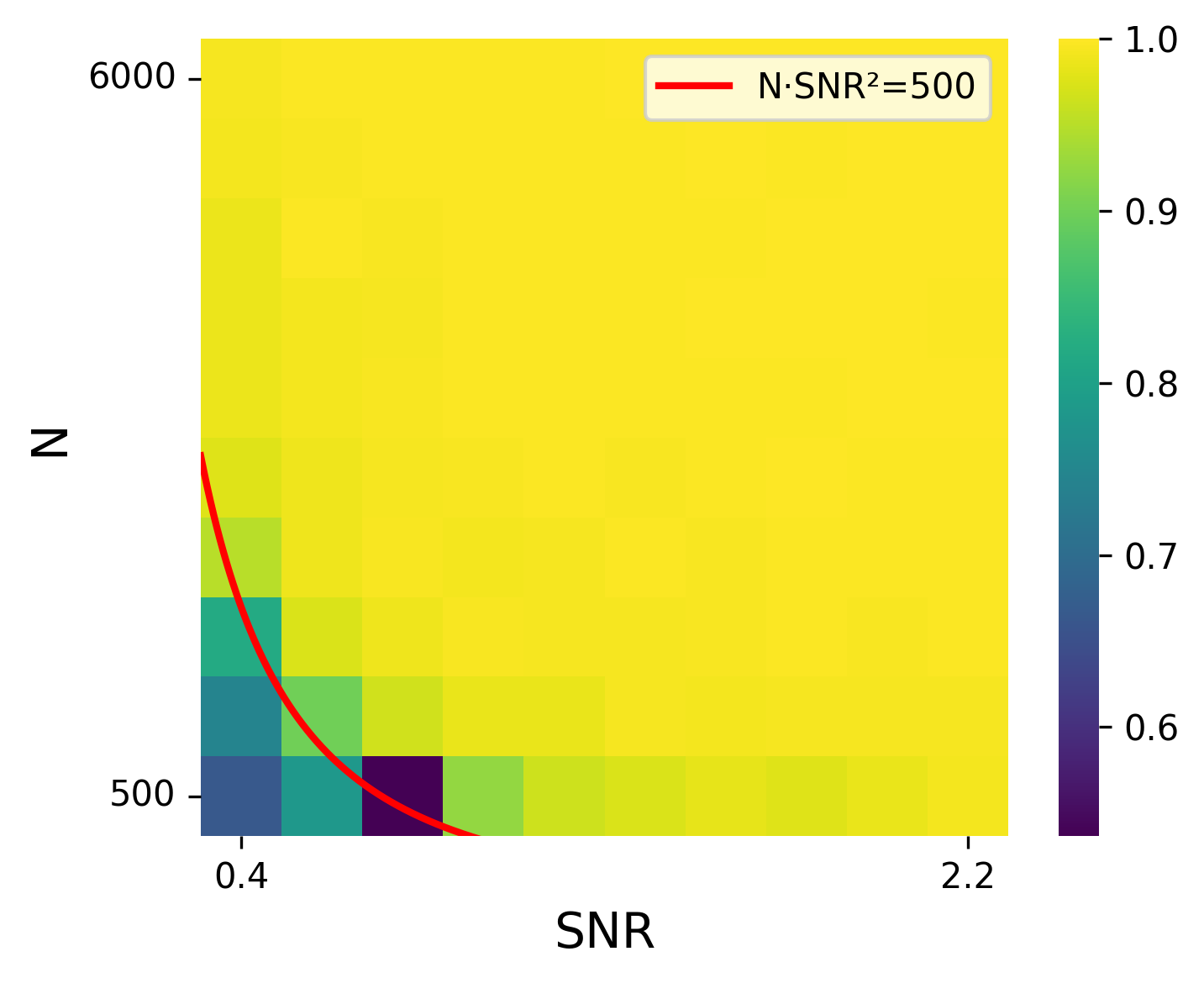}
        \caption{Clean Heatmap}
        \label{fig:clean_heatmap}
    \end{subfigure}
    \hfill
   \begin{subfigure}{0.23\linewidth}
        \centering
        \includegraphics[width=\linewidth]{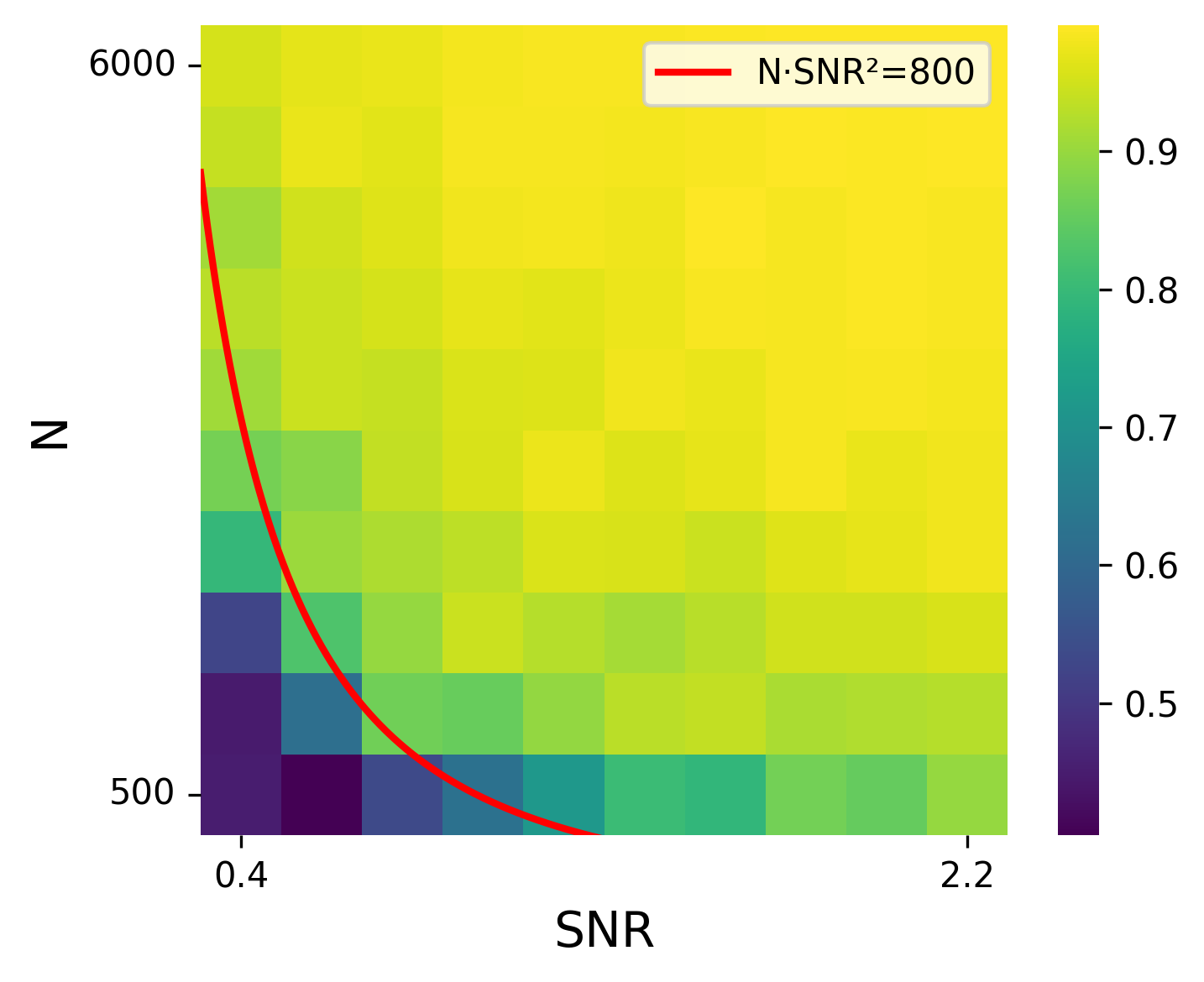}
        \caption{Robust Heatmap}
        \label{fig:robust_heatmap}
    \end{subfigure}
    \caption{
    Clean and robust test accuracy under adversarial training across various signal-to-noise ratios ($\operatorname{SNR}$) and sample sizes ($N$). (a)\&(b): results on synthetic data; (c)\&(d): results on real-world data.
    High test accuracy is colored in yellow, whereas low test accuracy is colored in purple.}
    \label{fig:all_heatmaps}
\end{figure*}
\begin{figure*}[t]
    \centering
    \begin{subfigure}{0.23\linewidth}
        \centering
        \includegraphics[width=\linewidth]{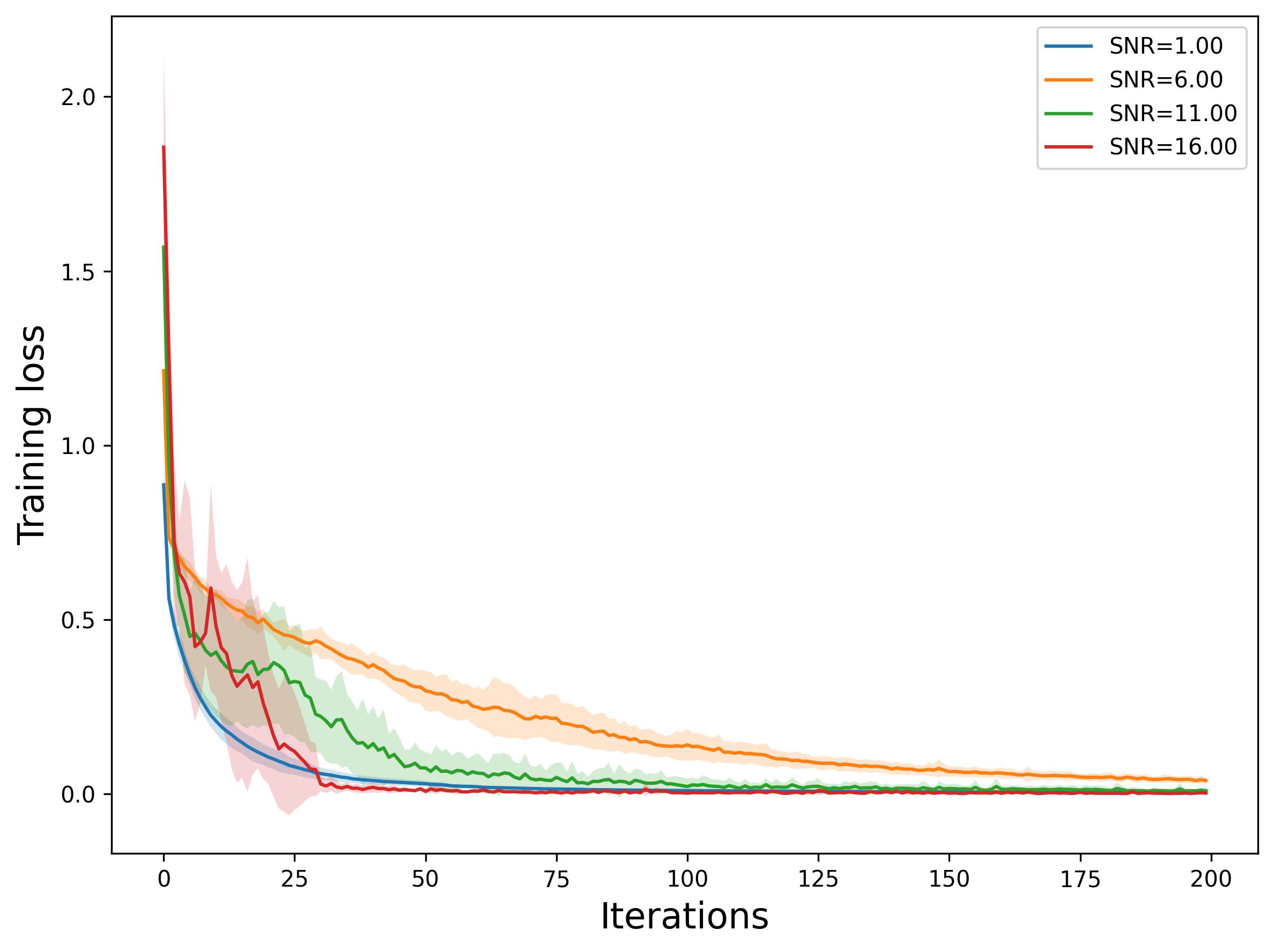}
        \caption{Training Loss}
        \label{fig:loss_fixn}
    \end{subfigure}
    \hfill
    \begin{subfigure}{0.23\linewidth}
        \centering
        \includegraphics[width=\linewidth]{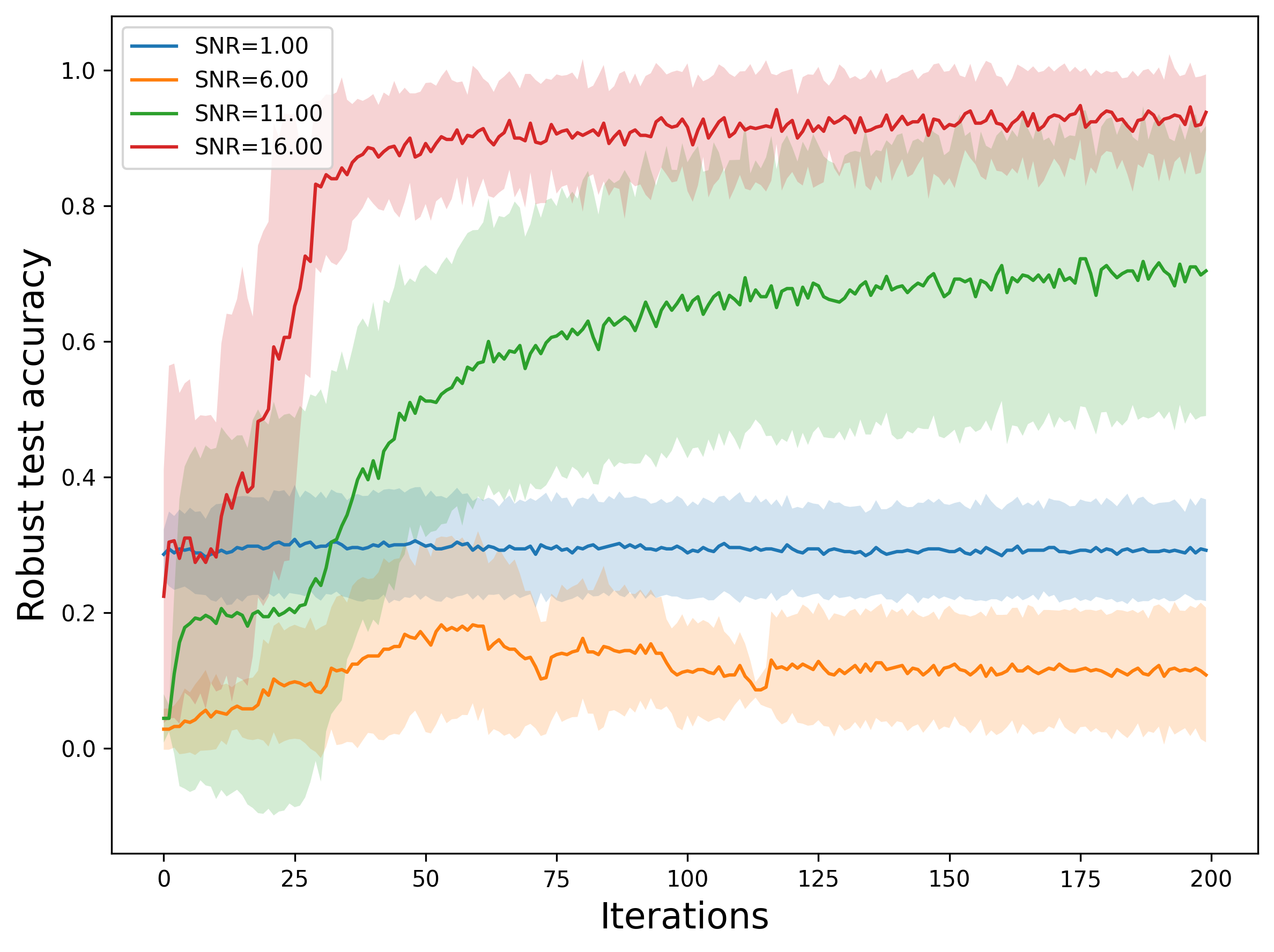}
        \caption{Robust Accuracy}
        \label{fig:rob_acc_fixn}
    \end{subfigure}
    \hfill
    \begin{subfigure}{0.23\linewidth}
        \centering
        \includegraphics[width=\linewidth]{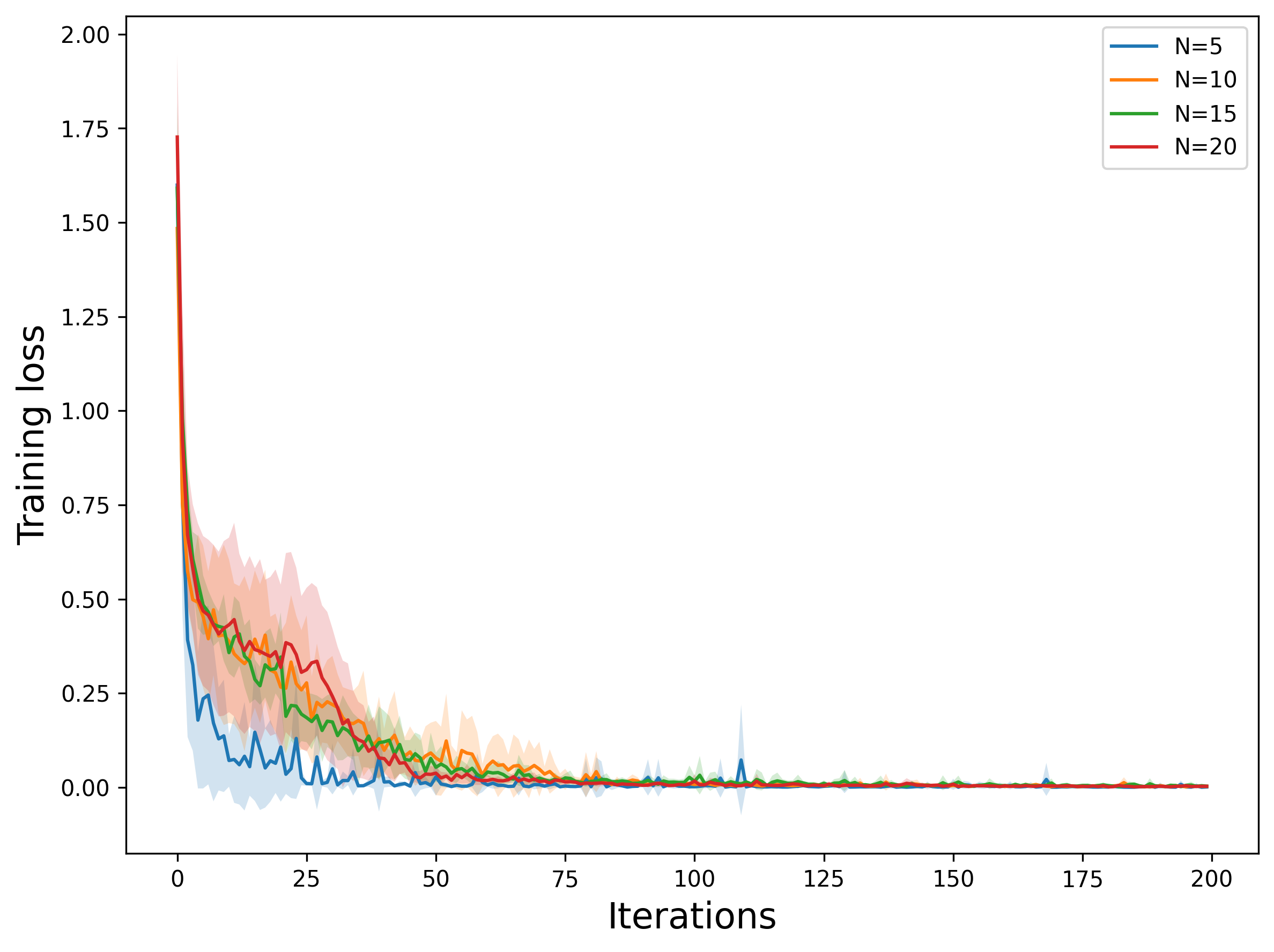}
        \caption{Training Loss}
        \label{fig:loss_fixsnr}
    \end{subfigure}
    \hfill
   \begin{subfigure}{0.23\linewidth}
        \centering
        \includegraphics[width=\linewidth]{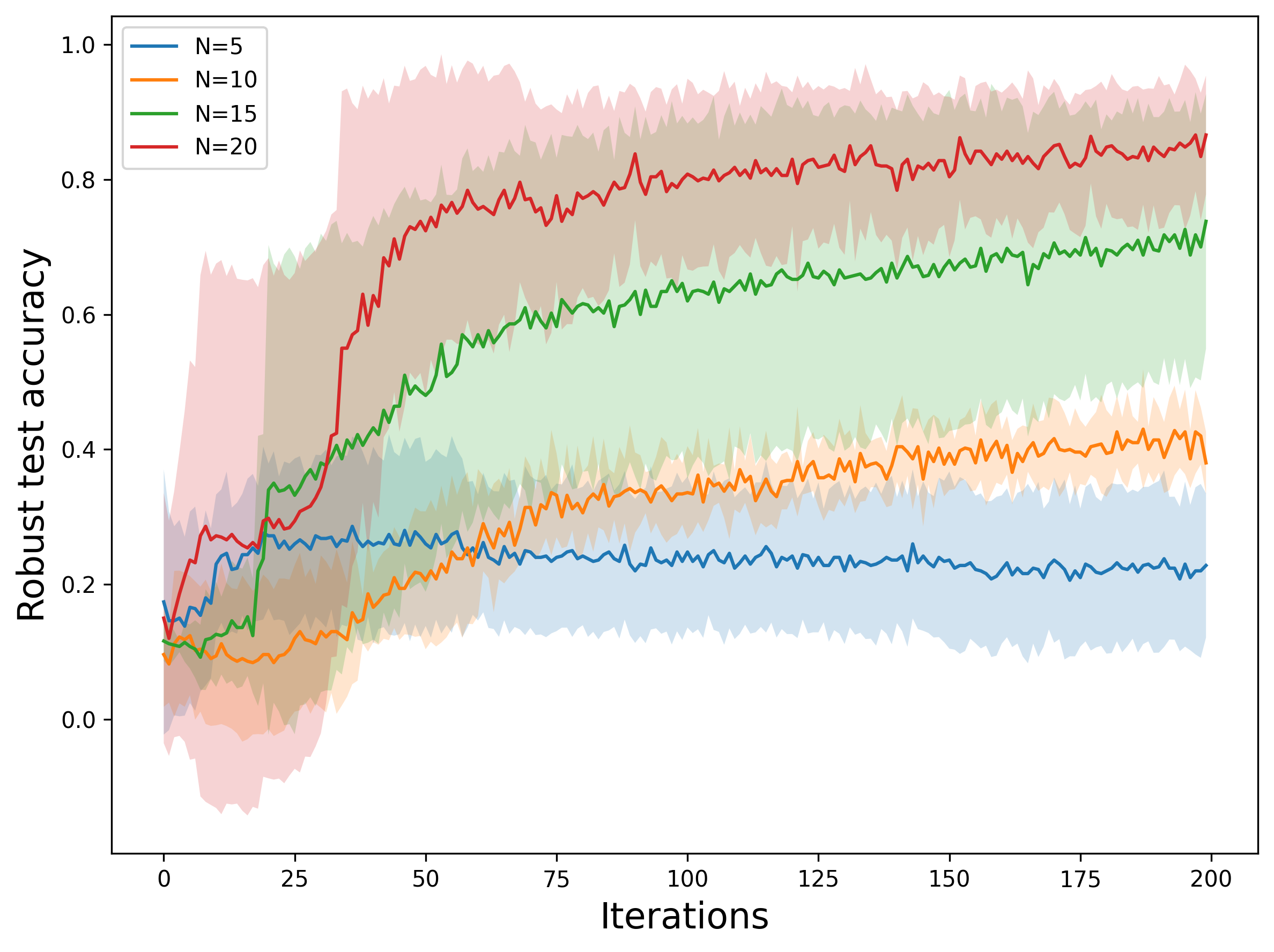}
        \caption{Robust Accuracy}
        \label{fig:rob_acc_fixsnr}
    \end{subfigure}
    \caption{
    (a)\&(b): Curves of robust training loss and robust test accuracy versus training iteration under fixed training data number $N = 22$.
    (c)\&(d): Curves of robust training loss and robust test accuracy versus training iteration under fixed data $\text{SNR} = 12$.
    }
    \label{fig:fixline}
    \vspace{-0.1in}
\end{figure*}

\section{Experiments}\label{sec:6}
\subsection{Experimental Setup}
\textbf{Datasets.}
For the experiments on synthetic data, we synthesize each sample following the distribution in Definition~\ref{def:data_gen}.
Specifically, we first formalize the two signal vectors in Definition~\ref{def:data_gen} as
$\boldsymbol{\mu}_{+} = \lVert \boldsymbol{\mu} \rVert_{2} \cdot [1, 0, \ldots, 0]^\top$
and
$\boldsymbol{\mu}_{-} = \lVert \boldsymbol{\mu} \rVert_{2} \cdot [0, 1, 0, \ldots, 0]^\top$,
where the signal dimension $d$ is set as $1024$. 
Then, for each generated sample, the number $M$ of tokens in this sample is set as $16$ and every noise vector in this sample is drawn independently from the Gaussian distribution $\boldsymbol{\xi}_i \sim \mathcal{N}(0, 0.4 \cdot \mathbb{I}_d)$. We generate up to $22$ samples for training and $100$ samples for evaluation.


For the experiments on real-world data, we employ MNIST, CIFAR-10, and Tiny-ImageNet.
To better simulate the feature learning setting in our theory, we transform each image to a vector with the following procedures:
(1)~Normalize the image $\ell_2$-norm to $5$ to keep the same signal strength.
(2)~Add independent Gaussian noise $\mathcal{N}(0,\sigma^2)$ to each pixel to create a noise map, with the variance $\sigma^2$ determining the SNR.
(3)~Superimpose this noise map onto the original clean image to produce the sample used for our experiments.

\textbf{Model architectures.} For synthetic data experiments, we adopt the two-layer Transformer defined in (\ref{eq:2trans}), where both of the hidden dimension $d_h$ and the value dimension $d_v$ are set as $128$.
For real-world data experiments, We implement realistic ViT architecture~\citep{dosovitskiy2020image} and a simplified ViT model consisting of two attention layers, each with four self-attention heads, followed by an MLP layer with ReLU activation.
The hidden dimension of this ViT is set as $128$.
Model parameters in all experiments are initialized by PyTorch’s default method, followed by an additional scaling factor of \(1/16\), matching the requirement of Condition~\ref{con:initialization}. 


\textbf{Model training \& evaluation.}
We use full-batch gradient descent to train all models in our experiments, with a learning rate of $0.1$. 
Each model will be trained until its training loss falls below a target threshold $0.01$.
Besides, in each adversarial training step, we leverage projected gradient descent (PGD; \citealt{madry2017towards}) and state-of-the-art APGD~\citep{croce2020reliable} to generate adversarial examples under both single $\ell_2$-norm and multi-norm threat models. For $\ell_2$-norm PGD attack, we set base attack strength \(\tau / \lVert \boldsymbol{\mu} \rVert_2 = 0.05\) for $20$ steps with a step size of $0.2\tau$, using per-token $L_2$-normalized gradient updates with projection onto the $L_2$ ball of radius $\tau$. Detailed parameters for APGD and multi-norm attacks are specified in Appendix~\ref{sec:b2},\ref{sec:b3},\ref{sec:b4}.
To assess the performance of trained models, we report both their clean and robust classification errors calculated on test datasets.




\subsection{Results Analysis}

\textbf{Phase Transition in benign overfitting.} 
We perform adversarial training 
with different number $N$ of training data ranging from 2 to 22 and different SNR ranging from 0.16 to 16 on synthetic data.
The clean and robust test accuracies of these models are collected and presented as heatmaps in Figures~\ref{fig:clean_heatmap_hc} and \ref{fig:robust_heatmap_hc}.
From the figure, we observe a clear decision boundary in the form $N \cdot \mathrm{SNR}^2 = \Omega(1)$ 
separates the high-accuracy and low-accuracy regions.
This observation aligns well with our Theorem~\ref{thm:main_thm} that $N \cdot \mathrm{SNR}^2 = \Omega(1)$ is necessary for models in adversarial training to produce benign overfitting.

To demonstrate the generality and robustness of our findings, we conduct an extensive suite of experiments spanning diverse datasets (MNIST, CIFAR-10, and Tiny-ImageNet), various threat models (PGD and state-of-the-art APGD attack~\citep{croce2020reliable}), and multiple norm constraints ($\ell_1$, $\ell_2$, and $\ell_\infty$). Notably, our conclusions remain remarkably consistent across all these settings. Specifically, results on MNIST are shown in Figures~\ref{fig:clean_heatmap} and~\ref{fig:robust_heatmap}, while detailed evaluations on other datasets, multi-norm attacks, and realistic ViT-Base models~\citep{dosovitskiy2020image} are provided in Appendices~\ref{sec:b2}, \ref{sec:b3}, and \ref{sec:b4}, respectively.

    
\textbf{Effects of signal-to-noise ratio and dataset size.}
We then fix the number $N$ of training data and plot curves of the training loss/robust test accuracy versus the training iteration under different SNR in Figures~\ref{fig:loss_fixn} and~\ref{fig:rob_acc_fixn}.
From them, we observe that as the training iteration increases, the model overfits to training data, but its robustness does not consistently increase unless the SNR is large.
We also fix the SNR and plot similar curves with different number of training data in Figures~\ref{fig:loss_fixsnr} and~\ref{fig:rob_acc_fixsnr}, where we find that even the model is overfitting, its robustness improves only when the training data number $N$ is large.
All these results indicate that benign overfitting can emerge in adversarial training only when the training data number $N$ and the data SNR are both not too small, which coincides with the requirements that $N \cdot \text{SNR}^2 = \Omega(1)$ or $N \cdot \text{SNR}^2 = \Omega(\frac{1}{\epsilon})$ in our Theorem~\ref{thm:main_thm}.


\textbf{Empirical validation under varying perturbation radii.} Our theoretical training dynamic of the different regimes is further supported by experiments in Appendix~\ref{sec:b1}, where we track the evolution of the attention mechanism and model weights across various perturbation radii $\tau$.

\section{Conclusion}
Our paper presents the first comprehensive theoretical analysis of the generalization behavior after adversarial training on a two-layer Vision Transformer. We demonstrate that, under appropriate relationships between signal-to-noise ratio and perturbation magnitude, adversarially trained ViTs can interpolate the training data with vanishing robust loss while still achieving small robust test error. Our analysis reveals three perturbation regimes that clarify how adversarial training shapes the learning dynamics of attention heads, from clean-like behavior to linear collapse and eventual failure beyond the benign regime. Empirical results on synthetic data and real world data corroborate the theory, aligning closely with the predicted conditions under which robust benign overfitting emerges.

\section*{Impact Statement}

This paper presents work whose goal is to advance the field of Machine
Learning. There are many potential societal consequences of our work, none
which we feel must be specifically highlighted here.

\nocite{langley00}

\bibliography{example_paper}
\bibliographystyle{icml2026}

\newpage
\appendix
\onecolumn
\section{Additional Related Work}
{
\textbf{Differences in Adversarial Robustness Between Transformer and CNN.} Several studies have compared Transformers and CNNs under adversarial attacks. ~\citet{bai2021transformers} observe that under unified training settings, Transformers are not inherently more robust, with their OOD generalization mainly attributed to self-attention. ~\citet{mo2022adversarial} show that under standard training, Transformers do not necessarily outperform CNNs under adversarial attack, and propose training strategies to improve ViT robustness. ~\citet{dingeto2024comparative} propose a regularization method that enables ViTs to exhibit stronger adversarial robustness than CNNs. These studies are largely empirical, focusing on performance differences, and do not analyze the learning dynamics or the role of attention in adversarial training.
}

{\section{Additional Experiments}
\subsection{Empirical validation of theoretical regimes}\label{sec:b1}
In this section, we follow a similar setting as the experiments on synthetic data in Section~\ref{sec:6}, and add more experiments about training dynamics with different perturbation $\tau$ radius.

\textbf{Experiments setting.}

We focus on tracking the dynamics of attention entropy, training loss, and \(W_V\) norm under benign  overfitting regime with different perturbation $\tau$ radius, and the key parameters are as follows:
\begin{itemize}
    \item $N = 25$
    \item $SNR = 16$
    \item $M = 2$
    \item $d = 1024$
    \item $d_h = 512$
    \item $d_v = 512$
    \item $\sigma_p = 0.05$
    \item $\frac{\tau}{\|\boldsymbol \mu\|_2} = 0.02,0.1,0.5$
\end{itemize}

\begin{figure}[h]
    \centering
   \begin{subfigure}{0.32\linewidth}
        \centering
        \includegraphics[width=\linewidth]{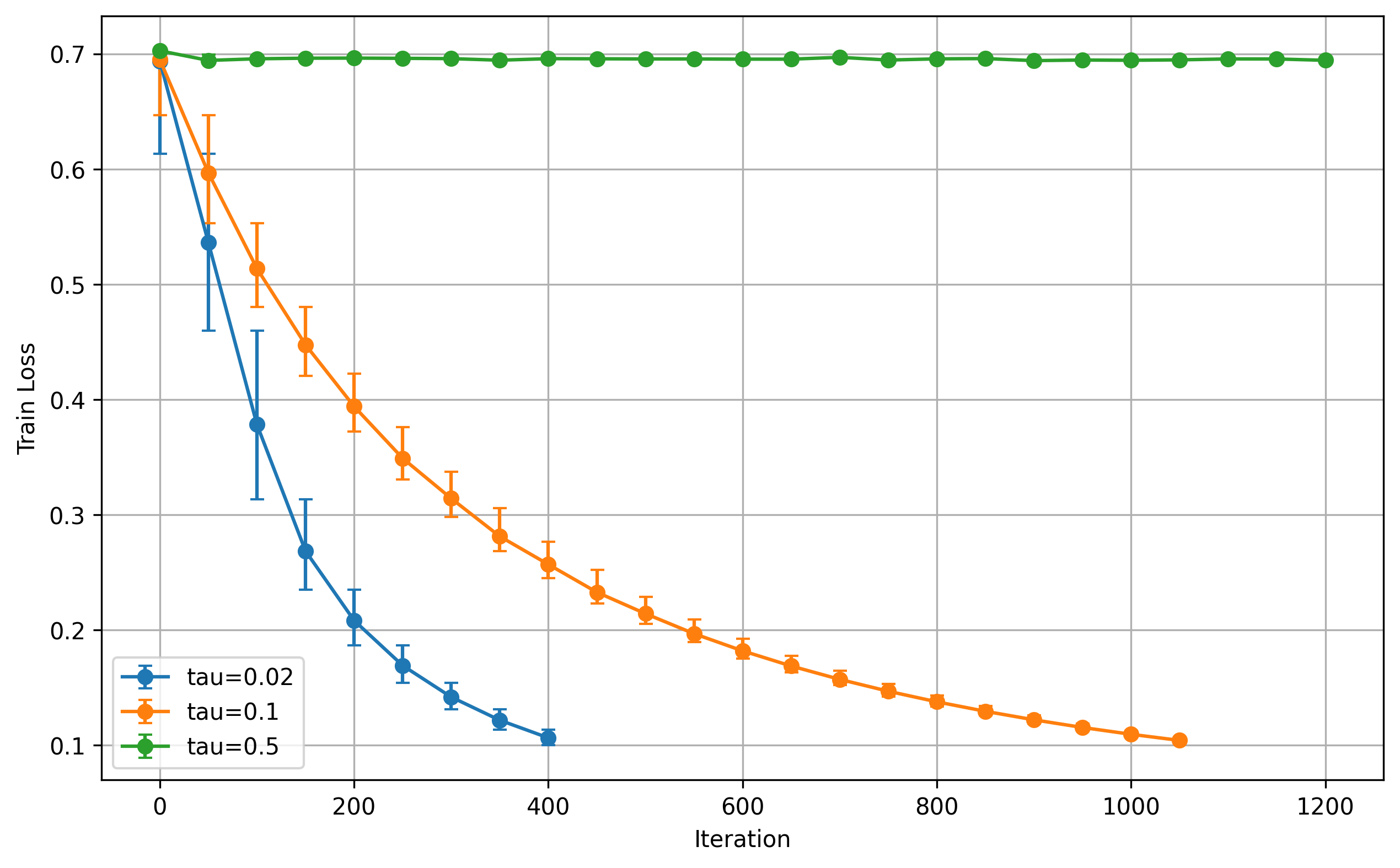}
        \caption{Train Loss}
        \label{fig:Train_Loss}
    \end{subfigure}
    \hfill
     \begin{subfigure}{0.32\linewidth}
        \centering
        \includegraphics[width=\linewidth]{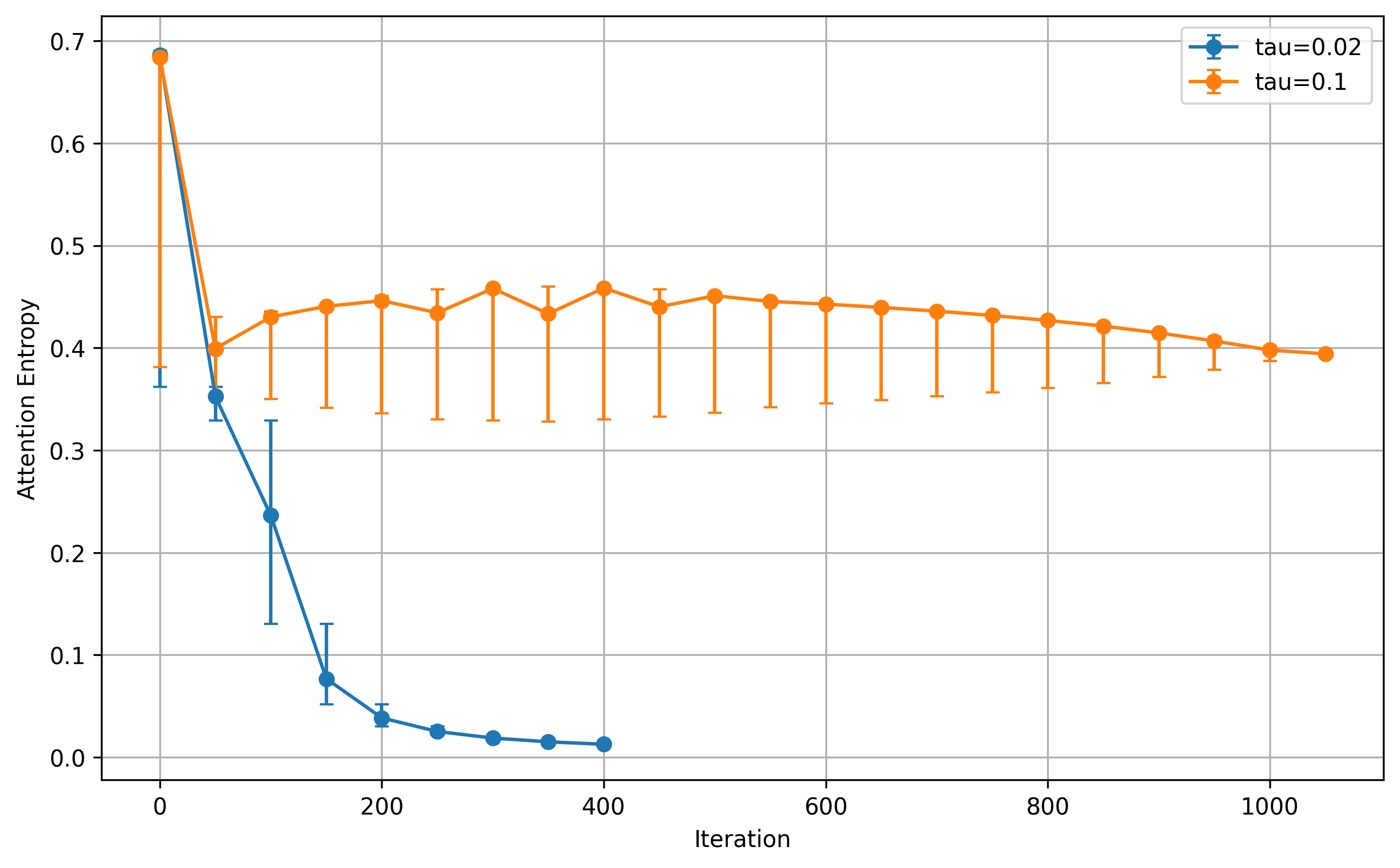}
        \caption{Attention Entropy}
        \label{fig:Attention_Entropy}
    \end{subfigure}
    \hfill
    \begin{subfigure}{0.32\linewidth}
        \centering
        \includegraphics[width=\linewidth]{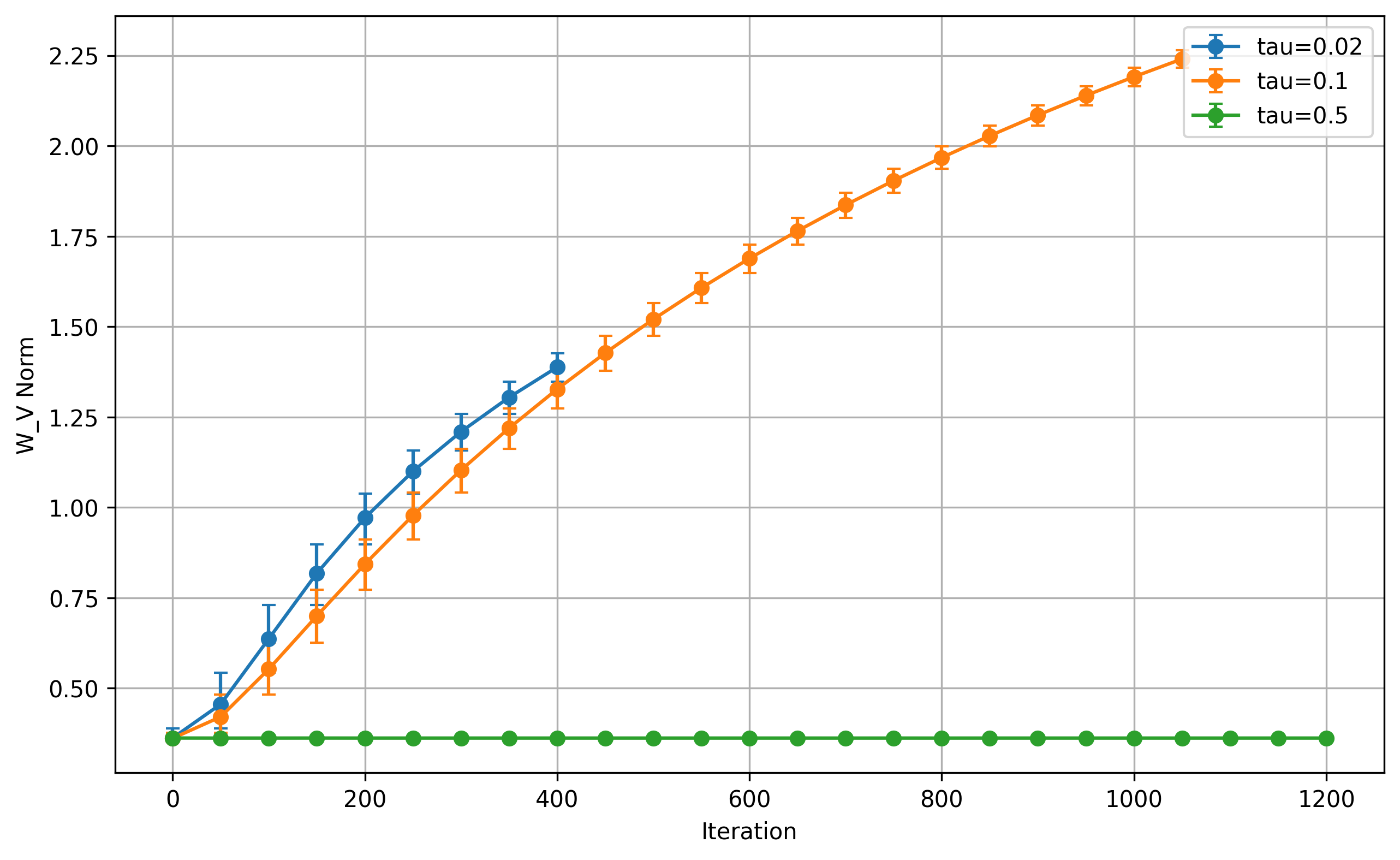}
        \caption{\(W_V\) Norm}
        \label{fig:W_V_Norm}
    \end{subfigure}
    \caption{Training dynamics with different perturbation $\tau$ radius: attention entropy, training loss, and \(W_V\) norm.}
    \label{fig:training_dynamics}
\end{figure}

\textbf{Experiments results.}

1. In Figure~\ref{fig:Train_Loss}, when \(\frac{\tau}{\|\boldsymbol{\mu}\|_2} = 0.02\) or \(0.1\), the adversarial training loss converges to zero, indicating that the model successfully interpolates all noise-corrupted training samples, consistent with the benign overfitting behavior in Theorem~\ref{thm:main_thm}. In contrast, when \(\frac{\tau}{\|\boldsymbol{\mu}\|_2} = 0.5\), the adversarial training loss fails to decrease, aligning with the non-convergence regime characterized in Theorem~\ref{coro:larger_pert}. Moreover, the case \(\frac{\tau}{\|\boldsymbol{\mu}\|_2} = 0.02\) exhibits a faster convergence rate than the \(\frac{\tau}{\|\boldsymbol{\mu}\|_2} = 0.1\) setting, highlighting the role of the attention mechanism in accelerating convergence, which is in agreement with our theoretical predictions in Theorem~\ref{thm:main_thm}.

2. In Figure~\ref{fig:Attention_Entropy}, when $\frac{\tau}{\|\boldsymbol{\mu}\|_2} = 0.02$, the attention entropy decreases to nearly zero, indicating that the attention mechanism correctly concentrates on the signal patch. This behavior is consistent with Case~1 of Theorem~\ref{thm:main_thm}, where the perturbation level is sufficiently small for the model to recover the underlying signal structure. 

In contrast, when $\frac{\tau}{\|\boldsymbol{\mu}\|_2} = 0.1$, the attention entropy fails to decrease and instead remains high, demonstrating that moderate perturbations hinder the learning of attention weights. As a result, the attention distribution remains nearly uniform rather than focusing on the signal patch. This phenomenon aligns with Case~2 of Theorem~\ref{thm:main_thm}, where the perturbation magnitude prevents the attention mechanism from identifying the true signal.

3. In Figure~\ref{fig:W_V_Norm}, when $\frac{\tau}{\|\boldsymbol{\mu}\|_2} = 0.1$, the $\|W_V\|_2$ norm exhibits the largest growth. This behavior is consistent with Case~2 of Theorem~\ref{thm:main_thm}, where the ViT effectively collapses into a linear model and the value projection $W_V$ becomes the dominant component driving the learning dynamics. 

In contrast, for $\frac{\tau}{\|\boldsymbol{\mu}\|_2} = 0.02$, the attention mechanism remains effective, so only mild updates to $W_V$ are required for the model to fit the noisy training data and achieve benign overfitting. When $\frac{\tau}{\|\boldsymbol{\mu}\|_2} = 0.5$, the perturbation is too large for the model to learn meaningful structure, resulting in $W_V$ failing to make progress during training.

\subsection{Additional Experiments on MNIST, CIFAR-10 and Tiny-ImageNet with APGD}\label{sec:b2}

In this section, we follow the same experimental setup as in Section~\ref{sec:6} for the MNIST dataset and further extend our evaluation by conducting additional experiments on both MNIST,CIFAR-10 and Tiny-ImageNet under APGD attacks. These results demonstrate that our theoretical insights continue to hold for larger and more complex datasets, as well as under stronger adversarial attacks.

\textbf{Experiments setting.} We conduct adversarial training using the state-of-the-art APGD attack model. 
For the MNIST dataset, we consider an attack strength of $\frac{\tau}{\|\boldsymbol{\mu}\|_2} = 0.05$, 
and vary the number of training samples $N$ from $1000$ to $6000$ and SNR from $0.4$ to $2$. 
For the CIFAR-10 dataset, we set a weaker attack strength of $\frac{\tau}{\|\boldsymbol{\mu}\|_2} = 0.01$, 
and vary the number of training samples $N$ from $1000$ to $10000$ and the SNR from $0.4$ to $10$.
For the Tiny-ImageNet dataset, we set a weaker attack strength of $\frac{\tau}{\|\boldsymbol{\mu}\|_2} = 0.10$, 
and vary the number of training samples $N$ from $100$ to $1000$ and the SNR from $0.4$ to $10$.
 
\textbf{Experiments results.} The clean and robust test accuracies on MNIST, CIFAR-10 and Tiny-ImageNet are collected and presented as heatmaps in Figures~\ref{fig:minst_heatmaps_agpd},~ \ref{fig:cifar10_heatmaps_agpd} and \ref{fig:imagenet_heatmaps_agpd}. In both figures, we observe a clear phase transition phenomenon. Moreover, as both the sample size $N$ and the SNR increase, the clean and robust test error consistently decreases. This behavior is fully aligned with our theoretical analysis as well as the empirical findings reported in previous experiments.

\begin{figure}[h]
    \centering
    \begin{subfigure}{0.24\linewidth}
        \centering
        \includegraphics[width=\linewidth]{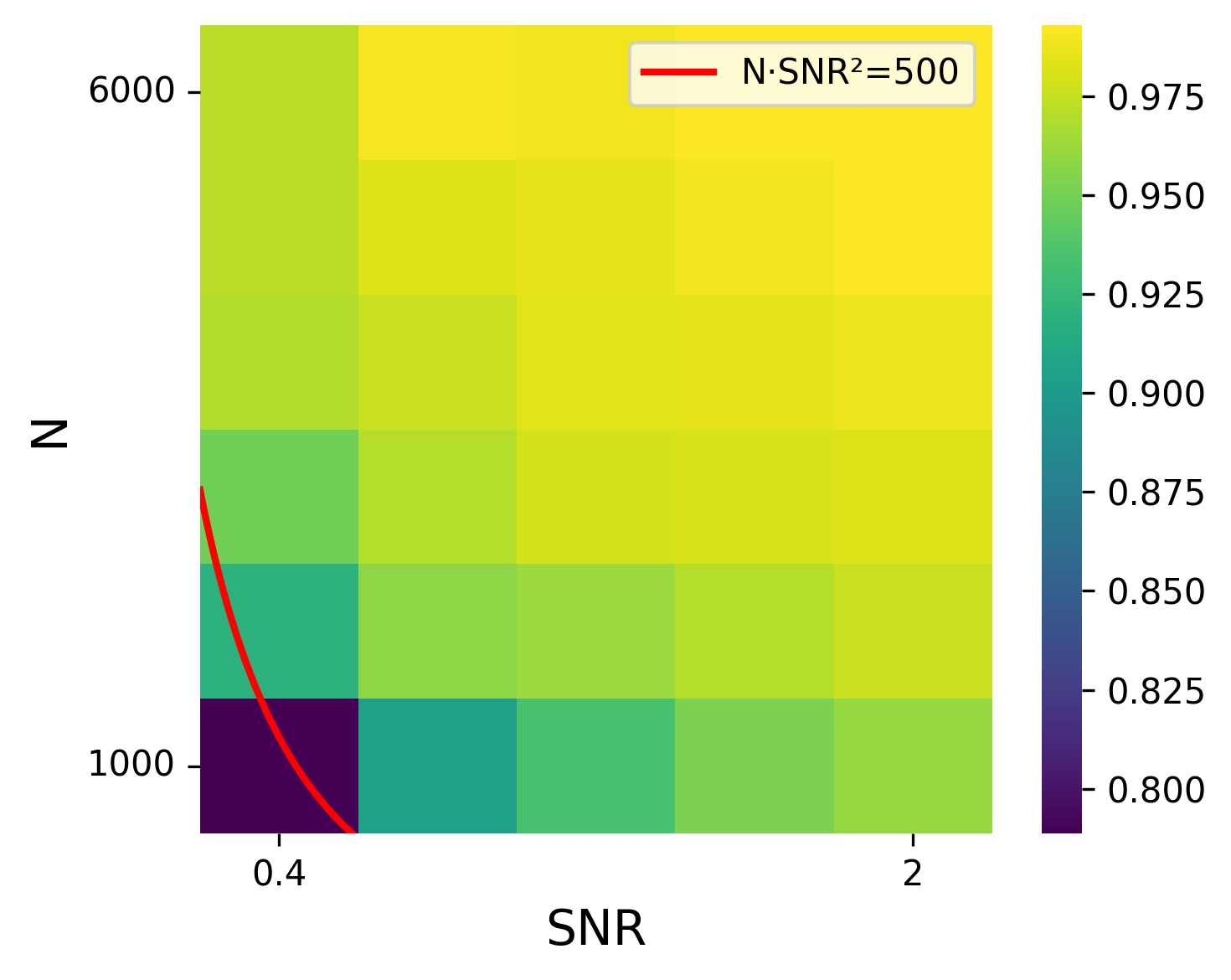}
        \caption{Clean Heatmap}
        \label{fig:minst_test_acc_apgd}
    \end{subfigure}
    \hfill
    \begin{subfigure}{0.24\linewidth}
        \centering
        \includegraphics[width=\linewidth]{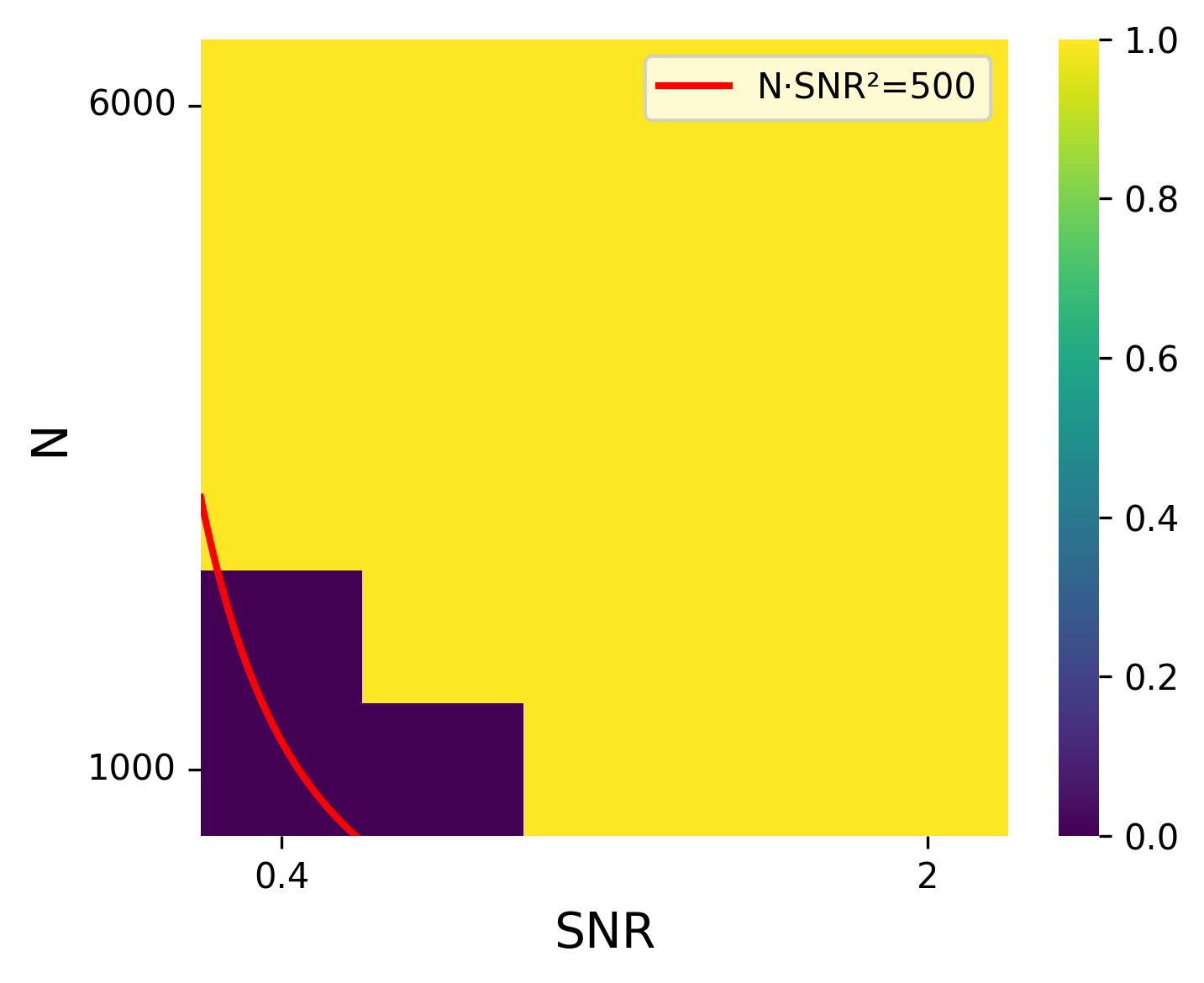}
        \caption{Clean Cutoff}
        \label{fig:minst_test_acc_apgd_cutoff}
    \end{subfigure}
    \hfill
    \begin{subfigure}{0.24\linewidth}
        \centering
        \includegraphics[width=\linewidth]{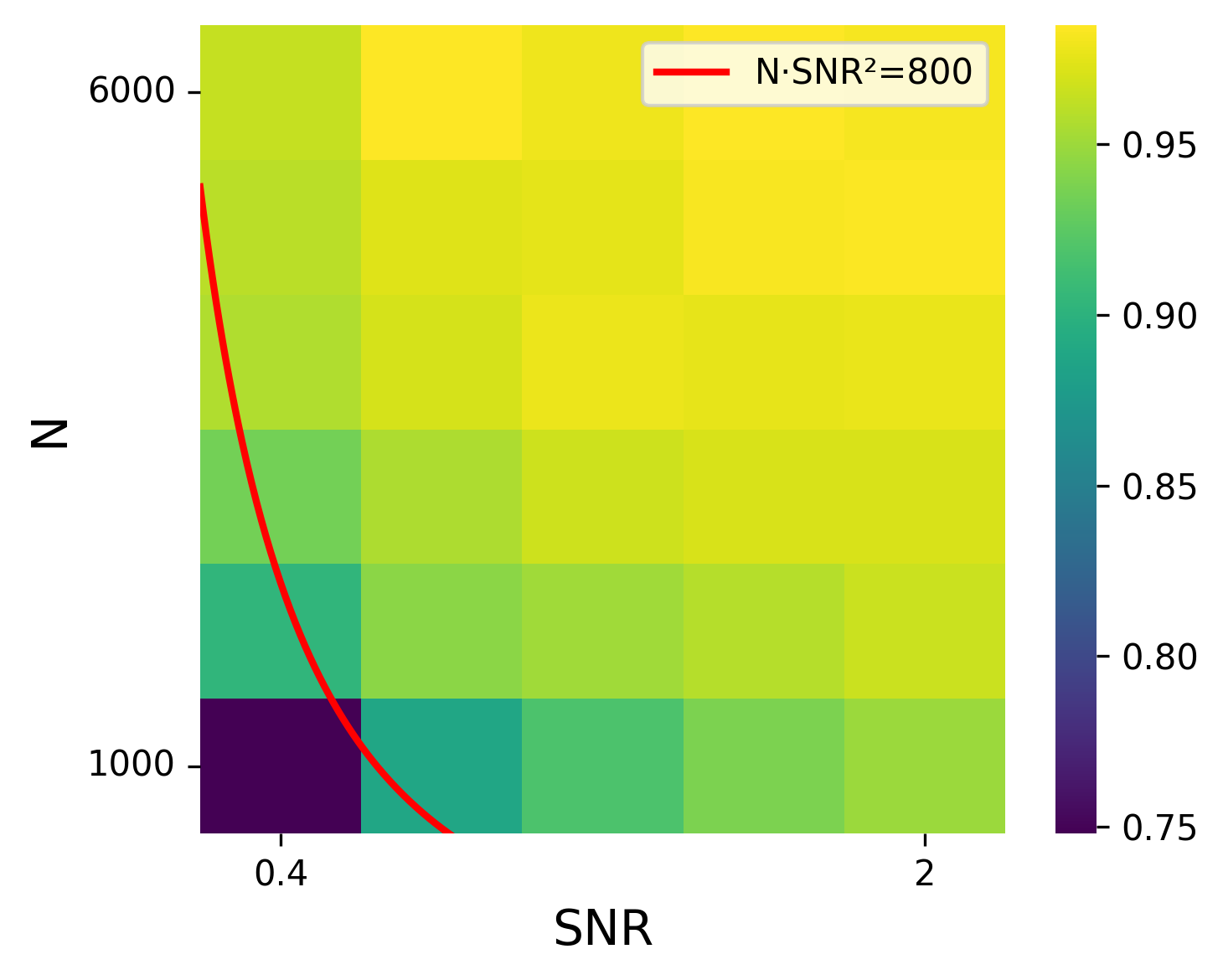}
        \caption{Robust Heatmap}
        \label{fig:minst_test_acc_ad_apgd}
    \end{subfigure}
    \hfill
    \begin{subfigure}{0.24\linewidth}
        \centering
        \includegraphics[width=\linewidth]{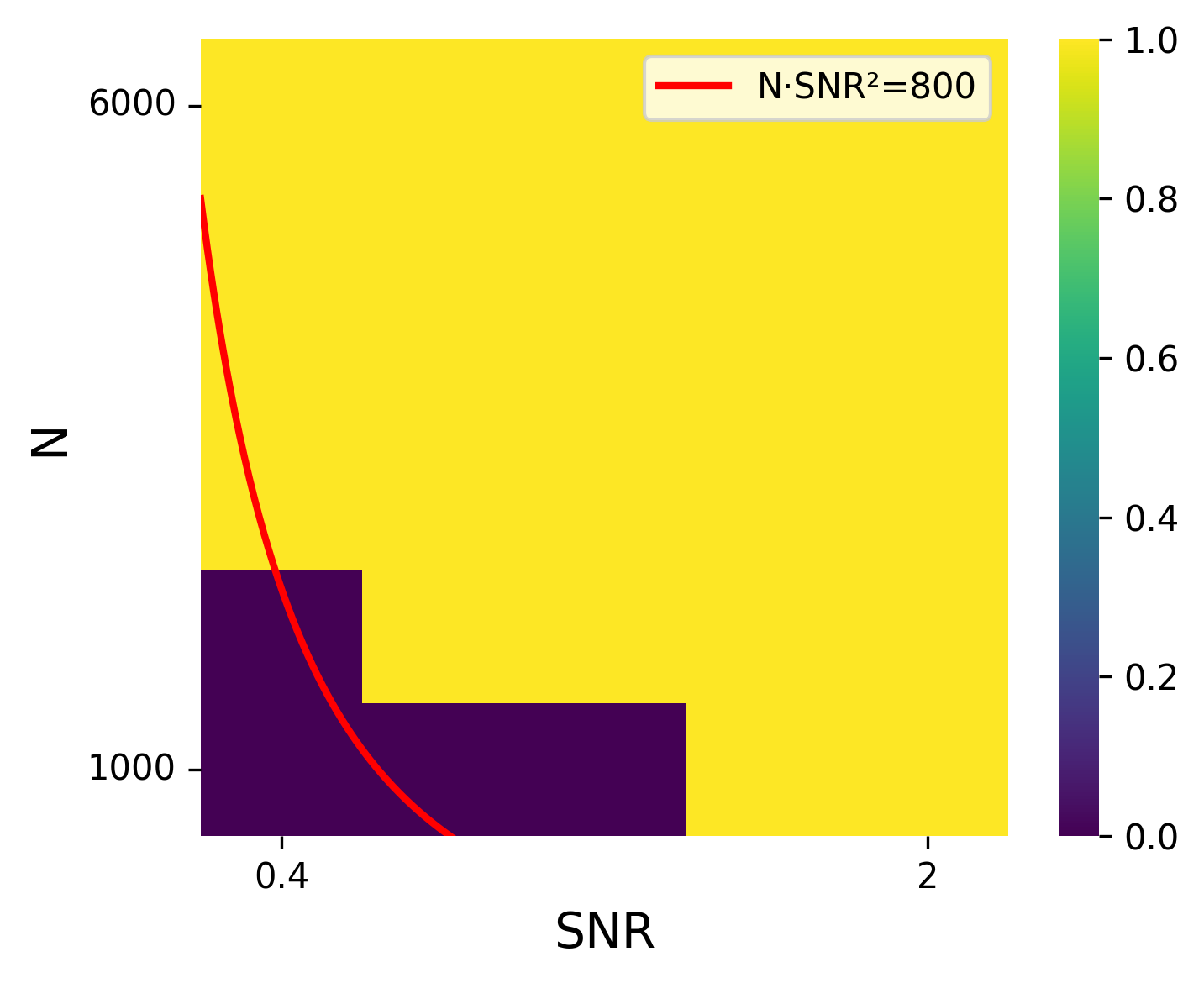}
        \caption{Robust Cutoff}
        \label{fig:minst_test_acc_ad_apgd_cutoff}
    \end{subfigure}
    \caption{Clean and robust test accuracy heatmaps on MNIST with APGD attack across various signal-to-noise ratios ($\operatorname{SNR}$) and sample sizes ($N$).  (b)\&(d) are a heatmap that applies a cutoff value 0.93.}
    \label{fig:minst_heatmaps_agpd}
\end{figure}

\begin{figure}[h]
    \centering
    \begin{subfigure}{0.24\linewidth}
        \centering
        \includegraphics[width=\linewidth]{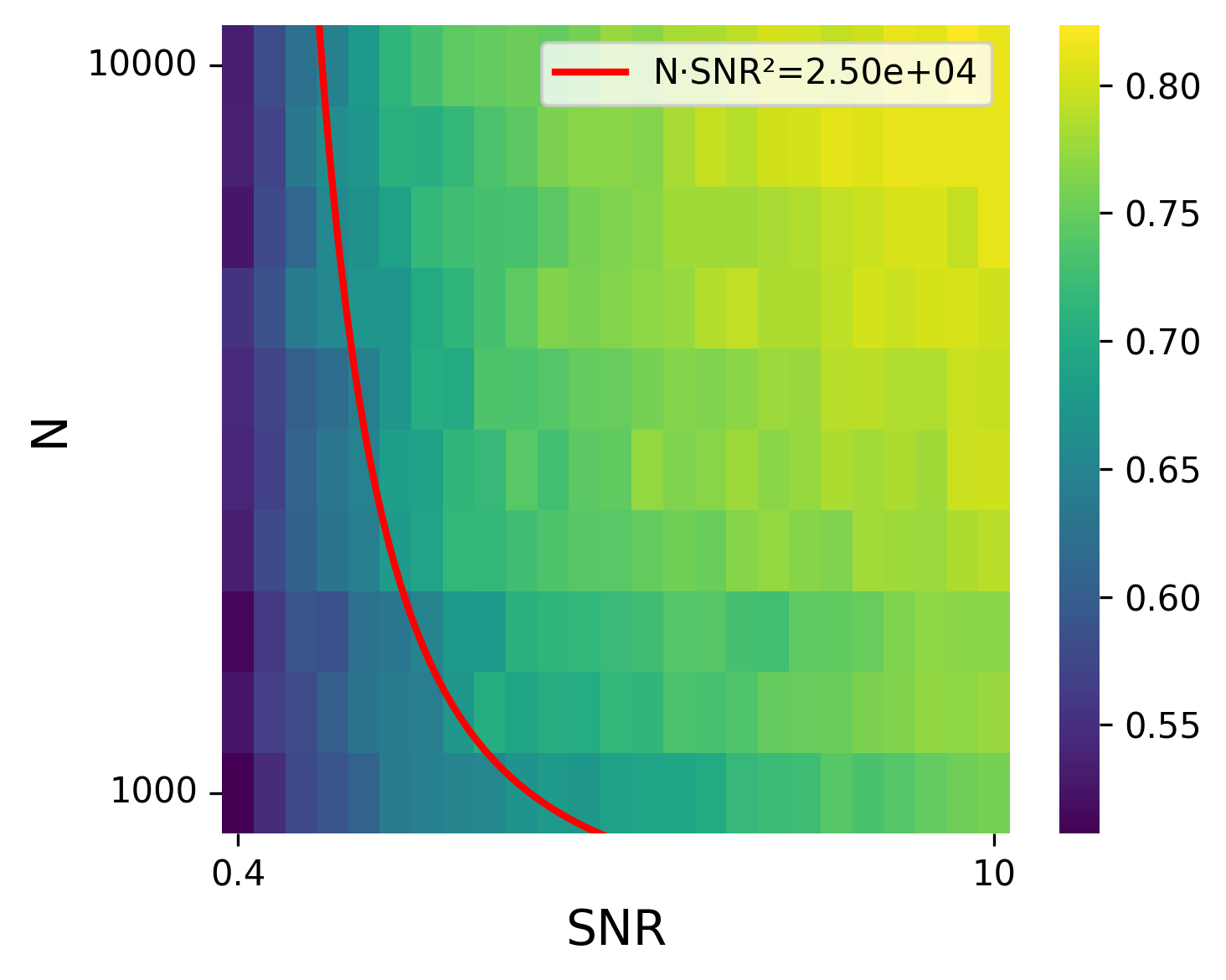}
        \caption{Clean Heatmap}
        \label{fig:test_acc_apgd}
    \end{subfigure}
    \hfill
    \begin{subfigure}{0.24\linewidth}
        \centering
        \includegraphics[width=\linewidth]{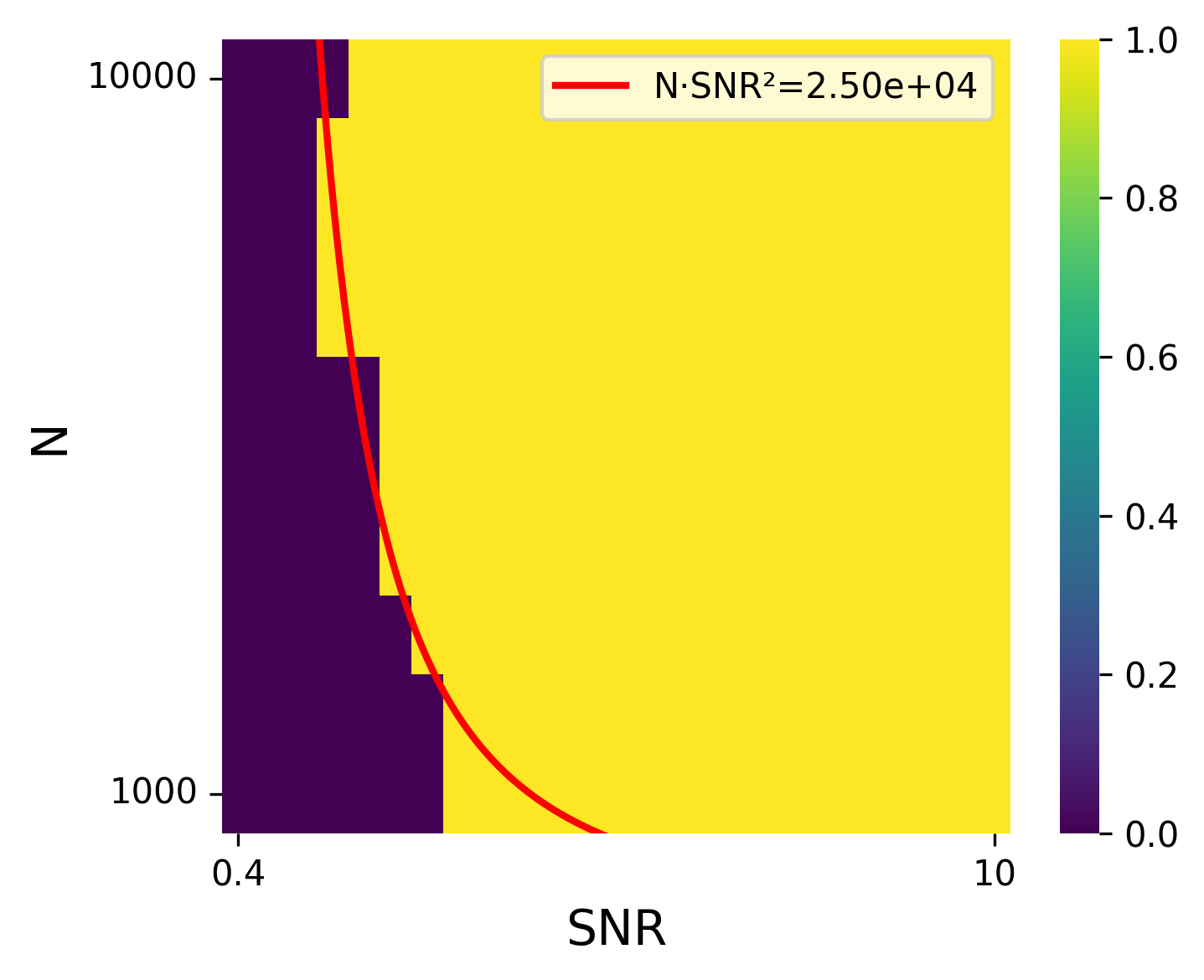}
        \caption{Clean Cutoff}
        \label{fig:test_acc_apgd_cutoff}
    \end{subfigure}
    \hfill
    \begin{subfigure}{0.24\linewidth}
        \centering
        \includegraphics[width=\linewidth]{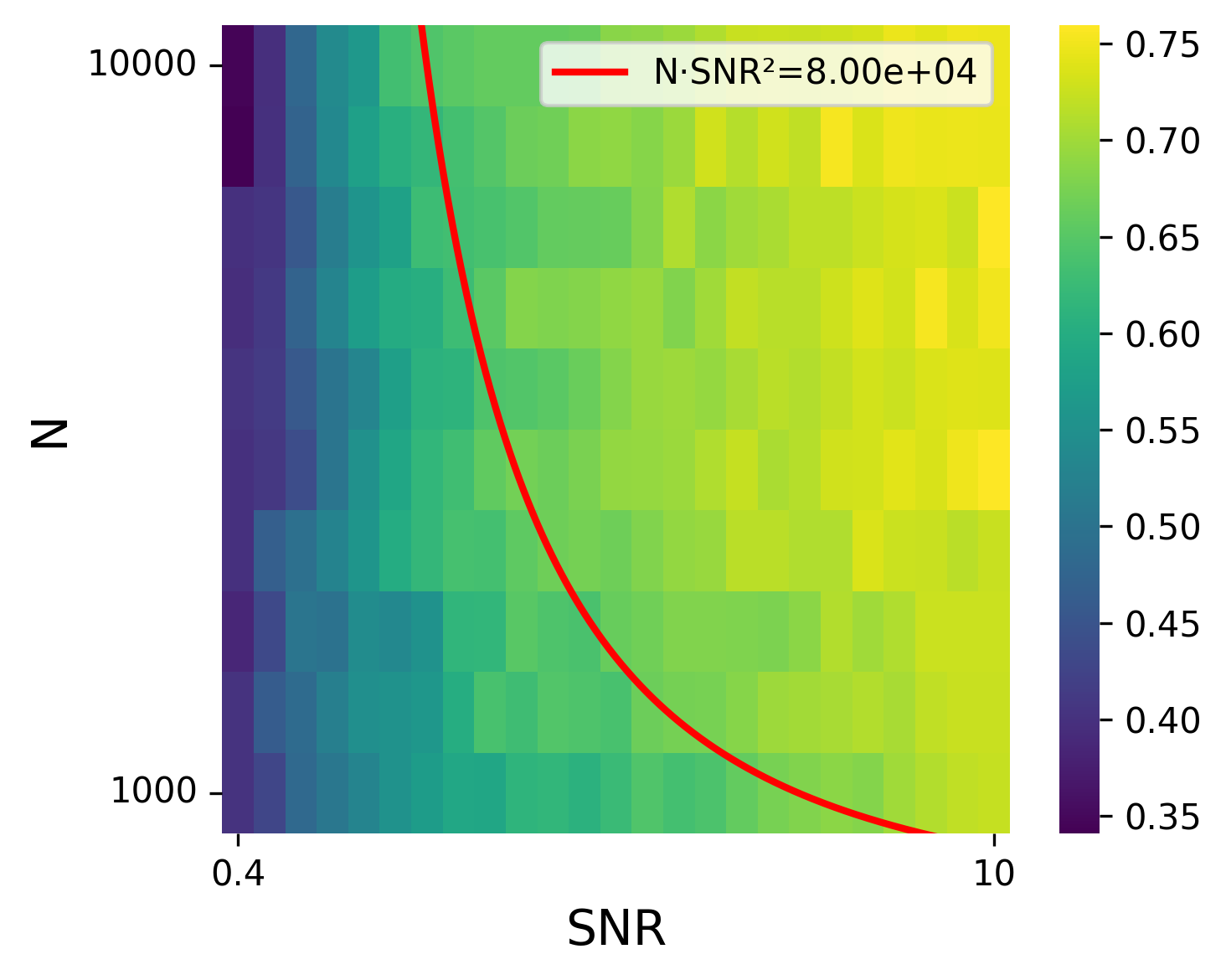}
        \caption{Robust Heatmap}
        \label{fig:test_acc_ad_apgd}
    \end{subfigure}
    \hfill
    \begin{subfigure}{0.24\linewidth}
        \centering
        \includegraphics[width=\linewidth]{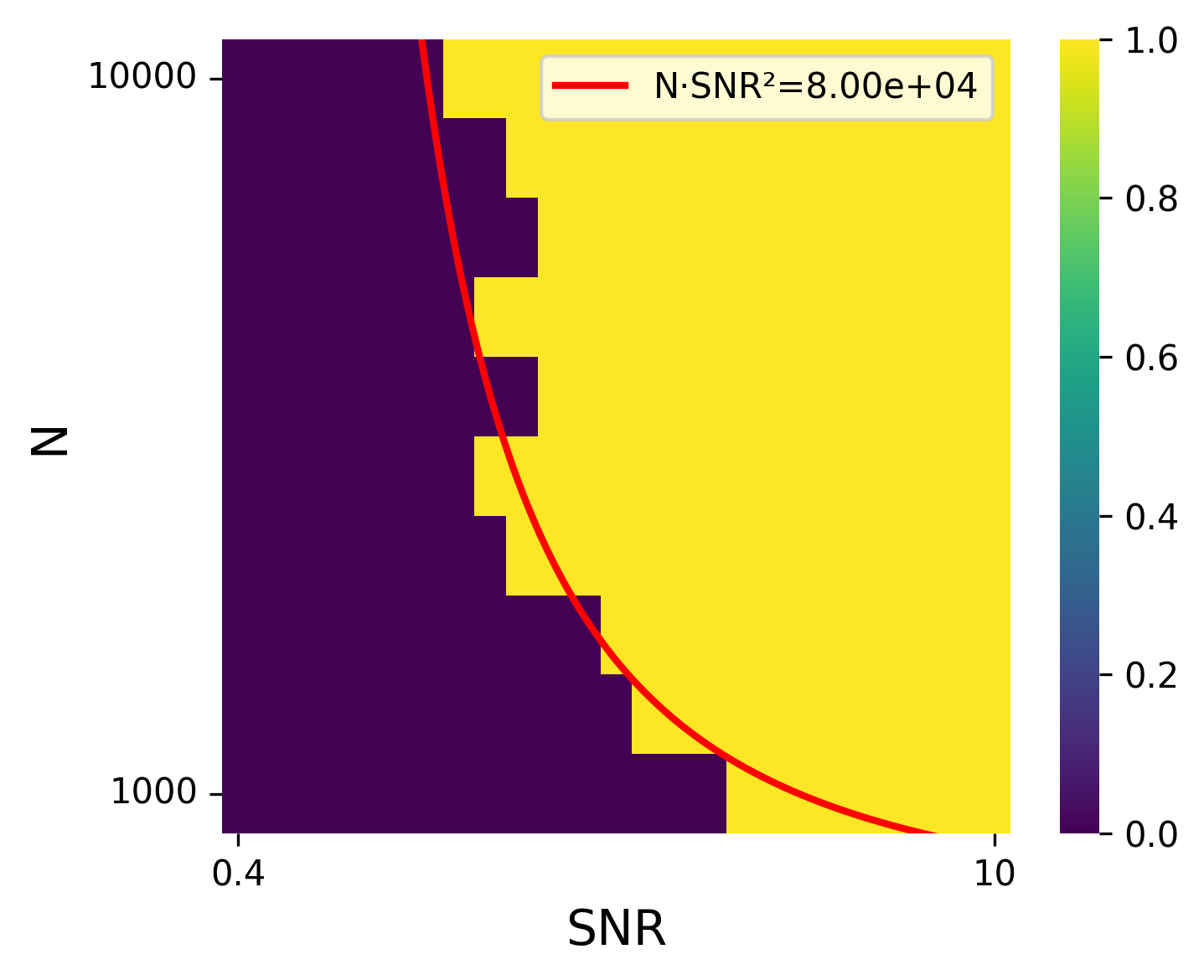}
        \caption{Robust Cutoff}
        \label{fig:test_acc_ad_apgd_cutoff}
    \end{subfigure}
    \caption{Clean and robust test accuracy heatmaps on CIFAR-10 with APGD attack across various signal-to-noise ratios ($\operatorname{SNR}$) and sample sizes ($N$). (b)\&(d) are a heatmap that applies a cutoff value 0.65.}
    \label{fig:cifar10_heatmaps_agpd}
\end{figure}

\begin{figure}[h]
    \centering
    \begin{subfigure}{0.24\linewidth}
        \centering
        \includegraphics[width=\linewidth]{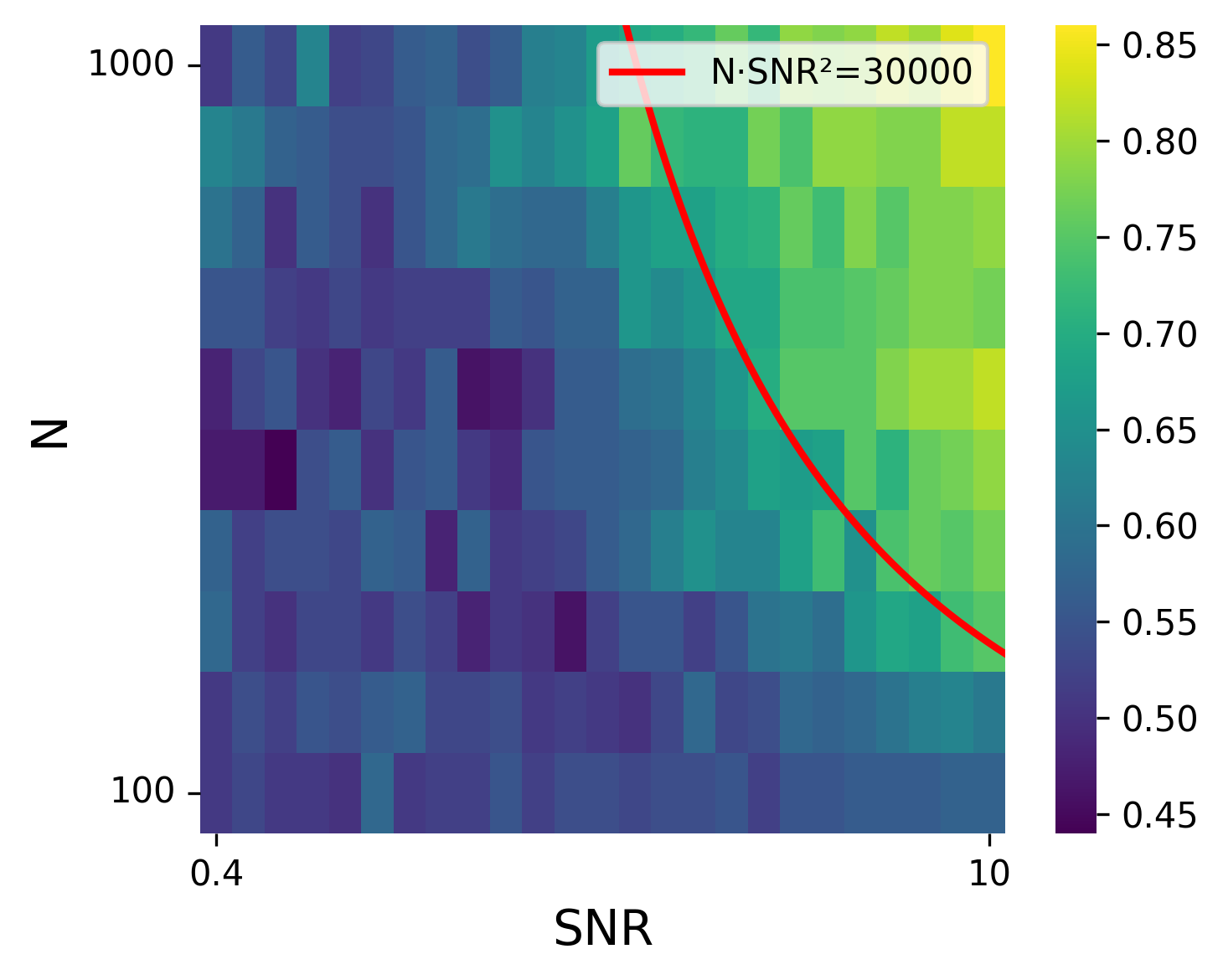}
        \caption{Clean Heatmap}
        \label{fig:test_acc_apgd}
    \end{subfigure}
    \hfill
    \begin{subfigure}{0.24\linewidth}
        \centering
        \includegraphics[width=\linewidth]{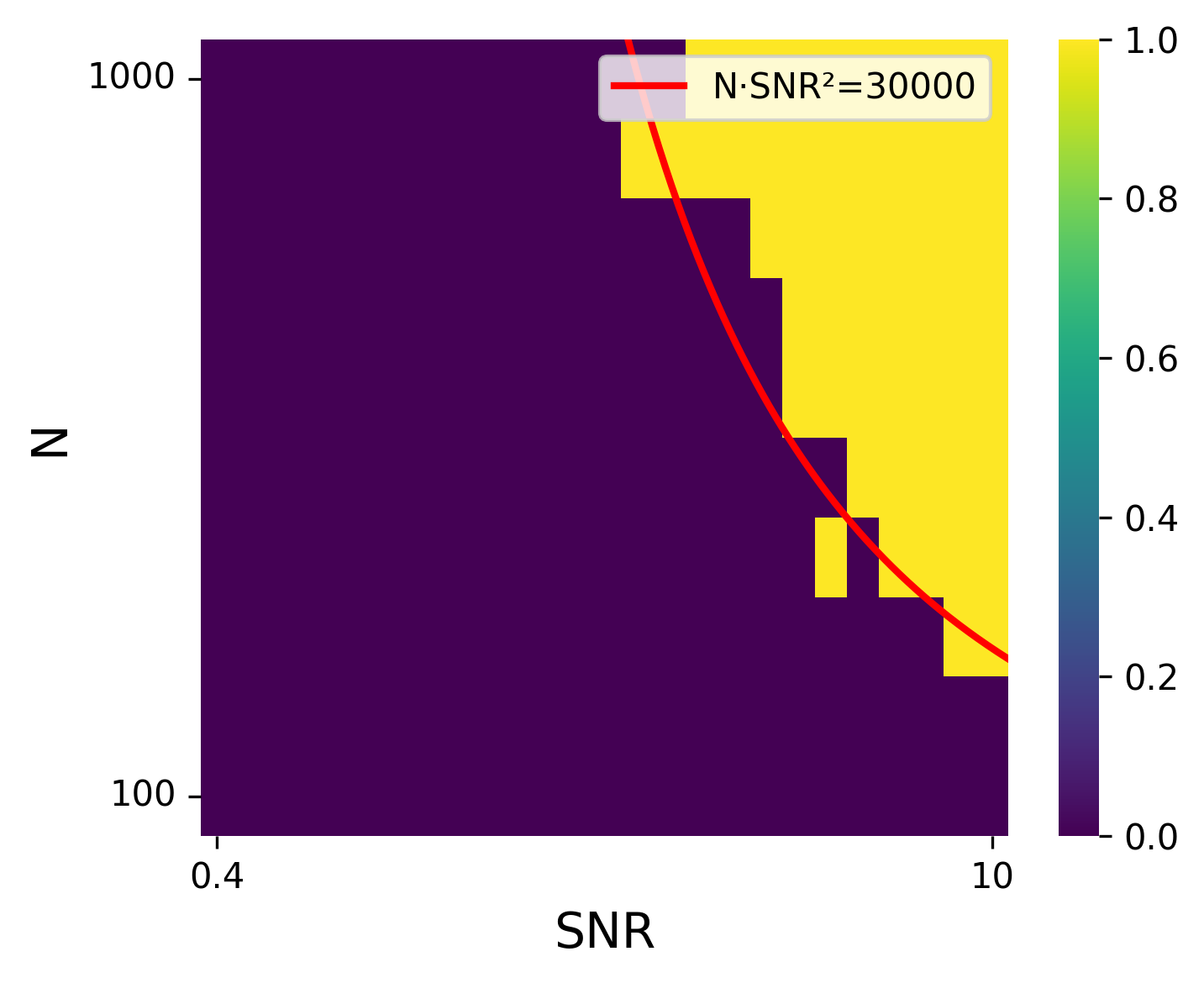}
        \caption{Clean Cutoff}
        \label{fig:test_acc_apgd_cutoff}
    \end{subfigure}
    \hfill
    \begin{subfigure}{0.24\linewidth}
        \centering
        \includegraphics[width=\linewidth]{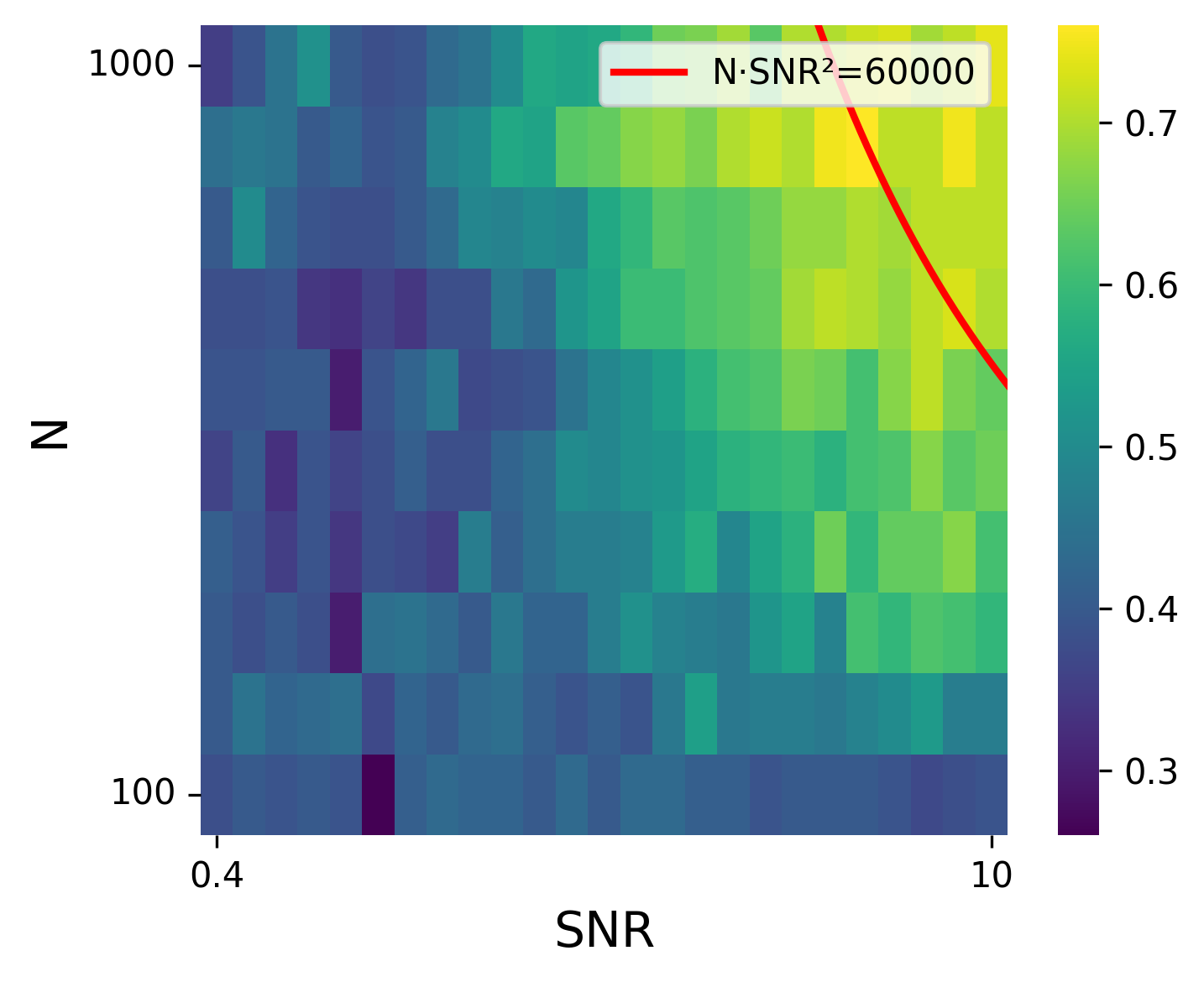}
        \caption{Robust Heatmap}
        \label{fig:test_acc_ad_apgd}
    \end{subfigure}
    \hfill
    \begin{subfigure}{0.24\linewidth}
        \centering
        \includegraphics[width=\linewidth]{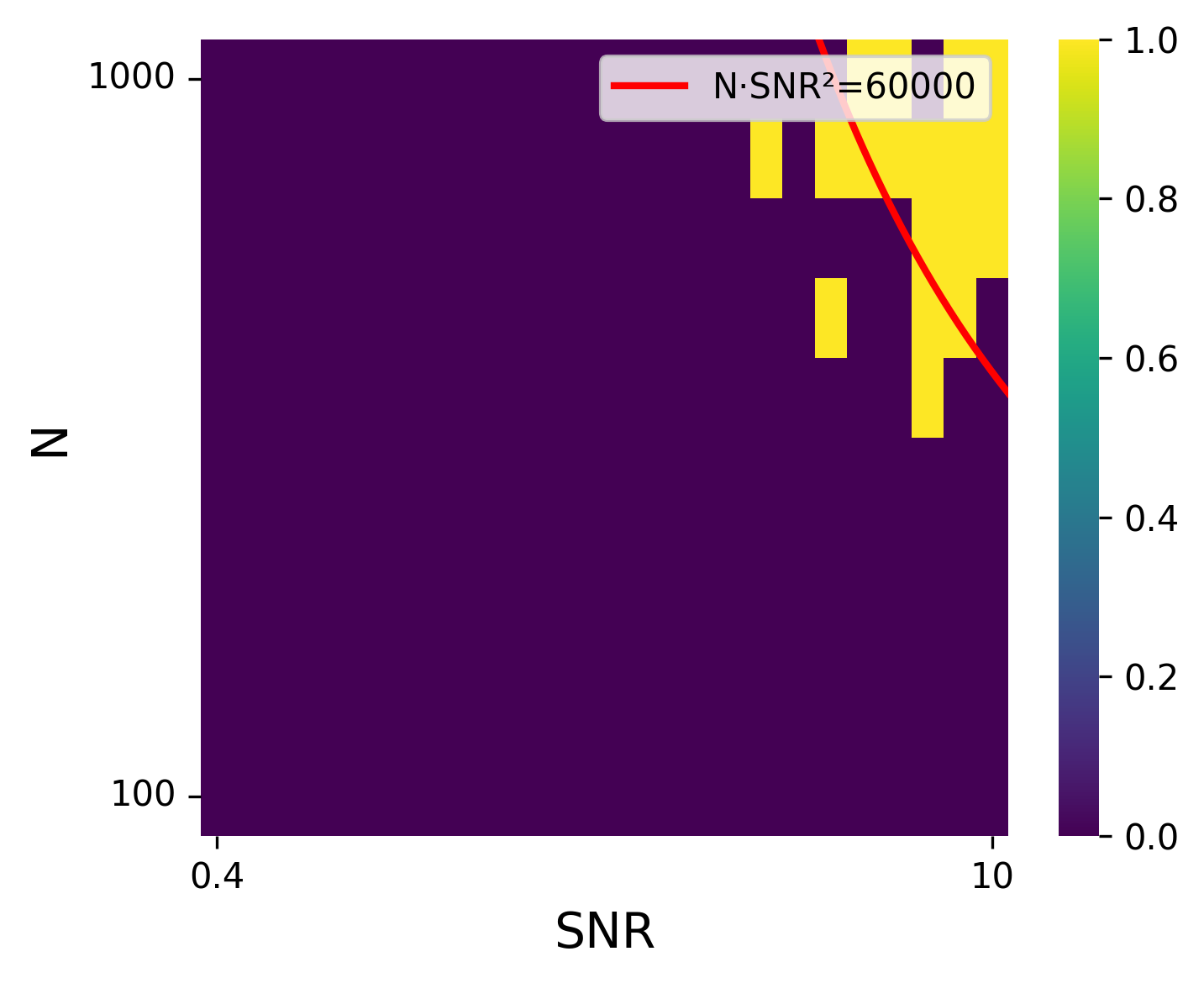}
        \caption{Robust Cutoff}
        \label{fig:test_acc_ad_apgd_cutoff}
    \end{subfigure}
    \caption{Clean and robust test accuracy heatmaps on Tiny Imagenet with APGD attack across various signal-to-noise ratios ($\operatorname{SNR}$) and sample sizes ($N$). (b)\&(d) are a heatmap that applies a cutoff value 0.70.}
    \label{fig:imagenet_heatmaps_agpd}
\end{figure}

\subsection{Additional Experiments on Multi-Norm Attacks}\label{sec:b3}

In this section, we follow the same experimental setup as in Section~\ref{sec:6} for the MNIST dataset and further extend our evaluation by conducting additional experiments on both MNIST and CIFAR-10 under multi-norm attacks. These results demonstrate that although the model’s robust test accuracy decreases under multi-norm attacks, our theoretical insights continue to hold.

\textbf{Experiments setting.} We perform adversarial training under a multi-norm PGD attack model spanning $l_1$,$l_2$ and $l_{\infty}$ perturbations.

For the \textbf{MNIST} dataset, we consider a base attack strength of \(\text{eps}=\frac{\tau_2}{\|\boldsymbol{\mu}\|_2}=0.05\), and set \((\frac{\tau_1}{\|\boldsymbol{\mu}\|_2},\frac{\tau_2}{\|\boldsymbol{\mu}\|_2},\frac{\tau_{\infty}}{\|\boldsymbol{\mu}\|_2})=(\text{eps}*20,\text{eps},\text{eps}/30)\). 
We vary vary the number of training samples $N$ from $1000$ to $6000$ and SNR from $0.4$ to $2$. 

For the \textbf{CIFAR-10} dataset, we set a weaker base attack strength of $\text{eps}=\frac{\tau_2}{\|\boldsymbol{\mu}\|_2} = 0.01$, and set\((\frac{\tau_1}{\|\boldsymbol{\mu}\|_2},\frac{\tau_2}{\|\boldsymbol{\mu}\|_2},\frac{\tau_{\infty}}{\|\boldsymbol{\mu}\|_2})=(\text{eps}*20,\text{eps},\text{eps}/30)\). 
We vary vary the number of training samples $N$ from $1000$ to $10000$ and the SNR from $0.4$ to $10$.

\textbf{Experiments results.}  The clean and robust test accuracies on MNIST and CIFAR-10 are collected and presented as heatmaps in Figures~\ref{fig:multi}. We observe that although the robust test accuracy decreases under multi-norm attacks, a phase transition phenomenon still persists, and increasing either the sample size $N$ or the SNR further reduces both clean and robust test error in a manner fully aligned with our theoretical results.

\begin{figure}[h]
    \centering
    \begin{subfigure}{0.24\linewidth}
        \centering
        \includegraphics[width=\linewidth]{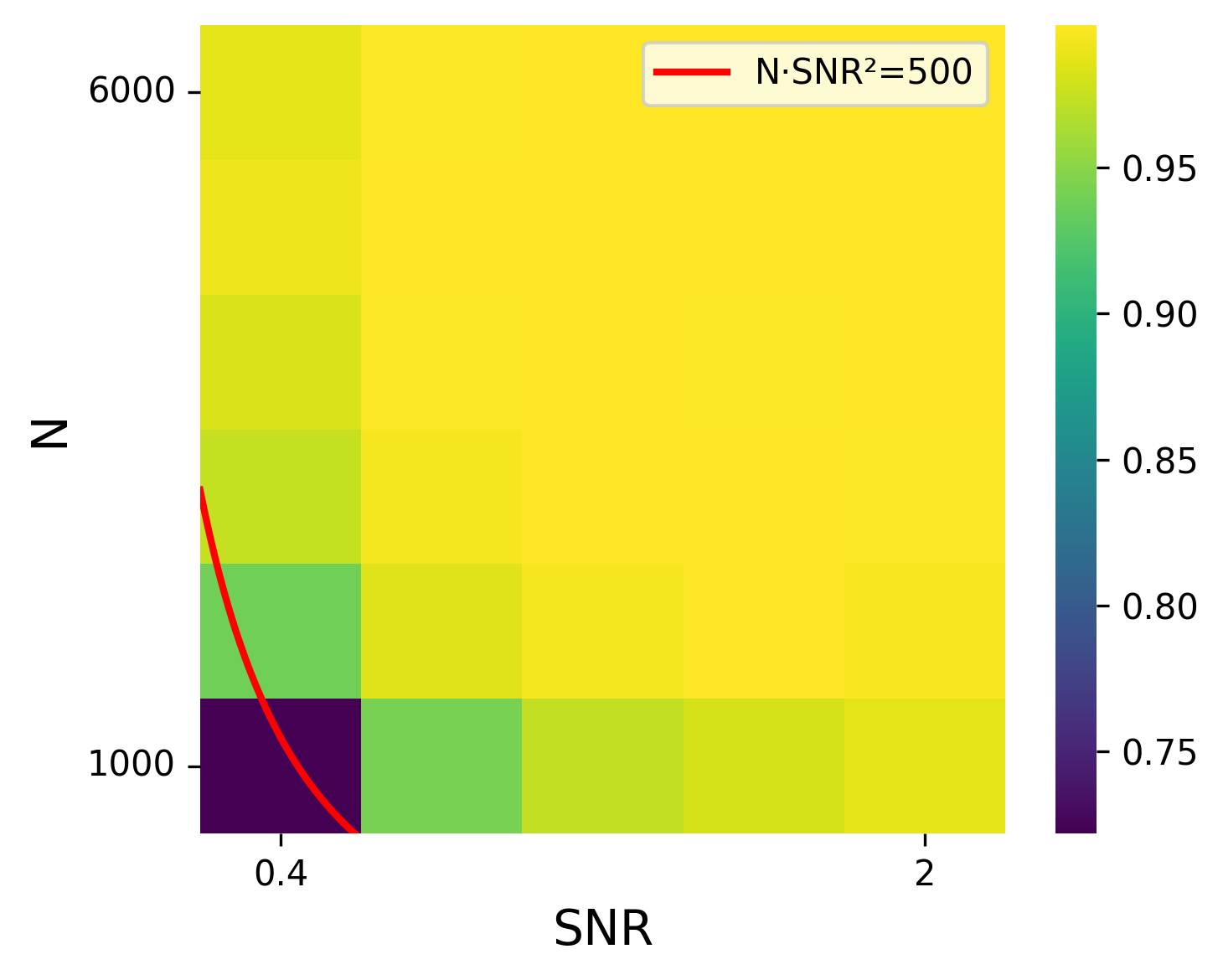}
        \caption{Clean Heatmap}
        \label{fig:test_acc_Mmulti}
    \end{subfigure}
    \hfill
    \begin{subfigure}{0.24\linewidth}
        \centering
        \includegraphics[width=\linewidth]{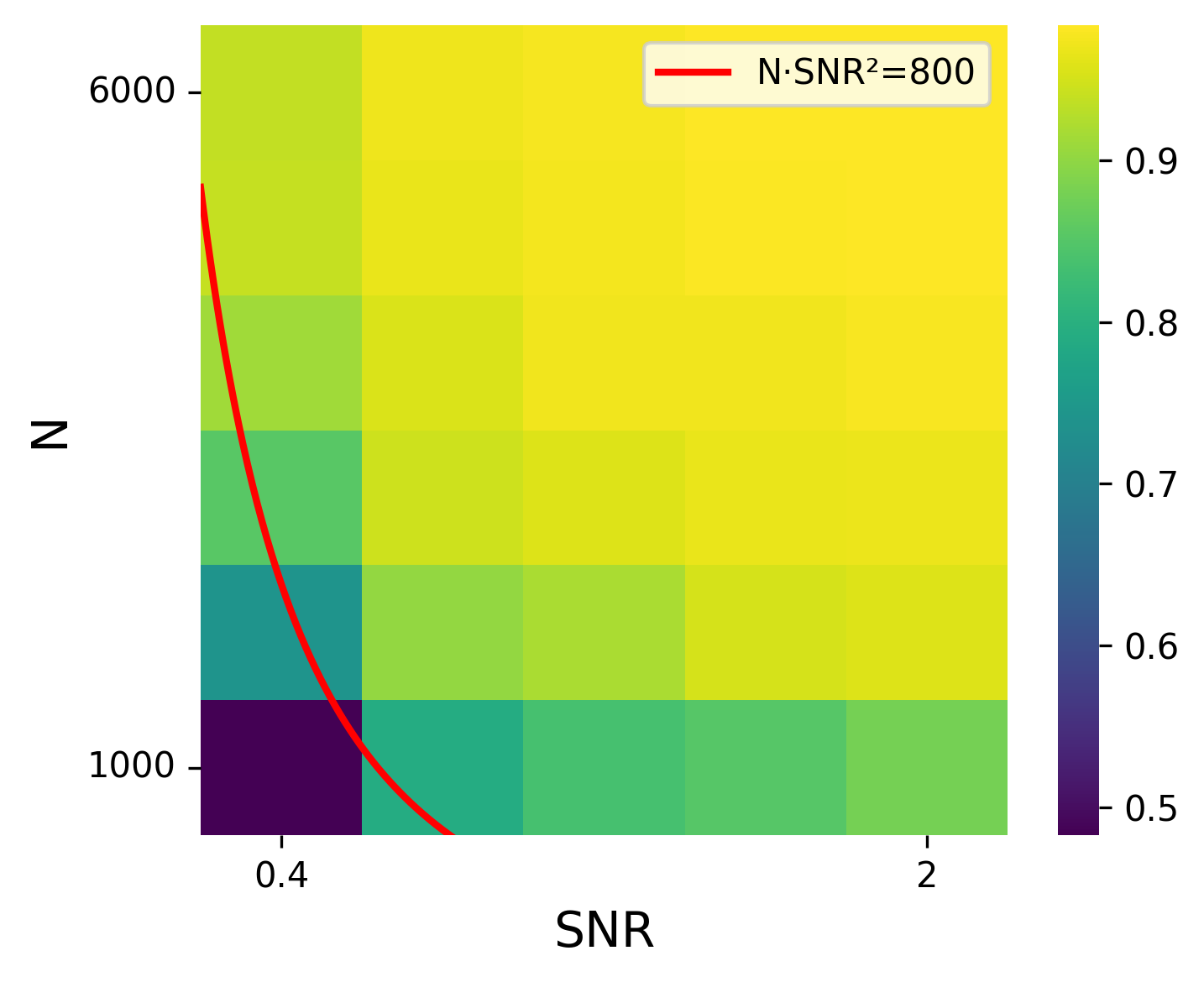}
        \caption{Robust Heatmap}
        \label{fig:test_acc_ad_Mmulti}
    \end{subfigure}
    \hfill
    \begin{subfigure}{0.24\linewidth}
        \centering
        \includegraphics[width=\linewidth]{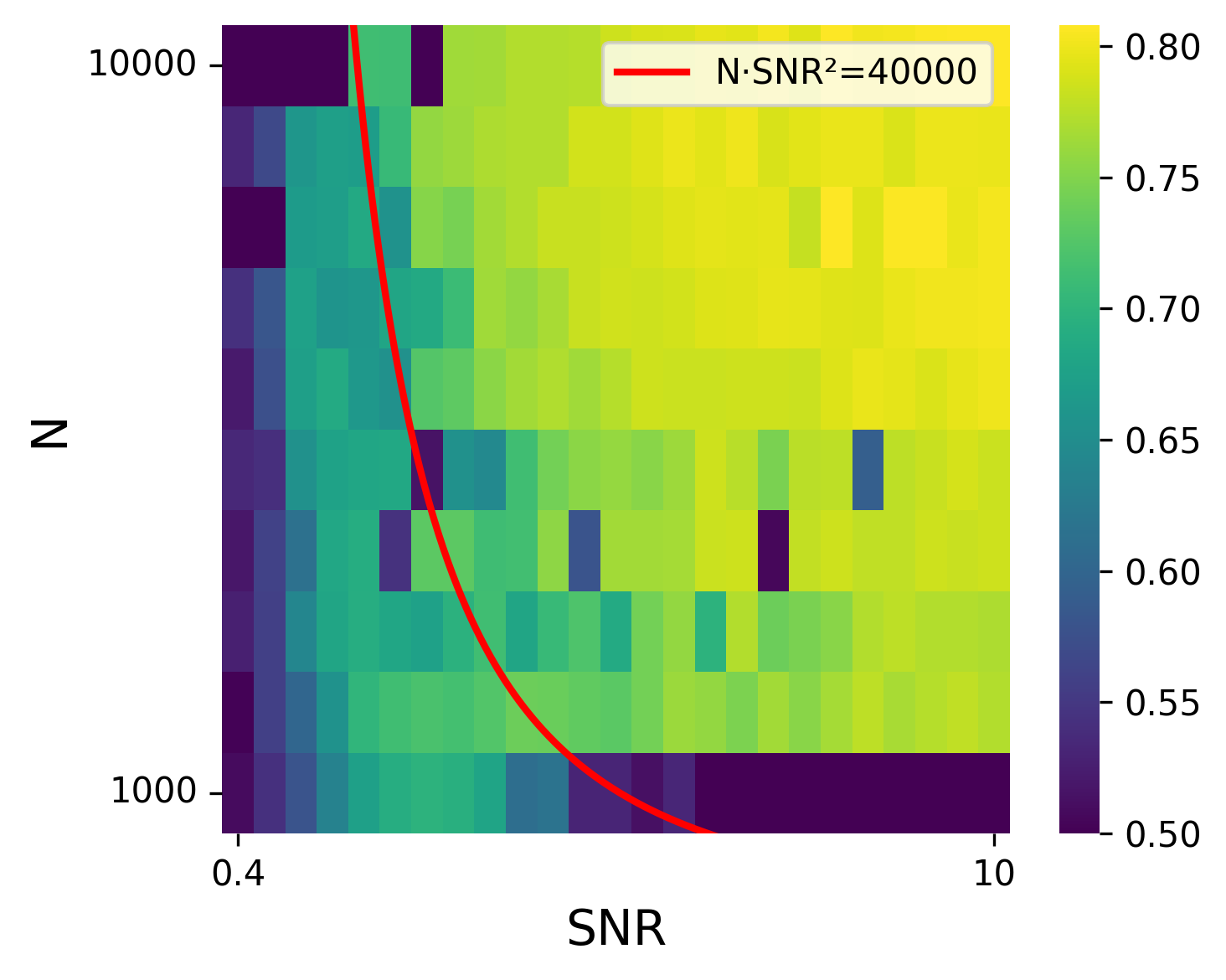}
        \caption{Clean Heatmap}
        \label{fig:test_acc_Cmulti}
    \end{subfigure}
    \hfill
    \begin{subfigure}{0.24\linewidth}
        \centering
        \includegraphics[width=\linewidth]{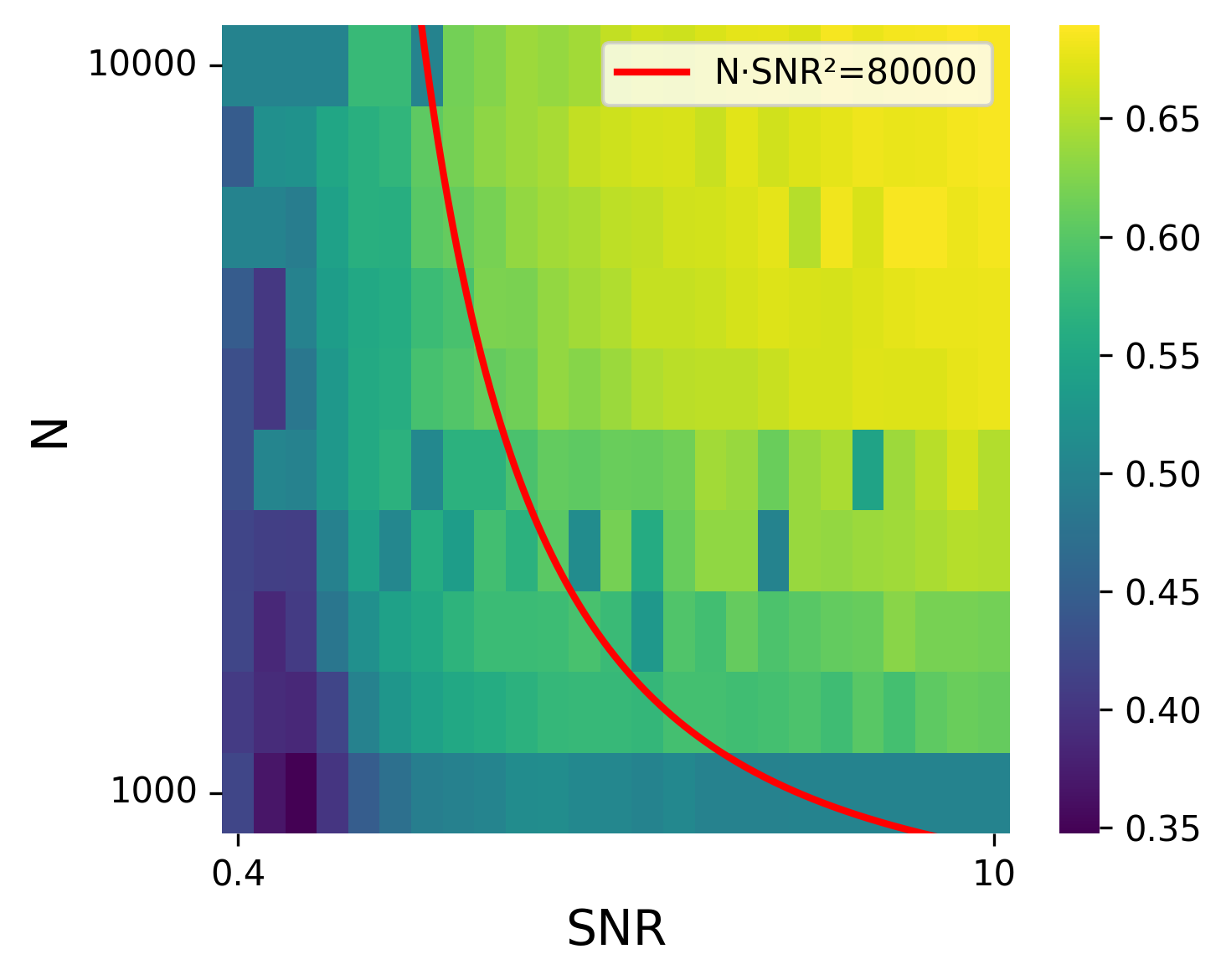}
        \caption{Robust Heatmap}
        \label{fig:test_acc_ad_Cmulti}
    \end{subfigure}
    \caption{Clean and robust test accuracy under multi-norm attack adversarial training across various signal-to-noise ratios ($\operatorname{SNR}$) and sample sizes ($N$). (a)\&(b): results on MNIST data; (c)\&(d): results on CIFAR-10 data.}
    \label{fig:multi}
\end{figure}

\subsection{Additional Experiments on realistic ViT}\label{sec:b4}

In this section, we conduct real-world experiments on image classification benchmarks, including MNIST, CIFAR-10, and Tiny-ImageNet, using a realistic ViT architecture~\citep{dosovitskiy2020image}. The results show that when our model is scaled from the simplified two-layer ViT to a full-fledged ViT model, our theoretical insights continue to hold.

\textbf{Experiments setting.}
We adopt \texttt{google/vit-base-patch16-224-in21k}~\citep{dosovitskiy2020image} as the backbone model to extend our analysis to real-world scenarios. To align with our theoretical setting, we freeze all parameters except for the QKV matrices in the final attention layer. This setup effectively treats all preceding layers as a fixed feature-extraction encoder, whose output serves as the input to the last Transformer layer.

We conduct adversarial training on MNIST, CIFAR-10, and Tiny-ImageNet, using PGD as the threat model with a perturbation radius of $\frac{\|\boldsymbol \mu\|_2}{20}$ and 5 attack steps.

\textbf{Experiments results.} The clean and robust test accuracies of ViT-base on MNIST, CIFAR-10 and Tiny-ImageNet are collected and presented as heatmaps in Figures~\ref{fig:vitpgd}. We observe that when the model is scaled from the simplified two-layer ViT to a realistic ViT architecture, a phase transition phenomenon still persists, and increasing either the sample size $N$ or the SNR further reduces both clean and robust test error in a manner fully aligned with our theoretical results.

\begin{figure*}[h]
    \centering
    \begin{subfigure}{0.32\textwidth}
        \includegraphics[width=\linewidth]{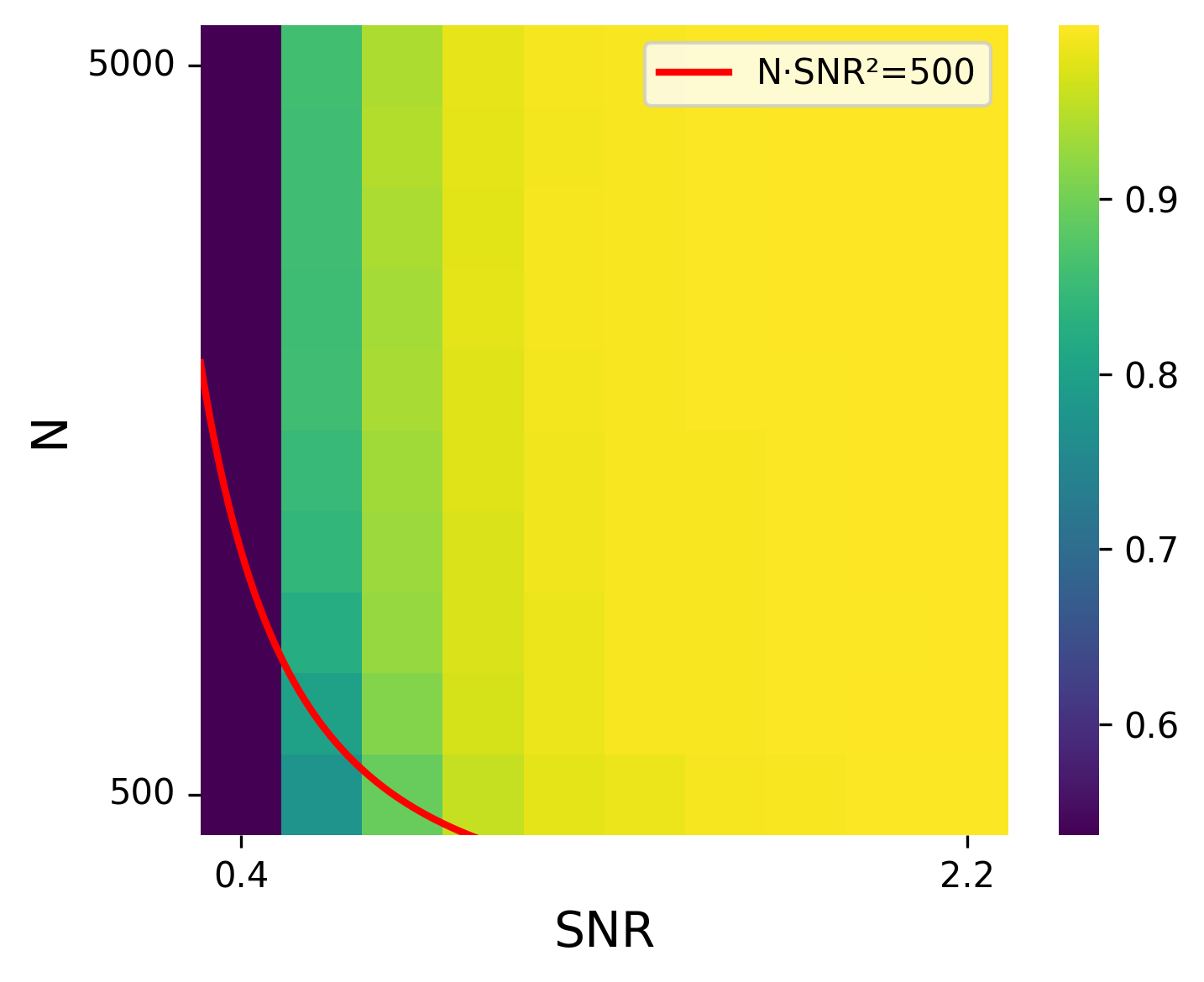}
    \end{subfigure}
    \begin{subfigure}{0.32\textwidth}
        \includegraphics[width=\linewidth]{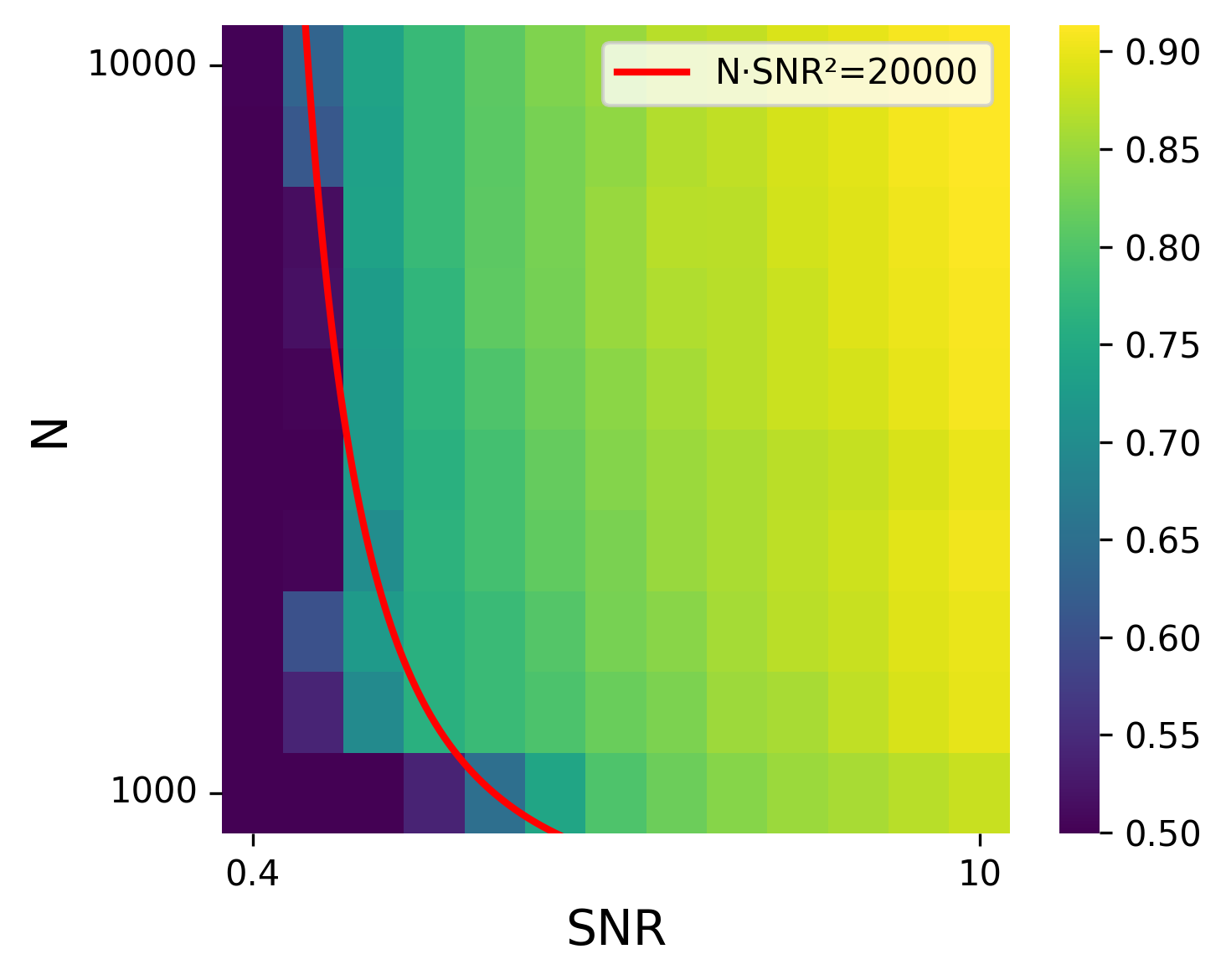}
    \end{subfigure}
    \begin{subfigure}{0.32\textwidth}
        \includegraphics[width=\linewidth]{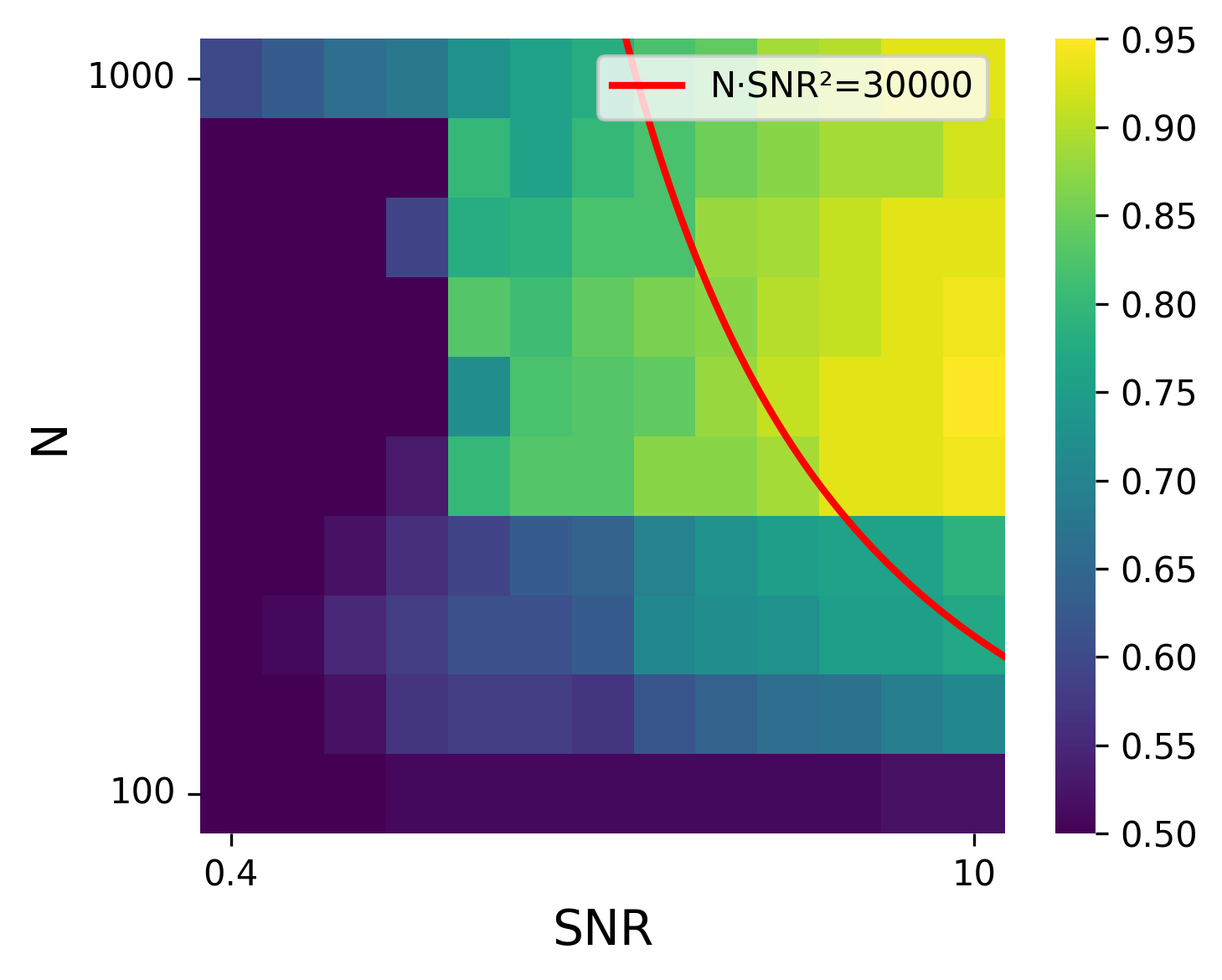}
    \end{subfigure}

    \vspace{0.2cm}

    \begin{subfigure}{0.32\textwidth}
        \includegraphics[width=\linewidth]{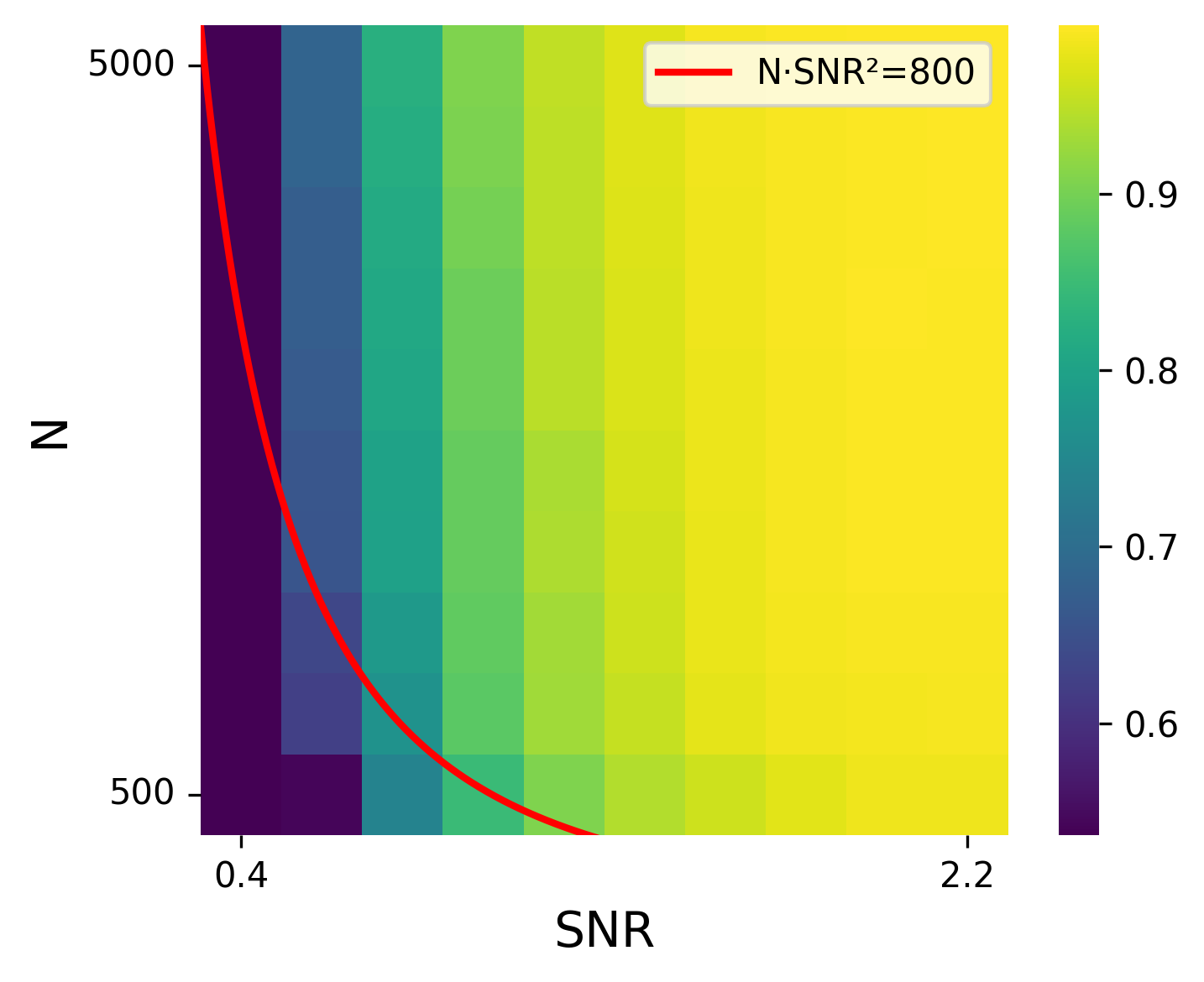}
        \caption{MNIST}
    \end{subfigure}
    \begin{subfigure}{0.32\textwidth}
        \includegraphics[width=\linewidth]{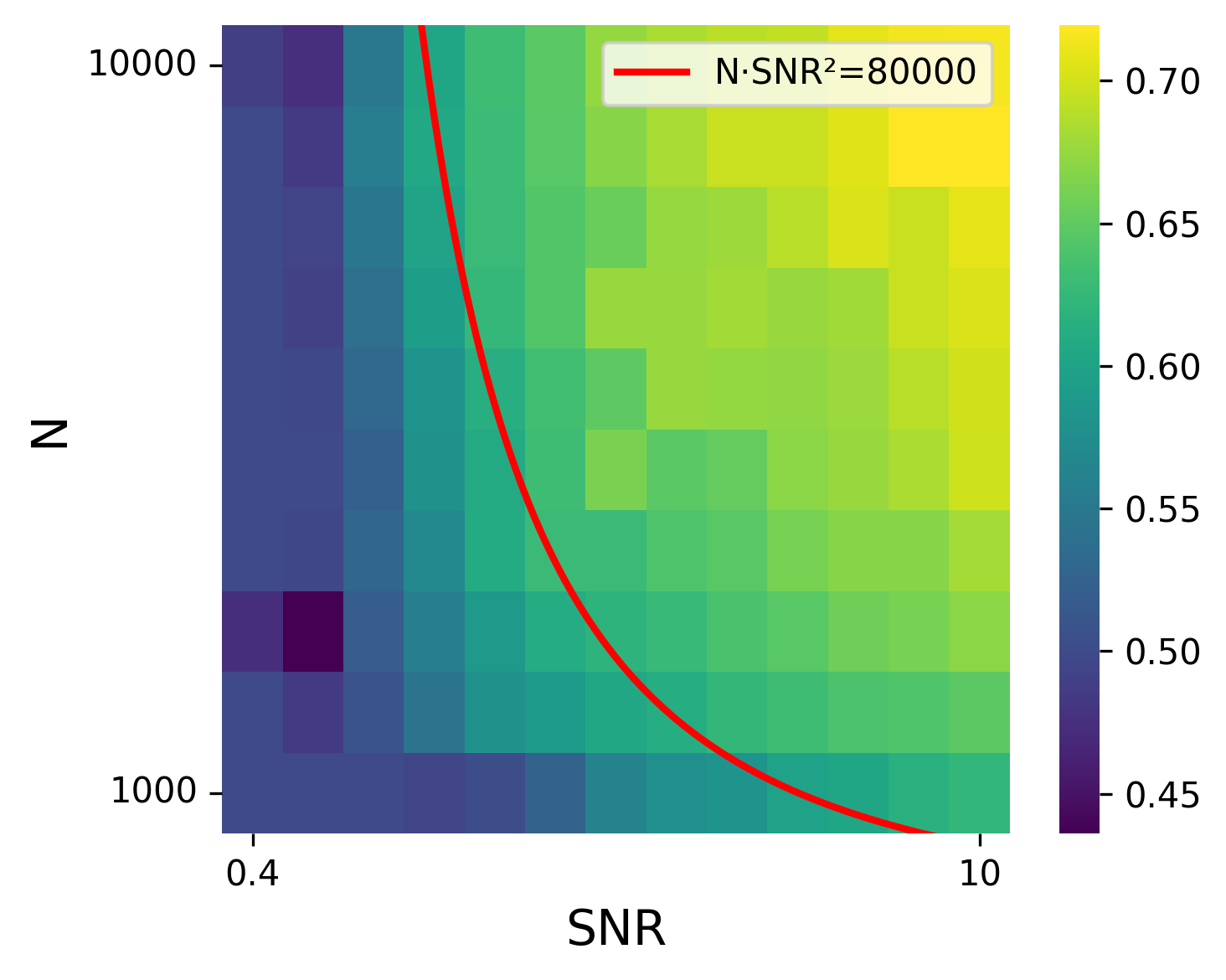}
        \caption{CIFAR-10}
    \end{subfigure}
    \begin{subfigure}{0.32\textwidth}
        \includegraphics[width=\linewidth]{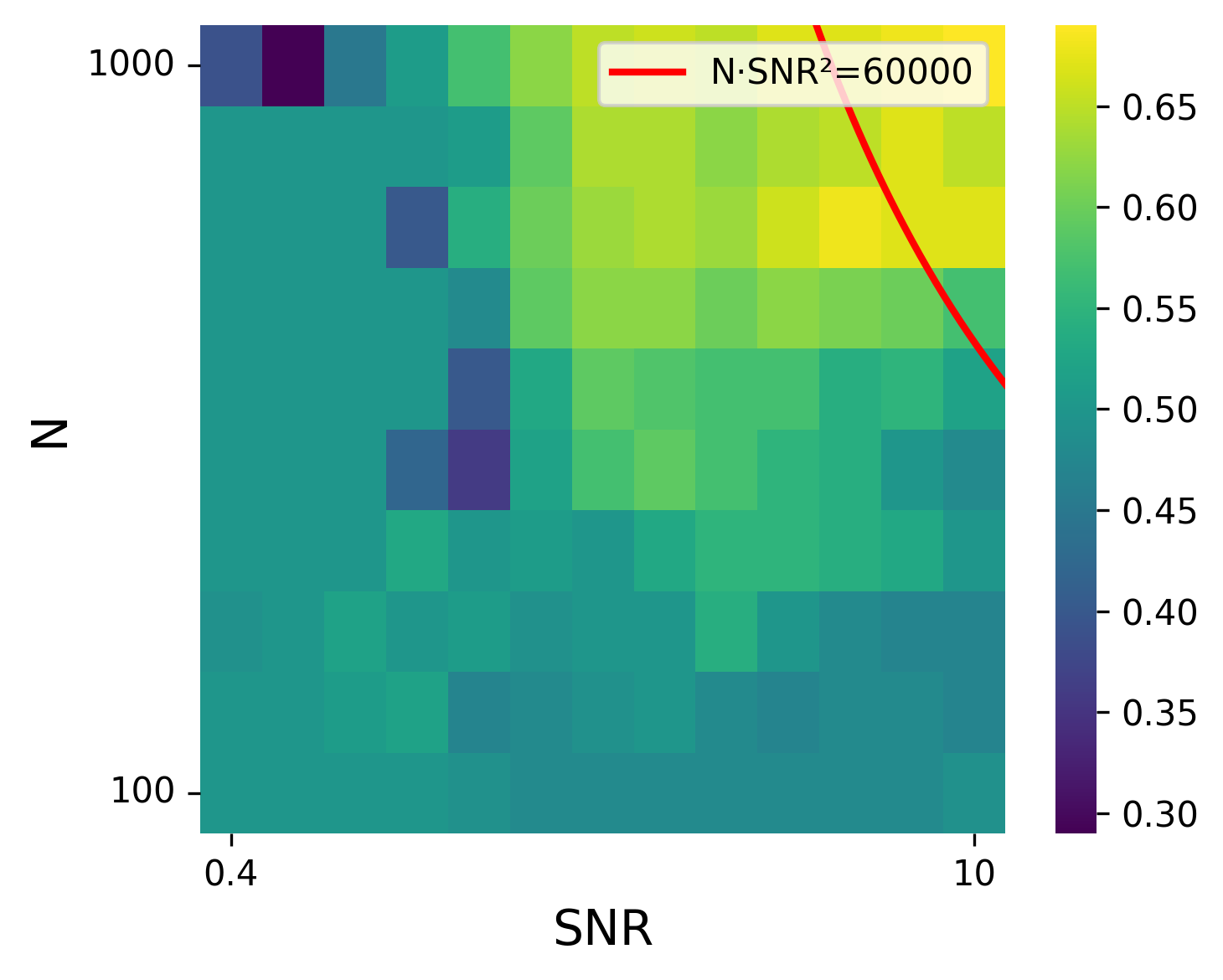}
        \caption{Tiny-ImageNet}
    \end{subfigure}

    \caption{Clean and robust test accuracy of ViT-base under adversarial training across various signal-to-noise ratios (SNR) and sample sizes (N). Top row: clean test accuracy. Bottom row: robust test accuracy. }
    \label{fig:vitpgd}
\end{figure*}

\section{Extension to MHA}\label{sec:mha}
{We let the parameters be \(\theta := \{(W_{Q,h}, W_{K,h}, W_{V,h})\}_{h=1}^H\), where \(W_{Q,h}, W_{K,h} \in \mathbb{R}^{d \times d_h}\) and \(W_{V,h} \in \mathbb{R}^{d \times d_v}\) for each \(h \in [H]\). Here \(H\) denotes the number of attention heads, which we treat as a fixed constant. Under this parameterization, the network can be written as:}

{
\[
f(\mathbf{X}, \theta)=\sum_{h=1}^Hf_h(\mathbf{X}, \theta)
\]
where,
\[
f_h(\mathbf{X}, \theta) = \frac{1}{M} \sum_{l=1}^{M} \varphi(\mathbf{x}_l^\top \mathbf{W}_{Q,h} \mathbf{W}_{K,h}^\top \mathbf{X}^\top) \mathbf{X} \mathbf{W}_{V,h} \boldsymbol{w}_O.
\]
}
{
The gradients in the multi-head attention module,
\(\frac{\partial f_{h}}{\partial W_{K,h}}, \frac{\partial f_{h}}{\partial W_{Q,h}}, \frac{\partial f_{h}}{\partial W_{V,h}}\),
remain unchanged. However, the gradient of the loss with respect to the output of each single head, i.e., \(\frac{\partial \ell}{\partial f_h}\), does change.}

{Intuitively, the model output increases by approximately an \(H\)-fold factor, which causes the scale of the loss \(\ell'\) to decrease accordingly.}

{More concretely, following our analysis of the signal attention head, $\ell'^{(t)}=\frac{1}{M}\pm o(1)$ stay when $t \leq T_2 = \Theta\left(\frac{1}{\eta(\|\mu\|_2 + \tau)^2 \|w_O\|_2^2}\right)$. Thus, this implies $f_h^{(T_2)}(\mathbf{X}, \theta)=o(1)$. The 
\( H \)-fold increase in the multi-head model outputs does not alter this result, so the effect of the changes in \(\frac{\partial \ell}{\partial f_h}\) can be ignored.}

{Therefore, under the MHA setting, the training dynamics of the model still follow those of the single-head attention case, and our conclusions remain unchanged.}
}
\section{Discussion on Multi-Norm Attacks}\label{sec:mul-norm}
 Our analysis covers perturbations under all norm types. For $\ell_\infty$ or $\ell_1$ norms, they can be mapped to $\ell_2$ through standard norm-equivalence. For any $\mathbf{x} \in \mathbb{R}^d$ and $1 \le p \le q \le \infty$, the following inequality holds:
    \begin{equation}
        \|\mathbf{x}\|_q \le d^{(\frac{1}{q} - \frac{1}{p})} \|\mathbf{x}\|_p.
    \end{equation}
    In particular, the $\ell_\infty$ and $\ell_1$ norms satisfy:
    \begin{equation}
        \|\mathbf{x}\|_\infty \le \|\mathbf{x}\|_2 \le \sqrt{d} \|\mathbf{x}\|_\infty, \quad \|\mathbf{x}\|_2 \le \|\mathbf{x}\|_1 \le \sqrt{d} \|\mathbf{x}\|_2.
    \end{equation}
    Thus, an $\ell_\infty$ or $\ell_1$ perturbation budget $\tau$ corresponds to an $\ell_2$ budget scaled by at most $\sqrt{d}$. This implies that all proofs still hold, with the only difference being the perturbation radius $\tau$.

     For the $\ell_0$-norm perturbation model, our theoretical lower bounds implicitly show that it cannot provide benign overfitting guarantees. According to Theorem~\ref{coro:larger_pert}, once the $\ell_2$-norm perturbation radius becomes sufficiently large (i.e., $\tau \ge \|\boldsymbol{\mu}\|_2$), the model incurs a large robust test error. This implies that even an $\ell_0$-norm radius of 1 can still lead to substantial robust test error in the worst case.

\section{Basic Calculation}
\subsection{Notion}
\begin{table}[h]
\centering
\caption{Notations}
\begin{tabular}{ll}
\toprule
\textbf{Symbols} & \textbf{Definitions} \\ \toprule
$x_{n,i}$ & the i-th token in the n-th training sample \\ 
 & if $i \in [M] \backslash \{1\}$, $x_{n,i} = \boldsymbol{\xi}_{n,i}$. \\ \hline
$\varphi_{n,i}^{(t)}$ & the i-th row of attention for the n-th sample, i.e., $\varphi_{n,i}^{(t)} := \varphi(x_{n,i}^\top \mathbf{W}_Q^{(t)} \mathbf{W}_K^{(t)\top} X_n^\top)$ \\ \hline
$S_+, S_-$ & the training samples with +1 labels and -1 labels, \\ 
 & i.e., $S_+ := \{n \in [N] : y_n = 1\}$, $S_- := \{n \in [N] : y_n = -1\}$ \\ \hline
$\boldsymbol{q}_+^{(t)}, \boldsymbol{q}_-^{(t)}, \boldsymbol{q}_{n,i}^{(t)}$ & vectorized Q, defined as $\boldsymbol{q}_+^{(t)} = \widetilde{\boldsymbol{\mu}}_+^\top \mathbf{W}_Q^{(t)}, \boldsymbol{q}_-^{(t)} = \widetilde{\boldsymbol{\mu}}_-^\top \mathbf{W}_Q^{(t)}, \boldsymbol{q}_{n,i}^{(t)} = \widetilde{\boldsymbol{\xi}}_{n,i} \mathbf{W}_Q^{(t)}$ \\ \hline
$\boldsymbol{k}_+^{(t)}, \boldsymbol{k}_-^{(t)}, \boldsymbol{k}_{n,i}^{(t)}$ & vectorized K, defined as $\boldsymbol{k}_+^{(t)} = \widetilde{\boldsymbol{\mu}}_+^\top \mathbf{W}_K^{(t)}, \boldsymbol{k}_-^{(t)} = \widetilde{\boldsymbol{\mu}}_-^\top \mathbf{W}_K^{(t)}, \boldsymbol{k}_{n,i}^{(t)} = \widetilde{\boldsymbol{\xi}}_{n,i} \mathbf{W}_K^{(t)}$ \\ \hline
$V_+^{(t)}, V_-^{(t)}, V_{n,i}^{(t)}$ & scalarized V, defined as $V_+^{(t)} := \widetilde{\boldsymbol{\mu}}_+^\top \mathbf{W}_V^{(t)} \boldsymbol{w}_O, V_-^{(t)} := \widetilde{\boldsymbol{\mu}}_-^\top \mathbf{W}_V^{(t)} \boldsymbol{w}_O, V_{n,i}^{(t)} := \widetilde{\boldsymbol{\xi}}_{n,i}^\top \mathbf{W}_V^{(t)} \boldsymbol{w}_O$ \\ \hline
$\alpha_{\pm,\pm}^{(t)}, \alpha_{n,\pm,i}^{(t)}$ & linear combinations coefficients for the dynamics of $\boldsymbol{q}_+^{(t)}$ and $\boldsymbol{q}_-^{(t)}$, \\ 
 & i.e., $\boldsymbol q_\pm^{(t+1)} - \boldsymbol q_\pm^{(t)} = \alpha_{\pm,\pm}^{(t)} \boldsymbol k_\pm^{(t)} + \sum_{n \in S_\pm} \sum_{i=2}^M \alpha_{n,\pm,i}^{(t)} \boldsymbol{k}_{n,i}^{(t)}$ \\ \hline
$\alpha_{n,i,\pm}^{(t)}, \alpha_{n,i,n',i'}^{(t)}$ & linear combinations coefficients for the dynamics of $\boldsymbol{q}_{n,i}^{(t)}$, \\ 
 & i.e., $\boldsymbol{q}_{n,i}^{(t+1)} - \boldsymbol{q}_{n,i}^{(t)} = \alpha_{n,i,+}^{(t)} \boldsymbol{k}_+^{(t)} + \alpha_{n,i,-}^{(t)} \boldsymbol{k}_-^{(t)} + \sum_{n'=1}^N \sum_{i'=2}^M \alpha_{n,i,n',i'}^{(t)} \boldsymbol{k}_{n',i'}^{(t)}$ \\ \hline
$\beta_{\pm,\pm}^{(t)}, \beta_{n,\pm,i}^{(t)}$ & linear combinations coefficients for the dynamics of $\boldsymbol{k}_+^{(t)}$ and $\boldsymbol{k}_-^{(t)}$, \\ 
 & i.e., $\boldsymbol k_\pm^{(t+1)} - \boldsymbol k_\pm^{(t)} = \beta_{\pm,\pm}^{(t)} \boldsymbol q_\pm^{(t)} + \sum_{n \in S_\pm} \sum_{i=2}^M \beta_{n,\pm,i}^{(t)} \boldsymbol{q}_{n,i}^{(t)}$ \\ \hline
$\beta_{n,i,\pm}^{(t)}, \beta_{n,i,n',i'}^{(t)}$ & linear combinations coefficients for the dynamics of $\boldsymbol{k}_{n,i}^{(t)}$, \\ 
 & i.e., $\boldsymbol{k}_{n,i}^{(t+1)} - \boldsymbol{k}_{n,i}^{(t)} = \beta_{n,i,+}^{(t)} \boldsymbol{q}_+^{(t)} + \beta_{n,i,-}^{(t)} \boldsymbol{q}_-^{(t)} + \sum_{n'=1}^N \sum_{i'=2}^M \beta_{n,i,n',i'}^{(t)} \boldsymbol{q}_{n',i'}^{(t)}$ \\ \hline
$\text{softmax}(\langle \boldsymbol q_\pm^{(t)}, \boldsymbol k_\pm^{(t)} \rangle)$ & a general references to $\frac{\exp(\langle \boldsymbol q_\pm^{(t)}, \boldsymbol k_\pm^{(t)} \rangle)}{\exp(\langle \boldsymbol q_\pm^{(t)}, \boldsymbol k_\pm^{(t)} \rangle) + \sum_{k=2}^M \exp(\langle \boldsymbol q_\pm^{(t)}, \boldsymbol{k}_{n,k}^{(t)} \rangle)}$ for $n \in S_+, i \in [M] \backslash \{1\}$, \\ 
 & and $\frac{\exp(\langle \boldsymbol q_\pm^{(t)}, \boldsymbol k_\pm^{(t)} \rangle)}{\exp(\langle \boldsymbol q_\pm^{(t)}, \boldsymbol k_\pm^{(t)} \rangle) + \sum_{k=2}^M \exp(\langle \boldsymbol q_\pm^{(t)}, \boldsymbol{k}_{n,k}^{(t)} \rangle)}$ for $n \in S_-, i \in [M] \backslash \{1\}$ \\ \hline
$\text{softmax}(\langle \boldsymbol q_\pm^{(t)}, \boldsymbol{k}_{n,j}^{(t)} \rangle)$ & a general references to $\frac{\exp(\langle \boldsymbol q_\pm^{(t)}, \boldsymbol{k}_{n,j}^{(t)} \rangle)}{\exp(\langle \boldsymbol q_\pm^{(t)}, \boldsymbol k_\pm^{(t)} \rangle) + \sum_{k=2}^M \exp(\langle \boldsymbol q_\pm^{(t)}, \boldsymbol{k}_{n,k}^{(t)} \rangle)}$ for $n \in S_+, i, j \in [M] \backslash \{1\}$, \\ 
 & and $\frac{\exp(\langle \boldsymbol q_\pm^{(t)}, \boldsymbol{k}_{n,j}^{(t)} \rangle)}{\exp(\langle \boldsymbol q_\pm^{(t)}, \boldsymbol k_\pm^{(t)} \rangle) + \sum_{k=2}^M \exp(\langle \boldsymbol q_\pm^{(t)}, \boldsymbol{k}_{n,k}^{(t)} \rangle)}$ for $n \in S_-, i, j \in [M] \backslash \{1\}$ \\ \hline
$\text{softmax}(\langle \boldsymbol{q}_{n,i}^{(t)}, \boldsymbol{k}_{n,j}^{(t)} \rangle)$ & a general references to $\frac{\exp(\langle \boldsymbol{q}_{n,i}^{(t)}, \boldsymbol{k}_{n,j}^{(t)} \rangle)}{\exp(\langle \boldsymbol{q}_{n,i}^{(t)}, \boldsymbol{k}_{n,+}^{(t)} \rangle) + \sum_{k=2}^{M} \exp(\langle \boldsymbol{q}_{n,i}^{(t)}, \boldsymbol{k}_{n,k}^{(t)} \rangle)}$ for $n \in S_+, i,j \in [M] \backslash \{1\}$, \\ 
&and $\frac{\exp(\langle \boldsymbol{q}_{n,i}^{(t)}, \boldsymbol{k}_{n,j}^{(t)} \rangle)}{\exp(\langle \boldsymbol{q}_{n,i}^{(t)}, \boldsymbol{k}_{n,-}^{(t)} \rangle) + \sum_{k=2}^{M} \exp(\langle \boldsymbol{q}_{n,i}^{(t)}, \boldsymbol{k}_{n,k}^{(t)} \rangle)}$ for $n \in S_-, i,j \in [M] \backslash \{1\}$
 \\
 \hline
$\Lambda_{n,\pm,j}^{(t)}, \Lambda_{n,i,\pm,j}^{(t)}$ & $\Lambda_{n,\pm,j}^{(t)} := \langle \boldsymbol{q}_\pm^{(t)}, \boldsymbol{k}_{\pm}^{(t)} \rangle - \langle \boldsymbol q_\pm^{(t)}, \boldsymbol k_{n,j}^{(t)} \rangle$, $\Lambda_{n,i,\pm,j}^{(t)} := \langle \boldsymbol{q}_{n,i}^{(t)}, \boldsymbol k_\pm^{(t)} \rangle - \langle \boldsymbol{q}_{n,i}^{(t)}, \boldsymbol{k}_{n,j}^{(t)} \rangle$ \\ \bottomrule
\end{tabular}
\end{table}
\subsection{Update rules}\label{sec:uprule}
We fix a universal perturbed input
$
\widetilde{\boldsymbol{X}} = \big[\widetilde{\boldsymbol{\mu}}, \widetilde{\boldsymbol{\xi}}_{n,2}, \ldots, \widetilde{\boldsymbol{\xi}}_{n,M}\big],
$
where $\widetilde{\boldsymbol{\mu}}_+ \in B(\boldsymbol{\mu}_+,\tau)$, $\widetilde{\boldsymbol{\mu}}_- \in B(\boldsymbol{\mu}_-,\tau)$, and $\widetilde{\boldsymbol{\xi}}_{n,i} \in B(\boldsymbol{\xi}_{n,i},\tau)$ are chosen once and remain fixed for all iterations $t$. 
These perturbations are universal and do not correspond to iteration-specific adversarial examples.
{\begin{definition}[Scalarized V]\label{def:sca_v}
Let $\mathbf{W}_V^{(t)}$ be the V matrix of the ViT at the $t$-th iteration of adversarial training. Then there exist coefficients $\gamma_{V,+}^{(t)}, \gamma_{V,-}^{(t)}, \rho_{V,n,i}^{(t)}$ such that
\begin{align*}
    \widetilde{\boldsymbol{\mu}}_+^\top \mathbf{W}_V^{(t)} \boldsymbol{w}_O &= \widetilde{\boldsymbol{\mu}}_+^\top \mathbf{W}_V^{(0)} \boldsymbol{w}_O + \gamma_{V,+}^{(t)} \|\boldsymbol{w}_O\|_2^2, \\
    \widetilde{\boldsymbol{\mu}}_-^\top \mathbf{W}_V^{(t)} \boldsymbol{w}_O &= \widetilde{\boldsymbol{\mu}}_-^\top \mathbf{W}_V^{(0)} \boldsymbol{w}_O + \gamma_{V,-}^{(t)} \|\boldsymbol{w}_O\|_2^2, \\
    \widetilde{\boldsymbol{\xi}}_{n,i}^\top \mathbf{W}_V^{(t)} \boldsymbol{w}_O &= \widetilde{\boldsymbol{\xi}}_{n,i}^\top \mathbf{W}_V^{(0)} \boldsymbol{w}_O + \rho_{V,n,i}^{(t)} \|\boldsymbol{w}_O\|_2^2
\end{align*}
for $i \in [M] \backslash \{1\}, n \in [N]$.

We further denote the $V_+^{(t)} := \widetilde{\boldsymbol{\mu}}_+^{\top} \mathbf{W}_V^{(t)} \boldsymbol{w}_O$, $V_-^{(t)} := \widetilde{\boldsymbol{\mu}}_-^{\top} \mathbf{W}_V^{(t)} \boldsymbol{w}_O$ and $V_{n,i}^{(t)} := \widetilde{\boldsymbol{\xi}}_{n,i}^{\top} \mathbf{W}_V^{(t)} \boldsymbol{w}_O$ with time-independent perturbations, and $\widetilde{V}_+^{(t)} := \widetilde{\boldsymbol{\mu}}_+^{(t)\top} \mathbf{W}_V^{(t)} \boldsymbol{w}_O$, $\widetilde{V}_-^{(t)} := \widetilde{\boldsymbol{\mu}}_-^{(t)\top} \mathbf{W}_V^{(t)} \boldsymbol{w}_O$ and $\widetilde{V}_{n,i}^{(t)} := \widetilde{\boldsymbol{\xi}}_{n,i}^{(t)\top} \mathbf{W}_V^{(t)} \boldsymbol{w}_O$ with time-dependent perturbations, where $\widetilde{\boldsymbol{\mu}}_+^{(t)},\widetilde{\boldsymbol{\mu}}_-^{(t)},\widetilde{\boldsymbol{\xi}}_{n,i}^{(t)}$ is the adversarial sample at $t$-th iteration. We refer to it as scalarized V.
\end{definition}}

Similarly, we define the vectorized queries and keys as $\boldsymbol q^{(t)},\boldsymbol k^{(t)}$ with time-independent perturbations, and as $\widetilde{\boldsymbol q}^{(t)},\widetilde{\boldsymbol k}^{(t)}$ with time-dependent perturbations.
\begin{definition}[Vectorized Q \& K]\label{def:vec_qk}
Let $\mathbf{W}_Q^{(t)}$ and $\mathbf{W}_K^{(t)}$ be the QK matrices of the ViT at the $t$-th iteration of adversarial training. Then we define the vectorized Q and vectorized K as follows
\begin{align*}
    \boldsymbol q_+^{(t)} &= \widetilde{\boldsymbol{\mu}}_+^{\top} \mathbf{W}_Q^{(t)}, \quad \boldsymbol q_-^{(t)} = \widetilde{\boldsymbol{\mu}}_-^\top \mathbf{W}_Q^{(t)}, \quad \boldsymbol q_{n,i}^{(t)} = \widetilde{\boldsymbol{\xi}}_{n,i}^\top \mathbf{W}_Q^{(t)}, \\
    \boldsymbol k_+^{(t)} &= \widetilde{\boldsymbol{\mu}}_+^\top \mathbf{W}_K^{(t)}, \quad \boldsymbol k_-^{(t)} = \widetilde{\boldsymbol{\mu}}_-^\top \mathbf{W}_K^{(t)}, \quad \boldsymbol k_{n,i}^{(t)} = \widetilde{\boldsymbol{\xi}}_{n,i}^\top \mathbf{W}_K^{(t)}, \\
    \widetilde{\boldsymbol q}_+^{(t)} &= \widetilde{\boldsymbol{\mu}}_+^{(t)\top} \mathbf{W}_Q^{(t)}, \quad \widetilde{\boldsymbol q}_-^{(t)} = \widetilde{\boldsymbol{\mu}}_-^{(t)\top} \mathbf{W}_Q^{(t)}, \quad \widetilde{\boldsymbol q}_{n,i}^{(t)} = \widetilde{\boldsymbol{\xi}}_{n,i}^{(t)\top} \mathbf{W}_Q^{(t)}, \\
    \widetilde{\boldsymbol k}_+^{(t)} &= \widetilde{\boldsymbol{\mu}}_+^{(t)\top} \mathbf{W}_K^{(t)}, \quad \widetilde{\boldsymbol k}_-^{(t)} = \widetilde{\boldsymbol{\mu}}_-^{(t)\top} \mathbf{W}_K^{(t)}, \quad \widetilde{\boldsymbol k}_{n,i}^{(t)} = \widetilde{\boldsymbol{\xi}}_{n,i}^{(t)\top} \mathbf{W}_K^{(t)}  
\end{align*}

for $i \in [M] \backslash \{1\}, n \in [N]$.
\end{definition}

\begin{definition}[Gradient Decomposition]\label{def:gra_de}
There exist coefficients $\alpha_{+,+}^{(t)}$, 
$\alpha_{n,+,i}^{(t)}$, 
$\alpha_{-,-}^{(t)}$, 
$\alpha_{n,-,i}^{(t)}$, 
$\alpha_{n,i,+}^{(t)}$, 
$\alpha_{n,i,-}^{(t)}$, 
$\alpha_{n,i,n',i'}^{(t)}$, 
$\beta_{+,+}^{(t)}$, 
$\beta_{n,+,i}^{(t)}$, 
$\beta_{-,-}^{(t)}$, 
$\beta_{n,-,i}^{(t)}$, 
$\beta_{n,i,+}^{(t)}$, 
$\beta_{n,i,-}^{(t)}$, 
$\beta_{n,i,n',i'}^{(t)}$ such that
\begin{align*}
\Delta \boldsymbol{q}_+^{(t)} &:= \boldsymbol{q}_+^{(t+1)} - \boldsymbol{q}_+^{(t)} = \alpha_{+,+}^{(t)} \boldsymbol{k}_+^{(t)} + \sum_{n \in S_+} \sum_{i=2}^{M} \alpha_{n,+,i}^{(t)} \boldsymbol{k}_{n,i}^{(t)}, \\
\Delta \boldsymbol{q}_-^{(t)} &:= \boldsymbol{q}_-^{(t+1)} - \boldsymbol{q}_-^{(t)} = \alpha_{-,-}^{(t)} \boldsymbol{k}_-^{(t)} + \sum_{n \in S_-} \sum_{i=2}^{M} \alpha_{n,-,i}^{(t)} \boldsymbol{k}_{n,i}^{(t)}, \\
\Delta \boldsymbol{q}_{n,i}^{(t)} &:= \boldsymbol{q}_{n,i}^{(t+1)} - \boldsymbol{q}_{n,i}^{(t)} = \alpha_{n,i,+}^{(t)} \boldsymbol{k}_+^{(t)} + \alpha_{n,i,-}^{(t)} \boldsymbol{k}_-^{(t)} + \sum_{n'=1}^{N} \sum_{i'=2}^{M} \alpha_{n,i,n',i'}^{(t)} \boldsymbol{k}_{n',i'}^{(t)}, \\
\Delta \boldsymbol{k}_+^{(t)} &:= \boldsymbol{k}_+^{(t+1)} - \boldsymbol{k}_+^{(t)} = \beta_{+,+}^{(t)} \boldsymbol{q}_+^{(t)} + \sum_{n \in S_+} \sum_{i=2}^{M} \beta_{n,+,i}^{(t)} \boldsymbol{q}_{n,i}^{(t)}, \\
\Delta \boldsymbol{k}_-^{(t)} &:= \boldsymbol{k}_-^{(t+1)} - \boldsymbol{k}_-^{(t)} = \beta_{-,-}^{(t)} \boldsymbol{q}_-^{(t)} + \sum_{n \in S_-} \sum_{i=2}^{M} \beta_{n,-,i}^{(t)} \boldsymbol{q}_{n,i}^{(t)}, \\
\Delta \boldsymbol{k}_{n,i}^{(t)} &:= \boldsymbol{k}_{n,i}^{(t+1)} - \boldsymbol{k}_{n,i}^{(t)} = \beta_{n,i,+}^{(t)} \boldsymbol{q}_+^{(t)} + \beta_{n,i,-}^{(t)} \boldsymbol{q}_-^{(t)} + \sum_{n'=1}^{N} \sum_{i'=2}^{M} \beta_{n,i,n',i'}^{(t)} \boldsymbol{q}_{n',i'}^{(t)}.
\end{align*}

for $i, i' \in [M]\backslash\{1\}$ and $n, n' \in [N]$.    
\end{definition}
\begin{lemma}
[Update Rule for V]\label{lem:update_v}
The coefficients $\gamma_{V,+}^{(t)}, \gamma_{V,-}^{(t)}, \rho_{V,n,i}^{(t)}$ defined in Definition 1 satisfy the following iterative equations:
\begin{align*}
    \gamma_{V,+}^{(t+1)} = \gamma_{V,+}^{(t)} \nonumber
     &- \frac{\eta \langle\widetilde{\boldsymbol{\mu}}_+,\widetilde{\boldsymbol{\mu}}_+^{(t)}\rangle}{NM} \sum_{n \in S_+} \widetilde{\ell}_n'^{(t)} \left( \frac{\exp(\langle \widetilde{\mathbf{q}}_+^{(t)}, \widetilde{\mathbf{k}}_+^{(t)} \rangle)}{\exp(\langle \widetilde{\mathbf{q}}_+^{(t)}, \widetilde{\mathbf{k}}_+^{(t)} \rangle) + \sum_{k=2}^M \exp(\langle \widetilde{\mathbf{q}}_+^{(t)}, \widetilde{\mathbf{k}}_{n,k}^{(t)} \rangle)} \right. \nonumber \\
    &\quad\quad\quad\quad\quad\quad\quad\quad\quad \left. + \sum_{j=2}^M \frac{\exp(\langle \widetilde{\mathbf{q}}_{n,j}^{(t)}, \widetilde{\mathbf{k}}_+^{(t)} \rangle)}{\exp(\langle \widetilde{\mathbf{q}}_{n,j}^{(t)}, \widetilde{\mathbf{k}}_+^{(t)} \rangle) + \sum_{k=2}^M \exp(\langle \widetilde{\mathbf{q}}_{n,j}^{(t)}, \widetilde{\mathbf{k}}_{n,k}^{(t)} \rangle)} \right)  
    \\&+\sum_{n \in S_+}\widetilde{\ell}_n'^{(t)}\sum_{i=2}^M\frac{-\eta\langle\widetilde{\boldsymbol{\mu}}_+,\widetilde{\boldsymbol{\xi}}_{n,i}^{(t)}\rangle}{NM}\bigg(\frac{\exp(\langle \widetilde{\mathbf{q}}_+^{(t)}, \widetilde{\mathbf{k}}_{n,i}^{(t)} \rangle)}{\exp(\langle \widetilde{\mathbf{q}}_+^{(t)}, \widetilde{\mathbf{k}}_+^{(t)} \rangle) + \sum_{k=2}^M \exp(\langle \widetilde{\mathbf{q}}_+^{(t)}, \widetilde{\mathbf{k}}_{n,k}^{(t)} \rangle)} 
    \\&\quad\quad\quad\quad\quad\quad\quad\quad\quad+\sum_{j=2}^M\frac{\exp(\langle \widetilde{\mathbf{q}}_{n,j}^{(t)}, \widetilde{\mathbf{k}}_{n,i}^{(t)} \rangle)}{\exp(\langle \widetilde{\mathbf{q}}_{n,j}^{(t)}, \widetilde{\mathbf{k}}_+^{(t)} \rangle) + \sum_{k=2}^M \exp(\langle \widetilde{\mathbf{q}}_{n,j}^{(t)}, \widetilde{\mathbf{k}}_{n,k}^{(t)} \rangle)}\bigg) \\
    \gamma_{V,-}^{(t+1)} = \gamma_{V,-}^{(t)} \nonumber
     &- \frac{\eta \langle\widetilde{\boldsymbol{\mu}}_-,\widetilde{\boldsymbol{\mu}}_-^{(t)}\rangle}{NM} \sum_{n \in S_-} \widetilde{\ell}_n'^{(t)} \left( \frac{\exp(\langle \widetilde{\mathbf{q}}_-^{(t)}, \widetilde{\mathbf{k}}_-^{(t)} \rangle)}{\exp(\langle \widetilde{\mathbf{q}}_-^{(t)}, \widetilde{\mathbf{k}}_-^{(t)} \rangle) + \sum_{k=2}^M \exp(\langle \widetilde{\mathbf{q}}_-^{(t)}, \widetilde{\mathbf{k}}_{n,k}^{(t)} \rangle)} \right. \nonumber \\
    &\quad\quad\quad\quad\quad\quad\quad\quad\quad \left. + \sum_{j=2}^M \frac{\exp(\langle \widetilde{\mathbf{q}}_{n,j}^{(t)}, \widetilde{\mathbf{k}}_-^{(t)} \rangle)}{\exp(\langle \widetilde{\mathbf{q}}_{n,j}^{(t)}, \widetilde{\mathbf{k}}_-^{(t)} \rangle) + \sum_{k=2}^M \exp(\langle \widetilde{\mathbf{q}}_{n,j}^{(t)}, \widetilde{\mathbf{k}}_{n,k}^{(t)} \rangle)} \right)  
    \\&+\sum_{n \in S_+}\widetilde{\ell}_n'^{(t)}\sum_{i=2}^M\frac{-\eta\langle\widetilde{\boldsymbol{\mu}}_-,\widetilde{\boldsymbol{\xi}}_{n,i}^{(t)}\rangle}{NM}\bigg(\frac{\exp(\langle \widetilde{\mathbf{q}}_-^{(t)}, \widetilde{\mathbf{k}}_{n,i}^{(t)} \rangle)}{\exp(\langle \widetilde{\mathbf{q}}_-^{(t)}, \widetilde{\mathbf{k}}_-^{(t)} \rangle) + \sum_{k=2}^M \exp(\langle \widetilde{\mathbf{q}}_-^{(t)}, \widetilde{\mathbf{k}}_{n,k}^{(t)} \rangle)} 
    \\&\quad\quad\quad\quad\quad\quad\quad\quad\quad+\sum_{j=2}^M\frac{\exp(\langle \widetilde{\mathbf{q}}_{n,j}^{(t)}, \widetilde{\mathbf{k}}_{n,i}^{(t)} \rangle)}{\exp(\langle \widetilde{\mathbf{q}}_{n,j}^{(t)}, \widetilde{\mathbf{k}}_-^{(t)} \rangle) + \sum_{k=2}^M \exp(\langle \widetilde{\mathbf{q}}_{n,j}^{(t)}, \widetilde{\mathbf{k}}_{n,k}^{(t)} \rangle)}\bigg), \\
    \rho_{V,n,i}^{(t+1)} = \rho_{V,n,i}^{(t)} 
    &  -\frac{\eta}{NM} \sum_{n' \in S_+} \widetilde{\ell}_{n'}'^{(t)} \left(\left( \langle\widetilde{\boldsymbol{\xi}}_{n,i},\widetilde{\boldsymbol{\mu}}_+^{(t)}\rangle \frac{\exp(\langle \widetilde{\mathbf{q}}_+^{(t)}, \widetilde{\mathbf{k}}_+^{(t)} \rangle)}{\exp(\langle \widetilde{\mathbf{q}}_+^{(t)}, \widetilde{\mathbf{k}}_+^{(t)} \rangle) + \sum_{k=2}^M \exp(\langle \widetilde{\mathbf{q}}_+^{(t)}, \widetilde{\mathbf{k}}_{n,k}^{(t)} \rangle)}  \right.\right.\nonumber \\
    &\left.\quad\quad\quad\quad\quad\quad\quad\quad + \sum_{j=2}^M \langle\widetilde{\boldsymbol{\xi}}_{n,i},\widetilde{\boldsymbol{\mu}}_+^{(t)}\rangle \frac{\exp(\langle \widetilde{\mathbf{q}}_{n,j}^{(t)}, \widetilde{\mathbf{k}}_+^{(t)} \rangle)}{\exp(\langle \widetilde{\mathbf{q}}_{n,j}^{(t)}, \widetilde{\mathbf{k}}_+^{(t)} \rangle) + \sum_{k=2}^M \exp(\langle \widetilde{\mathbf{q}}_{n,j}^{(t)}, \widetilde{\mathbf{k}}_{n,k}^{(t)} \rangle)} \right) \\
        &\quad\quad\quad\quad\quad+\sum_{i=2}^M \left(\langle \widetilde{\boldsymbol{\xi}}_{n,i}, \widetilde{\boldsymbol{\xi}}_{n',i'}^{(t)} \rangle
        \frac{\exp(\langle \widetilde{\mathbf{q}}_{n',i'}^{(t)}, \widetilde{\mathbf{k}}_{n',i'}^{(t)} \rangle)}{\exp(\langle \widetilde{\mathbf{q}}_{n',i'}^{(t)}, \widetilde{\mathbf{k}}_{n',i'}^{(t)} \rangle) + \sum_{k=2}^M \exp(\langle \widetilde{\mathbf{q}}_{n',i'}^{(t)}, \widetilde{\mathbf{k}}_{n',k}^{(t)} \rangle)} \right.\nonumber \\
    &\left.\left.\quad\quad\quad\quad\quad\quad\quad\quad + \sum_{j=2}^M  \langle \widetilde{\boldsymbol{\xi}}_{n,i}, \widetilde{\boldsymbol{\xi}}_{n',i'}^{(t)} \rangle 
        \frac{\exp(\langle \widetilde{\mathbf{q}}_{n',j}^{(t)}, \widetilde{\mathbf{k}}_{n',i'}^{(t)} \rangle)}{\exp(\langle \widetilde{\mathbf{q}}_{n',j}^{(t)}, \widetilde{\mathbf{k}}_{n',i'}^{(t)} \rangle) + \sum_{k=2}^M \exp(\langle \widetilde{\mathbf{q}}_{n',j}^{(t)}, \widetilde{\mathbf{k}}_{n',k}^{(t)} \rangle)}
    \right)\right) \nonumber \\
    & - \frac{\eta}{NM} \sum_{n' \in S_-} \left(\cdot\right)
\end{align*}

for $i \in [M] \backslash \{1\}, n \in [N]$.
\end{lemma}

\begin{proof}

The gradient of \( \mathbf{W}_V \) can be obtained using the chain rule as follows
\begin{align*}
\nabla_{\mathbf{W}_V} L_S(\theta) &= \frac{1}{N} \sum_{n=1}^{N} y_n \ell'(y_n f(\mathbf{X}_n, \theta)) \nabla_{\mathbf{W}_V} f(\mathbf{X}_n, \theta)  \\
&= \frac{1}{NM} \sum_{n=1}^{N} y_n \ell'_n(\theta) \left[ \boldsymbol{w}_O \sum_{l=1}^{M} \varphi(\mathbf{x}_{n,l} \mathbf{W}_Q \mathbf{W}_K^\top (\mathbf{X}_n)^\top) \mathbf{X}_n \right]^\top 
\end{align*}

Base on above, we have
\begin{align*}
\mathbf{x}^\top \nabla_{\mathbf{w}_V} \widetilde{L_S}(\theta) \boldsymbol{w}_O
&= \frac{1}{NM} \sum_{n \in S_+} \widetilde{\ell}'_n(\theta) \big( \langle \mathbf{x}, \widetilde{\boldsymbol{\boldsymbol{\mu}}}_+ \rangle \frac{\exp(\widetilde{\boldsymbol{\boldsymbol{\mu}}}_+^\top \mathbf{W}_Q \mathbf{W}_K^\top \widetilde{\boldsymbol{\boldsymbol{\mu}}}_+)}{\exp(\widetilde{\boldsymbol{\boldsymbol{\mu}}}_+^\top \mathbf{W}_Q \mathbf{W}_K^\top \widetilde{\boldsymbol{\boldsymbol{\mu}}}_+) + \sum_{k=2}^{M} \exp(\widetilde{\boldsymbol{\boldsymbol{\mu}}}_+^\top \mathbf{W}_Q \mathbf{W}_K^\top \widetilde{\boldsymbol{\boldsymbol{\xi}}}_{n,k})} \\
&+ \sum_{j=2}^{M} \langle \mathbf{x}, \widetilde{\boldsymbol{\boldsymbol{\mu}}}_+ \rangle \frac{\exp(\widetilde{\boldsymbol{\boldsymbol{\xi}}}_{n,j}^\top \mathbf{W}_Q \mathbf{W}_K^\top \widetilde{\boldsymbol{\boldsymbol{\mu}}}_+)}{\exp(\widetilde{\boldsymbol{\boldsymbol{\xi}}}_{n,j}^\top \mathbf{W}_Q \mathbf{W}_K^\top \widetilde{\boldsymbol{\boldsymbol{\mu}}}_+) + \sum_{k=2}^{M} \exp(\widetilde{\boldsymbol{\boldsymbol{\xi}}}_{n,j}^\top \mathbf{W}_Q \mathbf{W}_K^\top \widetilde{\boldsymbol{\boldsymbol{\xi}}}_{n,k})}\big) \\
&+ \sum_{i=2}^{M} \big(\langle \mathbf{x}, \widetilde{\boldsymbol{\boldsymbol{\xi}}}_{n,i} \rangle \frac{\exp(\widetilde{\boldsymbol{\boldsymbol{\mu}}}_+^\top \mathbf{W}_Q \mathbf{W}_K^\top \widetilde{\boldsymbol{\boldsymbol{\xi}}}_{n,i})}{\exp(\widetilde{\boldsymbol{\boldsymbol{\mu}}}_+^\top \mathbf{W}_Q \mathbf{W}_K^\top \widetilde{\boldsymbol{\boldsymbol{\mu}}}_+) + \sum_{k=2}^{M} \exp(\widetilde{\boldsymbol{\boldsymbol{\mu}}}_+^\top \mathbf{W}_Q \mathbf{W}_K^\top \widetilde{\boldsymbol{\boldsymbol{\xi}}}_{n,k})} \\
&+ \sum_{j=2}^{M} \langle \mathbf{x}, \widetilde{\boldsymbol{\boldsymbol{\xi}}}_{n,i} \rangle \frac{\exp(\widetilde{\boldsymbol{\boldsymbol{\xi}}}_{n,j}^\top \mathbf{W}_Q \mathbf{W}_K^\top \widetilde{\boldsymbol{\boldsymbol{\xi}}}_{n,i})}{\exp(\widetilde{\boldsymbol{\boldsymbol{\xi}}}_{n,j}^\top \mathbf{W}_Q \mathbf{W}_K^\top \widetilde{\boldsymbol{\boldsymbol{\mu}}}_+) + \sum_{k=2}^{M} \exp(\widetilde{\boldsymbol{\boldsymbol{\xi}}}_{n,j}^\top \mathbf{W}_Q \mathbf{W}_K^\top \widetilde{\boldsymbol{\boldsymbol{\xi}}}_{n,k})} \big)\bigg) \| \boldsymbol{w}_O \|_2^2
\\&+\frac{1}{NM} \sum_{n \in S_-} (\cdot)
\end{align*}

where the second equality we expand $X_n$ into vectors and make inner products with $x$, the third equality we materializing all the $x_{n,i}$ (e.g., $x_{n,1} = \boldsymbol{\mu}_+$ for $n \in S_+$). Note the orthogonality between $\boldsymbol{\mu}$ and $\boldsymbol{\xi}_{n,i}$, we can remove many of the terms in this equation. For any $x = \widetilde{\boldsymbol{\mu}}_+'\in B(\boldsymbol{\mu}_+,\tau)$, we have

{\begin{align*}
\widetilde{\boldsymbol{\mu}}_{+}^{'\top} \nabla_{\mathbf{w}_V} \widetilde{L_S}(\theta) \boldsymbol{w}_O
&= \frac{1}{NM} \sum_{n \in S_+} \widetilde{\ell}'_n(\theta) \bigg(\big( \langle \widetilde{\boldsymbol{\mu}}_{+}', \widetilde{\boldsymbol{\boldsymbol{\mu}}}_+ \rangle \frac{\exp(\widetilde{\boldsymbol{\boldsymbol{\mu}}}_+^\top \mathbf{W}_Q \mathbf{W}_K^\top \widetilde{\boldsymbol{\boldsymbol{\mu}}}_+)}{\exp(\widetilde{\boldsymbol{\boldsymbol{\mu}}}_+^\top \mathbf{W}_Q \mathbf{W}_K^\top \widetilde{\boldsymbol{\boldsymbol{\mu}}}_+) + \sum_{k=2}^{M} \exp(\widetilde{\boldsymbol{\boldsymbol{\mu}}}_+^\top \mathbf{W}_Q \mathbf{W}_K^\top \widetilde{\boldsymbol{\boldsymbol{\xi}}}_{n,k})} \\
&+ \sum_{j=2}^{M} \langle \widetilde{\boldsymbol{\mu}}_{+}', \widetilde{\boldsymbol{\boldsymbol{\mu}}}_+ \rangle \frac{\exp(\widetilde{\boldsymbol{\boldsymbol{\xi}}}_{n,j}^\top \mathbf{W}_Q \mathbf{W}_K^\top \widetilde{\boldsymbol{\boldsymbol{\mu}}}_+)}{\exp(\widetilde{\boldsymbol{\boldsymbol{\xi}}}_{n,j}^\top \mathbf{W}_Q \mathbf{W}_K^\top \widetilde{\boldsymbol{\boldsymbol{\mu}}}_+) + \sum_{k=2}^{M} \exp(\widetilde{\boldsymbol{\boldsymbol{\xi}}}_{n,j}^\top \mathbf{W}_Q \mathbf{W}_K^\top \widetilde{\boldsymbol{\boldsymbol{\xi}}}_{n,k})}\big) \\
&+ \sum_{i=2}^{M} \big( \langle \widetilde{\boldsymbol{\mu}}_{+}', \widetilde{\boldsymbol{\boldsymbol{\xi}}}_{n,i} \rangle \frac{\exp(\widetilde{\boldsymbol{\boldsymbol{\mu}}}_+^\top \mathbf{W}_Q \mathbf{W}_K^\top \widetilde{\boldsymbol{\boldsymbol{\xi}}}_{n,i})}{\exp(\widetilde{\boldsymbol{\boldsymbol{\mu}}}_+^\top \mathbf{W}_Q \mathbf{W}_K^\top \widetilde{\boldsymbol{\boldsymbol{\mu}}}_+) + \sum_{k=2}^{M} \exp(\widetilde{\boldsymbol{\boldsymbol{\mu}}}_+^\top \mathbf{W}_Q \mathbf{W}_K^\top \widetilde{\boldsymbol{\boldsymbol{\xi}}}_{n,k})} \\
&+ \sum_{j=2}^{M} \langle \widetilde{\boldsymbol{\mu}}_{+}', \widetilde{\boldsymbol{\boldsymbol{\xi}}}_{n,i} \rangle \frac{\exp(\widetilde{\boldsymbol{\boldsymbol{\xi}}}_{n,j}^\top \mathbf{W}_Q \mathbf{W}_K^\top \widetilde{\boldsymbol{\boldsymbol{\xi}}}_{n,i})}{\exp(\widetilde{\boldsymbol{\boldsymbol{\xi}}}_{n,j}^\top \mathbf{W}_Q \mathbf{W}_K^\top \widetilde{\boldsymbol{\boldsymbol{\mu}}}_+) + \sum_{k=2}^{M} \exp(\widetilde{\boldsymbol{\boldsymbol{\xi}}}_{n,j}^\top \mathbf{W}_Q \mathbf{W}_K^\top \widetilde{\boldsymbol{\boldsymbol{\xi}}}_{n,k})} \big)\bigg) \| \boldsymbol{w}_O \|_2^2
\end{align*}}

Then we have
{\begin{align*}
    &\widetilde{\boldsymbol{\mu}}_+^\top \mathbf{W}_V^{(t+1)} \boldsymbol{w}_O - \widetilde{\boldsymbol{\mu}}_+^\top \mathbf{W}_V^{(t)} \boldsymbol{w}_O = \widetilde{\boldsymbol{\mu}}_+^\top (-\eta \nabla_{\mathbf{W}_V} \widetilde{L}_S(\theta(t))) \boldsymbol{w}_O \nonumber \\
    &\quad= - \frac{\eta \langle\widetilde{\boldsymbol{\mu}}_+,\widetilde{\boldsymbol{\mu}}_+^{(t)}\rangle}{NM} \sum_{n \in S_+} \widetilde{\ell}_n'^{(t)} \left( \frac{\exp(\langle \widetilde{\mathbf{q}}_+^{(t)}, \widetilde{\mathbf{k}}_+^{(t)} \rangle)}{\exp(\langle \widetilde{\mathbf{q}}_+^{(t)}, \widetilde{\mathbf{k}}_+^{(t)} \rangle) + \sum_{k=2}^M \exp(\langle \widetilde{\mathbf{q}}_+^{(t)}, \widetilde{\mathbf{k}}_{n,k}^{(t)} \rangle)} \right. \nonumber \\
    &\quad \left. + \sum_{j=2}^M \frac{\exp(\langle \widetilde{\mathbf{q}}_{n,j}^{(t)}, \widetilde{\mathbf{k}}_+^{(t)} \rangle)}{\exp(\langle \widetilde{\mathbf{q}}_{n,j}^{(t)}, \widetilde{\mathbf{k}}_+^{(t)} \rangle) + \sum_{k=2}^M \exp(\langle \widetilde{\mathbf{q}}_{n,j}^{(t)}, \widetilde{\mathbf{k}}_{n,k}^{(t)} \rangle)} \right) \|\boldsymbol{w}_O\|_2^2 
    \\&+\sum_{n \in S_+}\widetilde{\ell}_n'^{(t)}\sum_{i=2}^M\frac{-\eta\langle\widetilde{\boldsymbol{\mu}}_+,\widetilde{\boldsymbol{\xi}}_{n,i}^{(t)}\rangle}{NM}\bigg(\frac{\exp(\langle \widetilde{\mathbf{q}}_+^{(t)}, \widetilde{\mathbf{k}}_{n,i}^{(t)} \rangle)}{\exp(\langle \widetilde{\mathbf{q}}_+^{(t)}, \widetilde{\mathbf{k}}_+^{(t)} \rangle) + \sum_{k=2}^M \exp(\langle \widetilde{\mathbf{q}}_+^{(t)}, \widetilde{\mathbf{k}}_{n,k}^{(t)} \rangle)} 
    \\&+\sum_{j=2}^M\frac{\exp(\langle \widetilde{\mathbf{q}}_{n,j}^{(t)}, \widetilde{\mathbf{k}}_{n,i}^{(t)} \rangle)}{\exp(\langle \widetilde{\mathbf{q}}_{n,j}^{(t)}, \widetilde{\mathbf{k}}_+^{(t)} \rangle) + \sum_{k=2}^M \exp(\langle \widetilde{\mathbf{q}}_{n,j}^{(t)}, \widetilde{\mathbf{k}}_{n,k}^{(t)} \rangle)}\bigg)\|\boldsymbol{w}_O\|_2^2
\end{align*}}

Dividing by $\|\boldsymbol{w}_O\|_2^2$ we get
\begin{align*}
    \gamma_{V,+}^{(t+1)} = \gamma_{V,+}^{(t)} \nonumber
     &- \frac{\eta \langle\widetilde{\boldsymbol{\mu}}_+,\widetilde{\boldsymbol{\mu}}_+^{(t)}\rangle}{NM} \sum_{n \in S_+} \widetilde{\ell}_n'^{(t)} \left( \frac{\exp(\langle \widetilde{\mathbf{q}}_+^{(t)}, \widetilde{\mathbf{k}}_+^{(t)} \rangle)}{\exp(\langle \widetilde{\mathbf{q}}_+^{(t)}, \widetilde{\mathbf{k}}_+^{(t)} \rangle) + \sum_{k=2}^M \exp(\langle \widetilde{\mathbf{q}}_+^{(t)}, \widetilde{\mathbf{k}}_{n,k}^{(t)} \rangle)} \right. \nonumber \\
    &\quad\quad\quad\quad\quad\quad\quad\quad\quad \left. + \sum_{j=2}^M \frac{\exp(\langle \widetilde{\mathbf{q}}_{n,j}^{(t)}, \widetilde{\mathbf{k}}_+^{(t)} \rangle)}{\exp(\langle \widetilde{\mathbf{q}}_{n,j}^{(t)}, \widetilde{\mathbf{k}}_+^{(t)} \rangle) + \sum_{k=2}^M \exp(\langle \widetilde{\mathbf{q}}_{n,j}^{(t)}, \widetilde{\mathbf{k}}_{n,k}^{(t)} \rangle)} \right)  
    \\&+\sum_{n \in S_+}\widetilde{\ell}_n'^{(t)}\sum_{i=2}^M\frac{-\eta\langle\widetilde{\boldsymbol{\mu}}_+,\widetilde{\boldsymbol{\xi}}_{n,i}^{(t)}\rangle}{NM}\bigg(\frac{\exp(\langle \widetilde{\mathbf{q}}_+^{(t)}, \widetilde{\mathbf{k}}_{n,i}^{(t)} \rangle)}{\exp(\langle \widetilde{\mathbf{q}}_+^{(t)}, \widetilde{\mathbf{k}}_+^{(t)} \rangle) + \sum_{k=2}^M \exp(\langle \widetilde{\mathbf{q}}_+^{(t)}, \widetilde{\mathbf{k}}_{n,k}^{(t)} \rangle)} 
    \\&\quad\quad\quad\quad\quad\quad\quad\quad\quad+\sum_{j=2}^M\frac{\exp(\langle \widetilde{\mathbf{q}}_{n,j}^{(t)}, \widetilde{\mathbf{k}}_{n,i}^{(t)} \rangle)}{\exp(\langle \widetilde{\mathbf{q}}_{n,j}^{(t)}, \widetilde{\mathbf{k}}_+^{(t)} \rangle) + \sum_{k=2}^M \exp(\langle \widetilde{\mathbf{q}}_{n,j}^{(t)}, \widetilde{\mathbf{k}}_{n,k}^{(t)} \rangle)}\bigg)
\end{align*}
This proves the update rule for $\gamma_{V,+}^{(t)}$. The proof for $\gamma_{V,-}^{(t)}$ and $\rho_{V,n,i}^{(t)}$ is similar to it.
\end{proof}

\begin{lemma}[Update Rule for QK, Lemma B.3 in \citet{jiang2024unveil}]
The dynamics of $\boldsymbol x^\top \boldsymbol{W}_Q \boldsymbol{W}_K \boldsymbol x$ can be characterized as follows:
\begin{equation}
\begin{aligned}
&\langle \boldsymbol{q}_+^{(t+1)}, \boldsymbol{k}_+^{(t+1)} \rangle - \langle \boldsymbol{q}_+^{(t)}, \boldsymbol{k}_+^{(t)} \rangle \\
&= \alpha_+^{(t)} \|\boldsymbol{k}_+^{(t)}\|_2^2 + \sum_{n \in S_+} \sum_{i=2}^{M} \alpha_{n,+,i}^{(t)} \langle \boldsymbol{k}_+^{(t)}, \boldsymbol{k}_{n,i}^{(t)} \rangle \\
&\quad + \beta_+^{(t)} \|\boldsymbol{q}_+^{(t)}\|_2^2 + \sum_{n \in S_+} \sum_{i=2}^{M} \beta_{n,+,i}^{(t)} \langle \boldsymbol{q}_+^{(t)}, \boldsymbol{q}_{n,i}^{(t)} \rangle \\
&\quad + \left( \alpha_+^{(t)} \boldsymbol{k}_+^{(t)} + \sum_{n \in S_+} \sum_{i=2}^{M} \alpha_{n,+,i}^{(t)} \boldsymbol{k}_{n,i}^{(t)} \right) \\
&\qquad \cdot \left( \beta_+^{(t)} \boldsymbol{q}_+^{(t)\top} + \sum_{n \in S_+} \sum_{i=2}^{M} \beta_{n,+,i}^{(t)} \boldsymbol{q}_{n,i}^{(t)\top} \right),
\end{aligned}
\end{equation}

\begin{equation}
\begin{aligned}
&\langle \boldsymbol{q}_-^{(t+1)}, \boldsymbol{k}_-^{(t+1)} \rangle - \langle \boldsymbol{q}_-^{(t)}, \boldsymbol{k}_-^{(t)} \rangle \\
&= \alpha_{-,-}^{(t)} \|\boldsymbol{k}_-^{(t)}\|_2^2 + \sum_{n \in S_-} \sum_{i=2}^{M} \alpha_{n,-,i}^{(t)} \langle \boldsymbol{k}_-^{(t)}, \boldsymbol{k}_{n,i}^{(t)} \rangle \\
&\quad + \beta_{-,-}^{(t)} \|\boldsymbol{q}_-^{(t)}\|_2^2 + \sum_{n \in S_-} \sum_{i=2}^{M} \beta_{n,-,i}^{(t)} \langle \boldsymbol{q}_-^{(t)}, \boldsymbol{q}_{n,i}^{(t)} \rangle \\
&\quad + \left( \alpha_{-,-}^{(t)} \boldsymbol{k}_-^{(t)} + \sum_{n \in S_-} \sum_{i=2}^{M} \alpha_{n,-,i}^{(t)} \boldsymbol{k}_{n,i}^{(t)} \right) \\
&\qquad \cdot \left( \beta_{-,-}^{(t)} \boldsymbol{q}_-^{(t)\top} + \sum_{n \in S_-} \sum_{i=2}^{M} \beta_{n,-,i}^{(t)} \boldsymbol{q}_{n,i}^{(t)\top} \right),
\end{aligned}
\end{equation}

\begin{equation}
\begin{aligned}
&\langle \boldsymbol{q}_{n,i}^{(t+1)}, \boldsymbol{k}_+^{(t+1)} \rangle - \langle \boldsymbol{q}_{n,i}^{(t)}, \boldsymbol{k}_+^{(t)} \rangle \\
&= \alpha_{n,i,+}^{(t)} \|\boldsymbol{k}_+^{(t)}\|_2^2 + \alpha_{n,i,-}^{(t)} \langle \boldsymbol{k}_+^{(t)}, \boldsymbol{k}_-^{(t)} \rangle + \sum_{n'=1}^{N} \sum_{l=2}^{M} \alpha_{n,i,n',l}^{(t)} \langle \boldsymbol{k}_+^{(t)}, \boldsymbol{k}_{n',l}^{(t)} \rangle \\
&\quad + \beta_{+,+}^{(t)} \langle \boldsymbol{q}_+^{(t)}, \boldsymbol{q}_{n,i}^{(t)} \rangle + \sum_{n' \in S_+} \sum_{l=2}^{M} \beta_{n',+,l}^{(t)} \langle \boldsymbol{q}_{n,i}^{(t)}, \boldsymbol{q}_{n',l}^{(t)} \rangle \\
&\quad + \left( \alpha_{n,i,+}^{(t)} \boldsymbol{k}_+^{(t)} + \alpha_{n,i,-}^{(t)} \boldsymbol{k}_-^{(t)} + \sum_{n'=1}^{N} \sum_{l=2}^{M} \alpha_{n,i,n',l}^{(t)} \boldsymbol{k}_{n',l}^{(t)} \right) \\
&\qquad \cdot \left( \beta_{+,+}^{(t)} \boldsymbol{q}_+^{(t)\top} + \sum_{n' \in S_+} \sum_{l=2}^{M} \beta_{n',+,l}^{(t)} \boldsymbol{q}_{n',l}^{(t)\top} \right),
\end{aligned}
\end{equation}

\begin{equation}
\begin{aligned}
&\langle \boldsymbol{q}_{n,i}^{(t+1)}, \boldsymbol{k}_-^{(t+1)} \rangle - \langle \boldsymbol{q}_{n,i}^{(t)}, \boldsymbol{k}_-^{(t)} \rangle \\
&= \alpha_{n,i,-}^{(t)} \|\boldsymbol{k}_-^{(t)}\|_2^2 + \alpha_{n,i,+}^{(t)} \langle \boldsymbol{k}_+^{(t)}, \boldsymbol{k}_-^{(t)} \rangle + \sum_{n'=1}^{N} \sum_{l=2}^{M} \alpha_{n,i,n',l}^{(t)} \langle \boldsymbol{k}_-^{(t)}, \boldsymbol{k}_{n',l}^{(t)} \rangle \\
&\quad + \beta_{-,-}^{(t)} \langle \boldsymbol{q}_-^{(t)}, \boldsymbol{q}_{n,i}^{(t)} \rangle + \sum_{n' \in S_-} \sum_{l=2}^{M} \beta_{n',-,l}^{(t)} \langle \boldsymbol{q}_{n,i}^{(t)}, \boldsymbol{q}_{n',l}^{(t)} \rangle \\
&\quad + \left( \alpha_{n,i,+}^{(t)} \boldsymbol{k}_+^{(t)} + \alpha_{n,i,-}^{(t)} \boldsymbol{k}_-^{(t)} + \sum_{n'=1}^{N} \sum_{l=2}^{M} \alpha_{n,i,n',l}^{(t)} \boldsymbol{k}_{n',l}^{(t)} \right) \\
&\qquad \cdot \left( \beta_{-,-}^{(t)} \boldsymbol{q}_-^{(t)\top} + \sum_{n' \in S_-} \sum_{l=2}^{M} \beta_{n',-,l}^{(t)} \boldsymbol{q}_{n',l}^{(t)\top} \right),
\end{aligned}
\end{equation}

\begin{equation}
\begin{aligned}
&\langle \boldsymbol{q}_+^{(t+1)}, \boldsymbol{k}_{n,j}^{(t+1)} \rangle - \langle \boldsymbol{q}_+^{(t)}, \boldsymbol{k}_{n,j}^{(t)} \rangle \\
&= \alpha_{+,+}^{(t)} \langle \boldsymbol{k}_+^{(t)}, \boldsymbol{k}_{n,j}^{(t)} \rangle + \sum_{n' \in S_+} \sum_{l=2}^{M} \alpha_{n',+,l}^{(t)} \langle \boldsymbol{k}_{n,j}^{(t)}, \boldsymbol{k}_{n',l}^{(t)} \rangle \\
&\quad + \beta_{n,j,+}^{(t)} \|\boldsymbol{q}_+^{(t)}\|_2^2 + \beta_{n,j,-}^{(t)} \langle \boldsymbol{q}_+^{(t)}, \boldsymbol{q}_-^{(t)} \rangle + \sum_{n'=1}^{N} \sum_{l=2}^{M} \beta_{n,j,n',l}^{(t)} \langle \boldsymbol{q}_+^{(t)}, \boldsymbol{q}_{n',l}^{(t)} \rangle \\
&\quad + \left( \alpha_{+,+}^{(t)} \boldsymbol{k}_+^{(t)} + \sum_{n' \in S_+} \sum_{l=2}^{M} \alpha_{n',+,l}^{(t)} \boldsymbol{k}_{n',l}^{(t)} \right) \\
&\qquad \cdot \left( \beta_{n,j,+}^{(t)} \boldsymbol{q}_+^{(t)\top} + \beta_{n,j,-}^{(t)} \boldsymbol{q}_-^{(t)\top} + \sum_{n'=1}^{N} \sum_{l=2}^{M} \beta_{n,j,n',l}^{(t)} \boldsymbol{q}_{n',l}^{(t)\top} \right),
\end{aligned}
\end{equation}

\begin{equation}
\begin{aligned}
&\langle \boldsymbol{q}_-^{(t+1)}, \boldsymbol{k}_{n,j}^{(t+1)} \rangle - \langle \boldsymbol{q}_-^{(t)}, \boldsymbol{k}_{n,j}^{(t)} \rangle \\
&= \alpha_{-,-}^{(t)} \langle \boldsymbol{k}_-^{(t)}, \boldsymbol{k}_{n,j}^{(t)} \rangle + \sum_{n' \in S_-} \sum_{l=2}^{M} \alpha_{n',-,l}^{(t)} \langle \boldsymbol{k}_{n,j}^{(t)}, \boldsymbol{k}_{n',l}^{(t)} \rangle \\
&\quad + \beta_{n,j,-}^{(t)} \|\boldsymbol{q}_-^{(t)}\|_2^2 + \beta_{n,j,+}^{(t)} \langle \boldsymbol{q}_+^{(t)}, \boldsymbol{q}_-^{(t)} \rangle + \sum_{n'=1}^{N} \sum_{l=2}^{M} \beta_{n,j,n',l}^{(t)} \langle \boldsymbol{q}_-^{(t)}, \boldsymbol{q}_{n',l}^{(t)} \rangle \\
&\quad + \left( \alpha_{-,-}^{(t)} \boldsymbol{k}_-^{(t)} + \sum_{n' \in S_-} \sum_{l=2}^{M} \alpha_{n',-,l}^{(t)} \boldsymbol{k}_{n',l}^{(t)} \right) \\
&\qquad \cdot \left( \beta_{n,j,+}^{(t)} \boldsymbol{q}_+^{(t)\top} + \beta_{n,j,-}^{(t)} \boldsymbol{q}_-^{(t)\top} + \sum_{n'=1}^{N} \sum_{l=2}^{M} \beta_{n,j,n',l}^{(t)} \boldsymbol{q}_{n',l}^{(t)\top} \right),
\end{aligned}
\end{equation}

\begin{equation}
\begin{aligned}
&\langle \boldsymbol{q}_{n,i}^{(t+1)}, \boldsymbol{k}_{n,j}^{(t+1)} \rangle - \langle \boldsymbol{q}_{n,i}^{(t)}, \boldsymbol{k}_{n,j}^{(t)} \rangle \\
&= \alpha_{n,i,+}^{(t)} \langle \boldsymbol{k}_+^{(t)}, \boldsymbol{k}_{n,j}^{(t)} \rangle + \alpha_{n,i,-}^{(t)} \langle \boldsymbol{k}_-^{(t)}, \boldsymbol{k}_{n,j}^{(t)} \rangle + \sum_{n'=1}^{N} \sum_{l=2}^{M} \alpha_{n,i,n',l}^{(t)} \langle \boldsymbol{k}_{n',l}^{(t)}, \boldsymbol{k}_{n,j}^{(t)} \rangle \\
&\quad + \beta_{n,j,+}^{(t)} \langle \boldsymbol{q}_+^{(t)}, \boldsymbol{q}_{n,i}^{(t)} \rangle + \beta_{n,j,-}^{(t)} \langle \boldsymbol{q}_-^{(t)}, \boldsymbol{q}_{n,i}^{(t)} \rangle + \sum_{n'=1}^{N} \sum_{l=2}^{M} \beta_{n,j,n',l}^{(t)} \langle \boldsymbol{q}_{n',l}^{(t)}, \boldsymbol{q}_{n,i}^{(t)} \rangle \\
&\quad + \left( \alpha_{n,i,+}^{(t)} \boldsymbol{k}_+^{(t)} + \alpha_{n,i,-}^{(t)} \boldsymbol{k}_-^{(t)} + \sum_{n'=1}^{N} \sum_{l=2}^{M} \alpha_{n,i,n',l}^{(t)} \boldsymbol{k}_{n',l}^{(t)} \right) \\
&\qquad \cdot \left( \beta_{n,j,+}^{(t)} \boldsymbol{q}_+^{(t)\top} + \beta_{n,j,-}^{(t)} \boldsymbol{q}_-^{(t)\top} + \sum_{n'=1}^{N} \sum_{l=2}^{M} \beta_{n,j,n',l}^{(t)} \boldsymbol{q}_{n',l}^{(t)\top} \right),
\end{aligned}
\end{equation}

 for  $i, j \in [M]\backslash\{1\}, n \in [N]$.
\end{lemma}
\section{Concentration Inequalities}\label{sec:con_ineq}

In this section, we will give some concentration inequalities that show some important properties of the data and the ViT parameters at random initialization.

\begin{lemma}[Lemma B.1 in~\citet{cao2022benign}]\label{lemma:num_pos}
Suppose that $\delta > 0$ and $n \geq 8 \log(4/\delta)$. Then with probability at least $1 - \delta$,
\[
\frac{N}{4} \leq |\{n \in [N] : y_n = 1\}|, |\{n \in [N] : y_n = -1\}| \leq \frac{3N}{4}.
\]
\end{lemma}

\begin{lemma}[Initialization of V, Lemma C.2 in \citet{jiang2024unveil}]\label{lemma:in_V}
Suppose that $\delta > 0$. Then with probability at least $1 - \delta$,
\[
|V_\pm^{(0)}| \leq d_h^{-\frac{1}{4}}, \quad |V_{n,i}^{(0)}| \leq d_h^{-\frac{1}{4}}
\]
for $i \in [M]\backslash\{1\}, n \in [N]$.
\end{lemma}

\begin{lemma}[Initialization of QK,  Lemma C.3 in~\citet{jiang2024unveil}]\label{lem:in_qk}
Suppose that $\delta > 0$. Then with probability at least $1 - \delta$,
\\
\begin{align*}
    \frac{\|\boldsymbol{\mu}\|_2^2 \sigma_p^2 d_h}{2} &\leq \|\boldsymbol{q}_\pm^{(0)}\|_2^2 \leq \frac{3\|\boldsymbol{\mu}\|_2^2 \sigma_p^2 d_h}{2}, \\
    \frac{\sigma_p^2 \sigma_h^2 d d_h}{2} &\leq \|\boldsymbol{q}_{n,i}^{(0)}\|_2^2 \leq \frac{3\sigma_p^2 \sigma_h^2 d d_h}{2}, \\
    \frac{\|\boldsymbol{\mu}\|_2^2 \sigma_p^2 d_h}{2} &\leq \|\boldsymbol{k}_\pm^{(0)}\|_2^2 \leq \frac{3\|\boldsymbol{\mu}\|_2^2 \sigma_p^2 d_h}{2}, \\
    \frac{\sigma_p^2 \sigma_h^2 d d_h}{2} &\leq \|\boldsymbol{k}_{n,i}^{(0)}\|_2^2 \leq \frac{3\sigma_p^2 \sigma_h^2 d d_h}{2}, \\
    |\langle \boldsymbol{q}_+^{(0)}, \boldsymbol{q}_-^{(0)} \rangle| &\leq 2\|\boldsymbol{\mu}\|_2^2 \sigma_h^2 \cdot \sqrt{d_h \log(6N^2 M^2 / \delta)}, \\
    |\langle \boldsymbol{q}_\pm^{(0)}, \boldsymbol{q}_{n,i}^{(0)} \rangle| &\leq 2\|\boldsymbol{\mu}\|_2 \sigma_p \sigma_h^2 d^{\frac{3}{2}} \cdot \sqrt{d_h \log(6N^2 M^2 / \delta)}, \\
    |\langle \boldsymbol{k}_\pm^{(0)}, \boldsymbol{k}_\pm^{(0)} \rangle| &\leq 2\|\boldsymbol{\mu}\|_2^2 \sigma_h^2 \cdot \sqrt{d_h \log(6N^2 M^2 / \delta)}, \\
    |\langle \boldsymbol{q}_\pm^{(0)}, \boldsymbol{k}_\pm^{(0)} \rangle| &\leq 2\|\boldsymbol{\mu}\|_2^2 \sigma_h^2 \cdot \sqrt{d_h \log(6N^2 M^2 / \delta)}, \\
    |\langle \boldsymbol{q}_\pm^{(0)}, \boldsymbol{k}_\mp^{(0)} \rangle| &\leq 2\|\boldsymbol{\mu}\|_2^2 \sigma_h^2 \cdot \sqrt{d_h \log(6N^2 M^2 / \delta)}, \\
    |\langle \boldsymbol{q}_{n,i}^{(0)}, \boldsymbol{k}_\pm^{(0)} \rangle| &\leq 2\|\boldsymbol{\mu}\|_2 \sigma_p \sigma_h^2 d^{\frac{3}{2}} \cdot \sqrt{d_h \log(6N^2 M^2 / \delta)}, \\
    |\langle \boldsymbol{q}_{n,i}^{(0)}, \boldsymbol{q}_{n',j}^{(0)} \rangle| &\leq 2\sigma_p^2 \sigma_h^2 d \cdot \sqrt{d_h \log(6N^2 M^2 / \delta)}, \\
    |\langle \boldsymbol{k}_{n,i}^{(0)}, \boldsymbol{k}_{n',j}^{(0)} \rangle| &\leq 2\sigma_p^2 \sigma_h^2 d \cdot \sqrt{d_h \log(6N^2 M^2 / \delta)}, \\
    |\langle \boldsymbol{k}_\pm^{(0)}, \boldsymbol{k}_{n,i}^{(0)} \rangle| &\leq 2\|\boldsymbol{\mu}\|_2 \sigma_p \sigma_h^2 d^{\frac{3}{2}} \cdot \sqrt{d_h \log(6N^2 M^2 / \delta)}, \\
    |\langle \boldsymbol{q}_\pm^{(0)}, \boldsymbol{k}_{n,i}^{(0)} \rangle| &\leq 2\|\boldsymbol{\mu}\|_2 \sigma_p \sigma_h^2 d^{\frac{3}{2}} \cdot \sqrt{d_h \log(6N^2 M^2 / \delta)}, \\
    |\langle \boldsymbol{q}_{n,i}^{(0)}, \boldsymbol{k}_{n',j}^{(0)} \rangle| &\leq 2\sigma_p^2 \sigma_h^2 d \cdot \sqrt{d_h \log(6N^2 M^2 / \delta)}
\end{align*}

for $i, j \in [M] \setminus \{1\}   \text{ and } n, n' \in [N]$.
\end{lemma}

\begin{lemma}[Lemma B.2 in \citet{cao2022benign} and Lemma B.4 in \citet{kou2023benign}]\label{lem:con_ineq}
Suppose that $\delta > 0$ and $d = \Omega(\log(4NM/\delta))$. Then with probability at least $1 - \delta$
\begin{align*}
    \frac{\sigma_p^2 d}{2} &\leq \|\boldsymbol{\xi}_{n,i}\|_2^2 \leq \frac{3\sigma_p^2 d}{2}, \\
    |\langle \boldsymbol{\xi}_{n,i}, \boldsymbol{\xi}_{n',i'} \rangle| &\leq 2\sigma_p^2 \cdot \sqrt{d \log(4N^2M^2/\delta)},\\
    \frac{\sigma_p^2 d}{2}-2\sigma_p\tau\sqrt{2\log(4NM/\delta)}-\tau^2 &\leq \|\widetilde{\boldsymbol{\xi}}_{n,i}\|_2^2 \leq \frac{3\sigma_p^2 d}{2}+2\sigma_p\tau\sqrt{2\log(4NM/\delta)}+\tau^2,\\
    (\|\boldsymbol{\mu}\|_2-\tau)^2&\leq\langle\widetilde{\boldsymbol{\mu}}_{\pm},\widetilde{\boldsymbol{\mu}}'_{\pm}\rangle\leq (\|\boldsymbol{\mu}\|_2+\tau)^2,\\
    |\langle\widetilde{\boldsymbol{\mu}}_{\pm},\widetilde{\boldsymbol{\xi}}_{n,i}\rangle|&\leq \|\boldsymbol{\mu}\|_2\tau+\sigma_p\tau\sqrt{2\log(4NM/\delta)}+\tau^2\\
    |\langle\widetilde{\boldsymbol{\xi}}_{n,i},\widetilde{\boldsymbol{\xi}}_{n',i'}\rangle|&\leq 2\sigma_p^2 \cdot \sqrt{d \log(4N^2M^2/\delta)}+2\sigma_p\tau\sqrt{2\log(4NM/\delta)}+\tau^2   
\end{align*}
for $i, i' \in [M]\backslash\{1\}, n, n' \in [N], i \neq i' \text{ or } n \neq n'$.
\end{lemma}

\section{Benign Overfitting in Case 1}
In this section, we consider the benign overfitting regime under the condition that $N \cdot \text{SNR}^2 = \Omega(1)$ and $\tau \leq O(\tfrac{\|\boldsymbol{\mu}\|_2}{\log d_h})$. We analyze the dynamics of $V_\pm$, $V_{n,i}$, the inner product $\boldsymbol q_\pm$, $\boldsymbol{q}_{n,i}$, and $\boldsymbol k_\pm$, $\boldsymbol{k}_{n,i}$ during adversarial training, and further give the upper bound for clean test error and robust test error. The proofs in this section are based on the results in Section~\ref{sec:con_ineq}, which hold with high probability.
\subsection{Stage i}\label{sec:s1}
In Stage I, $V_\pm^{(t)}$, $V_{n,i}^{(t)}$ begin to pull apart until $|V_\pm^{(t)}|$ is sufficiently larger than $|V_{n,i}^{(t)}|$. At the same time, the inner products of $q$ and $k$ maintain their magnitude.

\begin{lemma}[Gradient of Loss]\label{lem:gradient_loss}
As long as $\max\{|V_+^{(t)}|, |V_-^{(t)}|, |V_{n,i}^{(t)}|\} = o(1)$, we have $-\ell'(y_n f(\widetilde{\mathbf{X}}_n, \theta(t)))$ remains $1/2 \pm o(1)$.
\end{lemma}
\begin{proof}
    
 Note that $\ell(z) = \log(1 + \exp(-z))$ and $-\ell' = \exp(-z)/(1 + \exp(-z))$, without loss of generality, we assume $y_n = 1$, we have
\begin{align*}
    -\ell'(f(\widetilde{\mathbf{X}}_n, \theta(t))) &= \frac{1}{1 + \exp\left(\frac{1}{M} \sum_{l=1}^M \varphi(\widetilde{\mathbf{x}}_{n,l}^\top \mathbf{W}_Q^{(t)} \mathbf{W}_K^{(t)\top} \widetilde{\mathbf{X}}_n) \widetilde{\mathbf{X}}_n^\top \mathbf{W}_V^{(t)}\boldsymbol{w}_O\right)}.
\end{align*}
Note that
\begin{align*}
    -\max\{|V_+^{(t)}|, |V_-^{(t)}|, |V_{n,i}^{(t)}|\} \leq \frac{1}{M} \sum_{l=1}^M \varphi(\widetilde{\mathbf{x}}_{n,l}^\top \mathbf{W}_Q^{(t)\top} \widetilde{\mathbf{X}}_n^\top) \widetilde{\mathbf{X}}_n \mathbf{W}_V^{(t)}\boldsymbol{w}_O \leq \max\{|V_+^{(t)}|, |V_-^{(t)}|, |V_{n,i}^{(t)}|\}.
\end{align*}
Then we have
\begin{align*}
    -\ell'(f(\mathbf{\widetilde{X}}_n, \theta(t))) &\geq \frac{1}{1 + \exp(0 + o(1))} \geq \frac{1}{2 + o(1)} \geq \frac{1}{2} - o(1), \\
    -\ell'(f(\mathbf{\widetilde{X}}_n, \theta(t))) &\leq \frac{\exp(0 + o(1))}{1 + \exp(0 + o(1))} \leq \frac{1 + o(1)}{1 + 1 + o(1)} \leq \frac{1}{2} + o(1).
\end{align*}
\end{proof}

\begin{lemma}[Bound of Attention]\label{lem:bound_att}
As long as $|\langle \mathbf{q}_\pm^{(t)}, \mathbf{k}_\pm^{(t)} \rangle|, |\langle \mathbf{q}_{n,i}^{(t)}, \mathbf{k}_\pm^{(t)} \rangle|, |\langle \mathbf{q}_\pm^{(t)}, \mathbf{k}_{n,j}^{(t)} \rangle|, |\langle \mathbf{q}_{n,i}^{(t)}, \mathbf{k}_{n,j}^{(t)} \rangle| = o(1)$, we have
\begin{align*}
    \frac{1}{M} - o(1) &\leq \text{softmax}(\langle \mathbf{q}_\pm^{(t)}, \mathbf{k}_\pm^{(t)} \rangle) \leq \frac{1}{M} + o(1), \\
    \frac{1}{M} - o(1) &\leq \text{softmax}(\langle \mathbf{q}_{n,i}^{(t)}, \mathbf{k}_\pm^{(t)} \rangle) \leq \frac{1}{M} + o(1), \\
    \frac{1}{M} - o(1) &\leq \text{softmax}(\langle \mathbf{q}_\pm^{(t)}, \mathbf{k}_{n,j}^{(t)} \rangle) \leq \frac{1}{M} + o(1), \\
    \frac{1}{M} - o(1) &\leq \text{softmax}(\langle \mathbf{q}_{n,i}^{(t)}, \mathbf{k}_{n,j}^{(t)} \rangle) \leq \frac{1}{M} + o(1).
\end{align*}
\end{lemma}

\begin{proof}
    
 It is clear that $\exp(o(1)) = 1 + o(1)$. Therefore, as long as $|\langle \mathbf{q}_\pm^{(t)}, \mathbf{k}_\pm^{(t)} \rangle| = o(1)$, we have
\begin{align*}
    \frac{1}{M} - o(1) &= \frac{1}{1 + (M-1) + (M-1)o(1)} = \frac{1}{1 + (M-1)\exp(o(1))} = \\
    &= \frac{\exp(-o(1))}{\exp(-o(1)) + (M-1)\exp(o(1))} \leq \text{softmax}(\langle \mathbf{q}_\pm^{(t)}, \mathbf{k}_\pm^{(t)} \rangle) \leq \frac{\exp(o(1))}{\exp(o(1)) + (M-1)\exp(-o(1))} \\
    &= \frac{\exp(o(1))}{\exp(o(1)) + (M-1)} = \frac{1 + o(1)}{1 + o(1) + (M-1)} = \frac{1}{M} + o(1)
\end{align*}
Similarly, we have
\begin{align*}
    \frac{1}{M} - o(1) &\leq \text{softmax}(\langle \mathbf{q}_{n,i}^{(t)}, \mathbf{k}_\pm^{(t)} \rangle) \leq \frac{1}{M} + o(1), \\
    \frac{1}{M} - o(1) &\leq \text{softmax}(\langle \mathbf{q}_\pm^{(t)}, \mathbf{k}_{n,j}^{(t)} \rangle) \leq \frac{1}{M} + o(1), \\
    \frac{1}{M} - o(1) &\leq \text{softmax}(\langle \mathbf{q}_{n,i}^{(t)}, \mathbf{k}_{n,j}^{(t)} \rangle) \leq \frac{1}{M} + o(1).
\end{align*}
\end{proof}

\begin{lemma}[Upper bound of V]\label{lem:up_v}
Let $T_0 = \mathcal{O}\left(\frac{1}{\eta d_h^{\frac{1}{4}} (\|\boldsymbol{\mu}\|_2+\tau)^2 \|\boldsymbol{w}_O\|_2^2}\right)$. Then under the same conditions as Theorem~\ref{thm:main_thm} we have
\[
|V_+^{(t)}|, |V_-^{(t)}|, |V_{n,i}^{(t)}| = \mathcal{O}(d_h^{-\frac{1}{4}})
\]
for $t \in [0, T_0]$.
\end{lemma}

\begin{proof}
By Lemma~\ref{lem:update_v}, we have
{\begin{align*}
    |\gamma_{V,+}^{(t+1)} - \gamma_{V,+}^{(t)}| &\leq-  \frac{\eta \langle\widetilde{\boldsymbol{\mu}}_+,\widetilde{\boldsymbol{\mu}}_+^{(t)}\rangle}{NM} \sum_{n \in S_+} \widetilde{\ell}_n'^{(t)} \left( \frac{\exp(\langle \widetilde{\mathbf{q}}_+^{(t)}, \widetilde{\mathbf{k}}_+^{(t)} \rangle)}{\exp(\langle \widetilde{\mathbf{q}}_+^{(t)}, \widetilde{\mathbf{k}}_+^{(t)} \rangle) + \sum_{k=2}^M \exp(\langle \widetilde{\mathbf{q}}_+^{(t)}, \widetilde{\mathbf{k}}_{n,k}^{(t)} \rangle)} \right. \nonumber \\
    &\quad \left. + \sum_{j=2}^M \frac{\exp(\langle \widetilde{\mathbf{q}}_{n,j}^{(t)}, \widetilde{\mathbf{k}}_+^{(t)} \rangle)}{\exp(\langle \widetilde{\mathbf{q}}_{n,j}^{(t)}, \widetilde{\mathbf{k}}_+^{(t)} \rangle) + \sum_{k=2}^M \exp(\langle \widetilde{\mathbf{q}}_{n,j}^{(t)}, \widetilde{\mathbf{k}}_{n,k}^{(t)} \rangle)} \right) \|\boldsymbol{w}_O\|_2^2 
    \\&+\frac{-\eta\langle\widetilde{\boldsymbol{\mu}}_+,\widetilde{\boldsymbol{\xi}}_{n,i}^{(t)}\rangle}{NM}\sum_{n \in S_+}\sum_{i=2}^M\bigg(\frac{\exp(\langle \widetilde{\mathbf{q}}_+^{(t)}, \widetilde{\mathbf{k}}_{n,i}^{(t)} \rangle)}{\exp(\langle \widetilde{\mathbf{q}}_+^{(t)}, \widetilde{\mathbf{k}}_+^{(t)} \rangle) + \sum_{k=2}^M \exp(\langle \widetilde{\mathbf{q}}_+^{(t)}, \widetilde{\mathbf{k}}_{n,k}^{(t)} \rangle)} 
    \\&+\sum_{j=2}^M\frac{\exp(\langle \widetilde{\mathbf{q}}_{n,j}^{(t)}, \widetilde{\mathbf{k}}_{n,i}^{(t)} \rangle)}{\exp(\langle \widetilde{\mathbf{q}}_{n,j}^{(t)}, \widetilde{\mathbf{k}}_+^{(t)} \rangle) + \sum_{k=2}^M \exp(\langle \widetilde{\mathbf{q}}_{n,j}^{(t)}, \widetilde{\mathbf{k}}_{n,k}^{(t)} \rangle)}\bigg) \\
    &\leq \frac{\eta (\|\boldsymbol{\mu}\|_2+\tau)^2}{NM} \cdot \frac{3N}{4} (1 + (M - 1))+ \frac{\eta (\|\boldsymbol{\mu}\|_2\tau+\sigma_p\tau\sqrt{2\log(4NM/\delta)}+\tau^2)}{NM} \cdot \frac{3N}{4}  M(M - 1)\\
    &\leq O(\eta (\|\boldsymbol{\mu}\|_2+\tau)^2),
\end{align*}}
where the second inequality is by Lemma~\ref{lemma:num_pos} and Lemma~\ref{lem:con_ineq} and $-\widetilde{\ell}_n'^{(t)} \leq 1$.

Similarly, we have
\[
|\gamma_{V,-}^{(t+1)} - \gamma_{V,-}^{(t)}| \leq O(\eta (\|\boldsymbol{\mu}\|_2+\tau)^2).
\]
By Definition~\ref{def:sca_v}, we have
\begin{align*}
    |V_+^{(t)}| &= |V_+^{(0)} + \sum_{s=0}^{t-1} (\gamma_{V,+}^{(s+1)} - \gamma_{V,+}^{(s)}) \|\boldsymbol{w}_O\|_2^2| \\
    &\leq |V_+^{(0)}| + \sum_{s=0}^{t-1} |\gamma_{V,+}^{(s+1)} - \gamma_{V,+}^{(s)}| \cdot \|\boldsymbol{w}_O\|_2^2 \\
    &\leq d_h^{-\frac{1}{4}} + O(\eta (\|\boldsymbol{\mu}\|_2+\tau)^2) \cdot \|\boldsymbol{w}_O\|_2^2 \cdot O\left(\frac{1}{\eta d_h^{\frac{1}{4}} (\|\boldsymbol{\mu}\|_2+\tau)^2 \|\boldsymbol{w}_O\|_2^2}\right) \\
    &= O(d_h^{-\frac{1}{4}}),
\end{align*}
where the first inequality is by triangle inequality, the second inequality is by Lemma~\ref{lemma:in_V}. Similarly, we have $|V_-^{(t)}| = O(d_h^{-\frac{1}{4}})$.
By Lemma~\ref{lem:update_v}, we have
\begin{equation}\label{eq:rho_up1}
\begin{aligned}
    |\rho_{V,n,i}^{(t+1)} - \rho_{V,n,i}^{(t)}|
    &\leq \Bigg| -\frac{\eta}{NM} \sum_{n' \in S_+} \widetilde{\ell}_{n'}'^{(t)} \left(\left( \langle\widetilde{\boldsymbol{\xi}}_{n,i},\widetilde{\boldsymbol{\mu}}_+^{(t)}\rangle \frac{\exp(\langle \widetilde{\mathbf{q}}_+^{(t)}, \widetilde{\mathbf{k}}_+^{(t)} \rangle)}{\exp(\langle \widetilde{\mathbf{q}}_+^{(t)}, \widetilde{\mathbf{k}}_+^{(t)} \rangle) + \sum_{k=2}^M \exp(\langle \widetilde{\mathbf{q}}_+^{(t)}, \widetilde{\mathbf{k}}_{n,k}^{(t)} \rangle)}  \right.\right.\nonumber \\
    &\left.\quad\quad\quad\quad\quad\quad + \sum_{j=2}^M \langle\widetilde{\boldsymbol{\xi}}_{n,i},\widetilde{\boldsymbol{\mu}}_+^{(t)}\rangle \frac{\exp(\langle \widetilde{\mathbf{q}}_{n,j}^{(t)}, \widetilde{\mathbf{k}}_+^{(t)} \rangle)}{\exp(\langle \widetilde{\mathbf{q}}_{n,j}^{(t)}, \widetilde{\mathbf{k}}_+^{(t)} \rangle) + \sum_{k=2}^M \exp(\langle \widetilde{\mathbf{q}}_{n,j}^{(t)}, \widetilde{\mathbf{k}}_{n,k}^{(t)} \rangle)} \right) \\
        &\quad\quad\quad\quad\quad+\sum_{i=2}^M \left(\langle \widetilde{\boldsymbol{\xi}}_{n,i}, \widetilde{\boldsymbol{\xi}}_{n',i'}^{(t)} \rangle
        \frac{\exp(\langle \widetilde{\mathbf{q}}_{n',i'}^{(t)}, \widetilde{\mathbf{k}}_{n',i'}^{(t)} \rangle)}{\exp(\langle \widetilde{\mathbf{q}}_{n',i'}^{(t)}, \widetilde{\mathbf{k}}_{n',i'}^{(t)} \rangle) + \sum_{k=2}^M \exp(\langle \widetilde{\mathbf{q}}_{n',i'}^{(t)}, \widetilde{\mathbf{k}}_{n',k}^{(t)} \rangle)} \right.\nonumber \\
    &\left.\left.\quad\quad\quad\quad\quad\quad + \sum_{j=2}^M  \langle \widetilde{\boldsymbol{\xi}}_{n,i}, \widetilde{\boldsymbol{\xi}}_{n',i'}^{(t)} \rangle 
        \frac{\exp(\langle \widetilde{\mathbf{q}}_{n',j}^{(t)}, \widetilde{\mathbf{k}}_{n',i'}^{(t)} \rangle)}{\exp(\langle \widetilde{\mathbf{q}}_{n',j}^{(t)}, \widetilde{\mathbf{k}}_{n',i'}^{(t)} \rangle) + \sum_{k=2}^M \exp(\langle \widetilde{\mathbf{q}}_{n',j}^{(t)}, \widetilde{\mathbf{k}}_{n',k}^{(t)} \rangle)}
    \right)\right) \nonumber \\
    &\quad - \frac{\eta}{NM} \sum_{n' \in S_-} \left(\cdot\right)\Bigg| \\
    &\leq \frac{3\eta \sigma_p^2 d}{2NM} \cdot M + \frac{\eta}{NM} \cdot MN \cdot 2\sigma_p^2 \cdot \sqrt{d \log(4N^2M^2/\delta)}\\
    & \quad+ \frac{\eta (\|\boldsymbol{\mu}\|\tau+(2M-1)\sigma_p\tau\sqrt{3d/2}+M\tau^2)}{NM} \cdot N M \\
    &\leq \frac{2\eta \sigma_p^2 d}{N}+\frac{\eta (\|\boldsymbol{\mu}\|\tau+\sigma_p\tau\sqrt{2\log(4MN/\delta)}+\tau^2)}{NM} \cdot N M^2 \\
    &= O(\eta(\max\{(\|\boldsymbol{\mu}\|_2+\tau)^2, \frac{ \sigma_p^2 d}{N}\}))\\
    &= O(\eta(\|\boldsymbol{\mu}\|_2+\tau)^2)
\end{aligned}
\end{equation}

where the second inequality is by Lemma~\ref{lem:con_ineq} and $-\ell_n'^{(t)} \leq 1$, the third inequality is by $d = \widetilde{\Omega}(\epsilon^{-2} N^2 d_h) \geq 4N \sqrt{\log(4N^2 M^2/\delta)}$, the last inequality is by $N \cdot \text{SNR}^2 = \Omega(1)$. Then by Definition~\ref{def:sca_v}, we have
\begin{align*}
    |V_{n,i}^{(t)}| &= |V_{n,i}^{(0)} + \sum_{s=0}^{t-1} (\rho_{V,n,i}^{(s+1)} - \rho_{V,n,i}^{(s)}) \|\boldsymbol{w}_O\|_2^2| \\
    &\leq |V_{n,i}^{(0)}| + \sum_{s=0}^{t-1} |\rho_{V,n,i}^{(s+1)} - \rho_{V,n,i}^{(s)}| \cdot \|\boldsymbol{w}_O\|_2^2 \\
    &\leq d_h^{-\frac{1}{4}} + O(\eta (\|\boldsymbol{\mu}\|_2+\tau)^2) \cdot \|\boldsymbol{w}_O\|_2^2 \cdot O\left(\frac{1}{\eta d_h^{\frac{1}{4}} (\|\boldsymbol{\mu}\|_2+\tau)^2\|\boldsymbol{w}_O\|_2^2}\right) \\
    &= O(d_h^{-\frac{1}{4}}),
\end{align*}
where the first inequality is by triangle inequality, the second inequality is by Lemma C.2, which completes the proof.
\end{proof}

\begin{lemma}[Inner Products Hold Magnitude]\label{lem:innerpro_s1} Let $T_0 = O\left(\frac{1}{\eta d_h^{\frac{1}{4}}(\|\boldsymbol{\mu}\|_2+\tau)^2 \|{\boldsymbol{w}_O}\|_2^2}\right)$. Then under the same conditions as Theorem~\ref{thm:main_thm}, we have
\begin{align*}
&|\langle\boldsymbol{q}_{\pm}^{(t)},\boldsymbol{k}_{\pm}^{(t)} \rangle|, |\langle\boldsymbol{q}_{n,i}^{(t)},\boldsymbol{k}_{\pm}^{(t)} \rangle|, |\langle\boldsymbol{q}_{\pm}^{(t)},\boldsymbol{k}_{n,j}^{(t)} \rangle|, |\langle\boldsymbol{q}_{n,i}^{(t)},\boldsymbol{k}_{n',j}^{(t)} \rangle| \\
&= O\left( \max\{\|\boldsymbol{\mu}\|_2^2, \sigma_p^2 d\} \cdot \sigma_h^2 \cdot \sqrt{d_h} \log(6N^2M^2/\delta) \right),
\end{align*}
\begin{align*}
&|\langle\boldsymbol{q}_{\pm}^{(t)},\boldsymbol{q}_{\mp}^{(t)} \rangle|, |\langle\boldsymbol{q}_{n,i}^{(t)},\boldsymbol{q}_{\mp}^{(t)} \rangle|, |\langle\boldsymbol{q}_{n,i}^{(t)},\boldsymbol{q}_{n',j}^{(t)} \rangle| \\
&= O\left( \max\{\|\boldsymbol{\mu}\|_2^2, \sigma_p^2 d\} \cdot \sigma_h^2 \cdot \sqrt{d_h} \log(6N^2M^2/\delta) \right),
\end{align*}
\begin{align*}
&|\langle\boldsymbol{k}_{\pm}^{(t)},\boldsymbol{k}_{\mp}^{(t)} \rangle|, |\langle\boldsymbol{k}_{n,i}^{(t)},\boldsymbol{k}_{\pm}^{(t)} \rangle|, |\langle\boldsymbol{k}_{n,i}^{(t)},\boldsymbol{k}_{n',j}^{(t)} \rangle| \\
&= O\left( \max\{\|\boldsymbol{\mu}\|_2^2, \sigma_p^2 d\} \cdot \sigma_h^2 \cdot \sqrt{d_h} \log(6N^2M^2/\delta) \right),
\end{align*}
\[
\|\boldsymbol{q}_{\pm}^{(t)}\|_2^2, \|\boldsymbol{k}_{\pm}^{(t)}\|_2^2 = \Theta(\|\boldsymbol{\mu}\|_2^2 \sigma_h^2 d_h),
\]
\[
\|\boldsymbol{q}_{n,i}^{(t)}\|_2^2, \|\boldsymbol{k}_{n,i}^{(t)}\|_2^2 = \Theta(\sigma_p^2 \sigma_h^2 d d_h)
\]
for $i, j \in [M] \backslash \{1\}$, $n, n' \in [N]$ and $t \in [0, T_0]$.

\end{lemma}
The proof for Lemma~\ref{lem:innerpro_s1} is in Section~\ref{sec:pf_inner_pro1}. Note that $\sigma_h^2 \leq \min\{\|\boldsymbol{\mu}\|_2^{-2}, (\sigma_p^2 d)^{-1}\} \cdot d_h^{-\frac{1}{2}} \cdot (\log(6N^2M^2/\delta))^{-\frac{3}{2}}$, thus $O\left( \max\{\|\boldsymbol{\mu}\|_2^2, \sigma_p^2 d\} \cdot \sigma_h^2 \cdot \sqrt{d_h} \log(6N^2M^2/\delta) \right) = o(1)$.

\begin{lemma}[V's Beginning of Learning Signals]
Under the same conditions as Theorem~\ref{thm:main_thm}, there exists
\[
T_1 = \frac{10M(3M+1)N}{\eta d_h^{\frac{1}{4}} \left(N (\|\boldsymbol{\mu}\|_2-\tau)^2 -(N+30M^2)(\|\boldsymbol{\mu}\|\tau+2\sigma_p\tau\sqrt{2\log(4NM/\delta)}+\tau^2)- 60M^2  \sigma_p^2 d\right) \|\boldsymbol{w}_O\|_2}
\]
such that the first element of the vector $\widetilde{X}_n \mathbf{W}_V^{(t)}\boldsymbol{w}_O$ dominates its other elements, that is,
\[
V_+^{(t)} \geq 3M \cdot |V_{n,i}^{(t)}|, \quad \text{for all } n \in S_+, \; i \in [M] \setminus \{1\},
\]
\[
V_-^{(t)} \leq -3M \cdot |V_{n,i}^{(t)}|, \quad \text{for all } n \in S_-, \; i \in [M] \setminus \{1\}.
\]
\end{lemma}

\begin{proof}
Let $C$ be a constant larger than $10M(3M + 1)$. As long as

\[
N (\|\boldsymbol{\mu}\|_2-\tau)^2 -(N+30M^2)(\|\boldsymbol{\mu}\|\tau+2\sigma_p\tau\sqrt{2\log(4NM/\delta)}+\tau^2)- 60M^2  \sigma_p^2 d\geq\frac{10M(3M + 1)}{C} N (\|\boldsymbol{\mu}\|_2+\tau)^2.
\]
Thus, we further get

{
\[
T_1 = \frac{10M(3M + 1)N}{\eta d_h^{\frac{1}{4}} \left(N (\|\boldsymbol{\mu}\|_2-\tau)^2 -(N+30M^2)(\|\boldsymbol{\mu}\|\tau+2\sigma_p\tau\sqrt{2\log(4NM/\delta)}+\tau^2)- 60M^2  \sigma_p^2 d\right) \|\boldsymbol{w}_O\|_2}
\]
\[
\leq \frac{C}{\eta d_h^{\frac{1}{4}} (\|\boldsymbol{\mu}\|_2+\tau)^2 \|\boldsymbol{w}_O\|_2}
= \mathcal{O} \left( \frac{1}{\eta d_h^{\frac{1}{4}} (\|\boldsymbol{\mu}\|_2+\tau)^2 \|\boldsymbol{w}_O\|_2} \right),
\]
}
which satisfies the time condition in Lemma~\ref{lem:up_v} and Lemma~\ref{lem:innerpro_s1}. Then by Lemma~\ref{lem:gradient_loss} and Lemma~\ref{lem:bound_att} we have:
\[
-\ell_n^{\prime(t)} = \frac{1}{2} \pm o(1),
\]
\[
\frac{1}{M} - o(1) \leq \text{softmax}(\langle\boldsymbol{q}_+^{(t)},\boldsymbol{k}_+^{(t)} \rangle) \leq \frac{1}{M} + o(1),
\]
\[
\frac{1}{M} - o(1) \leq \text{softmax}(\langle\boldsymbol{q}_{n,i}^{(t)},\boldsymbol{k}_+^{(t)} \rangle) \leq \frac{1}{M} + o(1),
\]
\[
\frac{1}{M} - o(1) \leq \text{softmax}(\langle\boldsymbol{q}_{n,i}^{(t)},\boldsymbol{k}_{n,j}^{(t)} \rangle) \leq \frac{1}{M} + o(1).
\]
For $i,j \in [M] \setminus \{1\}$, $n \in [N]$, and $t \in [0, T_1]$, plugging these into the update rule for $\gamma_{V_+}^{(t)}$ shown in Lemma~\ref{lem:update_v}, we have:

\begin{align*}
\gamma_{V_+}^{(t+1)} - \gamma_{V_+}^{(t)} 
&= - \frac{\eta \langle\widetilde{\boldsymbol{\mu}}_+,\widetilde{\boldsymbol{\mu}}_+^{(t)}\rangle}{NM} \sum_{n \in S_+} \widetilde{\ell}_n'^{(t)} \left( \frac{\exp(\langle \widetilde{\mathbf{q}}_+^{(t)}, \widetilde{\mathbf{k}}_+^{(t)} \rangle)}{\exp(\langle \widetilde{\mathbf{q}}_+^{(t)}, \widetilde{\mathbf{k}}_+^{(t)} \rangle) + \sum_{k=2}^M \exp(\langle \widetilde{\mathbf{q}}_+^{(t)}, \widetilde{\mathbf{k}}_{n,k}^{(t)} \rangle)} \right. \nonumber \\
    &\quad\quad\quad\quad\quad\quad\quad\quad\quad \left. + \sum_{j=2}^M \frac{\exp(\langle \widetilde{\mathbf{q}}_{n,j}^{(t)}, \widetilde{\mathbf{k}}_+^{(t)} \rangle)}{\exp(\langle \widetilde{\mathbf{q}}_{n,j}^{(t)}, \widetilde{\mathbf{k}}_+^{(t)} \rangle) + \sum_{k=2}^M \exp(\langle \widetilde{\mathbf{q}}_{n,j}^{(t)}, \widetilde{\mathbf{k}}_{n,k}^{(t)} \rangle)} \right)  
    \\&\quad+\sum_{n \in S_+}\widetilde{\ell}_n'^{(t)}\sum_{i=2}^M\frac{-\eta\langle\widetilde{\boldsymbol{\mu}}_+,\widetilde{\boldsymbol{\xi}}_{n,i}^{(t)}\rangle}{NM}\bigg(\frac{\exp(\langle \widetilde{\mathbf{q}}_+^{(t)}, \widetilde{\mathbf{k}}_{n,i}^{(t)} \rangle)}{\exp(\langle \widetilde{\mathbf{q}}_+^{(t)}, \widetilde{\mathbf{k}}_+^{(t)} \rangle) + \sum_{k=2}^M \exp(\langle \widetilde{\mathbf{q}}_+^{(t)}, \widetilde{\mathbf{k}}_{n,k}^{(t)} \rangle)} 
    \\&\quad\quad\quad\quad\quad\quad\quad\quad\quad+\sum_{j=2}^M\frac{\exp(\langle \widetilde{\mathbf{q}}_{n,j}^{(t)}, \widetilde{\mathbf{k}}_{n,i}^{(t)} \rangle)}{\exp(\langle \widetilde{\mathbf{q}}_{n,j}^{(t)}, \widetilde{\mathbf{k}}_+^{(t)} \rangle) + \sum_{k=2}^M \exp(\langle \widetilde{\mathbf{q}}_{n,j}^{(t)}, \widetilde{\mathbf{k}}_{n,k}^{(t)} \rangle)}\bigg)\\
&\geq \frac{\eta (\|\boldsymbol{\mu}\|_2-\tau)_2^2}{NM} \cdot \frac{N}{4} \cdot \left( \frac{1}{2} \pm o(1) \right) \cdot M \left( \frac{1}{M} \pm o(1) \right)\\
&-\frac{\eta (\|\boldsymbol{\mu}\|\tau+2\sigma_p\tau\sqrt{2\log(4NM/\delta)}+\tau^2)}{NM} \cdot \frac{N}{4} \cdot \left( \frac{1}{2} \pm o(1) \right) \cdot M(M-1) \left( \frac{1}{M} \pm o(1) \right)\\
&\geq \frac{\eta \left((\|\boldsymbol{\mu}\|_2-\tau)_2^2-(\|\boldsymbol{\mu}\|\tau+2\sigma_p\tau\sqrt{2\log(4NM/\delta)}+\tau^2)\right)}{10M}
\end{align*}

Then by Definition~\ref{def:sca_v} and summing over $T_1$ steps, we have:

\begin{equation}\label{eq:19}
\begin{aligned}
V_+^{(T_1)} &\geq -|V_+^{(0)}| + T_1 \cdot \frac{\eta \left((\|\boldsymbol{\mu}\|_2-\tau)_2^2-(\|\boldsymbol{\mu}\|\tau+2\sigma_p\tau\sqrt{2\log(4NM/\delta)}+\tau^2)\right)}{10M} \cdot \|\boldsymbol{w}_O\|_2^2 \\
&= -d_h^{-\frac{1}{4}} + T_1 \cdot \frac{\eta \left((\|\boldsymbol{\mu}\|_2-\tau)_2^2-(\|\boldsymbol{\mu}\|\tau+2\sigma_p\tau\sqrt{2\log(4NM/\delta)}+\tau^2)\right)}{10M} \cdot \|\boldsymbol{w}_O\|_2^2
\end{aligned}    
\end{equation}

Similarly, we have:
\begin{equation}\label{eq:20}
V_-^{(T_1)} \leq d_h^{-\frac{1}{4}} - T_1 \cdot \frac{\eta \left((\|\boldsymbol{\mu}\|_2-\tau)_2^2-(\|\boldsymbol{\mu}\|\tau+2\sigma_p\tau\sqrt{2\log(4NM/\delta)}+\tau^2)\right)}{10M} \cdot \|\boldsymbol{w}_O\|_2^2
\end{equation}

Similarly, by the bound
{
\[
|\rho_{V_{n,i}^{(t+1)}} - \rho_{V_{n,i}^{(t)}}| \leq \frac{2 \eta \sigma_p^2 d}{N}+\eta (\|\boldsymbol{\mu}\|\tau+2\sigma_p\tau\sqrt{2\log(4NM/\delta)}+\tau^2)
\]
}

given in equation (\ref{eq:rho_up1}), we have

\begin{equation}\label{eq:21}
\begin{aligned}
|V_{n,i}^{(T_1)}| 
&\leq |V_{n,i}^{(0)}| + T_1 \cdot \left(\frac{2 \eta \sigma_p^2 d}{N}+\eta (\|\boldsymbol{\mu}\|\tau+2\sigma_p\tau\sqrt{2\log(4NM/\delta)}+\tau^2)\right) \cdot \|\boldsymbol{w}_O\|_2^2 \\
&= d_h^{-\frac{1}{4}} + T_1 \cdot \left(\frac{2 \eta \sigma_p^2 d}{N}+\eta (\|\boldsymbol{\mu}\|\tau+2\sigma_p\tau\sqrt{2\log(4NM/\delta)}+\tau^2)\right) \cdot \|\boldsymbol{w}_O\|_2^2
\end{aligned}
\end{equation}

According to equations~(\ref{eq:19}), (\ref{eq:20}), and (\ref{eq:21}), it is easy to verify that
\[
V_+^{(T_1)} - 3M \cdot |V_{n,i}^{(T_1)}| \geq 0 \quad \text{and} \quad V_-^{(T_1)} + 3M \cdot |V_{n,i}^{(T_1)}| \leq 0,
\]
which completes the proof.

\end{proof}

\subsection{Stage ii}\label{sec:stage2}
In stage II, $\langle\boldsymbol{q}_+,\boldsymbol{k}_+ \rangle, \langle\boldsymbol{q}_{n,i},\boldsymbol{k}_+ \rangle$ grows while $\langle\boldsymbol{q}_+,\boldsymbol{k}_{n,j} \rangle, \langle\boldsymbol{q}_{n,i},\boldsymbol{k}_{n,j} \rangle$ decreases, resulting in attention focusing more and more on the signals and less on the noises. By the results of stage I, we have the following conditions at the beginning of stage II
\begin{equation*}
V_+^{(T_1)} \geq 3M \cdot |V_{n,i}^{(T_1)}|,
\end{equation*}
\begin{equation*}
V_-^{(T_1)} \leq -3M \cdot |V_{n,i}^{(T_1)}|,
\end{equation*}

\begin{equation*}
|V_+^{(T_1)}|, |V_-^{(T_1)}|, |V_{n,i}^{(T_1)}| = O(d_h^{-\frac{1}{4}}),
\end{equation*}

\begin{equation*}
\begin{aligned}
&|\langle\boldsymbol{q}_\pm^{(T_1)},\boldsymbol{k}_\pm^{(T_1)} \rangle|, |\langle\boldsymbol{q}_{n,i}^{(T_1)},\boldsymbol{k}_\pm^{(T_1)} \rangle|, |\langle\boldsymbol{q}_\pm^{(T_1)},\boldsymbol{k}_{n,j}^{(T_1)} \rangle|, |\langle\boldsymbol{q}_{n,i}^{(T_1)},\boldsymbol{k}_{n',j}^{(T_1)} \rangle| \\
&= O \left( \max \{ \|\boldsymbol{\mu}\|_2^2, \sigma_p^2 d \} \cdot \sigma_h^2 \cdot \sqrt{d_h \log(6N^2M^2 / \delta) }\right),
\end{aligned}
\end{equation*}

\begin{equation*}
\begin{aligned}
&|\langle\boldsymbol{q}_\pm^{(T_1)},\boldsymbol{q}_\mp^{(T_1)} \rangle|, |\langle\boldsymbol{q}_{n,i}^{(T_1)},\boldsymbol{q}_\pm^{(T_1)} \rangle|, |\langle\boldsymbol{q}_{n,i}^{(T_1)},\boldsymbol{q}_{n',j}^{(T_1)} \rangle| \\
&= O \left( \max \{ \|\boldsymbol{\mu}\|_2^2 , \sigma_p^2 d \} \cdot \sigma_h^2 \cdot \sqrt{d_h \log(6N^2M^2 / \delta)} \right),
\end{aligned}
\end{equation*}

\begin{equation*}
\begin{aligned}
&|\langle\boldsymbol{k}_\pm^{(T_1)},\boldsymbol{k}_\pm^{(T_1)} \rangle|, |\langle\boldsymbol{k}_{n,i}^{(T_1)},\boldsymbol{k}_\pm^{(T_1)} \rangle|, |\langle\boldsymbol{k}_{n,i}^{(T_1)},\boldsymbol{k}_{n',j}^{(T_1)} \rangle| \\
&= O \left( \max \{ \|\boldsymbol{\mu}\|_2^2 , \sigma_p^2 d \} \cdot \sigma_h^2 \cdot \sqrt{d_h \log(6N^2M^2 / \delta) }\right),
\end{aligned}
\end{equation*}

\begin{equation*}
\|\boldsymbol{q}_\pm^{(T_1)} \|_2, \|\boldsymbol{k}_\pm^{(T_1)} \|_2 = \Theta(\|\boldsymbol{\mu}\|_2^2 \sigma_h^2 d_h),
\end{equation*}

\begin{equation*}
\|\boldsymbol{q}_{n,i}^{(T_1)} \|_2, \|\boldsymbol{k}_{n,i}^{(T_1)} \|_2 = \Theta(\sigma_p^2 \sigma_h^2 d  d_h)
\end{equation*}
for $i, j \in [M] \backslash \{1\}, n, n' \in [N]$.

\textit{Notations. To better characterize the gap between different inner products, we define the following notations:}
\begin{itemize}
    \item denote $\Lambda_{n,+,j}^{(t)} = \langle\boldsymbol{q}_+^{(t)},\boldsymbol{k}_+^{(t)} \rangle - \langle\boldsymbol{q}_+^{(t)},\boldsymbol{k}_{n,j}^{(t)} \rangle, \quad n \in S_+$.
    \item denote $\Lambda_{n,-,j}^{(t)} = \langle\boldsymbol{q}_-^{(t)},\boldsymbol{k}_-^{(t)} \rangle - \langle\boldsymbol{q}_-^{(t)},\boldsymbol{k}_{n,j}^{(t)} \rangle, \quad n \in S_-$.
    \item denote $\Lambda_{n,i,+,j}^{(t)} = \langle\boldsymbol{q}_{n,i}^{(t)},\boldsymbol{k}_+^{(t)} \rangle - \langle\boldsymbol{q}_{n,i}^{(t)},\boldsymbol{k}_{n,j}^{(t)} \rangle, \quad n \in S_+$.
    \item denote $\Lambda_{n,i,-,j}^{(t)} = \langle\boldsymbol{q}_{n,i}^{(t)},\boldsymbol{k}_-^{(t)} \rangle - \langle\boldsymbol{q}_{n,i}^{(t)},\boldsymbol{k}_{n,j}^{(t)} \rangle, \quad n \in S_-$.
\end{itemize}

\begin{lemma}[Upper bound of V]\label{lem:up_v_2}
Let $T_0 = O \left( \frac{1}{\eta (\|\boldsymbol{\mu}\|_2+\tau)^2 \|\boldsymbol{w}_O\|_2^2 \log(6N^2M^2 / \delta)} \right)$. Then under the same conditions as Theorem~\ref{thm:main_thm}, we have
\begin{equation*}
|V_+^{(t)}|, |V_-^{(t)}|, |V_{n,i}^{(t)}| = o(1)
\end{equation*}
for $t \in [0, T_0]$.
\end{lemma}
The proof of Lemma~\ref{lem:up_v_2} is similar to that of Lemma~\ref{lem:up_v}, except that the time $T_0$ is changed. Let $T_2 = \Omega \left( \frac{1}{\eta (\|\boldsymbol{\mu}\|_2+\tau)^2 \|\boldsymbol{w}_O\|_2^2 \log(6N^2M^2 / \delta)} \right)$, then by Lemma~\ref{lem:up_v_2} and Lemma~\ref{lem:gradient_loss} we have $\frac{1}{2} - o(1) \leq -\widetilde{\ell}_n'^{(t)} \leq \frac{1}{2} + o(1)$ for $n \in [N], t \in [T_1, T_2]$, which can simplify the calculations of $\alpha$ and $\beta$ defined in Definition~\ref{def:gra_de} by their bounds. Next we prove the following four propositions $\mathcal{B}(t), \mathcal{C}(t), \mathcal{D}(t), \mathcal{E}(t)$ by induction on $t$ for $t \in [T_1, T_2]$:

\begin{itemize}
    \item $\mathcal{B}(t)$:
    \begin{equation*}
    V_+^{(t)} \geq \eta C_3 (\|\boldsymbol{\mu}\|_2-\tau)^2 \|\boldsymbol{w}_O\|_2^2 (t - T_1)
    \end{equation*}
    \begin{equation*}
    V_- \leq\eta C_3 (\|\boldsymbol{\mu}\|_2-\tau)^2 \|\boldsymbol{w}_O\|_2^2 (t - T_1)
    \end{equation*}
    \begin{equation*}
    V_+^{(t)} \geq 3M \cdot |V_{n,i}^{(t)}|,
    \end{equation*}
    \begin{equation*}
    V_-^{(t)} \leq -3M \cdot |V_{n,i}^{(t)}|,
    \end{equation*}
    \begin{equation*}
    |V_+^{(t)}| \leq O(d_h^{-\frac{1}{4}}) + \eta C_4 (\|\boldsymbol{\mu}\|_2+\tau)^2 \|\boldsymbol{w}_O\|_2^2 (t - T_1)
    \end{equation*}
    \begin{equation*}
    |V_-^{(t)}| \leq O(d_h^{-\frac{1}{4}}) + \eta C_4 (\|\boldsymbol{\mu}\|_2+\tau)^2 \|\boldsymbol{w}_O\|_2^2 (t - T_1)
    \end{equation*}
for $i \in [M] \backslash \{1\}, n \in [N]$.

    \item $\mathcal{C}(t)$:
    \begin{equation*}
    \begin{aligned}
    & \|\boldsymbol q_\pm^{(t)}\|_2, \|\boldsymbol k_\pm^{(t)}\|_2 = \Theta(\|\boldsymbol{\mu}\|_2^2 \sigma_h^2 d_h), \\
    & \|\boldsymbol{q}_{n,i}^{(t)}\|_2, \|\boldsymbol{k}_{n,i}^{(t)}\|_2 = \Theta(\sigma_p^2 \sigma_h^2 d d d_h),
    \end{aligned}
    \end{equation*}
    
    \begin{equation*}
    |\langle\boldsymbol{q}_+^{(t)},\boldsymbol{q}_-^{(t)} \rangle|, |\langle\boldsymbol{q}_+^{(t)},\boldsymbol{q}_{n,i}^{(t)} \rangle|, |\langle\boldsymbol{q}_{n,i}^{(t)},\boldsymbol{q}_{n',j}^{(t)} \rangle| = o(1),
    \end{equation*}
    \begin{equation*}
    |\langle\boldsymbol{k}_+^{(t)},\boldsymbol{k}_-^{(t)} \rangle|, |\langle\boldsymbol{k}_+^{(t)},\boldsymbol{k}_{n,i}^{(t)} \rangle|, |\langle\boldsymbol{k}_{n,i}^{(t)},\boldsymbol{k}_{n',j}^{(t)} \rangle| = o(1),
    \end{equation*}
    \begin{equation*}
    \text{for } i, j \in [M] \backslash \{1\}, n, n' \in [N], i \neq j \text{ or } n \neq n'.
    \end{equation*}

\item $\mathcal{D}(t)$:
\begin{equation*}
\begin{aligned}
& \langle\boldsymbol{q}_+^{(t+1)},\boldsymbol{k}_+^{(t+1)} \rangle \geq \langle\boldsymbol{q}_+^{(t)},\boldsymbol{k}_+^{(t)} \rangle \\
& \langle\boldsymbol{q}_{n,i}^{(t+1)},\boldsymbol{k}_+^{(t+1)} \rangle \geq \langle\boldsymbol{q}_{n,i}^{(t)},\boldsymbol{k}_+^{(t)} \rangle \\
& \langle\boldsymbol{q}_+^{(t+1)},\boldsymbol{k}_{n,j}^{(t+1)} \rangle \leq \langle\boldsymbol{q}_+^{(t)},\boldsymbol{k}_{n,j}^{(t)} \rangle \\
& \langle\boldsymbol{q}_{n,i}^{(t+1)},\boldsymbol{k}_{n,j}^{(t+1)} \rangle \leq \langle\boldsymbol{q}_{n,i}^{(t)},\boldsymbol{k}_{n,j}^{(t)} \rangle
\end{aligned}
\end{equation*}
\begin{equation*}
\Lambda_{n,\pm,j}^{(t+1)} \geq \log \left( \exp(\Lambda_{n,\pm,j}^{(T_1)})  + \frac{\eta^2 C_8 (\|\boldsymbol{\mu}\|_2+\tau)^2\|\boldsymbol{\mu}\|_2^2 \|\boldsymbol{w}_O\|_2^2 d_h^{\frac{1}{2}}} {N \left( \log(6N^2 M^2 / \delta \right) )^2} \cdot (t - T_1)(t - T_1 + 1)\right)
\end{equation*}
\begin{equation*}
\begin{aligned}
\Lambda_{n,i,\pm,j}^{(t+1)}
&\ge \log\Bigg(\exp\big(\Lambda_{n,i,\pm,j}^{(T_1)}\big)
\\&+ \frac{\eta^2 C_8\,\big(\sigma_p^2 d
  + \sigma_p\tau\sqrt{2\log(4NM/\delta)}+\tau^2\big)
  \|\boldsymbol{\mu}\|_2^2\|\boldsymbol{w}_O\|_2^2 d_h}
 {N\big(\log(6N^2M^2/\delta)\big)^2}\cdot
\ (t-T_1)(t-T_1+1)\Bigg).
\end{aligned}
\end{equation*}
$\text{for } i, j \in [M] \backslash \{1\}, n \in [N].$

\item $\mathcal{E}(t)$:
\begin{equation*}
\begin{aligned}
& |\langle\boldsymbol q_\pm^{(t)},\boldsymbol k_\pm^{(t)} \rangle|, |\langle\boldsymbol{q}_{n,j}^{(t)},\boldsymbol k_\pm^{(t)} \rangle|, |\langle\boldsymbol{q}_{n,i}^{(t)},\boldsymbol{k}_{n,j}^{(t)} \rangle|, |\langle\boldsymbol{q}_{n,i}^{(t)},\boldsymbol{k}_{n',j}^{(t)} \rangle| \leq \log(d_h^{\frac{1}{4}}) \\
& |\langle\boldsymbol  q_\pm^{(t)},\boldsymbol k_\pm^{(t)} \rangle, |\langle\boldsymbol{q}_{n,i}^{(t)},\boldsymbol{k}_{n,j}^{(t)} \rangle| = o(1)
\end{aligned}
\end{equation*}
$\text{for } i, j \in [M] \backslash \{1\}, n, \bar{n} \in [N], n \neq \bar{n}.$
\end{itemize}

By the results of Stage I, we know that $\mathcal{B}(T_1), \mathcal{C}(T_1), \mathcal{E}(T_1)$ are true. To prove that $\mathcal{B}(t), \mathcal{C}(t), \mathcal{D}(t)$ and $\mathcal{E}(t)$ are true in Stage II, we will prove the following claims holds for $t \in [T_1, T_2]$:
\begin{claim}
\label{cl:1}$\mathcal{D}(T_1), \ldots, \mathcal{D}(t - 1),\mathcal{E}(T_1), \ldots, \mathcal{E}(t) \implies \mathcal{B}(t + 1)$
\end{claim}
\begin{claim}\label{cl:2}$\mathcal{B}(T_1), \ldots, \mathcal{B}(t), \mathcal{C}(T_1), \ldots, \mathcal{C}(t), \mathcal{D}(T_1), \ldots, \mathcal{D}(t - 1){, \mathcal{E}(T_1), \ldots, \mathcal{E}(t) } \implies \mathcal{D}(t)$
\end{claim}
\begin{claim}\label{cl:3}$\mathcal{B}(T_1), \ldots, \mathcal{B}(t), \mathcal{D}(T_1), \ldots, \mathcal{D}(t - 1), \mathcal{E}(T_1), \ldots, \mathcal{E}(t) \implies \mathcal{C}(t + 1)$
\end{claim}
\begin{claim}\label{cl:4}$\mathcal{B}(T_1), \ldots, \mathcal{B}(t), \mathcal{C}(T_1), \ldots, \mathcal{C}(t), \mathcal{D}(T_1), \ldots, \mathcal{D}(t - 1) ,\mathcal{E}(T_1), \ldots, \mathcal{E}(t)\implies \mathcal{E}(t + 1)$
\end{claim}

First, we emphasize that when the perturbation is sufficiently small, 
the softmax of the inner product $\langle \mathbf{q}, \mathbf{k} \rangle$ remains robust 
with respect to such perturbations. This is formally established in 
Lemma~\ref{lemma:rsoftmax1}. Importantly, this lemma enables us to uniformly control 
the behavior of the softmax function, rather than being restricted to the 
specific case of $\widetilde{q}^{(t)}$ and $\widetilde{k}^{(t)}$ at 
iteration~$t$.
\begin{lemma}\label{lemma:rsoftmax}
Suppose the perturbation satisfies $\tau \leq O(\tfrac{\|\boldsymbol{\boldsymbol{\mu}}\|_2}{\log d_h})$ and $t \geq \Omega \left( \frac{1}{\eta (\|\boldsymbol{\mu}\|_2+\tau)^2 \|\boldsymbol{w}_O\|_2^2 \log(6N^2M^2 / \delta)} \right)$($\mathcal{E}(t)$ holds). Then there exists a universal constant $C \leq e/2$ such that
\[
{\max\limits_{\widetilde{X}\in B(X,\tau)} 
{softmax}\!\left(\langle\mathbf{q}_+^{(t)} ,\mathbf{k}_{n,j}^{(t)}\rangle\right)}/
{\min\limits_{\widetilde{X}\in B(X,\tau)} 
{softmax}\!\left(\langle\mathbf{q}_+^{(t)} ,\mathbf{k}_{n,j}^{(t)}\rangle\right)}
\;\leq\; C.
\]
\end{lemma}

\begin{proof}
For any perturbed query-key pair, the softmax weight can be written as
\begin{equation}\label{eq:softmax_perturb}
\begin{aligned}
softmax\!\left(\langle\mathbf{q}_+^{(t)} ,\mathbf{k}_{n,j}^{(t)}\rangle\right)
&= \frac{\exp\!\left(\langle{\boldsymbol{q}}_+^{(t)}, {\boldsymbol{k}}_{n,j}^{(t)}\rangle\right)}{
\exp\!\left(\langle{\boldsymbol{q}}_+^{(t)}, {\boldsymbol{k}}_+^{(t)}\rangle\right)
+ \sum_{j'=2}^M \exp\!\left(\langle{\boldsymbol{q}}_+^{(t)}, {\boldsymbol{k}}_{n,j'}^{(t)}\rangle\right)} \\
&= \frac{1}{\exp\!\big(\langle{\boldsymbol{q}}_+^{(t)},{\boldsymbol{k}}_+^{(t)}\rangle 
- \langle{\boldsymbol{q}}_+^{(t)}, {\boldsymbol{k}}_{n,j}^{(t)}\rangle\big) 
+ \sum_{j'=2}^M \exp\!\big(\langle{\boldsymbol{q}}_+^{(t)},{\boldsymbol{k}}_{n,j'}^{(t)}\rangle 
- \langle{\boldsymbol{q}}_+^{(t)},{\boldsymbol{k}}_{n,j}^{(t)}\rangle\big)}.
\end{aligned}
\end{equation}

Next, we analyze the difference in the logits. Expanding the perturbation terms yields
\begin{equation}\label{eq:logit_diff}
\begin{aligned}
&\langle{\boldsymbol{q}}_+^{(t)}, {\boldsymbol{k}}_+^{(t)}\rangle 
- \langle{\boldsymbol{q}}_+^{(t)}, {\boldsymbol{k}}_{n,j}^{(t)}\rangle \\
&= \langle \boldsymbol{q}_+, \boldsymbol{k}_+ - \boldsymbol{k}_{n,j}\rangle
+ \langle \tau_+ \mathbf{W}_Q^{(t)}, \boldsymbol{k}_+ - \boldsymbol{k}_{n,j}\rangle \\
&\quad + \langle \boldsymbol{q}_+, \tau_+\mathbf{W}_K^{(t)} - \tau_{n,j}\mathbf{W}_K^{(t)}\rangle
+ \langle \tau_+ \mathbf{W}_Q^{(t)}, \tau_+\mathbf{W}_K^{(t)} - \tau_{n,j}\mathbf{W}_K^{(t)}\rangle \\
&= \langle \boldsymbol{q}_+^{(t)}, \boldsymbol{k}_+^{(t)} - \boldsymbol{k}_{n,j}^{(t)}\rangle \;\pm o(1).
\end{aligned}
\end{equation}

The first equality follows from decomposing the perturbed terms, while the second uses the perturbation bound 
$\tau \leq O(\tfrac{\|\boldsymbol{\boldsymbol{\mu}}\|_2}{\log d_h})$, together with the assumption that the logit magnitudes satisfy
\[
|\langle \boldsymbol{q}_\pm^{(t)}, \boldsymbol{k}_\pm^{(t)} \rangle|,\;
|\langle \boldsymbol{q}_\pm^{(t)}, \boldsymbol{k}_{n,j}^{(t)} \rangle|,\;
|\langle \boldsymbol{q}_{n,i}^{(t)}, \boldsymbol{k}_\pm^{(t)} \rangle|,\;
|\langle \boldsymbol{q}_{n,i}^{(t)}, \boldsymbol{k}_{n,j}^{(t)} \rangle|
\;\leq\; \log\!\left( d_h^{1/2}\right).
\]
An analogous relation holds for differences involving $\widetilde{\boldsymbol{k}}_{n,j'}^{(t)}$ as well.

Substituting~\eqref{eq:logit_diff} into~\eqref{eq:softmax_perturb}, we obtain
\[
\frac{1}{C}\,softmax\!\left(\langle\mathbf{q}_+^{(t)} ,\mathbf{k}_{n,j}^{(t)}\rangle\right)
\;\leq\;
softmax\!\left(\langle\mathbf{q}_+^{(t)} ,\mathbf{k}_{n,j}^{(t)}\rangle\right)
\;\leq\;
C\,softmax\!\left(\langle\mathbf{q}_+^{(t)} ,\mathbf{k}_{n,j}^{(t)}\rangle\right),
\]
for some absolute constant $1 \leq C \leq e/2$. This establishes the claim.

Thus, in the following discussion, we show that the updates related to $\boldsymbol q^{(t)}$ and $\boldsymbol k^{(t)}$ during adversarial training can be bounded, while the effect of the perturbation does not accumulate over time. Since $C$ is a very small function, when computing the single-step update, we approximate the perturbed softmax at the previous time step by its clean state.
\end{proof}
    
\subsubsection{ Proof of Claim~\ref{cl:1}}\label{sec:cl1}
By the results of Stage I, we have
\begin{equation*}
|\langle\boldsymbol q_\pm^{(T_1)},\boldsymbol k_\pm^{(T_1)} \rangle|, |\langle\boldsymbol q_\pm^{(T_1)},\boldsymbol{k}_{n,j}^{(T_1)} \rangle|, |\langle\boldsymbol{q}_{n,i}^{(T_1)},\boldsymbol k_\pm^{(T_1)} \rangle|, |\langle\boldsymbol{q}_{n,i}^{(T_1)},\boldsymbol{k}_{n,j}^{(T_1)} \rangle| = o(1)
\end{equation*}

Assume that $\mathcal{D}(T_1), \ldots, \mathcal{D}(t - 1)$ ($t \in [T_1, T_2]$) are true, then $\langle\boldsymbol q_\pm^{(s)},\boldsymbol k_\pm^{(s)}\rangle$, $\langle\boldsymbol{q}_{n,i}^{(s)},\boldsymbol k_\pm^{(s)}\rangle$ are monotonically non-decreasing and $\langle q_\pm^{(s)},\boldsymbol{k}_{n,j}^{(s)}\rangle$, $\langle\boldsymbol{q}_{n,i}^{(s)},\boldsymbol{k}_{n,j}^{(s)}\rangle$ are monotonically non-increasing for $s \in [T_1, t - 1]$, so we have
\begin{equation*}
\langle\boldsymbol q_\pm^{(s)},\boldsymbol k_\pm^{(s)} \rangle, \langle\boldsymbol{q}_{n,i}^{(s)},\boldsymbol k_\pm^{(s)} \rangle \geq -o(1),
\end{equation*}
\begin{equation*}
\langle\boldsymbol q_\pm^{(s)},\boldsymbol{k}_{n,j}^{(s)} \rangle, \langle\boldsymbol{q}_{n,i}^{(s)},\boldsymbol{k}_{n,j}^{(s)} \rangle \leq o(1),
\end{equation*}

for  $s \in [T_1, t]$.  Further we have the lower bounds for the attention on signal  $\boldsymbol{\mu}_+ \text{ as follows for } s \in [T_1, t]$:
\begin{equation}
\begin{aligned}
\text{softmax}(\langle\boldsymbol q_\pm^{(s)},\boldsymbol k_\pm^{(s)} \rangle) &\geq \frac{\exp(-o(1))}{\exp(-o(1)) + (M - 1) \exp(o(1))} \\
&= \frac{1}{1 + (M - 1) \exp(o(1))} \\
&= \frac{1}{1 + (M - 1) + (M - 1) o(1)} \\
&= \frac{1}{M} - o(1),
\end{aligned}
\end{equation} 
where the second equality is by  $\exp(o(1)) = 1 + o(1)$.  Similarly, we have 
\begin{equation}
\text{softmax}(\langle\boldsymbol{q}_{n,i}^{(s)},\boldsymbol k_\pm^{(s)} \rangle) \geq \frac{1}{M} - o(1).
\end{equation}

Plugging them in the update rule for  $\gamma_{V_+}^{(s)}$ shown in Lemma~\ref{lem:update_v} and we have
\begin{equation}
\begin{aligned}
&\gamma_{V_+}^{(s+1)} - \gamma_{V_+}^{(s)} \\
&= - \frac{\eta \langle\widetilde{\boldsymbol{\mu}}_+,\widetilde{\boldsymbol{\mu}}_+^{(t)}\rangle}{NM} \sum_{n \in S_+} \widetilde{\ell}_n'^{(t)} \left( \frac{\exp(\langle {\boldsymbol{q}}_+^{(t)}, {\boldsymbol{k}}_+^{(t)} \rangle)}{\exp(\langle {\boldsymbol{q}}_+^{(t)}, {\boldsymbol{k}}_+^{(t)} \rangle) + \sum_{k=2}^M \exp(\langle {\boldsymbol{q}}_+^{(t)}, {\boldsymbol{k}}_{n,k}^{(t)} \rangle)} \right. \nonumber \\
    &\quad\quad\quad\quad\quad\quad\quad\quad\quad \left. + \sum_{j=2}^M \frac{\exp(\langle {\boldsymbol{q}}_{n,j}^{(t)}, {\boldsymbol{k}}_+^{(t)} \rangle)}{\exp(\langle {\boldsymbol{q}}_{n,j}^{(t)}, {\boldsymbol{k}}_+^{(t)} \rangle) + \sum_{k=2}^M \exp(\langle {\boldsymbol{q}}_{n,j}^{(t)}, {\boldsymbol{k}}_{n,k}^{(t)} \rangle)} \right)  
    \\&+\sum_{n \in S_+}\widetilde{\ell}_n'^{(t)}\sum_{i=2}^M\frac{-\eta\langle\widetilde{\boldsymbol{\mu}}_+,\widetilde{\boldsymbol{\xi}}_{n,i}^{(t)}\rangle}{NM}\bigg(\frac{\exp(\langle {\boldsymbol{q}}_+^{(t)}, {\boldsymbol{k}}_{n,i}^{(t)} \rangle)}{\exp(\langle {\boldsymbol{q}}_+^{(t)}, {\boldsymbol{k}}_+^{(t)} \rangle) + \sum_{k=2}^M \exp(\langle {\boldsymbol{q}}_+^{(t)}, {\boldsymbol{k}}_{n,k}^{(t)} \rangle)} 
    \\&\quad\quad\quad\quad\quad\quad\quad\quad\quad+\sum_{j=2}^M\frac{\exp(\langle {\boldsymbol{q}}_{n,j}^{(t)}, {\boldsymbol{k}}_{n,i}^{(t)} \rangle)}{\exp(\langle {\boldsymbol{q}}_{n,j}^{(t)}, {\boldsymbol{k}}_+^{(t)} \rangle) + \sum_{k=2}^M \exp(\langle {\boldsymbol{q}}_{n,j}^{(t)}, {\boldsymbol{k}}_{n,k}^{(t)} \rangle)}\bigg)\\
&\geq \frac{\eta (\|\boldsymbol{\mu}\|_2-\tau)_2^2}{NMC} \cdot \frac{N}{4} \cdot \left( \frac{1}{2} \pm o(1) \right) \cdot M \left( \frac{1}{M} \pm o(1) \right)\\
&-\frac{\eta (\|\boldsymbol{\mu}\|\tau+\sigma_p\tau\sqrt{2\log(4NM/\delta)}+\tau^2)}{NMC} \cdot \frac{N}{4} \cdot \left( \frac{1}{2} \pm o(1) \right) \cdot M(M-1) \left( \frac{1}{M} \pm o(1) \right)\\
&\geq \frac{\eta \left((\|\boldsymbol{\mu}\|_2-\tau)^2-(\|\boldsymbol{\mu}\|\tau+\sigma_p\tau\sqrt{2\log(4NM/\delta)}+\tau^2)\right)}{10MC}=O\left(\frac{\eta (\|\boldsymbol{\mu}\|_2-{\tau})^2}{10MC}\right),    
\end{aligned}
\end{equation} 

for  $s \in [T_1, t]$.  The first inequality is by Lemma~\ref{lemma:rsoftmax1}. Then by Definition~\ref{def:sca_v} and taking a summation, we have
\begin{equation}\label{eq:24}
\begin{aligned}
V_+^{(t+1)} &\geq V_+^{(T_1)} + (t - T_1 + 1) \frac{\eta (\|\boldsymbol{\mu}\|_2-{\tau})^2 \|\boldsymbol{w}_O\|_2^2}{10MC} \\
&\geq \eta C_3 (\|\boldsymbol{\mu}\|_2-{\tau})^2 \|\boldsymbol{w}_O\|_2^2 (t - T_1 + 1),    
\end{aligned}
\end{equation}
where the last inequality is by  $V_+^{(T_1)} \geq 0$  and  $M = \Theta(1)$.  Similarly, we have
\begin{equation}
V_-^{(t+1)} \leq -\eta C_3 (\|\boldsymbol{\mu}\|_2- \textcolor{red}{\tau})^2 \|\boldsymbol{w}_O\|_2^2 (t - T_1 + 1).
\end{equation}

By $|\rho_{V_{n,i}^{(t+1)}} - \rho_{V_{n,i}^{(t)}}| \leq \eta(\|\boldsymbol{\mu}\|\tau+\sigma_p\tau\sqrt{2\log(4NM/\delta)}+\tau^2+ \frac{ 2\sigma_p^2 d}{N})$ in (\ref{eq:rho_up1}) and taking a summation, we have
\begin{align}\label{eq:25}
|V_{n,i}^{(t+1)}| &\leq |V_{n,i}^{(t)}| + (t - T_1 + 1) \eta(\|\boldsymbol{\mu}\|\tau+\sigma_p\tau\sqrt{2\log(4NM/\delta)}+\tau^2+ \frac{ 2\sigma_p^2 d}{N}) \|\boldsymbol{w}_O\|_2^2 .
\end{align} 

Combining (\ref{eq:24}) and (\ref{eq:25}) we have
\begin{equation}
\begin{aligned}
&V_+^{(t+1)} - 3M \cdot |V_{n,i}^{(t+1)}| \\
&\geq V_+^{(T_1)} + (t - T_1 + 1) \frac{\eta (\|\boldsymbol{\mu}\|_2-\tau)^2}{10M} \|\boldsymbol{w}_O\|_2^2 \\
&\quad- 3M \cdot (|V_{n,i}^{(T_1)}| + (t - T_1 + 1) \eta(\max\{\|\boldsymbol{\mu}\|\tau+\tau^2, \frac{ 2C_p^2\sigma_p^2 d}{N}\}) \|\boldsymbol{w}_O\|_2^2) \\
&\geq V_+^{(T_1)} - 3M \cdot |V_{n,i}^{(T_1)}| + (t - T_1 + 1) \frac{\eta (\|\boldsymbol{\mu}\|_2-\tau)^2}{10M} \|\boldsymbol{w}_O\|_2^2 \\
&\quad- 3M \cdot (|V_{n,i}^{(T_1)}| + (t - T_1 + 1) \eta(\max\{\|\boldsymbol{\mu}\|\tau+\tau^2, \frac{ 2C_p^2\sigma_p^2 d}{N}\}) \|\boldsymbol{w}_O\|_2^2) \\
&\geq 0,
\end{aligned}
\end{equation} 

where the last inequality is by $V_+^{(T_1)} \geq 3M \cdot |V_{n,i}^{(T_1)}|$ and requires $N \cdot \text{SNR}^2 \geq 60M^2 C_p^2$. The proof for $V_-^{(t+1)} \leq -3M \cdot |V_{n,i}^{(t+1)}|$ is the same.

Next, we prove the upper bound for $V_+$ and $V_{n,i}$. Based on the upper bound of attention($<1$) and $-\widetilde{\ell}'_n \leq 1$ we have
\begin{equation}
\begin{aligned}
\gamma_{V_+}^{(s+1)} &\leq \gamma_{V_+}^{(s)} - \frac{\eta}{NM} \sum_{n \in S_+} \widetilde{\ell}_n^{\prime(s)} ((\|\boldsymbol{\mu}\|_2+\tau)^2 + \sum_{j=2}^{M} (\|\boldsymbol{\mu}\|_2+\tau)^2) \\
&- \frac{\eta}{NM} \sum_{n \in S_+} \widetilde{\ell}_n^{\prime(s)} (\|\boldsymbol{\mu}\|\tau + \sigma_p \tau \sqrt{2\log(4NM/\delta)} + \tau^2)\cdot M \\
&\leq \gamma_{V_+}^{(s)} + \frac{3 \eta (\|\boldsymbol{\mu}\|_2+\tau)^2}{4} \\
&\leq \gamma_{V_+}^{(s)} + \eta C_4 (\|\boldsymbol{\mu}\|_2+\tau)^2
\end{aligned}
\end{equation} 

Then we can get that
\begin{equation}
\begin{aligned}
|V_+^{(t+1)}| &\leq V_+^{(T_1)} + (\gamma_{V_+}^{(t+1)} - \gamma_{V_+}^{(T_1)}) \|\boldsymbol{w}_O\|_2 \\
&\leq V_+^{(T_1)} + \sum_{s=T_1}^{t} \eta C_4 (\|\boldsymbol{\mu}\|_2+\tau)^2 \|\boldsymbol{w}_O\|_2^2 \\
&\leq O(d_h^{-\frac{1}{4}}) + \eta C_4 (\|\boldsymbol{\mu}\|_2+\tau)^2 \|\boldsymbol{w}_O\|_2^2 (t - T_1 + 1)
\end{aligned}
\end{equation} 

where the first inequality is by the monotonicity of $\gamma_{V_+}$ and the definition of $V_+$, the last inequality is by the result of stage 1 where $V_+^{(T_1)} = O(d^{-1})$. Similarly, we have
\begin{equation}
|V_-^{(t+1)}| \leq O(d_h^{-\frac{1}{4}}) + \eta C_4 (\|\boldsymbol{\mu}\|_2+\tau)^2 \|\boldsymbol{w}_O\|_2^2 (t - T_1 + 1)
\end{equation} 

which completes the proof for the upper bound of $V_+$.

Expanding (\ref{eq:25}) yields
\begin{equation}
\begin{aligned}
|V_{n,i}^{(t+1)}| &\leq |V_{n,i}^{(T_1)}| + \eta(\max\{\|\boldsymbol{\mu}\|\tau+\tau^2, \frac{ 2C_p^2\sigma_p^2 d}{N}\}) \cdot \|\boldsymbol{w}_O\|_2^2 (t - T_1 + 1) \\
&\leq O(d_h^{-\frac{1}{4}}) + \eta C_4 (\|\boldsymbol{\mu}\|_2+\tau)_2^2 \|\boldsymbol{w}_O\|_2 (t - T_1 + 1)
\end{aligned}
\end{equation} 

where the last inequality is by the result of phase 1 where $|V_{n,i}^{(T_1)}| = O(d_h^{-\frac{1}{4}})$ and the condition that $N \cdot \text{SNR}^2 \geq \Omega(1)$.

\subsubsection{Proof of Claim 2}
By the results of \ref{sec:low_qk}, we have the dynamic of $\langle\boldsymbol q,\boldsymbol k \rangle$ as follows
\begin{equation}
\begin{aligned}
&\langle\boldsymbol{q}_+^{(s+1)},\boldsymbol{k}_+^{(s+1)} \rangle - \langle\boldsymbol{q}_+^{(s)},\boldsymbol{k}_+^{(s)} \rangle \\
&\geq \frac{\eta^2 C_6 (\|\boldsymbol{\mu}\|_2-\tau)^4\|\boldsymbol{\mu}\|_2^2 \|\boldsymbol{w}_O\|_2^2 \sigma_h^2 d_h (s - T_1)}{N} \cdot \frac{1}{\exp(\Lambda_{n,+,j}^{(s)})},
\end{aligned}
\end{equation}
\begin{equation}
\begin{aligned}
&\langle\boldsymbol{q}_-^{(s+1)},\boldsymbol{k}_-^{(s+1)} \rangle - \langle\boldsymbol{q}_-^{(s)},\boldsymbol{k}_-^{(s)} \rangle \\
&\geq \frac{\eta^2 C_6 (\|\boldsymbol{\mu}\|_2-\tau)^4\|\boldsymbol{\mu}\|_2^2 \|\boldsymbol{w}_O\|_2^2 \sigma_h^2 d_h (s - T_1)}{N} \cdot \frac{1}{\exp(\Lambda_{n,-,j}^{(s)})},
\end{aligned} 
\end{equation}
\begin{equation}
\begin{aligned}
&\langle\boldsymbol{q}_+^{(s+1)},\boldsymbol{k}_{n,j}^{(s+1)} \rangle - \langle\boldsymbol{q}_+^{(s)},\boldsymbol{k}_{n,j}^{(s)} \rangle \\
&\leq -\frac{\eta^2 C_6 (\sigma_p^2 d+\sigma_p\tau\sqrt{2\log(4NM/\delta)}+\tau^2) (\|\boldsymbol{\mu}\|_2-\tau)^2\|\boldsymbol{\mu}\|_2^2 \|\boldsymbol{w}_O\|_2^2 \sigma_h^2 d_h (s - T_1)}{N} \cdot \frac{1}{\exp(\Lambda_{n,+,j}^{(s)})},
\end{aligned}
\end{equation}
\begin{equation}
\begin{aligned}
&\langle\boldsymbol{q}_-^{(s+1)},\boldsymbol{k}_{n,j}^{(s+1)} \rangle - \langle\boldsymbol{q}_-^{(s)},\boldsymbol{k}_{n,j}^{(s)} \rangle \\
&\leq -\frac{\eta^2 C_6 (\sigma_p^2 d+\sigma_p\tau\sqrt{2\log(4NM/\delta)}+\tau^2) (\|\boldsymbol{\mu}\|_2-\tau)^2\|\boldsymbol{\mu}\|_2^2 \|\boldsymbol{w}_O\|_2^2 \sigma_h^2 d_h (s - T_1)}{N} \cdot \frac{1}{\exp(\Lambda_{n,-,j}^{(s)})}
\end{aligned} 
\end{equation}
\begin{align}
\langle &\boldsymbol{q}_{n,i}^{(s+1)},\boldsymbol{k}_+^{(s+1)} \rangle - \langle\boldsymbol{q}_{n,i}^{(s)},\boldsymbol{k}_+^{(s)} \rangle \nonumber\\&\geq \frac{\eta^2 C_6 (\sigma_p^2 d+\sigma_p\tau\sqrt{2\log(4NM/\delta)}+\tau^2) (\|\boldsymbol{\mu}\|_2-\tau)^2\|\boldsymbol{\mu}\|_2^2 \|\boldsymbol{w}_O\|_2^2 \sigma_h^2 d_h (s - T_1)}{N} \cdot \frac{1}{\exp(\Lambda_{n,i,+,j})}  \end{align}
\begin{align}
&\langle\boldsymbol{q}_{n,i}^{(s+1)},\boldsymbol{k}_-^{(s+1)} \rangle - \langle\boldsymbol{q}_{n,i}^{(s)},\boldsymbol{k}_-^{(s)} \rangle \nonumber\\&\geq \frac{\eta^2 C_6 (\sigma_p^2 d+\sigma_p\tau\sqrt{2\log(4NM/\delta)}+\tau^2) (\|\boldsymbol{\mu}\|_2-\tau)^2\|\boldsymbol{\mu}\|_2^2 \|\boldsymbol{w}_O\|_2^2 \sigma_h^2 d_h (s - T_1)}{N} \cdot \frac{1}{\exp(\Lambda_{n,i,-,j})}     \end{align}
\begin{align}
&\langle\boldsymbol{q}_{n,i}^{(s+1)},\boldsymbol{k}_{n,j}^{(s+1)} \rangle - \langle\boldsymbol{q}_{n,i}^{(s)},\boldsymbol{k}_{n,j}^{(s)} \rangle \nonumber\\&\leq -\frac{\eta^2 C_6 (\sigma_p^2 d+\sigma_p\tau\sqrt{2\log(4NM/\delta)}+\tau^2)^2  \|\boldsymbol{\mu}\|_2^2 \|\boldsymbol{w}_O\|_2^2 \sigma_h^2 d_h (s - T_1)}{N} \cdot \frac{1}{\exp(\Lambda_{n,i,\pm,j})}
\end{align}
for $s \in [T_1, t]$. The seven equations above show that $\langle\boldsymbol q_\pm^{(s)},\boldsymbol k_\pm^{(s)} \rangle$, $\langle\boldsymbol{q}_{n,i}^{(s)},\boldsymbol k_\pm^{(s)} \rangle$ are monotonically increasing and $\langle\boldsymbol q_\pm^{(s)},\boldsymbol{k}_{n,j}^{(s)} \rangle$, $\langle\boldsymbol{q}_{n,i}^{(s)},\boldsymbol{k}_{n,j}^{(s)} \rangle$ are monotonically decreasing. Next, we provide the logarithmic increasing lower bounds of $\Lambda_{n,\pm,j}^{(s+1)}$ and $\Lambda_{n,\pm,j}^{(s+1)}$.

We have
\begin{equation}
\begin{aligned}
&\Lambda_{n,+,j}^{(s+1)} - \Lambda_{n,+,j}^{(s)}= (\langle\boldsymbol{q}_+^{(s+1)},\boldsymbol{k}_+^{(s+1)} \rangle - \langle\boldsymbol{q}_+^{(s)},\boldsymbol{k}_+^{(s)} \rangle) - (\langle\boldsymbol{q}_{n,j}^{(s+1)},\boldsymbol{k}_{n,j}^{(s+1)} \rangle - \langle\boldsymbol{q}_{n,j}^{(s)},\boldsymbol{k}_{n,j}^{(s)} \rangle) \\
&\geq \frac{\eta^2 C_6 (\|\boldsymbol{\mu}\|_2-\tau)^4\|\boldsymbol{\mu}\|_2^2 \|\boldsymbol{w}_O\|_2^2 \sigma_h^2 d_h (s - T_1)}{N} \cdot \frac{1}{\exp(\Lambda_{n,+,j}^{(s)})} \\
&+ \frac{\eta^2 C_6 (\sigma_p^2 d+\sigma_p\tau\sqrt{2\log(4NM/\delta)}+\tau^2) (\|\boldsymbol{\mu}\|_2-\tau)^2\|\boldsymbol{\mu}\|_2^2 \|\boldsymbol{w}_O\|_2^2 \sigma_h^2 d_h}{N} (s - T_1)(s - T_1 + 1) \cdot \frac{1}{\exp(\Lambda_{n,+,j}^{(s)})}\\
&\geq \frac{\eta^2 C_7 \max \{ \|\boldsymbol{\mu}\|_2^2, \sigma_p^2 d \} (\|\boldsymbol{\mu}\|_2-\tau)^2\|\boldsymbol{\mu}\|_2^2 \|\boldsymbol{w}_O\|_2^2 \sigma_h^2 d_h (s - T_1)}{N} \cdot \frac{1}{\exp(\Lambda_{n,+,j}^{(s)})}\\
&\geq \frac{\eta^2 C_7 (\|\boldsymbol{\mu}\|_2-\tau)^2\|\boldsymbol{\mu}\|_2^2 \|\boldsymbol{w}_O\|_2^2 d_h^{\frac{1}{2}} (s - T_1)^2}{N (\log(6N^2 M^2 / \delta))^2} \cdot \frac{1}{\exp(\Lambda_{n,+,j}^{(s)})}
\end{aligned}
\end{equation}

where the last inequality is by $\sigma_h^2 \geq \min\{\|\boldsymbol{\mu}\|_2^{-2}, (\sigma_p^2 d)^{-1}\} d_h^{-\frac{1}{2}} (\log(6N^2M^2 / \delta))^{-2}$. Multiply both sides simultaneously by $\exp(\Lambda_{n,+,j}^{(s)})$ and get
\begin{equation}
\exp(\Lambda_{n,+,j}^{(s)}) (\Lambda_{n,+,j}^{(s+1)} - \Lambda_{n,+,j}^{(s)}) \geq \frac{\eta^2 C_7 (\|\boldsymbol{\mu}\|_2-\tau)^2\|\boldsymbol{\mu}\|_2^2 \|\boldsymbol{w}_O\|_2^2 d_h^{\frac{1}{2}} (s - T_1)}{N (\log(6N^2M^2 / \delta))^2} \cdot \frac{1}{\exp(\Lambda_{n,+,j}^{(s)})}.
\end{equation} 

Taking a summation from $T_1$ to $t$ and get
\begin{align}
\sum_{s=T_1}^{t} \exp(\Lambda_{n,+,j}^{(s)}) (\Lambda_{n,+,j}^{(s+1)} - \Lambda_{n,+,j}^{(s)}) \geq \sum_{s=T_1}^{t} \frac{\eta^2 C_7 (\|\boldsymbol{\mu}\|_2-\tau)^2\|\boldsymbol{\mu}\|_2^2 \|\boldsymbol{w}_O\|_2^2 d_h^{\frac{1}{2}} (s - T_1)}{N (\log(6N^2M^2 / \delta))^2} \\
\geq \frac{\eta^2 C_8 (\|\boldsymbol{\mu}\|_2-\tau)^2\|\boldsymbol{\mu}\|_2^2 \|\boldsymbol{w}_O\|_2^2 d_h^{\frac{1}{2}}}{N (\log(6N^2M^2 / \delta))^2} \cdot (t - T_1)(t - T_1 + 1).
\end{align} 

By the property that $\Lambda_{n,+,j}^{(t)}$ is monotonically increasing, we have
\begin{equation}
\int_{\Lambda_{n,+,j}^{(T_1)}}^{T_1} \exp(x) dx \geq \sum_{s=T_1}^{t} \exp(\Lambda_{n,+,j}^{(s)} )(\Lambda_{n,+,j}^{(s+1)} - \Lambda_{n,+,j}^{(s)}) \geq \frac{\eta^2 C_8 (\|\boldsymbol{\mu}\|_2-\tau)^2\|\boldsymbol{\mu}\|_2^2 \|\boldsymbol{w}_O\|_2^2 d_h^{\frac{1}{2}}}{N (\log(6N^2M^2 / \delta))^2} \cdot (t - T_1)(t - T_1 + 1).
\end{equation} 

By
$\int_{\Lambda_{n,+,j}^{(T_1)}}^{\Lambda_{n,+,j}^{(t+1)}} \exp(x) dx = \exp(\Lambda_{n,+,j}^{(t+1)}) - \exp(\Lambda_{n,+,j}^{(T_1)})$ we get
\[\Lambda_{n,+,j}^{(t+1)} \geq \log \left( \exp(\Lambda_{n,+,j}^{(T_1)}) + \frac{\eta^2 C_8 (\|\boldsymbol{\mu}\|_2-\tau)^2\|\boldsymbol{\mu}\|_2^2 \|\boldsymbol{w}_O\|_2^2 d_h^{\frac{1}{2}}}{N (\log(6N^2M^2/\delta))^2} \cdot (t - T_1)(t - T_1 + 1) \right).\]

Similarly, we have
\[
\Lambda_{n,-,j}^{(t+1)} \geq \log \left( \exp(\Lambda_{n,-,j}^{(T_1)}) + \frac{\eta^2 C_8 (\|\boldsymbol{\mu}\|_2-\tau)^2\|\boldsymbol{\mu}\|_2^2 \|\boldsymbol{w}_O\|_2^2 d_h^{\frac{1}{2}}}{N (\log(6N^2M^2/\delta))^2} \cdot (t - T_1)(t - T_1 + 1) \right).
\]

Thus, we have
\begin{equation}
\begin{aligned}
&\Lambda_{n,i,+,j}^{(s+1)} - \Lambda_{n,i,+,j}^{(s)} = \langle (\boldsymbol{q}_{n,i}^{(s+1)},\boldsymbol{k}_{+}^{(s+1)}) - (\boldsymbol{q}_{n,i}^{(s)},\boldsymbol{k}_{+}^{(s)}) \rangle - \langle (\boldsymbol{q}_{n,i}^{(s+1)},\boldsymbol{k}_{n,j}^{(s+1)}) - (\boldsymbol{q}_{n,i}^{(s)},\boldsymbol{k}_{n,j}^{(s)}) \rangle \\
&\geq \frac{\eta^2 C_6 (\sigma_p^2 d+\sigma_p\tau\sqrt{2\log(4NM/\delta)}+\tau^2) (\|\boldsymbol{\mu}\|_2-\tau)^2\|\boldsymbol{\mu}\|_2^2 \|\boldsymbol{w}_O\|_2^2 \sigma_h^2 d_h (s - T_1)}{N} \cdot \frac{1}{\exp(\Lambda_{n,i,+,j}^{(s)})} \\
&+ \frac{\eta^2 C_6 (\sigma_p^2 d+\sigma_p\tau\sqrt{2\log(4NM/\delta)}+\tau^2)^2 \|\boldsymbol{\mu}\|_2^2 \|\boldsymbol{w}_O\|_2^2 \sigma_h^2 d_h (s - T_1)}{N} \cdot \frac{1}{\exp(\Lambda_{n,i,+,j}^{(s)})} \\
&\geq \frac{\eta^2 C_7 \max\{\|\boldsymbol{\mu}\|_2^2, \sigma_p^2 d\}  (\sigma_p^2 d+\sigma_p\tau\sqrt{2\log(4NM/\delta)}+\tau^2)\|\boldsymbol{\mu}\|_2^2 \|\boldsymbol{w}_O\|_2^2 \sigma_h^2 d_h (s - T_1)}{N} \cdot \frac{1}{\exp(\Lambda_{n,i,+,j}^{(s)})} \\
&\geq \frac{\eta^2 C_7 (\sigma_p^2 d+\sigma_p\tau\sqrt{2\log(4NM/\delta)}+\tau^2) \|\boldsymbol{\mu}\|_2^2 \|\boldsymbol{w}_O\|_2^2 d_h^{\frac{1}{2}} (s - T_1)}{N (\log(6N^2M^2/\delta))^2} \cdot \frac{1}{\exp(\Lambda_{n,i,+,j}^{(s)})}.
\end{aligned}
\end{equation}
Then using the similar method as for \(\Lambda_{n,i,+,j}^{(t)}\), we get
\[
\Lambda_{n,i,+,j}^{(t+1)} \geq \log \left( \exp(\Lambda_{n,i,+,j}^{(T_1)}) + \frac{\eta^2 C_8 (\sigma_p^2 d+\sigma_p\tau\sqrt{2\log(4NM/\delta)}+\tau^2) \|\boldsymbol{\mu}\|_2^2 \|\boldsymbol{w}_O\|_2^2 d_h^{\frac{1}{2}}}{N (\log(6N^2M^2/\delta))^2} \cdot (t - T_1)(t - T_1 + 1) \right),
\]
which complete the proof. The proof for Claim 3 is in Section~\ref{sec:f8}.

\subsubsection{Proof of Claim 4}
By the results of ~\ref{sec:f6}, we have
\begin{align*}
\langle\boldsymbol{q}_+^{(s+1)},\boldsymbol{k}_+^{(s+1)} \rangle &- \langle\boldsymbol{q}_+^{(s)},\boldsymbol{k}_+^{(s)} \rangle \leq \frac{\eta C_{10} \|\boldsymbol{\mu}\|_2^2(\|\boldsymbol{\mu}\|_2+\tau)^2 \sigma_h^2 d_h}{\exp(\langle\boldsymbol{q}_+^{(s)},\boldsymbol{k}_+^{(s)} \rangle)}
\end{align*}
for $s \in [T_1, t]$. Further we have
\begin{equation}\label{eq:46}
\begin{aligned}
\exp(\langle\boldsymbol{q}_+^{(s+1)},\boldsymbol{k}_+^{(s+1)} \rangle) &\leq \exp\Big(\langle\boldsymbol{q}_+^{(s)},\boldsymbol{k}_+^{(s)} \rangle +\frac{\eta C_{10} \|\boldsymbol{\mu}\|_2^2(\|\boldsymbol{\mu}\|_2+\tau)^2 \sigma_h^2 d_h}{\exp(\langle\boldsymbol{q}_+^{(s)},\boldsymbol{k}_+^{(s)} \rangle)}
\\&
= \exp(\langle\boldsymbol{q}_+^{(s)},\boldsymbol{k}_+^{(s)} \rangle) \cdot \exp\left(\frac{\eta C_{10} \|\boldsymbol{\mu}\|_2^2(\|\boldsymbol{\mu}\|_2+\tau)^2 \sigma_h^2 d_h}{\exp(\langle\boldsymbol{q}_+^{(s)},\boldsymbol{k}_+^{(s)} \rangle)}\right)
\\&
\leq C_{11} \exp(\langle\boldsymbol{q}_+^{(s)},\boldsymbol{k}_+^{(s)} \rangle).
\end{aligned}
\end{equation}

For the last inequality, by $\eta \leq \widetilde{O}(\min\{\|\boldsymbol{\mu}\|_2^{-2}, (\sigma_p^2 d)^{-1}\} \cdot d_h^{-\frac{1}{2}})$, $\sigma_h^2 \leq \min\{\|\boldsymbol{\mu}\|_2^{-2}, (\sigma_p^2 d)^{-1}\} d_h^{-\frac{1}{2}} (\log(6N^2M^2/\delta))^{-\frac{3}{2}}$, $\langle\boldsymbol{q}_+^{(T_1)},\boldsymbol{k}_+^{(T_1)} \rangle = o(1)$ and the monotonicity of $\langle\boldsymbol{q}_+^{(s)},\boldsymbol{k}_+^{(s)} \rangle$ for $s \in [T_1, t]$, we have $\exp\left(\frac{\eta C_{10} \|\boldsymbol{\mu}\|_2^2(\|\boldsymbol{\mu}\|_2+\tau)^2 \sigma_h^2 d_h}{\exp(\langle\boldsymbol{q}_+^{(s)},\boldsymbol{k}_+^{(s)} \rangle)}\right) \leq \exp(o(1)) \leq C_{11}$. Multiplying both sides by $\left(\langle\boldsymbol{q}_+^{(s+1)},\boldsymbol{k}_+^{(s+1)} \rangle - \langle\boldsymbol{q}_+^{(s)},\boldsymbol{k}_+^{(s)} \rangle\right)$ simultaneously gives

\begin{equation}
\begin{aligned}
   &\exp(\langle\boldsymbol{q}_+^{(s+1)},\boldsymbol{k}_+^{(s+1)} \rangle) 
   \left(\langle\boldsymbol{q}_+^{(s+1)},\boldsymbol{k}_+^{(s+1)} \rangle - \langle\boldsymbol{q}_+^{(s)},\boldsymbol{k}_+^{(s)} \rangle\right) \\
   &\quad \leq C_{11} \exp\!\left(\langle\boldsymbol{q}_+^{(s)},\boldsymbol{k}_+^{(s)} \rangle\right) 
   \cdot \left(\langle\boldsymbol{q}_+^{(s+1)},\boldsymbol{k}_+^{(s+1)} \rangle - \langle\boldsymbol{q}_+^{(s)},\boldsymbol{k}_+^{(s)} \rangle\right) \\
   &\quad \leq \eta C_{12} \|\boldsymbol{\mu}\|_2^2(\|\boldsymbol{\mu}\|_2+\tau)^2 \sigma_h^2 d_h,
\end{aligned}
\end{equation}

where the last inequality is by plugging (\ref{eq:46}). Taking a summation, we obtain
\begin{equation}
\begin{aligned}
   &\int_{\langle\boldsymbol{q}_+^{(T_1)},\boldsymbol{k}_+^{(T_1)} \rangle}^{\langle\boldsymbol{q}_+^{(t+1)},\boldsymbol{k}_+^{(t+1)} \rangle} \exp(x) \, dx \\
   &\quad \leq \sum_{s=T_1}^{t} \exp(\langle\boldsymbol{q}_+^{(s+1)},\boldsymbol{k}_+^{(s+1)} \rangle) 
   \left(\langle\boldsymbol{q}_+^{(s+1)},\boldsymbol{k}_+^{(s+1)} \rangle - \langle\boldsymbol{q}_+^{(s)},\boldsymbol{k}_+^{(s)} \rangle\right) \\
   &\quad \leq \sum_{s=T_1}^{t} \eta C_{12} \|\boldsymbol{\mu}\|_2^2(\|\boldsymbol{\mu}\|_2+\tau)^2 \sigma_h^2 d_h \\
   &\quad \leq T_2 \cdot \eta C_{12} \|\boldsymbol{\mu}\|_2^2(\|\boldsymbol{\mu}\|_2+\tau)^2 \sigma_h^2 d_h \\
   &\quad \leq \frac{d_h^{1/2}}{\log^2(6N^2 M^2 / \delta)}.
\end{aligned}
\end{equation}
where the first inequality is due to $\langle\boldsymbol{q}_+^{(s)},\boldsymbol{k}_+^{(s)} \rangle$ is monotone increasing, the last inequality is by $T_2 = \Theta(\eta^{-1} \|\boldsymbol{\mu}\|_2^{-2} \|\boldsymbol{w}_O\|_2^{-2} \log(6N^2M^2/\delta)^{-1})$ and $\sigma_h^2 \leq \min\{\|\boldsymbol{\mu}\|_2^{-2}, (\sigma_p^2 d)^{-1}\} d_h^{-\frac{1}{2}} (\log(6N^2M^2/\delta))^{-\frac{3}{2}}$. By
\[
\int_{\langle\boldsymbol{q}_+^{(T_1)},\boldsymbol{k}_+^{(T_1)} \rangle}^{\langle\boldsymbol{q}_+^{(t+1)},\boldsymbol{k}_+^{(t+1)} \rangle} \exp(x) dx = \exp(\langle\boldsymbol{q}_+^{(t+1)},\boldsymbol{k}_+^{(t+1)} \rangle) - \exp(\langle\boldsymbol{q}_+^{(T_1)},\boldsymbol{k}_+^{(T_1)} \rangle),
\]
we have
\begin{equation}
\langle\boldsymbol{q}_+^{(t+1)},\boldsymbol{k}_+^{(t+1)} \rangle \leq \log\left(\langle\boldsymbol{q}_+^{(T_1)},\boldsymbol{k}_+^{(T_1)} \rangle + \frac{d_h^{\frac{1}{2}}}{\log(6N^2M^2/\delta)}\right) \leq \log\left(d_h^{\frac{1}{2}}\right),
\end{equation}
By the results of \ref{sec:f6}, we also have
\begin{equation}
\langle\boldsymbol{q}_-^{(s+1)},\boldsymbol{k}_-^{(s+1)} \rangle - \langle\boldsymbol{q}_-^{(s)},\boldsymbol{k}_-^{(s)} \rangle \leq \frac{\eta C_{10} \|\boldsymbol{\mu}\|_2^2(\|\boldsymbol{\mu}\|_2+\tau)^2 \sigma_h^2 d_h}{\exp(\langle\boldsymbol{q}_-^{(s)},\boldsymbol{k}_-^{(s)} \rangle)}
\end{equation}
\begin{equation}
\langle\boldsymbol{q}_{\pm}^{(s+1)},\boldsymbol{k}_{n,j}^{(s+1)} \rangle - \langle\boldsymbol{q}_{\pm}^{(s)},\boldsymbol{k}_{n,j}^{(s)} \rangle \geq -\frac{\eta C_{10} \sigma_p^2 d (\|\boldsymbol{\mu}\|_2+\tau)^2 \sigma_h^2 d_h}{N} \cdot \exp(\langle\boldsymbol{q}_{\pm}^{(s)},\boldsymbol{k}_{n,j}^{(s)} \rangle).
\end{equation}
\begin{equation}
\langle\boldsymbol{q}_{n,i}^{(s+1)},\boldsymbol{k}_{\pm}^{(s+1)} \rangle - \langle\boldsymbol{q}_{n,i}^{(s)},\boldsymbol{k}_{\pm}^{(s)} \rangle \leq \frac{\eta C_{10} \sigma_p^2 d (\|\boldsymbol{\mu}\|_2+\tau)^2 \sigma_h^2 d_h}{N \exp(\langle\boldsymbol{q}_{n,i}^{(s)},\boldsymbol{k}_{\pm}^{(s)} \rangle)}.
\end{equation}
\begin{equation}
\langle\boldsymbol{q}_{n,i}^{(s+1)},\boldsymbol{k}_{n,j}^{(s+1)} \rangle - \langle\boldsymbol{q}_{n,i}^{(s)},\boldsymbol{k}_{n,j}^{(s)} \rangle \geq -\frac{\eta C_{10} \sigma_p^2 d(\sigma_p^2 d+\sigma_p\tau\sqrt{2\log(4NM/\delta)}+\tau^2) \sigma_h^2 d_h}{N} \cdot \exp(\langle\boldsymbol{q}_{n,i}^{(s)},\boldsymbol{k}_{n,j}^{(s)} \rangle).
\end{equation}
Then using the similar method as for $\langle\boldsymbol{q}_+^{(t+1)},\boldsymbol{k}_+^{(t+1)} \rangle$, we get
\begin{equation}
\langle\boldsymbol{q}_-^{(t+1)},\boldsymbol{k}_-^{(t+1)} \rangle \leq \log\left(d_h^{\frac{1}{2}}\right),
\end{equation}
\begin{equation}
\langle\boldsymbol{q}_{\pm}^{(t+1)},\boldsymbol{k}_{n,j}^{(t+1)} \rangle \geq -\log\left(d_h^{\frac{1}{2}}\right),
\end{equation}
\begin{equation}
\langle\boldsymbol{q}_{n,i}^{(t+1)},\boldsymbol{k}_{\pm}^{(t+1)} \rangle \leq \log\left(d_h^{\frac{1}{2}}\right),
\end{equation}

\begin{equation}
\langle\boldsymbol{q}_{n,i}^{(t+1)},\boldsymbol{k}_{n,j}^{(t+1)} \rangle \geq -\log\left(d_h^{\frac{1}{2}}\right).
\end{equation}

Next we provide the upper bound for $|\langle\boldsymbol{q}_{\pm}^{(t+1)},\boldsymbol{k}_{\pm}^{(t+1)} \rangle|, |\langle\boldsymbol{q}_{n,i}^{(t+1)},\boldsymbol{k}_{n',j}^{(t+1)} \rangle|$. By the results of~\ref{sec:f7}, we have
\begin{equation}
\sum_{s=T_1}^{t} |\alpha_{+,+}^{(s)}|, \sum_{s=T_1}^{t} |\alpha_{-,-}^{(s)}|, \sum_{s=T_1}^{t} |\beta_{+,+}^{(s)}|, \sum_{s=T_1}^{t} |\beta_{-,-}^{(s)}|, \sum_{s=T_1}^{t} |\beta_{n,i,+}^{(s)}|, \sum_{s=T_1}^{t} |\beta_{n,i,-}^{(s)}| = O\left(N^{\frac{1}{2}} d_h^{-\frac{1}{4}}\right),
\end{equation}
for $i \in [M]\backslash\{1\}, n \in S_{\pm}$.
\begin{equation}
\sum_{s=T_1}^{t} |\alpha_{n,+,i}^{(s)}|, \sum_{s=T_1}^{t} |\alpha_{n,-,i}^{(s)}| = O\left(N^{-\frac{1}{2}} d_h^{-\frac{1}{4}}\right),
\end{equation}
for $i \in [M]\backslash\{1\}, n \in S_{\pm}$.
\begin{equation}
\sum_{s=T_1}^{t} |\beta_{n,+,i}^{(s)}|, \sum_{s=T_1}^{t} |\beta_{n,-,i}^{(s)}| = O\left(\text{SNR} \cdot N^{-\frac{1}{2}} d_h^{-\frac{1}{4}}\right)
\end{equation}
for $i \in [M]\backslash\{1\}, n \in S_{\pm}$.
\begin{equation}
\sum_{s=T_1}^{t} |\alpha_{n,i,n',j}^{(s)}|, \sum_{s=T_1}^{t} |\beta_{n,j,n',i}^{(s)}| = O\left(d^{-\frac{1}{2}} d_h^{-\frac{1}{4}} \log(6N^2M^2/\delta)\right)
\end{equation}
for $i, j \in [M]\backslash\{1\}, n, n' \in [N], n \neq n'$. Plugging these into the update rule of $\langle\boldsymbol{q}_{\pm}^{(t)},\boldsymbol{k}_{\pm}^{(t)} \rangle, \langle\boldsymbol{q}_{n,i}^{(t)},\boldsymbol{k}_{n,j}^{(t)} \rangle$ and assume that propositions $\mathcal{C}(T_1), \ldots, \mathcal{C}(t)$ hold, we have
\begin{equation}
\begin{aligned}
    |\langle\boldsymbol{q}_+^{(t+1)},\boldsymbol{k}_-^{(t+1)} \rangle| &\leq |\langle\boldsymbol{q}_+^{(T_1)},\boldsymbol{k}_-^{(T_1)} \rangle| + \sum_{s=T_1}^{t} |\langle\boldsymbol{q}_+^{(s+1)},\boldsymbol{k}_-^{(s+1)} \rangle - \langle\boldsymbol{q}_+^{(s)},\boldsymbol{k}_-^{(s)} \rangle| \\
    &\leq |\langle\boldsymbol{q}_+^{(T_1)},\boldsymbol{k}_-^{(T_1)} \rangle| \\
    &\quad + \sum_{s=T_1}^{t} \left| \alpha_{+,+}^{(s)} \langle\boldsymbol{k}_+^{(s)},\boldsymbol{k}_-^{(s)} \rangle + \sum_{n \in S_+} \sum_{i=2}^{M} \alpha_{n,+,i}^{(s)} \langle\boldsymbol{k}_{n,i}^{(s)},\boldsymbol{k}_-^{(s)} \rangle \right. \\
    &\quad \left. + \beta_{-,-}^{(s)} \langle\boldsymbol{q}_+^{(s)},\boldsymbol{q}_-^{(s)} \rangle + \sum_{n \in S_-} \sum_{i=2}^{M} \beta_{n,-,i}^{(s)} \langle\boldsymbol{q}_{n,i}^{(s)},\boldsymbol{q}_+^{(s)} \rangle \right. \\
    &\quad \left. + \left( \alpha_{+,+}^{(s)}\boldsymbol{k}_+^{(s)} + \sum_{n \in S_+} \sum_{i=2}^{M} \alpha_{n,+,i}^{(s)}\boldsymbol{k}_{n,i}^{(s)} \right) \right. \\
    &\quad \left. \cdot \left( \beta_{-,-}^{(s) \top}\boldsymbol{q}_-^{(s) \top} + \sum_{n \in S_-} \sum_{i=2}^{M} \beta_{n,-,i}^{(s) \top}\boldsymbol{q}_{n,i}^{(s) \top} \right) \right| \\
    &\leq |\langle\boldsymbol{q}_+^{(T_1)},\boldsymbol{k}_-^{(T_1)} \rangle| \\
    &\quad + \sum_{s=T_1}^{t} |\alpha_{+,+}^{(s)}| |\langle\boldsymbol{k}_+^{(s)},\boldsymbol{k}_-^{(s)} \rangle| + \sum_{n \in S_+} \sum_{i=2}^{M} \sum_{s=T_1}^{t} |\alpha_{n,+,i}^{(s)}| |\langle\boldsymbol{k}_{n,i}^{(s)},\boldsymbol{k}_-^{(s)} \rangle| \\
    &\quad + \sum_{s=T_1}^{t} |\beta_{-,-}^{(s)}| |\langle\boldsymbol{q}_+^{(s)},\boldsymbol{q}_-^{(s)} \rangle| + \sum_{n \in S_-} \sum_{i=2}^{M} \sum_{s=T_1}^{t} |\beta_{n,-,i}^{(s)}| |\langle\boldsymbol{q}_{n,i}^{(s)},\boldsymbol{q}_+^{(s)} \rangle| \\
    &\quad + \{ \text{lower order term} \} \\
    &= |\langle\boldsymbol{q}_+^{(T_1)},\boldsymbol{k}_-^{(T_1)} \rangle| \\
    &\quad + O\left(N^{\frac{1}{2}} d_h^{-\frac{1}{4}}\right) \cdot o(1) + N \cdot M \cdot O\left(N^{-\frac{1}{2}} d_h^{-\frac{1}{4}}\right) \cdot o(1) \\
    &\quad + O\left(N^{\frac{1}{2}} d_h^{-\frac{1}{4}}\right) \cdot o(1) + N \cdot M \cdot O\left(\text{SNR} \cdot N^{-\frac{1}{2}} d_h^{-\frac{1}{4}}\right) \cdot o(1) \\
    &= |\langle\boldsymbol{q}_+^{(T_1)},\boldsymbol{k}_-^{(T_1)} \rangle| + o\left(N^{\frac{1}{2}} d_h^{-\frac{1}{4}}\right) + o\left(\text{SNR} \cdot N^{\frac{1}{2}} d_h^{-\frac{1}{4}}\right) \\
    &= o(1),
\end{aligned}
\end{equation}
where the first inequality is by triangle inequality, the last equality is by $|\langle\boldsymbol{q}_+^{(T_1)},\boldsymbol{k}_-^{(T_1)} \rangle| = o(1)$ and $d_h = \widetilde{\Omega}\left(\max\{\text{SNR}^4, \text{SNR}^{-4}\} N^2 \epsilon^{-2}\right)$. Similarly we have $|\langle\boldsymbol{q}_-^{(t+1)},\boldsymbol{k}_+^{(t+1)} \rangle| = o(1)$.
\begin{equation}
\begin{aligned}
    |\langle\boldsymbol{q}_{n,i}^{(t+1)},\boldsymbol{k}_{\overline{n},j}^{(t+1)} \rangle| &\leq |\langle\boldsymbol{q}_{n,i}^{(T_1)},\boldsymbol{k}_{\overline{n},j}^{(T_1)} \rangle| + \sum_{s=T_1}^{t} |\langle\boldsymbol{q}_{n,i}^{(s+1)},\boldsymbol{k}_{\overline{n},j}^{(s+1)} \rangle - \langle\boldsymbol{q}_{n,i}^{(s)},\boldsymbol{k}_{\overline{n},j}^{(s)} \rangle| \\
    &\leq |\langle\boldsymbol{q}_{n,i}^{(T_1)},\boldsymbol{k}_{\overline{n},j}^{(T_1)} \rangle| \\
    &\quad + \sum_{s=T_1}^{t} \left| \alpha_{n,i,+}^{(s)} \langle\boldsymbol{k}_+^{(s)},\boldsymbol{k}_{\overline{n},j}^{(s)} \rangle + \alpha_{n,i,-}^{(s)} \langle\boldsymbol{k}_-^{(s)},\boldsymbol{k}_{\overline{n},j}^{(s)} \rangle + \sum_{n'=1}^{N} \sum_{l=2}^{M} \alpha_{n,i,n',l}^{(s)} \langle\boldsymbol{k}_{n',l}^{(s)},\boldsymbol{k}_{\overline{n},j}^{(s)} \rangle \right. \\
    &\quad \left. + \beta_{\overline{n},j,+}^{(s)} \langle\boldsymbol{q}_+^{(s)},\boldsymbol{q}_{n,i}^{(s)} \rangle + \beta_{\overline{n},j,-}^{(s)} \langle\boldsymbol{q}_-^{(s)},\boldsymbol{q}_{n,i}^{(s)} \rangle + \sum_{n'=1}^{N} \sum_{l=2}^{M} \beta_{\overline{n},j,n',l}^{(s)} \langle\boldsymbol{q}_{n',l}^{(s)},\boldsymbol{q}_{n,i}^{(s)} \rangle \right. \\
    &\quad \left. + \left( \alpha_{n,i,+}^{(s)}\boldsymbol{k}_+^{(s)} + \alpha_{n,i,-}^{(s)}\boldsymbol{k}_-^{(s)} + \sum_{n'=1}^{N} \sum_{l=2}^{M} \alpha_{n,i,n',l}^{(s)}\boldsymbol{k}_{n',l}^{(s)} \right) \right. \\
    &\quad \left. \cdot \left( \beta_{\overline{n},j,+}^{(s) \top}\boldsymbol{q}_+^{(s) \top} + \beta_{\overline{n},j,-}^{(s) \top}\boldsymbol{q}_-^{(s) \top} + \sum_{n'=1}^{N} \sum_{l=2}^{M} \beta_{\overline{n},j,n',l}^{(s) \top}\boldsymbol{q}_{n',l}^{(s) \top} \right) \right| \\
    &\leq |\langle\boldsymbol{q}_{n,i}^{(T_1)},\boldsymbol{k}_{\overline{n},j}^{(T_1)} \rangle| \\
    &\quad + \sum_{s=T_1}^{t} |\alpha_{n,i,+}^{(s)}| |\langle\boldsymbol{k}_+^{(s)},\boldsymbol{k}_{\overline{n},j}^{(s)} \rangle| + \sum_{s=T_1}^{t} |\alpha_{n,i,-}^{(s)}| |\langle\boldsymbol{k}_-^{(s)},\boldsymbol{k}_{\overline{n},j}^{(s)} \rangle| \\
    &\quad + \sum_{s=T_1}^{t} |\alpha_{n,i,\overline{n},j}^{(s)}| |\langle\boldsymbol{k}_{\overline{n},j}^{(s)},\boldsymbol{k}_{\overline{n},j}^{(s)} \rangle| + \sum_{l=2}^{M} \sum_{s=T_1}^{t} |\alpha_{n,i,n,l}^{(s)}| |\langle\boldsymbol{k}_{n,l}^{(s)},\boldsymbol{k}_{\overline{n},j}^{(s)} \rangle| \\
    &\quad + \sum_{n' \neq \overline{n} \land (l \neq j \lor n' \neq \overline{n})} \sum_{s=T_1}^{t} \sum_{n'=1}^{N} \sum_{l=2}^{M} |\alpha_{n,i,n',l}^{(s)}| |\langle\boldsymbol{k}_{n',l}^{(s)},\boldsymbol{k}_{\overline{n},j}^{(s)} \rangle| \\
    &\quad + \sum_{s=T_1}^{t} |\beta_{\overline{n},j,+}^{(s)}| |\langle\boldsymbol{q}_+^{(s)},\boldsymbol{q}_{n,i}^{(s)} \rangle| + \sum_{s=T_1}^{t} |\beta_{\overline{n},j,-}^{(s)}| |\langle\boldsymbol{q}_-^{(s)},\boldsymbol{q}_{n,i}^{(s)} \rangle| \\
    &\quad + \sum_{s=T_1}^{t} |\beta_{\overline{n},j,n,i}^{(s)}| |\langle\boldsymbol{q}_{n,i}^{(s)},\boldsymbol{q}_{n,i}^{(s)} \rangle| + \sum_{l=2}^{M} \sum_{s=T_1}^{t} |\beta_{\overline{n},j,n,l}^{(s)}| |\langle\boldsymbol{q}_{n,l}^{(s)},\boldsymbol{q}_{n,i}^{(s)} \rangle| \\
    &\quad + \sum_{n' \neq \overline{n} \land (l \neq i \lor n' \neq n)} \sum_{s=T_1}^{t} \sum_{n'=1}^{N} \sum_{l=2}^{M} |\beta_{\overline{n},j,n',l}^{(s)}| |\langle\boldsymbol{q}_{n',l}^{(s)},\boldsymbol{q}_{n,i}^{(s)} \rangle| \\
    &\quad + \{\text{lower order term}\} \\
    &= |\langle\boldsymbol{q}_{n,i}^{(T_1)},\boldsymbol{k}_{\overline{n},j}^{(T_1)} \rangle| \\
    &\quad + O(d_h^{-\frac{1}{4}}) \cdot o(1) + O(d^{-\frac{1}{2}} d_h^{-\frac{1}{4}} \log(6N^2M^2/\delta)) \cdot \Theta(\sigma_p^2 \sigma_h^2 d d_h) \\
    &\quad + M \cdot O(d_h^{-\frac{1}{4}}) \cdot o(1) + N \cdot M \cdot O(d^{-\frac{1}{2}} d_h^{-\frac{1}{4}} \log(6N^2M^2/\delta)) \cdot o(1) \\
    &\quad + O(N^{\frac{1}{2}} d_h^{-\frac{1}{4}}) \cdot o(1) \\
    &= |\langle\boldsymbol{q}_{n,i}^{(T_1)},\boldsymbol{k}_{\overline{n},j}^{(T_1)} \rangle| + o(d_h^{-\frac{1}{4}}) \\
    &\quad + O(d^{-\frac{1}{2}} d_h^{\frac{1}{4}}) + o(N d^{-\frac{1}{2}} d_h^{-\frac{1}{4}} \log(6N^2M^2/\delta)) \\
    &= o(1),
\end{aligned}
\end{equation}
where the first inequality is by triangle inequality, the second equality is by 
\(
\sigma_h^2 \leq \min\{\|\boldsymbol{\mu}\|_2^{-2}, (\sigma_p^2 d)^{-1}\} \cdot d_h^{-\frac{1}{2}} \cdot (\log(6N^2M^2/\delta))^{-\frac{3}{2}},
\)
the last equality is by $|\langle\boldsymbol{q}_{n,i}^{(T_1)},\boldsymbol{k}_{\overline{n},j}^{(T_1)} \rangle| = o(1)$, $d = \widetilde{\Omega}(\epsilon^{-2} N^2 d_h)$ and $d_h = \widetilde{\Omega}\left(\max\{\text{SNR}^4, \text{SNR}^{-4}\} N^2 \epsilon^{-2}\right)$.
\subsection{Stage iii}\label{sec:stage3}
In Stage III, the outputs of ViT grow up and the loss derivatives are no longer at $o(1)$. We will carefully compute the growth rate of $V_{\pm}$ and $V_{n,i}$ while keeping monitoring the monotonicity of $\langle q, k \rangle$. By substituting $t = T_2 = \Theta\left(\frac{1}{\eta (\|\boldsymbol{\mu}\|_2+\tau)^2 \|\boldsymbol{w}_O\|_2^2}\right)$ into propositions $\mathcal{B}(t)$, $\mathcal{C}(t)$, $\mathcal{D}(t)$, $\mathcal{E}(t)$ in Stage II, we have the following conditions at the beginning of stage III
\begin{align*}
    |V_+^{(T_2)}|, |V_-^{(T_2)}|&, |V_{n,i}^{(T_2)}| = o(1), \\
    V_+^{(T_2)} &\geq 3M \cdot |V_{n,i}^{(T_2)}|, \\
    V_-^{(T_2)} &\leq -3M \cdot |V_{n,i}^{(T_2)}|, \\
    \|\boldsymbol{q}_+^{(T_2)}\|_2^2, \|\boldsymbol{k}_+^{(T_2)}\|_2^2 &= \Theta(\|\boldsymbol{\mu}\|_2^2 \sigma_h^2 d_h), \\
    \|\boldsymbol{q}_{n,i}^{(T_2)}\|_2^2, \|\boldsymbol{k}_{n,i}^{(T_2)}\|_2^2 &= \Theta(\sigma_p^2 \sigma_h^2 d d_h),
\end{align*}
\[
|\langle\boldsymbol{q}_+^{(T_2)},\boldsymbol{q}_-^{(T_2)} \rangle|, |\langle\boldsymbol{q}_+^{(T_2)},\boldsymbol{q}_{n,i}^{(T_2)} \rangle|, |\langle\boldsymbol{q}_{n,i}^{(T_2)},\boldsymbol{q}_{n',j}^{(T_2)} \rangle| = o(1),
\]
\[
|\langle\boldsymbol{k}_+^{(T_2)},\boldsymbol{k}_-^{(T_2)} \rangle|, |\langle\boldsymbol{k}_+^{(T_2)},\boldsymbol{k}_{n,i}^{(T_2)} \rangle|, |\langle\boldsymbol{k}_{n,i}^{(T_2)},\boldsymbol{k}_{n',j}^{(T_2)} \rangle| = o(1),
\]
for $i, j \in [M]\backslash\{1\}, n, n' \in [N], i \neq j \text{ or } n \neq n'$.
\begin{equation*}
\Lambda_{n,\pm,j}^{(T_2)} \geq \log\left(\exp(\Lambda_{n,\pm,j}^{(T_1)}) + \Theta\left(\frac{d_h^{\frac{1}{2}}}{N(\log(6N^2M^2/\delta))^3}\right)\right) 
\end{equation*}
\begin{equation*}
\Lambda_{n,i,\pm,j}^{(T_2)} \geq \log\left(\exp(\Lambda_{n,i,\pm,j}^{(T_1)}) + \Theta\left(\frac{\sigma_p^2 d d_h^{\frac{1}{2}}}{N\|\boldsymbol{\mu}\|_2^2 (\log(6N^2M^2/\delta))^3}\right)\right)
\end{equation*}
\[
|\langle\boldsymbol{q}_+^{(T_2)},\boldsymbol{k}_+^{(T_2)} \rangle|, |\langle\boldsymbol{q}_+^{(T_2)},\boldsymbol{k}_{n,j}^{(T_2)} \rangle|, |\langle\boldsymbol{q}_{n,i}^{(T_2)},\boldsymbol{k}_+^{(T_2)} \rangle|, |\langle\boldsymbol{q}_{n,i}^{(T_2)},\boldsymbol{k}_{n',j}^{(T_2)} \rangle| \leq \log(d_h^{\frac{1}{2}})
\]
\[
|\langle\boldsymbol{q}_+^{(T_2)},\boldsymbol{k}_-^{(T_2)} \rangle|, |\langle\boldsymbol{q}_{n,i}^{(T_2)},\boldsymbol{k}_{\overline{n},j}^{(T_2)} \rangle| = o(1)
\]
for $i, j \in [M]\backslash\{1\}, n, \overline{n} \in [N], n \neq \overline{n}$.

Let $T_3 = \Theta\left(\frac{1}{\eta \epsilon (\|\boldsymbol{\mu}\|_2+\tau)^2 \|\boldsymbol{w}_O\|_2^2}\right)$. Next we prove the following four propositions $\mathcal{F}(t)$, $\mathcal{G}(t)$, $\mathcal{H}(t)$, $\mathcal{I}(t)$ by induction on $t$ for $t \in [T_2, T_3]$:
\begin{itemize}
    \item $\mathcal{F}(t)$:
    \begin{align*}
        &\qquad\qquad\qquad\qquad\qquad\qquad\qquad V_+^{(t)} \geq 3M \cdot |V_{n,i}^{(t)}|, \\
        &\qquad\qquad\qquad\qquad\qquad\qquad\qquad V_-^{(t)} \leq -3M \cdot |V_{n,i}^{(t)}|, \\
        &\qquad\qquad\qquad\qquad\qquad\qquad\qquad\quad |V_{n,i}^{(t)}| = o(1), \\
        &\log\left( \exp(V_+^{(T_2)})  + \eta C_{17} (\|\boldsymbol{\mu}\|_2-\tau)^2 \|\boldsymbol{w}_O\|_2^2 (t - T_2)\right) \leq V_+^{(t)}\leq 2 \log\left(O\left(\frac{1}{\epsilon}\right)\right), \\
        &-2 \log\left(O\left(\frac{1}{\epsilon}\right)\right) \leq V_-^{(t)} \leq -\log\left(\exp(-V_-^{(T_2)}) + \eta C_{17} (\|\boldsymbol{\mu}\|_2-\tau)^2 \|\boldsymbol{w}_O\|_2^2 (t - T_2)\right)
    \end{align*}
    for $i \in [M]\backslash\{1\}, n \in [N]$.
    \item $\mathcal{G}(t)$:
    \begin{align*}
        \|\boldsymbol q_\pm^{(t)}\|_2^2, \|\boldsymbol k_\pm^{(t)}\|_2^2 &= \Theta(\|\boldsymbol{\mu}\|_2^2 \sigma_h^2 d_h), \\
        \|\boldsymbol{q}_{n,i}^{(t)}\|_2^2, \|\boldsymbol{k}_{n,i}^{(t)}\|_2^2 &= \Theta\left(\sigma_p^2 \sigma_h^2 d d_h\right), \\
        |\langle\boldsymbol{q}_+^{(t)},\boldsymbol{q}_-^{(t)} \rangle|, |\langle\boldsymbol q_\pm^{(t)},\boldsymbol{q}_{n,i}^{(t)} \rangle|, &|\langle\boldsymbol{q}_{n,i}^{(t)},\boldsymbol{q}_{n',j}^{(t)} \rangle| = o(1), \\
        |\langle\boldsymbol{k}_+^{(t)},\boldsymbol{k}_-^{(t)} \rangle|, |\langle\boldsymbol k_\pm^{(t)},\boldsymbol{k}_{n,i}^{(t)} \rangle|, &|\langle\boldsymbol{k}_{n,i}^{(t)},\boldsymbol{k}_{n',j}^{(t)} \rangle| = o(1)
    \end{align*}
    for $i, j \in [M]\backslash\{1\}, n, n' \in [N], i \neq j \text{ or } n \neq n'$.
    \item $\mathcal{H}(t)$:
    \begin{align*}
        \langle\boldsymbol q_\pm^{(t+1)},\boldsymbol k_\pm^{(t+1)} \rangle &\geq \langle\boldsymbol q_\pm^{(t)},\boldsymbol k_\pm^{(t)} \rangle, \\
        \langle\boldsymbol{q}_{n,i}^{(t+1)},\boldsymbol k_\pm^{(t+1)} \rangle &\geq \langle\boldsymbol{q}_{n,i}^{(t)},\boldsymbol k_\pm^{(t)} \rangle, \\
        \langle\boldsymbol q_\pm^{(t+1)},\boldsymbol{k}_{n,j}^{(t+1)} \rangle &\leq \langle\boldsymbol q_\pm^{(t)},\boldsymbol{k}_{n,j}^{(t)} \rangle, \\
        \langle\boldsymbol{q}_{n,i}^{(t+1)},\boldsymbol{k}_{n,j}^{(t+1)} \rangle &\leq \langle\boldsymbol{q}_{n,i}^{(t)},\boldsymbol{k}_{n,j}^{(t)} \rangle
    \end{align*}
    for $i, j \in [M]\backslash\{1\}, n \in [N]$.
    \item $\mathcal{I}(t)$:
    \begin{align*}
        |\langle\boldsymbol q_\pm^{(t)},\boldsymbol k_\pm^{(t)} \rangle|, |\langle\boldsymbol q_\pm^{(t)},\boldsymbol{k}_{n,j}^{(t)} \rangle|, |\langle\boldsymbol{q}_{n,i}^{(t)},\boldsymbol k_\pm^{(t)} \rangle|, |\langle\boldsymbol{q}_{n,i}^{(t)},\boldsymbol{k}_{n,j}^{(t)} \rangle| &\leq \log(\epsilon^{-1} d_h^{\frac{1}{2}}), \\
        |\langle\boldsymbol q_\pm^{(t)},\boldsymbol k_\mp^{(t)} \rangle|, |\langle\boldsymbol{q}_{n,i}^{(t)},\boldsymbol{k}_{\overline{n},j}^{(t)} \rangle| &= o(1)
    \end{align*}
    for $i, j \in [M]\backslash\{1\}, n, n' \in [N], n \neq \overline{n}$.
\end{itemize}

By the results of Stage II, we know that $\mathcal{F}(T_2)$, $\mathcal{G}(T_2)$, $\mathcal{I}(T_2)$ are true. To prove that $\mathcal{F}(t)$, $\mathcal{G}(t)$, $\mathcal{H}(t)$ and $\mathcal{I}(t)$ are true in stage 3, we will prove the following claims holds for $t \in [T_2, T_3]$:

\begin{itemize}
    \item Claim 5. $\mathcal{H}(T_2), \ldots, \mathcal{H}(t-1),\mathcal{I}(T_2), \ldots, \mathcal{I}(t) \implies \mathcal{F}(t+1)$
    \item Claim 6. $\mathcal{F}(t), \mathcal{G}(t), \mathcal{H}(T_2), \ldots, \mathcal{H}(t-1), \mathcal{I}(T_2), \ldots, \mathcal{I}(t-1)\implies \mathcal{H}(t)$
    \item Claim 7. $\mathcal{F}(T_2), \ldots, \mathcal{F}(t), \mathcal{G}(t), \mathcal{H}(T_2), \ldots, \mathcal{H}(t-1), \mathcal{I}(T_2), \ldots, \mathcal{I}(t) \implies \mathcal{G}(t+1)$
    \item Claim 8. $\mathcal{F}(T_2), .., \mathcal{F}(t), \mathcal{G}(T_2), .., \mathcal{G}(t), \mathcal{H}(T_2), .., \mathcal{H}(t-1),\mathcal{I}(T_2), .., \mathcal{I}(t) \implies \mathcal{I}(t+1)$
\end{itemize}
\subsubsection{Proof of Claim 5}\label{sec:cl5}

The proofs for $V_+^{(t)} \geq 3M \cdot |V_{n,i}^{(t)}|$ and $V_-^{(t)} \leq -3M \cdot |V_{n,i}^{(t)}|$ are the same as for ~\ref{sec:cl1}. Based on $\mathcal{H}(T_2), \ldots, \mathcal{H}(t)$ where $\langle\boldsymbol q_\pm^{(s)},\boldsymbol k_\pm^{(s)} \rangle$ and $\langle\boldsymbol{q}_{n,i}^{(s)},\boldsymbol k_\pm^{(s)} \rangle$ are monotonically non-decreasing and $\max_j \langle\boldsymbol q_\pm^{(s)},\boldsymbol{k}_{n,j}^{(s)} \rangle$, $\max_j \langle\boldsymbol{q}_{n,i}^{(s)},\boldsymbol{k}_{n,j}^{(s)} \rangle$ are monotonically non-increasing for $s \in [T_2, t-1]$, we have
\begin{equation}\label{eq:62}
    \Lambda_{n,\pm,j}^{(s)} \geq \Lambda_{n,\pm,j}^{(T_2)} \geq \log\left(\exp(\Lambda_{n,\pm,j}^{(T_1)}) + \Theta\left(\frac{d_h^{\frac{1}{2}}}{N(\log(6N^2M^2/\delta))^3}\right)\right),
\end{equation}
\begin{equation}\label{eq:63}
    \Lambda_{n,i,\pm,j}^{(s)} \geq \Lambda_{n,i,\pm,j}^{(T_2)} \geq \log\left(\exp(\Lambda_{n,i,\pm,j}^{(T_1)}) + \Theta\left(\frac{\sigma_p^2 d d_h^{\frac{1}{2}}}{N\|\boldsymbol{\mu}\|_2^2 (\log(6N^2M^2/\delta))^3}\right)\right)
\end{equation}
for $i, j \in [M]\backslash\{1\}, n \in [N], s \in [T_2, t]$. We further get
\begin{equation}
\begin{aligned}\label{eq:64}
    \frac{\exp(\langle\boldsymbol q_\pm^{(s)},\boldsymbol k_\pm^{(s)} \rangle)}{\exp(\langle\boldsymbol q_\pm^{(s)},\boldsymbol k_\pm^{(s)} \rangle) + \sum_{j'=2}^{M} \exp(\langle\boldsymbol q_\pm^{(s)},\boldsymbol{k}_{n,j'}^{(s)} \rangle)} \leq \frac{\exp(\langle\boldsymbol q_\pm^{(s)},\boldsymbol{k}_{n,j}^{(s)} \rangle)}{C \exp(\langle\boldsymbol q_\pm^{(s)},\boldsymbol k_\pm^{(s)} \rangle)} = \frac{1}{C \exp(\Lambda_{n,\pm,j}^{(s)})}
\\
    \leq \frac{1}{C \exp(\Lambda_{n,\pm,j}^{(T_1)}) + \Theta\left(\frac{d_h^{\frac{1}{2}}}{N(\log(6N^2M^2/\delta))^3}\right)} = O\left(\frac{N(\log(6N^2M^2/\delta))^3}{d_h^{\frac{1}{2}}}\right).    
\end{aligned}
\end{equation}

For the first inequality, by the monotonicity of $\langle\boldsymbol q_\pm^{(s)},\boldsymbol k_\pm^{(s)} \rangle$ ($\langle\boldsymbol q_\pm^{(s)},\boldsymbol k_\pm^{(s)} \rangle$ is increasing and $\langle\boldsymbol q_\pm^{(s)},\boldsymbol{k}_{n,j}^{(s)} \rangle$ is decreasing), there exist a constant $C$ such that $C \exp(\langle\boldsymbol q_\pm^{(s)},\boldsymbol k_\pm^{(s)} \rangle) \geq \exp(\langle\boldsymbol q_\pm^{(s)},\boldsymbol k_\pm^{(s)} \rangle) + \sum_{j'=2}^{M} \exp(\langle\boldsymbol q_\pm^{(s)},\boldsymbol{k}_{n,j'}^{(s)} \rangle)$. The second inequality is by plugging \ref{eq:62}. Similarly, we have
\begin{equation}\label{eq:65}
\begin{aligned}
    &\frac{\exp(\langle\boldsymbol{q}_{n,i}^{(s)},\boldsymbol{k}_{n,j}^{(s)} \rangle)}{\exp(\langle\boldsymbol{q}_{n,i}^{(s)},\boldsymbol k_\pm^{(s)} \rangle) + \sum_{j'=2}^{M} \exp(\langle\boldsymbol{q}_{n,i}^{(s)},\boldsymbol{k}_{n,j'}^{(s)} \rangle)} \leq \frac{1}{C \exp(\Lambda_{n,i,\pm,j}^{(s)})}
\\
   & = O\left(\frac{N\|\boldsymbol{\mu}\|_2^2 (\log(6N^2M^2/\delta))^3}{\sigma_p^2 d d_h^{\frac{1}{2}}}\right).    
\end{aligned}
\end{equation}
Plugging (\ref{eq:64}) and (\ref{eq:65}) into the update rule of $\rho_{V,n,i}$ in Lemma ~\ref{lem:update_v} and get
{\small
\begin{equation}
\begin{aligned}
    &|\rho_{V,n,i}^{(s+1)} - \rho_{V,n,i}^{(s)}| \leq \frac{\eta}{NM} |\widetilde{\ell}_n^{'(s)}| \cdot \langle\widetilde{\boldsymbol{\xi}}_{n,i},\widetilde{\boldsymbol{\xi}}_{n,i}^{(t)}\rangle \cdot \left(O\left(\frac{N(\log(6N^2M^2/\delta))^3}{d_h^{\frac{3}{2}}}\right) + O\left(\frac{N\|\boldsymbol{\mu}\|_2^2 (\log(6N^2M^2/\delta))^3}{\sigma_p^2 d d_h^{\frac{1}{2}}}\right)\right) \\
    &\quad + \frac{\eta}{NM} \sum_{n' \neq n \lor i \neq i'} |\widetilde{\ell}_{n'}^{'(t)}| \cdot \langle \widetilde{\boldsymbol{\xi}}_{n,i}, \widetilde{\boldsymbol{\xi}}_{n',i'}^{(t)} \rangle \cdot \left(O\left(\frac{N(\log(6N^2M^2/\delta))^3}{d_h^{\frac{3}{2}}}\right) + O\left(\frac{N\|\boldsymbol{\mu}\|_2^2 (\log(6N^2M^2/\delta))^3}{\sigma_p^2 d d_h^{\frac{1}{2}}}\right)\right)+o(1) \\
    &\leq \frac{3\eta (\sigma_p^2 d+o(1))}{2NM} \left(O\left(\frac{N(\log(6N^2M^2/\delta))^3}{d_h^{\frac{1}{2}}}\right) + O\left(\frac{N\|\boldsymbol{\mu}\|_2^2 (\log(6N^2M^2/\delta))^3}{\sigma_p^2 d d_h^{\frac{1}{2}}}\right)\right) \\
    &+ \frac{\eta}{NM} N  M  (2{\sigma}_p^2 \sqrt{d \log(4N^2M^2/\delta)} \sqrt{d}+o(1)) \cdot \left(O\left(\frac{N(\log(6N^2M^2/\delta))^3}{d_h^{\frac{1}{2}}}\right) + O\left(\frac{N\|\boldsymbol{\mu}\|_2^2 (\log(6N^2M^2/\delta))^3}{\sigma_p^2 d d_h^{\frac{1}{2}}}\right)\right) \\
    &\leq \frac{2\eta}{NM} \left(O\left(\frac{N (\sigma_p^2 d +o(1))(\log(6N^2M^2/\delta))^3}{d_h^{\frac{3}{2}}}\right) + O\left(\frac{N\|\boldsymbol{\mu}\|_2^2 (\log(6N^2M^2/\delta))^3}{d_h^{\frac{3}{2}}}\right)\right) \\
    &= O\left(\frac{\eta (\sigma_p^2 d +o(1))(\log(6N^2M^2/\delta))^3}{d_h^{\frac{1}{2}}} + \frac{\eta \|\boldsymbol{\mu}\|_2^2 (\log(6N^2M^2/\delta))^3}{d_h^{\frac{1}{2}}}\right)
\end{aligned}
\end{equation}
}
where the second inequality is by Lemma ~\ref{lem:con_ineq} and $|\widetilde{\ell}_n^{(t)}| \leq 1$. For the last inequality, since $d = \widetilde{\Omega}(\epsilon^{-2} N^2 d_h)$, we have $N \cdot M \cdot 2{\sigma}_p^2 \sqrt{d \log(4N^2M^2/\delta)} \leq \frac{1}{2} {\sigma}_p^2 d$. By Definition~\ref{def:sca_v} and taking a summation we have
{\small
\begin{equation}
\begin{aligned}
    &|V_{n,i}^{(t+1)}| \leq |V_{n,i}^{(T_2)}| + \sum_{s=T_2}^{t} |\rho_{V,n,i}^{(s+1)} - \rho_{V,n,i}^{(s)}| \cdot \|\boldsymbol{w}_O\|_2^2 \\
    &\leq |V_{n,i}^{(T_2)}| + T_3 \cdot |\rho_{V,n,i}^{(t+1)} - \rho_{V,n,i}^{(t)}| \cdot \|\boldsymbol{w}_O\|_2^2 \\
    &\leq o(1) + \Theta\left(\frac{1}{\eta \epsilon (\|\boldsymbol{\mu}\|_2+\tau)^2 \|\boldsymbol{w}_O\|_2^2}\right) \cdot O\left(\frac{\eta (\sigma_p^2 d +o(1))(\log(6N^2M^2/\delta))^3}{d_h^{\frac{1}{2}}} + \frac{\eta \|\boldsymbol{\mu}\|_2^2 (\log(6N^2M^2/\delta))^3}{d_h^{\frac{1}{2}}}\right) \cdot \|\boldsymbol{w}_O\|_2^2 \\
    &= o(1) + O\left(\frac{(\sigma_p^2 d +o(1))(\log(6N^2M^2/\delta))^3}{\epsilon (\|\boldsymbol{\mu}\|_2+\tau)^2 d_h^{\frac{1}{2}}} + \frac{(\log(6N^2M^2/\delta))^3}{\epsilon d_h^{\frac{1}{2}}}\right) \\
    &= o(1) + o(1) \\
    &= o(1),
\end{aligned}
\end{equation}
}
where the first equality is by $N \cdot \text{SNR}^2 \geq \Omega(1)$, 
the second equality is by 
$d_h = \widetilde{\Omega}\!\left(\max\{\text{SNR}^4, \text{SNR}^{-4}\} N^2 \epsilon^{-2}\right)$. 
Then we have a constant upper bound for the sum of $V_{n,i}$ as follows:
\begin{equation*}
\begin{aligned}
   \sum_{i \in [M]\backslash\{1\}} |V_{n,i}^{(s)}| 
   &= (M - 1) \cdot o(1) \leq C_{15},
\end{aligned}
\end{equation*}
for $n \in [N], s \in [T_2, t]$.  

Expanding (\ref{eq:64}) and (\ref{eq:65}), we have
\begin{equation}
\begin{aligned}
   \frac{\exp(\langle\boldsymbol q_\pm^{(s)},\boldsymbol{k}_{n,j}^{(s)} \rangle)}
   {\exp(\langle\boldsymbol q_\pm^{(s)},\boldsymbol k_\pm^{(s)} \rangle) + \sum_{j'=2}^{M} \exp(\langle\boldsymbol q_\pm^{(s)},\boldsymbol{k}_{n,j'}^{(s)} \rangle)} = O\!\left(\frac{N (\log(6N^2M^2/\delta))^3}{d_h^{3/2}}\right) = o(1),
\end{aligned}
\end{equation}
where the equality is by 
$d_h = \widetilde{\Omega}\!\left(\max\{\text{SNR}^4, \text{SNR}^{-4}\} N^2 \epsilon^{-2}\right)$.  

Similarly,
\begin{equation}
\begin{aligned}
   \frac{\exp(\langle\boldsymbol{q}_{n,i}^{(s)},\boldsymbol{k}_{n,j}^{(s)} \rangle)}
   {\exp(\langle\boldsymbol{q}_{n,i}^{(s)},\boldsymbol k_\pm^{(s)} \rangle) + \sum_{j'=2}^{M} \exp(\langle\boldsymbol{q}_{n,i}^{(s)},\boldsymbol{k}_{n,j'}^{(s)} \rangle)} = O\!\left(\frac{N \|\boldsymbol{\mu}\|_2^2 (\log(6N^2M^2/\delta))^3}{\sigma_p^2 d\, d_h^{1/2}}\right)= o(1),
\end{aligned}
\end{equation}
where the equality is by the same choice of $d_h$.  

Then we have
\begin{equation}
\begin{aligned}
   \text{softmax}(\langle\boldsymbol q_\pm^{(s)},\boldsymbol k_\pm^{(s)} \rangle) 
   &\geq 1 - (M - 1)\cdot o(1)\geq 1 - o(1),
\end{aligned}
\end{equation}
\begin{equation}
\begin{aligned}
   \text{softmax}(\langle\boldsymbol{q}_{n,i}^{(s)},\boldsymbol k_\pm^{(s)} \rangle) 
   &\geq 1 - (M - 1)\cdot o(1)\geq 1 - o(1),
\end{aligned}
\end{equation}
for $i \in [M]\backslash\{1\}, n \in [N], s \in [T_2, t]$.  

Thus, in the following discussion, we omit the noise-related terms as $o(1)$.

Next we provide the bounds for $-\widetilde{\ell}_n^{'(s)}$.  
Note that $\ell(z) = \log(1 + \exp(-z))$ and $-\ell'(z) = \exp(-z)/(1 + \exp(-z))$.  
Without loss of generality, assume $y_n = 1$. We have
\begin{equation}
\begin{aligned}
   -\ell'(f(\widetilde{\boldsymbol X}_n, \theta(s))) 
   &= \frac{1}{1 + \exp\!\left(\tfrac{1}{M} \sum_{l=1}^{M} \varphi(\widetilde{\boldsymbol x}_{n,l}^{(s)\top} \mathbf{W}_Q^{(s)} \mathbf{W}_K^{(s)\top} (\widetilde{\boldsymbol X}_n^{(s)})^{\top}) \widetilde{\boldsymbol X}_n^{(s)} \mathbf{W}_V^{(s)}\boldsymbol{w}_O\right)} \\
   &= \frac{1}{M} \Big( \text{softmax}(\langle\boldsymbol q_\pm^{(s)},\boldsymbol k_\pm^{(s)} \rangle) 
      + \sum_{l=2}^{M} \text{softmax}(\langle\boldsymbol{q}_{n,l}^{(s)},\boldsymbol k_\pm^{(s)} \rangle) \Big) 
      \cdot \widetilde{\boldsymbol{\mu}}_+^{(s)\top} \mathbf{W}_V^{(s)}\boldsymbol{w}_O \\
   &\quad + \sum_{j \in [M]\backslash\{1\}} 
      \Big( \text{softmax}(\langle\boldsymbol{q}_{n,j}^{(s)},\boldsymbol{k}_{n,j}^{(s)} \rangle) 
      + \sum_{l=2}^{M} \text{softmax}(\langle\boldsymbol{q}_{n,l}^{(s)},\boldsymbol{k}_{n,j}^{(s)} \rangle) \Big) 
      \cdot \widetilde{\boldsymbol{\xi}}_{n,j}^{(s)\top} \mathbf{W}_V^{(s)}\boldsymbol{w}_O \\
   &= \frac{1}{M} \Big( M \cdot (1 - o(1)) \cdot V_+^{(s)} 
      + M \cdot o(1) \cdot \sum_{j \in [M]\backslash\{1\}} V_{n,i}^{(s)} \Big) \\
   &\geq \tfrac{1}{2} V_+^{(s)},
\end{aligned}
\end{equation}
for $s \in [T_2, t]$,  
where the second equality is by plugging equations above, and the inequality follows from $V_+^{(s)} \geq 3M \cdot |V_{n,i}^{(s)}|$.  

Similarly, we have
\begin{equation}\label{eq:74}
\begin{aligned}
   \frac{1}{M} \sum_{l=1}^{M} \varphi(\widetilde{\boldsymbol x}_{n,l}^{(s)\top} \mathbf{W}_Q^{(s)} \mathbf{W}_K^{(s)\top} (\widetilde{\boldsymbol X}_n^{(s)})^{\top}) \widetilde{\boldsymbol X}_n^{(s)} \mathbf{W}_V^{(s)}\boldsymbol{w}_O
   &\leq \max_{i \in [M]\backslash\{1\}} \{V_+^{(s)}, V_{n,i}^{(s)}\} = V_+^{(s)}.
\end{aligned}
\end{equation}

Then we have
\begin{equation}\label{eq:75}
\begin{aligned}
    -\ell'(f(\widetilde{X}_n, \theta(s))) 
   &= \frac{1}{1 + \exp\!\left(\tfrac{1}{M} \sum_{l=1}^{M} \varphi(\widetilde{x}_{n,l}^{(s)\top} \mathbf{W}_Q^{(s)} \mathbf{W}_K^{(s)\top} (\widetilde{X}_n^{(s)})^{\top}) \widetilde{X}_n^{(s)} \mathbf{W}_V^{(s)}\boldsymbol{w}_O\right)} \\
    &\geq \frac{1}{1 + \exp(V_+^{(s)})} \\
    &\geq \frac{C_{16}}{\exp(V_+^{(s)})}
\end{aligned}
\end{equation}
where the first inequality is by plugging (\ref{eq:74}). For the last inequality, note that $V_+^{(T_2)} \geq 0$ and $V_+^{(s)}$ is monotonically increasing, so there exist a constant $C_{16}$ such that $\frac{1}{1 + \exp(V_+^{(s)})} \geq \frac{C_{16}}{\exp(V_+^{(s)})}$. We also have the upper bound
\begin{equation}\label{eq:76}
\begin{aligned}
    -\ell'(f(\boldsymbol X_n, \theta(s))) &= \frac{1}{1 + \exp\left(\frac{1}{M} \sum_{l=1}^{M} \varphi(\boldsymbol x_{n,l}^\top \mathbf{W}_Q^{(s)} \mathbf{W}_K^{(s) \top} (\boldsymbol X_n)^\top)\boldsymbol X_n \mathbf{W}_V^{(s)}\boldsymbol{w}_O\right)} \\
    &\leq \frac{1}{1 + \exp(V_+^{(s)}/2)} \\
    &\leq \frac{1}{\exp(V_+^{(s)}/2)}
\end{aligned}
\end{equation}

Then by the update rule of $\gamma_{V,+}^{(t)}$ and in Lemma~\ref{lem:update_v} and get
\begin{equation}
\begin{aligned}
    \gamma_{V,+}^{(s+1)} - \gamma_{V,+}^{(s)} &= -\eta \langle\widetilde{\boldsymbol{\mu}}_+,\widetilde{\boldsymbol{\mu}}_+^{(s)}\rangle \sum_{n \in S_+} \widetilde{\ell}_n^{'(s)} \left(\frac{\exp(\langle\boldsymbol{q}_+^{(s)},\boldsymbol{k}_+^{(s)} \rangle)}{\exp(\langle\boldsymbol{q}_+^{(s)},\boldsymbol{k}_+^{(s)} \rangle) + \sum_{k=2}^{M} \exp(\langle\boldsymbol{q}_+^{(s)},\boldsymbol{k}_{n,k}^{(s)} \rangle)} \right. \\
    &\quad \left. + \sum_{j=2}^{M} \frac{\exp(\langle\boldsymbol{q}_{n,j}^{(s)},\boldsymbol{k}_+^{(s)} \rangle)}{\exp(\langle\boldsymbol{q}_{n,j}^{(s)},\boldsymbol{k}_+^{(s)} \rangle) + \sum_{k=2}^{M} \exp(\langle\boldsymbol{q}_{n,j}^{(s)},\boldsymbol{k}_{n,k}^{(s)} \rangle)} \right)+o(1) \\
    &\geq -\eta (\|\boldsymbol{\mu}\|_2-\tau)^2 \sum_{n \in S_+} \widetilde{\ell}_n^{'(s)} (M \cdot (1 - o(1))) \\
    &\geq \eta (\|\boldsymbol{\mu}\|_2-\tau)^2 \cdot \frac{N}{4} \cdot (1 - o(1)) \cdot \frac{C_{16}}{\exp(V_+^{(s)})} \\
    &\geq \eta C_{17} (\|\boldsymbol{\mu}\|_2-\tau)_2^2 \frac{1}{\exp(V_+^{(s)})}
\end{aligned}
\end{equation}
where the second inequality is by (\ref{eq:75}). Then by definition~\ref{def:sca_v}, we get
\begin{equation}
\begin{aligned}
    V_+^{(s+1)} - V_+^{(s)} &= (\gamma_{V,+}^{(s+1)} - \gamma_{V,+}^{(s)}) \|\boldsymbol{w}_O\|_2^2 \geq \frac{\eta C_{17} (\|\boldsymbol{\mu}\|_2-\tau)^2 \|\boldsymbol{w}_O\|_2^2}{\exp(V_+^{(s)})} \\
    \end{aligned}
\end{equation}
Multiply both sides simultaneously by $\exp(V_+^{(s)})$ and get
\begin{equation}
\begin{aligned}
    \exp(V_+^{(s)})(V_+^{(s+1)} - V_+^{(s)}) &\geq \eta C_{17} (\|\boldsymbol{\mu}\|_2-\tau)^2 \|\boldsymbol{w}_O\|_2^2
\end{aligned}
\end{equation}
Taking a summation from $T_2$ to $t$ and get
\begin{equation}
\begin{aligned}
    \sum_{s=T_2}^{t} \exp(V_+^{(s)})(V_+^{(s+1)} - V_+^{(s)}) &\geq \sum_{s=T_2}^{t} \eta C_{17} (\|\boldsymbol{\mu}\|_2-\tau)^2 \|\boldsymbol{w}_O\|_2^2 \\
    &\geq \eta C_{17} (\|\boldsymbol{\mu}\|_2-\tau)^2 \|\boldsymbol{w}_O\|_2^2 (t - T_2 + 1)
\end{aligned}
\end{equation}

By the property that $V_+^{(s)}$ is monotonically increasing, we have
\begin{equation}
\begin{aligned}
    \int_{V_+^{(T_2)}}^{V_+^{(t+1)}} \exp(x) dx &\geq \sum_{s=T_2}^{t} \exp(V_+^{(s)})(V_+^{(s+1)} - V_+^{(s)}) \\
    &\geq \eta C_{17} (\|\boldsymbol{\mu}\|_2-\tau)^2 \|\boldsymbol{w}_O\|_2^2 (t - T_2 + 1)
\end{aligned}
\end{equation}

By $\int_{V_+^{(T_2)}/2}^{V_+^{(t+1)}/2} \exp(x) dx = \exp(V_+^{(t+1)}/2) - \exp(V_+^{(T_2)}/2)$ we get
\begin{equation}
\begin{aligned}
    V_+^{(t+1)} &\geq \log \left( \exp(V_+^{(T_2)}/2) + \eta C_{17} (\|\boldsymbol{\mu}\|_2-\tau)^2 \|\boldsymbol{w}_O\|_2^2 (t - T_2 + 1) \right) \\
    &\geq \eta C_{17} (\|\boldsymbol{\mu}\|_2-\tau)^2 \|\boldsymbol{w}_O\|_2^2 (t - T_2 + 1)
\end{aligned}
\end{equation}

Similarly, we have
\begin{equation}
\begin{aligned}
    V_-^{(t+1)} &\leq -\log \left( \exp\left(V_-^{(T_2)}  \right) 
    + \eta C_{17} (\|\boldsymbol{\mu}\|_2-\tau)_2^2 \|\boldsymbol{w}_O\|_2^2 (t - T_2 + 1)\right)
\end{aligned}
\end{equation}

Next we provide upper bounds for $V_+^{(t+1)}$ and $V_-^{(t+1)}$. By the update rule of $\gamma_{V,+}^{(t)}$ and in Lemma~\ref{lem:update_v} we have
\begin{equation}
\begin{aligned}
\gamma_{V,+}^{(s+1)} - \gamma_{V,+}^{(s)}
&= - \eta \langle\widetilde{\boldsymbol{\mu}}_+,\widetilde{\boldsymbol{\mu}}_+^{(s)}\rangle \sum_{n \in S_+} \widetilde{\ell}_n^{'(s)}
\Bigg(
    \frac{\exp\big(\langle\boldsymbol{q}_+^{(s)},\boldsymbol{k}_+^{(s)} \rangle\big)}
         {\exp\big(\langle\boldsymbol{q}_+^{(s)},\boldsymbol{k}_+^{(s)} \rangle\big)
          + \sum_{k=2}^{M} \exp\big(\langle\boldsymbol{q}_+^{(s)},\boldsymbol{k}_{n,k}^{(s)} \rangle\big)} \\[6pt]
&\qquad\qquad\qquad
    + \sum_{j=2}^{M}
    \frac{\exp\big(\langle\boldsymbol{q}_{n,j}^{(s)},\boldsymbol{k}_+^{(s)} \rangle\big)}
         {\exp\big(\langle\boldsymbol{q}_{n,j}^{(s)},\boldsymbol{k}_+^{(s)} \rangle\big)
          + \sum_{k=2}^{M} \exp\big(\langle\boldsymbol{q}_{n,j}^{(s)},\boldsymbol{k}_{n,k}^{(s)} \rangle\big)}
\Bigg)+o(1) \\[6pt]
&\leq \eta (\|\boldsymbol{\mu}\|_2+\tau)^2 \sum_{n \in S_+} \big(-\widetilde{\ell}_n^{'(s)}\big) \cdot M \\[4pt]
&\leq \eta (\|\boldsymbol{\mu}\|_2+\tau)^2 \cdot \frac{3N}{4} \cdot \exp\!\big(-V_+^{(s)}/2\big) \\[4pt]
&= \frac{3\eta (\|\boldsymbol{\mu}\|_2+\tau)^2}{4}\, \exp\!\big(-V_+^{(s)}/2\big).
\end{aligned}
\end{equation}
where the second inequality is by (\ref{eq:76}). Then by definition~\ref{def:sca_v}, we get
\begin{equation}
\begin{aligned}
    V_+^{(s+1)} - V_+^{(s)} &= (\gamma_{V,+}^{(s+1)} - \gamma_{V,+}^{(s)}) \|\boldsymbol{w}_O\|_2^2 \\
    &\leq \frac{3\eta (\|\boldsymbol{\mu}\|_2+\tau)^2 \|\boldsymbol{w}_O\|_2^2}{4 \exp(V_+^{(s)}/2)}
\end{aligned}
\end{equation}

Further we have
\begin{equation}\label{eq:82}
\begin{aligned}
    \exp(V_+^{(s+1)}/2) &\leq \exp(V_+^{(s)}/2 + \frac{3\eta (\|\boldsymbol{\mu}\|_2+\tau)^2 \|\boldsymbol{w}_O\|_2^2}{8 \exp(V_+^{(s)}/2)} \\
    &= \exp(V_+^{(s)}/2) \cdot \exp(\frac{3\eta (\|\boldsymbol{\mu}\|_2+\tau)^2 \|\boldsymbol{w}_O\|_2^2}{8 \exp(V_+^{(s)}/2)}) \\
    &\leq C_{18} \exp(V_+^{(s)}/2)
\end{aligned}
\end{equation}

For the last inequality, by $\eta \leq \widetilde{O}(\min\{\|\boldsymbol{\mu}\|_2^{-2}, (\sigma_p^2 d)^{-1}\} \cdot d_h^{-\frac{1}{2}})$, $V_+^{(T_2)} = \Theta(1)$ and the monotonicity of $V_+^{(s)}$, we have $\exp(\frac{3\eta \|\boldsymbol{\mu}\|_2^2 \|\boldsymbol{w}_O\|_2^2}{8 \exp(V_+^{(s)}/2)}) \leq C_{18}$. Multiplying both sides by $(V_+^{(s+1)}/2 - V_+^{(s)}/2)$ simultaneously gives
\begin{equation}
\begin{aligned}
    \exp(V_+^{(s)})(V_+^{(s+1)}/2 - V_+^{(s)}/2) &\leq C_{18} \exp(V_+^{(s)}/2)(V_+^{(s+1)}/2 - V_+^{(s)}/2) \\
    &\leq \frac{3\eta C_{18} (\|\boldsymbol{\mu}\|_2+\tau)^2 \|\boldsymbol{w}_O\|_2^2}{8}
\end{aligned}
\end{equation}

where the last inequality is by plugging (\ref{eq:82}). Taking a summation we have
\begin{equation}
\begin{aligned}
    \int_{V_+^{(T_2)}/2}^{V_+^{(t+1)}/2} \exp(x) dx &\leq \sum_{s=T_2}^{T_3} \exp(V_+^{(s+1)}/2)(V_+^{(s+1)}/2 - V_+^{(s)}/2) \\
    &\leq \sum_{s=T_2}^{T_3} \frac{3\eta C_{18} (\|\boldsymbol{\mu}\|_2+\tau)^2 \|\boldsymbol{w}_O\|_2^2}{8} \\
    &\leq \Theta\left(\frac{1}{\eta \epsilon (\|\boldsymbol{\mu}\|_2+\tau)^2 \|\boldsymbol{w}_O\|_2^2}\right) \cdot \frac{3\eta C_{18} (\|\boldsymbol{\mu}\|_2+\tau)^2 \|\boldsymbol{w}_O\|_2^2}{8} 
= O\left(\frac{1}{\epsilon}\right)
\end{aligned}
\end{equation}
By $\int_{V_+^{(T_2)}/2}^{V_+^{(t+1)}/2} \exp(x) dx = \exp(V_+^{(t+1)}/2) - \exp(V_+^{(T_2)}/2)$ we have
\begin{equation*}
    V_+^{(t+1)} \leq 2 \log \left( \exp(V_+^{(T_2)}/2) + O\left(\frac{1}{\epsilon}\right) \right) = 2 \log \left( O\left(\frac{1}{\epsilon}\right) \right)
\end{equation*}

Similarly, we have
\begin{equation*}
    V_-^{(t+1)} \geq -2 \log \left( O\left(\frac{1}{\epsilon}\right) \right)
\end{equation*}
\subsubsection{Proof of Claim 6}

By $\mathcal{H}(T_2), \ldots, \mathcal{H}(t-1)$, we have $\text{softmax}(\langle\boldsymbol q_\pm^{(t)},\boldsymbol k_\pm^{(t)} \rangle), \text{softmax}(\langle\boldsymbol{q}_{n,i}^{(t)},\boldsymbol k_\pm^{(t)} \rangle) = 1 - o(1)$ and $\text{softmax}(\langle\boldsymbol q_\pm^{(t)},\boldsymbol{k}_{n,j}^{(t)} \rangle), \text{softmax}(\langle\boldsymbol{q}_{n,i}^{(t)},\boldsymbol{k}_{n,j}^{(t)} \rangle) = o(1)$, which have been proved in ~\ref{sec:cl5}. By the results of ~\ref{sec:low_ab}, we have the signs of $\alpha$ and $\beta$ as follows:
\begin{equation*}
\begin{aligned}
    \alpha_{+,+}^{(t)}, \alpha_{-,-}^{(t)}, \beta_{+,+}^{(t)}, \beta_{-,-}^{(t)}, \alpha_{n,i,+}^{(t)}, \alpha_{n,i,-}^{(t)}, \beta_{n,+,i}^{(t)}, \beta_{n,-,i}^{(t)} &\geq 0, \\
    \alpha_{n,+,i}^{(t)}, \alpha_{n,-,i}^{(t)}, \alpha_{n,i,n,j}^{(t)}, \beta_{n,i,+}^{(t)}, \beta_{n,i,-}^{(t)}, \beta_{n,j,n,i}^{(t)} &\leq 0.
\end{aligned}
\end{equation*}

Then combined with $\mathcal{G}(T)$ and we have the dynamics of $\langle\boldsymbol q,\boldsymbol k \rangle$ as follows:
\begin{equation}
\begin{aligned}
    \langle\boldsymbol{q}_+^{(t+1)},\boldsymbol{k}_+^{(t+1)} \rangle - \langle\boldsymbol{q}_+^{(t)},\boldsymbol{k}_+^{(t)} \rangle &= \alpha_{+,+}^{(t)} \|\boldsymbol{k}_+^{(t)}\|_2^2 + \sum_{n \in S_+} \sum_{i=2}^M \alpha_{n,+,i}^{(t)} \langle\boldsymbol{k}_+^{(t)},\boldsymbol{k}_{n,i}^{(t)} \rangle \\
    &\quad + \beta_{+,+}^{(t)} \|\boldsymbol{q}_+^{(t)}\|_2^2 + \sum_{n \in S_+} \sum_{i=2}^M \beta_{n,+,i}^{(t)} \langle\boldsymbol{q}_+^{(t)},\boldsymbol{q}_{n,i}^{(t)} \rangle \\
    &\quad + \left( \alpha_{+,+}^{(t)}\boldsymbol{k}_+^{(t)} + \sum_{n \in S_+} \sum_{i=2}^M \alpha_{n,+,i}^{(t)}\boldsymbol{k}_{n,i}^{(t)} \right) \\
    &\quad \cdot \left( \beta_{+,+}^{(t)}\boldsymbol{q}_+^{(t)} + \sum_{n \in S_+} \sum_{i=2}^M \beta_{n,+,i}^{(t)}\boldsymbol{q}_{n,i}^{(t)} \right)^\top \\
    &= \alpha_{+,+}^{(t)} \|\boldsymbol{k}_+^{(t)}\|_2^2 + \beta_{+,+}^{(t)} \|\boldsymbol{q}_+^{(t)}\|_2^2 + \{\text{lower order term\}} \\
    &\geq 0
\end{aligned}
\end{equation}

Similarly, we have
\begin{align*}
        \langle\boldsymbol q_\pm^{(t+1)},\boldsymbol k_\pm^{(t+1)} \rangle -\langle\boldsymbol q_\pm^{(t)},\boldsymbol k_\pm^{(t)} \rangle&\geq 0, \\
        \langle\boldsymbol{q}_{n,i}^{(t+1)},\boldsymbol k_\pm^{(t+1)} \rangle-\langle\boldsymbol{q}_{n,i}^{(t)},\boldsymbol  k_\pm^{(t)} \rangle &\geq ,0 \\
        \langle\boldsymbol q_\pm^{(t+1)},\boldsymbol{k}_{n,j}^{(t+1)} \rangle -\langle\boldsymbol q_\pm^{(t)},\boldsymbol{k}_{n,j}^{(t)} \rangle&\leq 0, \\
        \langle\boldsymbol{q}_{n,i}^{(t+1)},\boldsymbol{k}_{n,j}^{(t+1)} \rangle -\langle\boldsymbol{q}_{n,i}^{(t)},\boldsymbol{k}_{n,j}^{(t)} \rangle&\leq 0
    \end{align*}
which completes the proof. The proof for Claim 7 is in Section ~\ref{sec:pf_cl7}
\subsubsection{Proof of Claim 8}

By the results of ~\ref{sec:up_qk9}, we have
\begin{equation}\label{eq:93}
\langle\boldsymbol{q}_+^{(t+1)},\boldsymbol{k}_+^{(t+1)} \rangle - \langle\boldsymbol{q}_+^{(t)},\boldsymbol{k}_+^{(t)} \rangle 
\leq \frac{\eta C_{10} (\|\boldsymbol{\mu}\|_2+\tau)^2\|\boldsymbol{\mu}\|_2^2 \sigma_h^2 d_h \log \left( O\!\left( \tfrac{1}{\epsilon} \right) \right)}{\exp\!\left(\langle\boldsymbol{q}_+^{(t)},\boldsymbol{k}_+^{(t)} \rangle\right)},
\end{equation}
Further we have
\begin{equation*}
\begin{aligned}
&\exp(\langle\boldsymbol{q}_+^{(t+1)},\boldsymbol{k}_+^{(t+1)} \rangle) \leq \exp\left( \langle\boldsymbol{q}_+^{(t)},\boldsymbol{k}_+^{(t)} \rangle + \frac{\eta C_{10} (\|\boldsymbol{\mu}\|_2+\tau)^2\|\boldsymbol{\mu}\|_2^2 \sigma_h^2 d_h \log \left( O\left( \frac{1}{\epsilon} \right) \right)}{\exp(\langle\boldsymbol{q}_+^{(t)},\boldsymbol{k}_+^{(t)} \rangle)} \right)
\\
&= \exp\left( \langle\boldsymbol{q}_+^{(t)},\boldsymbol{k}_+^{(t)} \rangle \right) \cdot \exp\left( \frac{\eta C_{10} (\|\boldsymbol{\mu}\|_2+\tau)^2\|\boldsymbol{\mu}\|_2^2 \sigma_h^2 d_h \log \left( O\left( \frac{1}{\epsilon} \right) \right)}{\exp(\langle\boldsymbol{q}_+^{(t)},\boldsymbol{k}_+^{(t)} \rangle)} \right) \\
&\leq C_{11} \exp\left( \langle\boldsymbol{q}_+^{(t)},\boldsymbol{k}_+^{(t)} \rangle \right).    
\end{aligned}
\end{equation*}

For the last inequality, by $\eta \leq \widetilde{O}(\min\{\boldsymbol{\mu}\|_2^{-2}, (\sigma_p^2 d)^{-1}\} \cdot d_h^{-\frac{1}{2}})$, $\sigma_h^2 \leq \min\{\|\boldsymbol{\mu}\|_2^{-2}, (\sigma_p^2 d)^{-1}\} \cdot d_h^{-\frac{1}{2}} \cdot (\log(6N^2M^2/\delta))^{-\frac{3}{2}}$, $\langle\boldsymbol{q}_+^{(T_1)},\boldsymbol{k}_+^{(T_1)} \rangle = o(1)$ and the monotonicity of $\langle\boldsymbol{q}_+^{(s)},\boldsymbol{k}_+^{(s)} \rangle$ for $s \in [T_1, t]$, we have $\exp\left( \frac{\eta C_{10} \|\boldsymbol{\mu}\|_2^4 \sigma_h^2 d_h \log \left( O\left( \frac{1}{\epsilon} \right) \right)}{\exp(\langle\boldsymbol{q}_+^{(t)},\boldsymbol{k}_+^{(t)} \rangle)} \right) \leq \exp(o(1)) \leq C_{11}$. Multiplying both sides by $\left( \langle\boldsymbol{q}_+^{(t+1)},\boldsymbol{k}_+^{(t+1)} \rangle - \langle\boldsymbol{q}_+^{(t)},\boldsymbol{k}_+^{(t)} \rangle \right)$ simultaneously gives
\begin{equation}
\begin{aligned}
&\exp(\langle\boldsymbol{q}_+^{(t+1)},\boldsymbol{k}_+^{(t+1)} \rangle) \left( \langle\boldsymbol{q}_+^{(t+1)},\boldsymbol{k}_+^{(t+1)} \rangle - \langle\boldsymbol{q}_+^{(t)},\boldsymbol{k}_+^{(t)}\rangle \right) \\
& \leq C_{11} \exp\left( \langle\boldsymbol{q}_+^{(t)},\boldsymbol{k}_+^{(t)} \rangle \right) \cdot \left( \langle\boldsymbol{q}_+^{(t+1)},\boldsymbol{k}_+^{(t+1)} \rangle - \langle\boldsymbol{q}_+^{(t)},\boldsymbol{k}_+^{(t)} \rangle \right) \\
&\leq \eta C_{12} (\|\boldsymbol{\mu}\|_2+\tau)^2\|\boldsymbol{\mu}\|_2^2 \sigma_h^2 d_h \log \left( O\left( \frac{1}{\epsilon} \right) \right), 
\end{aligned}
\end{equation}
where the last inequality is by plugging (\ref{eq:93}). Taking a summation we have
\begin{equation}
\begin{aligned}
\int_{\langle\boldsymbol{q}_+^{(T_2)},\boldsymbol{k}_+^{(T_2)} \rangle}^{\langle\boldsymbol{q}_+^{(t+1)},\boldsymbol{k}_+^{(t+1)} \rangle} \exp(x) \, dx 
&\leq \sum_{s=T_2}^{t} \exp\!\big(\langle\boldsymbol{q}_+^{(s+1)},\boldsymbol{k}_+^{(s+1)} \rangle\big) 
   \Big( \langle\boldsymbol{q}_+^{(s+1)},\boldsymbol{k}_+^{(s+1)} \rangle - \langle\boldsymbol{q}_+^{(s)},\boldsymbol{k}_+^{(s)} \rangle \Big) \\
&\leq \sum_{s=T_2}^{t} \eta C_{12} (\|\boldsymbol{\mu}\|_2+\tau)^2 \|\boldsymbol{\mu}\|_2^2 \sigma_h^2 d_h 
   \log \!\left( O\!\left( \tfrac{1}{\epsilon} \right) \right) \\
&\leq T_3 \cdot \eta C_{12} (\|\boldsymbol{\mu}\|_2+\tau)^2 \|\boldsymbol{\mu}\|_2^2 \sigma_h^2 d_h 
   \log \!\left( O\!\left( \tfrac{1}{\epsilon} \right) \right) \\
&= O\!\left( \frac{d_h^{\tfrac{1}{2}} \log \!\left( O\!\left( \tfrac{1}{\epsilon} \right) \right)}{\epsilon (\log(6N^2M^2/\delta))^{\tfrac{3}{2}}} \right).
\end{aligned}
\end{equation}
where the first inequality is due to $\langle\boldsymbol{q}_+^{(s)},\boldsymbol{k}_+^{(s)} \rangle$ is monotone increasing, the last equality is by $T_3 = \Theta(\eta^{-1} \epsilon^{-1} (\|\boldsymbol{\mu}\|_2+\tau)^{-2} \|\boldsymbol{w}_O\|_2^{-2})$, $\|\boldsymbol{w}_O\|_2^2 = \Theta(1)$ and $\sigma_h^2 \leq \min\{\|\boldsymbol{\mu}\|_2^{-2}, (\sigma_p^2 d)^{-1}\} \cdot d_h^{-\frac{1}{2}} \cdot (\log(6N^2M^2/\delta))^{-\frac{3}{2}}$. By $\int_{\langle\boldsymbol{q}_+^{(T_2)},\boldsymbol{k}_+^{(T_2)} \rangle}^{\langle\boldsymbol{q}_+^{(t+1)},\boldsymbol{k}_+^{(t+1)} \rangle} \exp(x) dx = \exp(\langle\boldsymbol{q}_+^{(t+1)},\boldsymbol{k}_+^{(t+1)} \rangle) - \exp(\langle\boldsymbol{q}_+^{(T_2)},\boldsymbol{k}_+^{(T_2)} \rangle)$, we have
\[
\langle\boldsymbol{q}_+^{(t+1)},\boldsymbol{k}_+^{(t+1)} \rangle \leq \log \left( \exp(\langle\boldsymbol{q}_+^{(T_2)},\boldsymbol{k}_+^{(T_2)} \rangle) + O\left( \frac{d_h^{\frac{1}{2}} \log \left( O\left( \frac{1}{\epsilon} \right) \right)}{\epsilon (\log(6N^2M^2/\delta))^{\frac{3}{2}}} \right) \right) \leq \log(\epsilon^{-1} d_h^{\frac{1}{2}})\] 
where the last inequality is by $\langle\boldsymbol{q}_+^{(T_2)},\boldsymbol{k}_+^{(T_2)} \rangle \leq \log(d_h^{\frac{1}{2}})$. By the results of ~\ref{sec:up_qk9}, we also have
\begin{equation}
\langle\boldsymbol{q}_-^{(t+1)},\boldsymbol{k}_-^{(t+1)} \rangle - \langle\boldsymbol{q}_-^{(t)},\boldsymbol{k}_-^{(t)} \rangle \leq \frac{\eta C_{10} \|\boldsymbol{\mu}\|_2^2(\|\boldsymbol{\mu}\|_2+\tau)^2 \sigma_h^2 d_h \log \left( O\left( \frac{1}{\epsilon} \right) \right)}{\exp(\langle q-^{(t)},\boldsymbol{k}_-^{(t)} \rangle)}. 
\end{equation}
\begin{equation}
\langle\boldsymbol{q}_{\pm}^{(s+1)},\boldsymbol{k}_{n,j}^{(s+1)} \rangle - \langle\boldsymbol{q}_{\pm}^{(s)},\boldsymbol{k}_{n,j}^{(s)} \rangle \geq -\frac{\eta C_{10} \sigma_p^2 d (\|\boldsymbol{\mu}\|_2+\tau)^2 \sigma_h^2 d_h\log \left( O\left( \frac{1}{\epsilon} \right)\right)}{N} \cdot \exp(\langle\boldsymbol{q}_{\pm}^{(s)},\boldsymbol{k}_{n,j}^{(s)} \rangle).
\end{equation}
\begin{equation}
\langle\boldsymbol{q}_{n,i}^{(s+1)},\boldsymbol{k}_{\pm}^{(s+1)} \rangle - \langle\boldsymbol{q}_{n,i}^{(s)},\boldsymbol{k}_{\pm}^{(s)} \rangle \leq \frac{\eta C_{10} \sigma_p^2 d (\|\boldsymbol{\mu}\|_2+\tau)^2 \sigma_h^2 d_h\log \left( O\left( \frac{1}{\epsilon} \right)\right)}{N \exp(\langle\boldsymbol{q}_{n,i}^{(s)},\boldsymbol{k}_{\pm}^{(s)} \rangle)}.
\end{equation}
\begin{equation}
\begin{aligned}
\langle\boldsymbol{q}_{n,i}^{(s+1)},\boldsymbol{k}_{n,j}^{(s+1)} \rangle - \langle\boldsymbol{q}_{n,i}^{(s)},\boldsymbol{k}_{n,j}^{(s)} \rangle \geq -&\frac{\eta C_{10} \sigma_p^2 d(\sigma_p^2 d+\sigma_p\tau\sqrt{2\log(4NM/\delta)}+\tau^2) \sigma_h^2 d_h\log \left( O\left( \frac{1}{\epsilon} \right)\right)}{N} \\&\cdot \exp(\langle\boldsymbol{q}_{n,i}^{(s)},\boldsymbol{k}_{n,j}^{(s)} \rangle).
\end{aligned}
\end{equation}
Then using the similar method as for $\langle\boldsymbol{q}_+^{(t+1)},\boldsymbol{k}_+^{(t+1)} \rangle$, we get
\begin{equation}
\begin{aligned}
\langle\boldsymbol{q}_-^{(t+1)},\boldsymbol{k}_-^{(t+1)} \rangle &\leq \log(\epsilon^{-1} d_h^{\frac{1}{2}}), \\
\langle\boldsymbol q_\pm^{(t+1)},\boldsymbol{k}_{n,j}^{(t+1)} \rangle &\geq -\log(\epsilon^{-1} d_h^{\frac{1}{2}}), \\
\langle\boldsymbol{q}_{n,i}^{(t+1)},\boldsymbol k_\pm^{(t+1)} \rangle &\leq \log(\epsilon^{-1} d_h^{\frac{1}{2}}), \\
\langle\boldsymbol{q}_{n,i}^{(t+1)},\boldsymbol{k}_{n,j}^{(t+1)} \rangle &\geq -\log(\epsilon^{-1} d_h^{\frac{1}{2}}),
\end{aligned}
\end{equation}

Next we provide the upper bound for $|\langle\boldsymbol q_\pm^{(t+1)},\boldsymbol k_\pm^{(t+1)} \rangle|, |\langle\boldsymbol{q}_{n,i}^{(t+1)},\boldsymbol{k}_{n',j}^{(t+1)} \rangle|$. By the results of~\ref{sec:f10}, we have
\begin{equation}
\begin{aligned}
\sum_{s=T_2}^{t} |\beta_{n,+,i}^{(t)}|, \sum_{s=T_2}^{t} |\beta_{n,-,i}^{(t)}| &= O\left(\frac{\text{SNR}^2 (\log(6N^2M^2/\delta))^3 \log \left( O\left( \frac{1}{\epsilon} \right) \right)}{\epsilon d_h^{\frac{1}{2}}}\right), 
\end{aligned}
\end{equation}
for $i \in [M]\backslash\{1\}, n \in S_\pm$.
\begin{equation}
\begin{aligned}
&\sum_{s=T_2}^{t} |\alpha_{+,+}^{(t)}|, \sum_{s=T_2}^{t} |\alpha_{-,-}^{(t)}|, \sum_{s=T_2}^{t} |\beta_{+,+}^{(t)}|, \sum_{s=T_2}^{t} |\beta_{-,-}^{(t)}|, \sum_{s=T_2}^{t} |\alpha_{n,i,+}^{(t)}|, \sum_{s=T_2}^{t} |\beta_{n,i,-}^{(t)}| \\&= O\left(\frac{N (\log(6N^2M^2/\delta))^3 \log \left( O\left( \frac{1}{\epsilon} \right) \right)}{\epsilon d_h^{\frac{1}{2}}}\right),
\end{aligned}
\end{equation}
for $i \in [M]\backslash\{1\}, n \in S_\pm$.
\begin{equation}
\begin{aligned}
&\sum_{s=T_2}^{t} |\alpha_{n,+,i}^{(t)}|, \sum_{s=T_2}^{t} |\alpha_{n,-,i}^{(t)}|, \sum_{s=T_2}^{t} |\alpha_{n,i,+}^{(t)}|, \sum_{s=T_2}^{t} |\alpha_{n,i,-}^{(t)}|, \sum_{s=T_2}^{t} |\alpha_{n,i,n,j}^{(t)}|, \sum_{s=T_2}^{t} |\beta_{n,j,n,i}^{(t)}| \\&= O\left(\frac{(\log(6N^2M^2/\delta))^3 \log \left( O\left( \frac{1}{\epsilon} \right) \right)}{\epsilon d_h^{\frac{1}{2}}}\right), 
\end{aligned}
\end{equation}
for $i, j \in [M]\backslash\{1\}, n \in S_\pm$.
\begin{equation}
\begin{aligned}
\sum_{s=T_2}^{t} |\alpha_{n,i,n',j}^{(t)}|, \sum_{s=T_2}^{t} |\beta_{n,j,n',i}^{(t)}| &= O\left(\frac{(\log(6N^2M^2/\delta))^4 \log \left( O\left( \frac{1}{\epsilon} \right) \right)}{\epsilon d_h^{\frac{1}{2}}}\right) 
\end{aligned}
\end{equation}
for $i, j \in [M]\backslash\{1\}, n, n' \in [N], n \neq n'$. Plugging these and proposition $\mathcal{G}(t)$ into the update rule of $|\langle\boldsymbol q_\pm^{(t)},\boldsymbol k_\mp^{(t)} \rangle|, |\langle\boldsymbol{q}_{n,i}^{(t)},\boldsymbol{k}_{\bar{n},j}^{(t)} \rangle|$ and get
\begin{equation}
\begin{aligned}
&|\langle\boldsymbol{q}_+^{(t+1)},\boldsymbol{k}_-^{(t+1)} \rangle| \leq |\langle\boldsymbol{q}_+^{(T_2)},\boldsymbol{k}_-^{(T_2)} \rangle| + \sum_{s=T_2}^{t} |\langle\boldsymbol{q}_+^{(s+1)},\boldsymbol{k}_-^{(s+1)} \rangle - \langle\boldsymbol{q}_+^{(s)},\boldsymbol{k}_-^{(s)} \rangle| \\
&\leq |\langle\boldsymbol{q}_+^{(T_2)},\boldsymbol{k}_-^{(T_2)} \rangle| \\
&\quad + \sum_{s=T_2}^{t} \left| \alpha_{+,+}^{(s)} \langle\boldsymbol{k}_+^{(s)},\boldsymbol{k}_-^{(s)} \rangle + \sum_{n \in S_+} \sum_{i=2}^{M} \alpha_{n,+,i}^{(s)} \langle\boldsymbol{k}_{n,i}^{(s)},\boldsymbol{k}_-^{(s)} \rangle \right. \\
&\quad \left. + \beta_{-,-}^{(s)} \langle\boldsymbol{q}_+^{(s)},\boldsymbol{q}_-^{(s)} \rangle + \sum_{n \in S_-} \sum_{i=2}^{M} \beta_{n,-,i}^{(s)} \langle\boldsymbol{q}_{n,i}^{(s)},\boldsymbol{q}_-^{(s)} \rangle \right| \\
&\quad + \left( \alpha_{+,+}^{(s)}\boldsymbol{k}_+^{(s)} + \sum_{n \in S_+} \sum_{i=2}^{M} \alpha_{n,+,i}^{(s)}\boldsymbol{k}_{n,i}^{(s)} \right) \\
&\quad \cdot \left( \beta_{-,-}^{(s)}\boldsymbol{q}_-^{(s)\top} + \sum_{n \in S_-} \sum_{i=2}^{M} \beta_{n,-,i}^{(s)}\boldsymbol{q}_{n,i}^{(s)\top} \right) \bigg| \\
&\leq |\langle\boldsymbol{q}_+^{(T_2)},\boldsymbol{k}_-^{(T_2)} \rangle| \\
&\quad + \sum_{s=T_2}^{t} |\alpha_{+,+}^{(t)}| |\langle\boldsymbol{k}_+^{(t)},\boldsymbol{k}_-^{(t)} \rangle| + \sum_{n \in S_+} \sum_{i=2}^{M} \sum_{s=T_2}^{t} |\alpha_{n,+,i}^{(t)}| |\langle\boldsymbol{k}_{n,i}^{(t)},\boldsymbol{k}_-^{(t)} \rangle| \\
&\quad + \sum_{s=T_2}^{t} |\beta_{-,-}^{(t)}| |\langle\boldsymbol{q}_+^{(t)},\boldsymbol{q}_-^{(t)} \rangle| + \sum_{n \in S_-} \sum_{i=2}^{M} \sum_{s=T_2}^{t} |\beta_{n,-,i}^{(t)}| |\langle\boldsymbol{q}_{n,i}^{(t)},\boldsymbol{q}_+^{(t)} \rangle| \\
&\quad + \{\text{lower order term}\} \\
&= |\langle\boldsymbol{q}_+^{(T_2)},\boldsymbol{k}_-^{(T_2)} \rangle| \\
&\quad + O\left(\frac{N (\log(6N^2M^2/\delta))^3 \log \left( O\left( \frac{1}{\epsilon} \right) \right)}{\epsilon d_h^{\frac{1}{2}}}\right) \cdot o(1) + N \cdot M \cdot O\left(\frac{\text{SNR}^2 (\log(6N^2M^2/\delta))^3 \log \left( O\left( \frac{1}{\epsilon} \right) \right)}{\epsilon d_h^{\frac{1}{2}}}\right) \cdot o(1) \\
&\quad + O\left(\frac{N (\log(6N^2M^2/\delta))^3 \log \left( O\left( \frac{1}{\epsilon} \right) \right)}{\epsilon d_h^{\frac{1}{2}}}\right) \cdot o(1) + N \cdot M \cdot O\left(\frac{\text{SNR}^2 (\log(6N^2M^2/\delta))^3 \log \left( O\left( \frac{1}{\epsilon} \right) \right)}{\epsilon d_h^{\frac{1}{2}}}\right) \cdot o(1) \\
&= |\langle\boldsymbol{q}_+^{(T_2)},\boldsymbol{k}_-^{(T_2)} \rangle| + o\left(\frac{N (\log(6N^2M^2/\delta))^3 \log \left( O\left( \frac{1}{\epsilon} \right) \right)}{\epsilon d_h^{\frac{1}{2}}}\right) + o\left(\frac{N \cdot \text{SNR}^2 (\log(6N^2M^2/\delta))^3 \log \left( O\left( \frac{1}{\epsilon} \right) \right)}{\epsilon d_h^{\frac{1}{2}}}\right) \\
&= o(1),
\end{aligned}
\end{equation}

where the first inequality is by triangle inequality, the second inequality is by results in \ref{sec:uprule}, the last equality is by $|\langle\boldsymbol{q}_+^{(T_2)},\boldsymbol{k}_-^{(T_2)} \rangle| = o(1)$ and $d_h = \widetilde{\Omega}\left(\max\{\text{SNR}^4, \text{SNR}^{-4}\}N^2\epsilon^{-2}\right)$. Similarly we have $|\langle\boldsymbol{q}_-^{(t+1)},\boldsymbol{k}_+^{(t+1)} \rangle| = o(1)$ and $|\langle\boldsymbol{q}_{n,i}^{(t+1)},\boldsymbol{k}_{\bar{n},j}^{(t+1)} \rangle|=o(1)$.

\begin{lemma}[Convergence of Training Loss, Lemma D.7 in \citep{jiang2024unveil}]\label{sec:convergence}
There exist $T = \frac{C_{19}}{\eta \epsilon (\|\boldsymbol{\mu}\|_2+\tau)^2 \|\boldsymbol{w}_O\|_2^2}$ such that
\begin{equation}
L_S(\theta(T)) \leq \epsilon
\end{equation}

\begin{proof}
    As we have the same conditions at the end of stage III as \citep{jiang2024unveil}, thus we have:
    Substituting $t = T = \frac{C_{19}}{\eta \epsilon (\|\boldsymbol \mu\|_2+\tau)^2 \|\boldsymbol w_O\|_2^2}$ into propositions $\mathcal{F}(t)$ and get
\begin{align*}
V_+^{(t)} &\geq \log \left( \exp(V_+^{(T_2)}) + \eta C_{17} (\|\boldsymbol \mu\|_2-\tau)^2 \|\boldsymbol w_O\|_2^2 (t - T_2) \right) \\
&\geq \log \left( \exp(V_+^{(T_2)}) + \frac{C_{20}}{\epsilon} \right) \\
&\geq \log \left( \frac{C_{20}}{\epsilon} \right),
\end{align*}

\[
|V_{n,i}^{(t)}| = O(1).
\]
Thus, for  $n \in S_+$, { we bound } $f(\widetilde{\boldsymbol X}_n, \theta(t))$ { as follows} :
\begin{equation}
    f(\widetilde{\boldsymbol X}_n,\theta)=\tfrac{1}{M} \sum_{l=1}^{M} \varphi(\widetilde{\boldsymbol x}_{n,l}^{\top} \mathbf{W}_Q \mathbf{W}_K^{\top} (\widetilde{\boldsymbol X}_n^{})^{\top}) \widetilde{\boldsymbol X}_n \mathbf{W}_V\boldsymbol{w}_O\geq \log\left(\frac{1}{\epsilon}\right)
\end{equation}
And
\begin{align*}
\widetilde{\ell}_n^{(t)} &= \log \left( 1 + \exp(-f(\widetilde{\boldsymbol X}_n, \theta(t))) \right) \\
&\leq \exp(-f(\widetilde{\boldsymbol X}_n, \theta(t))) \\
&\leq \exp \left( -\log \left( \frac{1}{\epsilon} \right) \right) \\
&\leq \epsilon.
\end{align*}
Similarly, we have $\widetilde{\ell}_n^{(t)} \leq \epsilon$ for $n \in S_-$. Therefore, we have 
\[
L_S(\theta(T)) = \frac{1}{N} \sum_{n=1}^{N} \widetilde{\ell}_n^{(t)} \leq \epsilon.
\]

\end{proof}
\end{lemma}
\subsection{Test error}\label{sec:test_error}
In this section, we denote clean $V_+,V_-$ and $V_{\boldsymbol{\xi}}$ as $\boldsymbol{\mu}_+^{\top}\mathbf{W}_V{\boldsymbol{w}_O},\boldsymbol{\mu}_-^{\top}\mathbf{W}_V{\boldsymbol{w}_O}$ and $\boldsymbol{\xi}^{\top}\mathbf{W}_V{\boldsymbol{w}_O}$, and perturbed $\widetilde{V}_+,\widetilde{V}_-$ and $\widetilde{V}_{\boldsymbol{\xi}}$ as $\widetilde{\boldsymbol{\mu}}_+^{\top}\mathbf{W}_V{\boldsymbol{w}_O},\widetilde{\boldsymbol{\mu}}_-^{\top}\mathbf{W}_V{\boldsymbol{w}_O}$ and $\widetilde{\boldsymbol{\xi}}^{\top}\mathbf{W}_V{\boldsymbol{w}_O}$.
\subsubsection{Clean test error}
\begin{theorem}\label{thm:2}
Under Assumption \ref{ass:1}, in the theoretical analysis of test error in the second and third stages of benign overfitting, we define $g(\boldsymbol{\xi})$ as 
\(V_{\boldsymbol{\xi}}^{(t)} =  \left\langle\boldsymbol{\xi},  \mathbf{W}_{V}^{(t)}{\boldsymbol{w}_O} \right\rangle \). Then, we know that for any $x \geq 0$, if $g : \mathbb{R}^n \rightarrow \mathbb{R}$ is a Lipschitz function and $c$ is a constant, the following inequality holds for the test loss.
\begin{equation*}
\mathbb{P}(\sum_{j=2}^M \alpha_i (g(\boldsymbol{\xi}_j) - \mathbb{E}g(\boldsymbol{\xi}_j) )\geq x) \leq \exp \left( -\frac{cx^2}{\sigma_p^2(\sum_{j=2}^M \alpha_j^2)  \left\| \mathbf{W}_{V}^{(t)} {\boldsymbol{w}_O} \right\|_2^2} \right). 
\end{equation*}
    
\end{theorem} 

\begin{proof}

 According to Theorem 5.2.2 in ~\citet{vershynin2018high}, we know that for any $x \geq 0$, if $g : \mathbb{R}^n \rightarrow \mathbb{R}$ is a Lipschitz function, it holds that
\begin{equation}\label{eq:con1}
    \mathbb{P}(\sum_{i=1}^N \alpha_i (g(\boldsymbol{\xi}_i) - \mathbb{E}g(\boldsymbol{\xi}_i) ) \geq x) \leq \exp \left( -\frac{cx^2}{\sigma_p^2(\sum_{i=1}^N \alpha_i^2) \|g\|_{\text{Lip}}^2} \right)
\end{equation}
where $g(\boldsymbol{\xi})$ is defined as \(|V_{\boldsymbol{\xi}}^{(t)}| =  |\langle\boldsymbol{\xi},  \mathbf{W}_{V}^{(t)}{\boldsymbol{w}_O} \rangle |\), we have
\begin{align*}
|g(\boldsymbol{\xi}) - g(\boldsymbol{\xi}')| &= \left|  \left|\langle\boldsymbol{\xi},  \mathbf{W}_{V}^{(t)}{\boldsymbol{w}_O} \rangle\right| - \left|\langle\boldsymbol{\xi}',  \mathbf{W}_{V}^{(t)}{\boldsymbol{w}_O} \rangle \right| \right| \\
&\leq  \left| \left\langle\boldsymbol{\xi}-\boldsymbol{\xi}',  \mathbf{W}_{V}^{(t)}{\boldsymbol{w}_O} \right\rangle \right| \\
&\leq  \|\mathbf{W}_{V}^{(t)}{\boldsymbol{w}_O}\|_2 \|\boldsymbol{\xi} -\boldsymbol{\xi}'\|_2
\end{align*}
So, we can get
\begin{equation}\label{eq:glip}
\|g\|_{\text{Lip}} \leq  \left\| \mathbf{W}_{V}^{(t)}{\boldsymbol{w}_O} \right\|_2.     
\end{equation} 
By plugging (\ref{eq:glip}) into (\ref{eq:con1}), we can get the result.  
\end{proof}

The following inequality holds according to the update rules of the $V$ vector in Lemma~\ref{lem:update_v}, the first equality is derived from the update of triangle inequality, and the second equality is due to the initialization of the $V$ vector.
\begin{equation}
\begin{aligned}
\left\| \mathbf{W}_V^{(t)} {\boldsymbol{w}_O} \right\|_2 &\leq \left\| \mathbf{W}_V^{(0)} {\boldsymbol{w}_O} \right\|_2 + \sum_{t'=0}^{t-1} \left\| \mathbf{W}_V^{(t'+1)}{\boldsymbol{w}_O} - \mathbf{W}_V^{(t')}{\boldsymbol{w}_O} \right\|_2 \\
&= \left\| \mathbf{W}_V^{(0)} {\boldsymbol{w}_O} \right\|_2 + t O \left(  \eta \cdot \max \left\{ \left\|\boldsymbol{\mu} \right\|_2, \sigma_p \sqrt{d} \right\} \cdot \left\| {\boldsymbol{w}_O} \right\|^2 \right) \\
&= O \left( \sigma_V \left\| {\boldsymbol{w}_O} \right\|_2 \sqrt{d} + t \eta \left\| {\boldsymbol{w}_O} \right\|^2 \max \left\{ \left\|\boldsymbol{\mu} \right\|_2, \sigma_p \sqrt{d} \right\} \right)\\
&\leq O \left(  t \eta \left\| {\boldsymbol{w}_O} \right\|^2 \max \left\{ \left\|\boldsymbol{\mu} \right\|_2, \sigma_p \sqrt{d} \right\} \right)
\end{aligned}
\end{equation}

Since $g(\boldsymbol{\xi})$ as \(|V_{\boldsymbol{\xi}}^{(t)}| =  |\langle\boldsymbol{\xi},  \mathbf{W}_{V}^{(t)}{\boldsymbol{w}_O} \rangle |\), and since $\left\langle\boldsymbol{\xi},  \mathbf{W}_{V}^{(t)}{\boldsymbol{w}_O} \right\rangle \sim \mathcal{N}(0, \|\mathbf{W}_{V}^{(t)}{\boldsymbol{w}_O}\|_2^2 \sigma_p^2)$, so we can get:
\begin{equation*}
\mathbb{E} g(\boldsymbol{\xi}) =  \mathbb{E} |\langle\boldsymbol{\xi},  \mathbf{W}_{V}^{(t)}{\boldsymbol{w}_O} \rangle |=\sqrt{\frac{2}{\pi}}\|\mathbf{W}_{V}^{(t)}{\boldsymbol{w}_O}\|_2\sigma_p
\end{equation*}

The test error can be interpreted as the probability that the noise term dominates the signal term. 
Formally, this corresponds to the event that the cumulative contribution of the random perturbation 
exceeds the deterministic signal margin.
After centralization, we can apply Theorem~\ref{thm:2} to obtain a high-probability upper bound on the test error.

\begin{align*}
P(y(f(\theta, \mathbf{X})) \leq 0) &= P \left(  \left[ (\sum_{i=1}^MS_{i,1})  (V_{+}^{(t)} -  V_{-}^{(t)}) + \sum_{j=2}^M((\sum_{i=1}^MS_{i,j} )V_{\boldsymbol{\xi}_j}^{(t)} )\right]\leq0 \right)\\
&\leq P\left(\sum_{j=2}^M((\sum_{i=1}^MS_{i,j} )|V_{\boldsymbol{\xi}_j}^{(t)}| ) \geq (\sum_{i=1}^MS_{i,1})  \left(V_{+}^{(t)} - V_{-}^{(t)}\right)\right) \\
&= P\left(\sum_{j=2}^M \alpha_j(g(\boldsymbol{\xi}_j) - E\{g(\boldsymbol{\xi}_j)\} \geq  \alpha_1 \left(V_{+}^{(t)} - V_{-}^{(t)}\right)- \sigma_p\sqrt{\frac{2}{{\pi}} }(\sum_{j=2}^M\alpha_j)\ \left\|\mathbf{W}_{V}^{(t)} {\boldsymbol{w}_O}\right\|_2\right) \\
&\leq \exp\left[-\frac{c_2 \left( \alpha_1\left(V_{+}^{(t)} - V_{-}^{(t)}\right) - \sigma_p\sqrt{\frac{2}{{\pi}} }(\sum_{j=2}^M\alpha_j)\ \left\|\mathbf{W}_{V}^{(t)} {\boldsymbol{w}_O}\right\|_2\right)^2}{\sigma_p^2 (\sum_{j=2}^M \alpha_j^2) \left( \left\|\mathbf{W}_{V}^{(t)} {\boldsymbol{w}_O}\right\|_2\right)^2}\right] \\
&\leq \exp\left(\frac{c_4}{\pi}\right) \cdot \exp\left[-\frac{c_5}{2} \left(\frac{ \alpha_1\left(V_{+}^{(t)} - V_{-}^{(t)}\right)}{\sigma_p \sqrt{\sum_{j=2}^M \alpha_j^2} \left\|\mathbf{W}_{-V}^{(t)} {\boldsymbol{w}_O}\right\|_2}\right)^2\right]
\end{align*}

We denote $\alpha_j$ as $\sum_iS_{i,j}$, where $S_{i,j}$ is $softmax(\langle \mathbf{q}_i^{(t)}, \mathbf{k}_j^{(t)} \rangle)$. When the subscript is 1, it represents the signal, and when the subscript is from 2 to M, it represents noise. 

Then, by the lower bound of $V_\pm^{(t)}$ and upper bound of $\left\| \mathbf{W}_V^{(t)} {\boldsymbol{w}_O} \right\|_2$, we can further bound the test error with following inequality:
\begin{align*}
P(y(f(\theta, \mathbf{X}) \leq 0) &\leq \exp\left(\frac{c_4}{\pi}\right) \cdot \exp\left[-\frac{c_5}{2} \left(\frac{ \alpha_1\left(V_{+}^{(t)} - V_{-}^{(t)}\right)}{\sigma_p \sqrt{\sum_{j=2}^M \alpha_j^2} \left\|\mathbf{W}_{V}^{(t)} {\boldsymbol{w}_O}\right\|_2}\right)^2\right] \\
&\leq \exp\left(\frac{c_{12}}{\pi}\right) \exp\left[-\frac{c_{13}}{2} O\left(\frac{V_+^{(t)}-V_-^{(t)}}{\frac{\sigma_p(V_+^{(t)}-V_-^{(t)})}{ \|\boldsymbol{\mu}\|_2}}\right)^2\right]\\
&=\exp\left(\frac{c_{12}}{\pi}\right) \exp\left[-\frac{c_{13}}{2} O\left(d SNR^{2}\right)\right]
\end{align*}
where the second inequality is by $\mathbf{W}_{V}^{(t)} {\boldsymbol{w}_O}$ is almost aligned with $V_+^{(t)}$ and $V_-^{(t)}$, thus $\|\mathbf{W}_{V}^{(t)} {\boldsymbol{w}_O}\|_2\sim\frac{V_+^{(t)}-V_-^{(t)}}{\|\boldsymbol{\mu}\|_2}$.
\subsubsection{Robust test error}
We start by writing the prediction score under perturbation $\widetilde{X}$ as
\begin{equation*}
    \begin{aligned}
        yf(\theta,\widetilde{X})
        =     \Big(\sum_{i=1}^M \widetilde{S}_{i1}\Big) \big(\widetilde{V}_{+}^{(t)} + \widetilde{V}_{-}^{(t)}\big) 
            + \sum_{j=2}^M \Big(\sum_{i=1}^M \widetilde{S}_{ij}\Big) \widetilde{V}_{\boldsymbol{\xi}_j}^{(t)},
    \end{aligned}
\end{equation*}

We first take the first term as an example and derive an upper bound on the maximum discrepancy between the perturbed input $\widetilde{X} \in B(X,\tau)$ and the clean input.

\begin{align*}
     S_{11}V_{+}^{(t)}-\widetilde{S}_{11}\widetilde{V}_{+}^{(t)} &=({S}_{11}-\widetilde{S}_{11})V_+^{(t)}+\widetilde{S}_{11}(V_{+}^{(t)}-\widetilde{V}_{+}^{(t)})
    \\&\leq (1-1/C)S_{11}{V}_+^{(t)}+\widetilde{S}_{11} |\langle \widetilde{\boldsymbol{\mu}}_+-\boldsymbol{\mu}_+, \mathbf{W}_{V}^{(t)}{\boldsymbol{w}_O} \rangle|\\
    &\leq(1-1/C)S_{11}{V}_+^{(t)}+\widetilde{S}_{11}\big\| \mathbf{W}_{V}^{(t)} {\boldsymbol{w}_O} \big\|  \tau,
\end{align*}
where the inequality comes from Lemma~\ref{lemma:rsoftmax}, and the definition of $V_+^{(t)}$. This shows that the deviation can be bounded linearly in the perturbation magnitude, with only a small residual term since $(C-1)=o(1)$. 
Similarly, an analogous upper bound holds for $\widetilde{S}_{i1}\widetilde{V}_{\pm}^{(t)} - S_{i1}V_{\pm}^{(t)}$ and $\widetilde{S}_{ij}\widetilde{V}_{\boldsymbol{\xi}_j}^{(t)} - S_{ij}V_{\boldsymbol{\xi}_j}^{(t)}$ for $i \in [M], \, j \in [M]\backslash\{1\}$.

Aggregating the deviations across all components, the worst-case perturbation satisfies
\begin{equation}\label{eq:max_mar}
    \begin{aligned}
        & y f(\theta,X)-\min_{\widetilde{X}\in B(X,\tau)} y f(\theta,\widetilde{X})  \\&\quad\quad\quad\leq  
            \Big(\sum_i {S}_{i1}\Big)  \|\mathbf{W}_{V}^{(t)}{\boldsymbol{w}_O}\| \tau 
            + \sum_{j=2}^M \Big(\sum_i {S}_{ij}\Big) \|\mathbf{W}_{V}^{(t)}{\boldsymbol{w}_O}\| \tau +(1-1/C)S_{11}(\widetilde{V}_+^{(t)}+\widetilde{V}_+^{(t)})+o(1)
         \\
        &\quad\quad\quad\lesssim \Bigg(\sum_{j=1}^M \sum_i {S}_{ij} \Bigg)
         \big\| \mathbf{W}_{V}^{(t)} {\boldsymbol{w}_O} \big\|  \tau+(1-1/C)\Big(\sum_i {S}_{i1}\Big)(\widetilde{V}_+^{(t)}+\widetilde{V}_-^{(t)})\\
         &\quad\quad\quad=M\big\| \mathbf{W}_{V}^{(t)} {\boldsymbol{w}_O} \big\|  \tau+(1-1/C)\Big(\sum_i {S}_{i1}\Big)(\widetilde{V}_+^{(t)}+\widetilde{V}_-^{(t)})
    \end{aligned}
\end{equation}

where the first inequality comes from $\widetilde{V}_{\boldsymbol{\xi}_j}^{(t)}=o(1)$ for $j\in[M]\backslash\{1\}$. Thus the adversarial effect scales with both the cumulative magnitude of the perturbed coefficients and the operator norm of the weight matrices. Then we can bound the robust test error:

\begin{align*}
P&\Bigg(\min_{\widetilde{X}\in B(X,\tau)} y f(\theta,\widetilde{X}) \leq 0\Bigg) = P\Bigg(y f(\theta,X) 
+ \big(y f(\theta,X)-\min_{\widetilde{X}\in B(X,\tau)} y f(\theta,\widetilde{X}) \big) \leq 0\Bigg) \\
&\leq P\Bigg(\sum_{j=2}^M \alpha_j \big(g(\boldsymbol{\xi}_j) - E\{g(\boldsymbol{\xi}_j)\}\big) 
\geq \alpha_1 \big(V_{+}^{(t)} - V_{-}^{(t)}\big) 
- \sigma_p \sqrt{\tfrac{2}{\pi}} \Big(\sum_{j=2}^M \alpha_j\Big) \big\|\mathbf{W}_{V}^{(t)} {\boldsymbol{w}_O}\big\|_2
\\&\quad\quad\quad- M \big\|\mathbf{W}_{V}^{(t)} {\boldsymbol{w}_O}\big\|_2 \tau-(1-1/C)\alpha_1 ({V}_+^{(t)}-{V}_-^{(t)}) \Bigg) \\
&\leq \exp\!\left(\tfrac{c_{12}}{\pi}\right) 
   \exp\!\left[-\tfrac{c_{13}}{2} 
   \left(\frac{\alpha_1 (V_{+}^{(t)} - V_{-}^{(t)})}
   {\sigma_p \sqrt{\sum_{j=2}^M \alpha_j^2} \|\mathbf{W}_{V}^{(t)} {\boldsymbol{w}_O}\|_2}
   - \frac{M\|\mathbf{W}_{V}^{(t)} {\boldsymbol{w}_O}\|\tau}
   {\sigma_p \sqrt{\sum_{j=2}^M \alpha_j^2}\|\mathbf{W}_{V}^{(t)} {\boldsymbol{w}_O}\|_2} \right)^2\right] \\
    &\leq \exp\!\left(\tfrac{c_{12}}{\pi}\right) 
   \exp\!\left[-\tfrac{c_{13}}{2} \, O\!\left(\sqrt{d}SNR(1-\frac{\tau}{\|\boldsymbol{\mu}\|_2})\right)^2 \right] \\
\end{align*}

The first inequality follows from (\ref{eq:max_mar}). 
The third inequality uses the fact that $C \leq e/2$, which we absorb into the constant term. 
The last inequality follows from the bound on $V_{\pm}^{(t)}$ and  $\|\mathbf{W}_{V}^{(t)} {\boldsymbol{w}_O}\|_2$.

This completes the proof.
\section{Benign Overfitting in Case 2}
Stage I stay same with ~\ref{sec:s1}. Next, we aim to prove that under condition $\tau=(1-o(1))\|\boldsymbol{\mu}\|$, the attention component in a ViT will remain in its initialization state and fail to learn meaningful signal-to-signal or noise-to-signal interactions. This is because, under such perturbations, any newly emerging margin can be immediately neutralized, preventing the signal-to-signal attention from accumulating advantages. In this regime, the ViT effectively degenerates into a linear model.

At the end of Stage~I, since $V_{+}^{T_1} \geq 3M |V_{n,i}^{T_1}|$ and $q = o(1)$ at this point, the perturbation has no significant effect. Consequently, $\langle\boldsymbol{q}_+,\boldsymbol{k}_+ \rangle$ and$\langle\boldsymbol{q}_{n,i},\boldsymbol{k}_+ \rangle$ experiences a temporary increase, but it does not exceed $\Theta(\log C)$ (as we will demonstrate later). Therefore, we assume that after Stage~II, when $\langle q, k \rangle$ has stabilized and the loss derivatives are no longer at the $o(1)$ scale. However, $|V_{n,i}^{T_3}|$ is not $o(1)$; it is of the same order as $|V_{+}^{T_3}|$, which implies that a larger SNR is required. The following conditions hold at the beginning of Stage~III.
\begin{align*}
    |V_+^{(T_2)}|, |V_-^{(T_2)}|,& |V_{n,i}^{(T_2)}| = o(1), \\
    V_+^{(T_2)} &\geq 3M \cdot |V_{n,i}^{(T_2)}|, \\
    V_-^{(T_2)} &\leq -3M \cdot |V_{n,i}^{(T_2)}|, \\
    \|\boldsymbol{q}_+^{(T_2)}\|_2^2, \|\boldsymbol{k}_+^{(T_2)}\|_2^2 &= \Theta(\log{C}), \\
    \|\boldsymbol{q}_{n,i}^{(T_2)}\|_2^2, \|\boldsymbol{k}_{n,i}^{(T_2)}\|_2^2 &= \Theta(\log{C}),
\end{align*}
\[
|\langle\boldsymbol{q}_+^{(T_2)},\boldsymbol{q}_-^{(T_2)} \rangle|, |\langle\boldsymbol{q}_+^{(T_2)},\boldsymbol{q}_{n,i}^{(T_2)} \rangle|, |\langle\boldsymbol{q}_{n,i}^{(T_2)},\boldsymbol{q}_{n',j}^{(T_2)} \rangle| = o(1),
\]
\[
|\langle\boldsymbol{k}_+^{(T_2)},\boldsymbol{k}_-^{(T_2)} \rangle|, |\langle\boldsymbol{k}_+^{(T_2)},\boldsymbol{k}_{n,i}^{(T_2)} \rangle|, |\langle\boldsymbol{k}_{n,i}^{(T_2)},\boldsymbol{k}_{n',j}^{(T_2)} \rangle| = o(1),
\]
for $i, j \in [M]\backslash\{1\}, n, n' \in [N], i \neq j \text{ or } n \neq n'$.

\[
|\langle\boldsymbol{q}_+^{(T_2)},\boldsymbol{k}_+^{(T_2)} \rangle|, |\langle\boldsymbol{q}_+^{(T_2)},\boldsymbol{k}_{n,j}^{(T_2)} \rangle|, |\langle\boldsymbol{q}_{n,i}^{(T_2)},\boldsymbol{k}_+^{(T_2)} \rangle|, |\langle\boldsymbol{q}_{n,i}^{(T_2)},\boldsymbol{k}_{n',j}^{(T_2)} \rangle| =\Theta(\log{C})
\]
\[
|\langle\boldsymbol{q}_+^{(T_2)},\boldsymbol{k}_-^{(T_2)} \rangle|, |\langle\boldsymbol{q}_{n,i}^{(T_2)},\boldsymbol{k}_{\overline{n},j}^{(T_2)} \rangle| = o(1)
\]
for $i, j \in [M]\backslash\{1\}, n, \overline{n} \in [N], n \neq \overline{n}$.

Let $T_3 = \Theta\left(\frac{M}{\eta \epsilon (\|\boldsymbol{\mu}\|_2+\tau)^2 \|\boldsymbol{w}_O\|_2^2}\right)$. Next we prove the following four propositions $\mathcal{J}(t)$, $\mathcal{K}(t)$, $\mathcal{L}(t)$ by induction on $t$ for $t \in [T_2, T_3]$:

\begin{itemize}
    \item $\mathcal{J}(t)$:
    \begin{align*}
        &\qquad\qquad\qquad\qquad\qquad\qquad\qquad\qquad  V_+^{(t)} \geq 3M \cdot |V_{n,i}^{(t)}|, \\
        &\qquad\qquad\qquad\qquad\qquad\qquad\qquad\qquad V_-^{(t)} \leq -3M \cdot |V_{n,i}^{(t)}|, \\
        &\qquad\qquad\qquad\qquad\qquad\qquad\qquad\qquad\quad |V_{n,i}^{(t)}| = o(1), \\
        &\log\left( \exp(V_+^{(T_2)})  + \frac{\eta}{M} C_{17} (\|\boldsymbol{\mu}\|_2-\tau)^2 \|\boldsymbol{w}_O\|_2^2 (t - T_2)\right) \leq V_+^{(t)}\leq 2 \log\left(O\left(\frac{1}{\epsilon}\right)\right), \\
        &-2 \log\left(O\left(\frac{1}{\epsilon}\right)\right) \leq V_-^{(t)} \leq -\log\left(\exp(-V_-^{(T_2)}) + \frac{\eta}{M} C_{17} (\|\boldsymbol{\mu}\|_2-\tau)^2 \|\boldsymbol{w}_O\|_2^2 (t - T_2)\right)
    \end{align*}
    for $i \in [M]\backslash\{1\}, n \in [N]$.
    \item $\mathcal{K}(t)$:
    \begin{align*}
        \|\boldsymbol q_\pm^{(t)}\|_2^2, \|\boldsymbol k_\pm^{(t)}\|_2^2 &= \Theta(logC), \\
        \|\boldsymbol{q}_{n,i}^{(t)}\|_2^2, \|\boldsymbol{k}_{n,i}^{(t)}\|_2^2 &= \Theta\left(logC\right), \\
        |\langle\boldsymbol{q}_+^{(t)},\boldsymbol{q}_-^{(t)} \rangle|, |\langle\boldsymbol q_\pm^{(t)},\boldsymbol{q}_{n,i}^{(t)} \rangle|&, |\langle\boldsymbol{q}_{n,i}^{(t)},\boldsymbol{q}_{n',j}^{(t)} \rangle| = o(1), \\
        |\langle\boldsymbol{k}_+^{(t)},\boldsymbol{k}_-^{(t)} \rangle|, |\langle\boldsymbol k_\pm^{(t)},\boldsymbol{k}_{n,i}^{(t)} \rangle|&, |\langle\boldsymbol{k}_{n,i}^{(t)},\boldsymbol{k}_{n',j}^{(t)} \rangle| = o(1)
    \end{align*}
    for $i, j \in [M]\backslash\{1\}, n, n' \in [N], i \neq j \text{ or } n \neq n'$,  $C=O(1)$.
    
    \item $\mathcal{L}(t)$:
    \begin{align*}
        |\langle\boldsymbol q_\pm^{(t)},\boldsymbol k_\pm^{(t)} \rangle|, |\langle\boldsymbol q_\pm^{(t)},\boldsymbol{k}_{n,j}^{(t)} \rangle|, |\langle\boldsymbol{q}_{n,i}^{(t)},\boldsymbol k_\pm^{(t)} \rangle|, |\langle\boldsymbol{q}_{n,i}^{(t)},\boldsymbol{k}_{n,j}^{(t)} \rangle| &=\Theta(\log{C}), \\
        |\langle\boldsymbol q_\pm^{(t)},\boldsymbol k_\mp^{(t)} \rangle|, |\langle\boldsymbol{q}_{n,i}^{(t)},\boldsymbol{k}_{\overline{n},j}^{(t)} \rangle| &= o(1)
    \end{align*}
    for $i, j \in [M]\backslash\{1\}, n, n' \in [N], n \neq \overline{n}$,  $C=O(1)$.
\end{itemize}

By the results of Stage II, we know that $\mathcal{F}(T_1)$, $\mathcal{G}(T_2)$, $\mathcal{I}(T_2)$ are true. To prove that $\mathcal{F}(t)$, $\mathcal{G}(t)$, $\mathcal{H}(t)$ and $\mathcal{I}(t)$ are true in stage III, we will prove the following claims holds for $t \in [T_2, T_3]$:
\begin{itemize}
    \item Claim 9. $\mathcal{L}(T_2), \ldots, \mathcal{L}(t) \implies \mathcal{J}(t+1)$
    \item Claim 10. $\mathcal{J}(t), \mathcal{L}(t),  \mathcal{K}(t) \implies \mathcal{K}(t+1)$
    \item Claim 11. $\mathcal{J}(t),  \mathcal{K}(t),  \mathcal{L}(t),  \implies \mathcal{L}(t+1)$
\end{itemize}

\subsection{Proof of Claim 9}
The proofs for $V_+^{(t)} \geq 3M \cdot |V_{n,i}^{(t)}|$ and $V_-^{(t)} \leq -3M \cdot |V_{n,i}^{(t)}|$ are the same as for ~\ref{sec:cl1}.

 we provide the bounds for $-\widetilde{\ell}_n^{'(s)}$.  
Note that $\ell(z) = \log(1 + \exp(-z))$ and $-\ell'(z) = \exp(-z)/(1 + \exp(-z))$.  
Without loss of generality, assume $y_n = 1$. We have
\begin{equation}
\begin{aligned}
   -\widetilde{\ell}'(f(\widetilde{\boldsymbol X}_n, \theta(s))) 
   &= \frac{1}{1 + \exp\!\left(\tfrac{1}{M} \sum_{l=1}^{M} \varphi(\widetilde{\boldsymbol x}_{n,l}^\top \mathbf{W}_Q^{(s)} \mathbf{W}_K^{(s)\top} (\widetilde{\boldsymbol X}_n)^\top) \widetilde{\boldsymbol X}_n \mathbf{W}_V^{(s)}\boldsymbol{w}_O\right)} \\
   &= \frac{1}{M} \Big( \text{softmax}(\langle \widetilde{\boldsymbol q}_\pm^{(s)}, \widetilde{\boldsymbol k}_\pm^{(s)} \rangle) 
      + \sum_{l=2}^{M} \text{softmax}(\langle \widetilde{\boldsymbol{q}}_{n,l}^{(s)}, \widetilde{\boldsymbol k}_\pm^{(s)} \rangle) \Big) 
      \cdot \widetilde{\boldsymbol{\mu}}_+^\top \mathbf{W}_V^{(s)}\boldsymbol{w}_O \\
   &\quad + \sum_{j \in [M]\backslash\{1\}} 
      \Big( \text{softmax}(\langle \widetilde{\boldsymbol{q}}_{n,j}^{(s)}, \widetilde{\boldsymbol{k}_{n,j}}^{(s)} \rangle) 
      + \sum_{l=2}^{M} \text{softmax}(\langle \widetilde{\boldsymbol{q}}_{n,l}^{(s)}, \widetilde{\boldsymbol{k}}_{n,j}^{(s)} \rangle) \Big) 
      \cdot \widetilde{\boldsymbol{\xi}}_{n,j}^\top \mathbf{W}_V^{(s)}\boldsymbol{w}_O \\
   &= \frac{1}{M} \Big( M \cdot \frac{C}{C+M-1} \cdot V_+^{(s)} 
      + M \cdot \frac{1}{C+M-1} \cdot \sum_{j \in [M]\backslash\{1\}} V_{n,i}^{(s)} \Big) \\
   &\geq \tfrac{1}{2M} V_+^{(s)},
\end{aligned}
\end{equation}
for $s \in [T_2, t]$,  
where the second equality is by $\mathcal{L}(t)$ that $\Lambda_{n,\pm,j}^{(s)}=\Theta(\log{C})$ and $\Lambda_{n,i,\pm,j}^{(s)}=\Theta(\log{C})$, and the inequality follows from $V_+^{(s)} \geq 3M \cdot |V_{n,i}^{(s)}|$.  

Similarly, we have
\begin{equation}
\begin{aligned}
   \tfrac{1}{M} \sum_{l=1}^{M} \varphi(\widetilde{\boldsymbol x}_{n,l}^\top \mathbf{W}_Q^{(s)} \mathbf{W}_K^{(s)\top} (\widetilde{\boldsymbol X}_n)^\top) \widetilde{\boldsymbol X}_n \mathbf{W}_V^{(s)}\boldsymbol{w}_O
   &\leq \max_{i \in [M]\backslash\{1\}} \{V_+^{(s)}, V_{n,i}^{(s)}\} \\
   &= V_+^{(s)}.
\end{aligned}
\end{equation}
Then, we have
\begin{equation}\label{eq:75_2}
\begin{aligned}
    -\widetilde{\ell}'(f(\widetilde{\boldsymbol X}_n, \theta(s))) &= \frac{1}{1 + \exp\left(\tfrac{1}{M} \sum_{l=1}^{M} \varphi(\widetilde{\boldsymbol x}_{n,l}^\top \mathbf{W}_Q^{(s)} \mathbf{W}_K^{(s)\top} (\widetilde{\boldsymbol X}_n)^\top) \widetilde{\boldsymbol X}_n \mathbf{W}_V^{(s)}\boldsymbol{w}_O\right)} \\
    &\geq \frac{1}{1 + \exp(V_+^{(s)})} \\
    &\geq \frac{C_{16}}{\exp(V_+^{(s)})}
\end{aligned}
\end{equation}
For the last inequality, note that $V_+^{(T_2)} \geq 0$ and $V_+^{(s)}$ is monotonically increasing, so there exist a constant $C_{16}$ such that $\frac{1}{1 + \exp(V_+^{(s)})} \geq \frac{C_{16}}{\exp(V_+^{(s)})}$. We also have the upper bound
\begin{equation}\label{eq:76_2}
\begin{aligned}
    -\ell'(f(X_n, \theta(s))) &= \frac{1}{1 + \exp\left(\frac{1}{M} \sum_{l=1}^{M} \varphi(x_{n,l}^\top \mathbf{W}_Q^{(s)} \mathbf{W}_K^{(s) \top} (X_n)^\top) X_n \mathbf{W}_V^{(s)}\boldsymbol{w}_O\right)} \\
    &\leq \frac{1}{1 + \exp(V_+^{(s)}/2M)} \\
    &\leq \frac{1}{\exp(V_+^{(s)}/2M)}
\end{aligned}
\end{equation}

By the update rule of $\gamma_{V,+}^{(t)}$ and in Lemma~\ref{lem:update_v} and get
\begin{equation}
\begin{aligned}
    \gamma_{V,+}^{(s+1)} - \gamma_{V,+}^{(s)} &=- \frac{\eta \langle\widetilde{\boldsymbol{\mu}}_+,\widetilde{\boldsymbol{\mu}}_+^{(t)}\rangle}{NM} \sum_{n \in S_+} \widetilde{\ell}_n'^{(t)} \left( \frac{\exp(\langle \widetilde{\mathbf{q}}_+^{(t)}, \widetilde{\mathbf{k}}_+^{(t)} \rangle)}{\exp(\langle \widetilde{\mathbf{q}}_+^{(t)}, \widetilde{\mathbf{k}}_+^{(t)} \rangle) + \sum_{k=2}^M \exp(\langle \widetilde{\mathbf{q}}_+^{(t)}, \widetilde{\mathbf{k}}_{n,k}^{(t)} \rangle)} \right. \nonumber \\
    &\quad\quad\quad\quad\quad\quad\quad\quad\quad \left. + \sum_{j=2}^M \frac{\exp(\langle \widetilde{\mathbf{q}}_{n,j}^{(t)}, \widetilde{\mathbf{k}}_+^{(t)} \rangle)}{\exp(\langle \widetilde{\mathbf{q}}_{n,j}^{(t)}, \widetilde{\mathbf{k}}_+^{(t)} \rangle) + \sum_{k=2}^M \exp(\langle \widetilde{\mathbf{q}}_{n,j}^{(t)}, \widetilde{\mathbf{k}}_{n,k}^{(t)} \rangle)} \right)  
    \\&+\sum_{n \in S_+}\widetilde{\ell}_n'^{(t)}\sum_{i=2}^M\frac{-\eta\langle\widetilde{\boldsymbol{\mu}}_+,\widetilde{\boldsymbol{\xi}}_{n,i}^{(t)}\rangle}{NM}\bigg(\frac{\exp(\langle \widetilde{\mathbf{q}}_+^{(t)}, \widetilde{\mathbf{k}}_{n,i}^{(t)} \rangle)}{\exp(\langle \widetilde{\mathbf{q}}_+^{(t)}, \widetilde{\mathbf{k}}_+^{(t)} \rangle) + \sum_{k=2}^M \exp(\langle \widetilde{\mathbf{q}}_+^{(t)}, \widetilde{\mathbf{k}}_{n,k}^{(t)} \rangle)} 
    \\&\quad\quad\quad\quad\quad\quad\quad\quad\quad+\sum_{j=2}^M\frac{\exp(\langle \widetilde{\mathbf{q}}_{n,j}^{(t)}, \widetilde{\mathbf{k}}_{n,i}^{(t)} \rangle)}{\exp(\langle \widetilde{\mathbf{q}}_{n,j}^{(t)}, \widetilde{\mathbf{k}}_+^{(t)} \rangle) + \sum_{k=2}^M \exp(\langle \widetilde{\mathbf{q}}_{n,j}^{(t)}, \widetilde{\mathbf{k}}_{n,k}^{(t)} \rangle)}\bigg) \\
    &\geq -\frac{\eta (\|\boldsymbol{\mu}\|_2-\tau)^2}{NM}\ \sum_{n \in S_+} \ell_n^{'(s)} (M \cdot \frac{C}{C+M-1})+\frac{\eta \tau\|\boldsymbol{\mu}\|_2}{NM}\ \sum_{n \in S_+} \ell_n^{'(s)} (M \cdot \frac{1}{C+M-1}) \\
    &\geq \frac{\eta (\|\boldsymbol{\mu}\|_2-\tau)^2}{NM} \cdot \frac{N}{4}  \cdot \frac{C_{16}}{\exp(V_+^{(s)})}-  \frac{\eta \tau\|\boldsymbol{\mu}\|_2}{NM} \cdot \frac{N}{4}  \cdot \frac{C_{16}}{\exp(V_+^{(s)}/2M)}\\
    &\geq  \frac{\eta C_{17}(\|\boldsymbol{\mu}\|_2-\tau)^2}{M} \frac{1}{\exp(V_+^{(s)})}
\end{aligned}
\end{equation}
where the second inequality is by (\ref{eq:75_2})(\ref{eq:76_2}), the last inequality is by $N\cdot SNR^2=\Omega(\frac{1}{\epsilon})$ Then by definition~\ref{def:sca_v}, we get
\begin{equation}
\begin{aligned}
    V_+^{(s+1)} - V_+^{(s)} &= (\gamma_{V,+}^{(s+1)} - \gamma_{V,+}^{(s)}) \|\boldsymbol{w}_O\|_2^2 \geq \frac{\eta C_{17} (\|\boldsymbol{\mu}\|_2-\tau)^2 \|\boldsymbol{w}_O\|_2^2}{M\exp(V_+^{(s)})} \\
    \end{aligned}
\end{equation}
Multiply both sides simultaneously by $\exp(V_+^{(s)})$ and get
\begin{equation}
\begin{aligned}
    \exp(V_+^{(s)})(V_+^{(s+1)} - V_+^{(s)}) &\geq \frac{\eta}{M} C_{17} (\|\boldsymbol{\mu}\|_2-\tau)^2 \|\boldsymbol{w}_O\|_2^2
\end{aligned}
\end{equation}
Taking a summation from $T_2$ to $t$ and get
\begin{equation}
\begin{aligned}
    \sum_{s=T_2}^{t} \exp(V_+^{(s)})(V_+^{(s+1)} - V_+^{(s)}) &\geq \sum_{s=T_2}^{t} \frac{\eta}{M} C_{17} (\|\boldsymbol{\mu}\|_2-\tau)^2 \|\boldsymbol{w}_O\|_2^2 \\
    &\geq \frac{\eta}{M} C_{17} (\|\boldsymbol{\mu}\|_2-\tau)^2 \|\boldsymbol{w}_O\|_2^2 (t - T_2 + 1)
\end{aligned}
\end{equation}

By the property that $V_+^{(s)}$ is monotonically increasing, we have
\begin{equation}
\begin{aligned}
    \int_{V_+^{(T_2)}}^{V_+^{(t+1)}} \exp(x) dx &\geq \sum_{s=T_2}^{t} \exp(V_+^{(s)})(V_+^{(s+1)} - V_+^{(s)}) \\
    &\geq \frac{\eta}{M} C_{17} (\|\boldsymbol{\mu}\|_2-\tau)^2 \|\boldsymbol{w}_O\|_2^2 (t - T_2 + 1)
\end{aligned}
\end{equation}

By $\int_{V_+^{(T_2)}}^{V_+^{(t+1)}} \exp(x) dx = \exp(V_+^{(t+1)}) - \exp(V_+^{(T_2)})$ we get
\begin{equation}
\begin{aligned}
    V_+^{(t+1)} &\geq \log \left( \exp(V_+^{(T_2)}) + \frac{\eta}{M} C_{17} (\|\boldsymbol{\mu}\|_2-\tau)^2 \|\boldsymbol{w}_O\|_2^2 (t - T_2 + 1) \right)   
\end{aligned}
\end{equation}

Similarly, we have
\begin{equation}
\begin{aligned}
    V_-^{(t+1)} &\leq -\log \left( \exp\left(V_-^{(T_2)} \right)  
    + \eta C_{17} (\|\boldsymbol{\mu}\|_2-\tau)^2 \|\boldsymbol{w}_O\|_2^2 (t - T_2 + 1)\right)
\end{aligned}
\end{equation}

Similarly, we have the lower bound, the proofs are the same as for ~\ref{sec:cl5} 

\subsection{Proof of Claim 10}\label{sec:cl10}
We first consider the increment of $\langle\boldsymbol q,\boldsymbol k\rangle$ at the $t$-th update when using the original clean data. We then show that, at this step, the effect introduced by the perturbation under adversarial samples dominates the increment learned from the clean data; hence $\langle\boldsymbol q,\boldsymbol k\rangle$ remains stable and bounded.

By the update rule of $\langle\boldsymbol q,\boldsymbol k\rangle$ we have
\begin{equation}
\begin{aligned}
    &\langle\boldsymbol{q}_+^{(t+1)},\boldsymbol{k}_+^{(t+1)} \rangle - \langle\boldsymbol{q}_+^{(t)},\boldsymbol{k}_+^{(t)} \rangle \\
    &= \alpha_{+,+}^{(t)} \|\boldsymbol{k}_+^{(t)}\|_2^2 + \beta_{+,+}^{(t)} \|\boldsymbol{q}_+^{(t)}\|_2^2 +\text{\{lower order term\}}
\end{aligned}
\end{equation}

Subsequently, we establish upper bounds for $\alpha$ and $\beta$ for clean data.
\begin{equation}\label{eq:103_2}
\begin{aligned}
    \alpha_{+,+}^{(t)} &= \frac{\eta}{NM} \sum_{n \in S_+} -\ell_n^{'(t)} \|\boldsymbol{\mu}\|_2^2 \\
    &\cdot \left( V_+^{(t)} \frac{\exp(\langle\boldsymbol{q}_+^{(t)},\boldsymbol{k}_+^{(t)} \rangle)}{\exp(\langle\boldsymbol{q}_+^{(t)},\boldsymbol{k}_+^{(t)} \rangle) + \sum_{j=2}^{M} \exp(\langle\boldsymbol{q}_+^{(t)},\boldsymbol{k}_{n,j}^{(t)} \rangle)} \right. \\
    &\quad - \left( \frac{\exp(\langle\boldsymbol{q}_+^{(t)},\boldsymbol{k}_+^{(t)} \rangle)}{\exp(\langle\boldsymbol{q}_+^{(t)},\boldsymbol{k}_+^{(t)} \rangle) + \sum_{j=2}^{M} \exp(\langle\boldsymbol{q}_+^{(t)},\boldsymbol{k}_{n,j}^{(t)} \rangle)} \right)^2 \\
    &\quad - \sum_{i=2}^{M} (V_{n,i}^{(t)} \cdot \frac{\exp(\langle\boldsymbol{q}_+^{(t)},\boldsymbol{k}_+^{(t)} \rangle)}{\exp(\langle\boldsymbol{q}_+^{(t)},\boldsymbol{k}_+^{(t)} \rangle) + \sum_{j=2}^{M} \exp(\langle\boldsymbol{q}_+^{(t)},\boldsymbol{k}_{n,j}^{(t)} \rangle)} \\
    &\quad \left. \cdot \frac{\exp(\langle\boldsymbol{q}_+^{(t)},\boldsymbol{k}_{n,i}^{(t)} \rangle)}{\exp(\langle\boldsymbol{q}_+^{(t)},\boldsymbol{k}_+^{(t)} \rangle) + \sum_{j=2}^{M} \exp(\langle\boldsymbol{q}_+^{(t)},\boldsymbol{k}_{n,j}^{(t)} \rangle)} \right)\\
    &\leq \frac{\eta}{NM} \sum_{n \in S_+} \|\boldsymbol{\mu}\|_2^2 (V_{+}^{(t)})\\
    &\leq  \frac{3\eta}{2M}  \|\boldsymbol{\mu}\|_2^2 V_{+}^{(t)}
\end{aligned}
\end{equation}

Then, we can then compute the incremental growth of $\langle\boldsymbol{q}_+^{(t)},\boldsymbol{k}_+^{(t)}\rangle$ after one update step on clean data.

\begin{equation}
\begin{aligned}
    &\langle\boldsymbol{q}_+^{(t+1)},\boldsymbol{k}_+^{(t+1)} \rangle - \langle\boldsymbol{q}_+^{(t)},\boldsymbol{k}_+^{(t)} \rangle \\
    &= \alpha_{+,+}^{(t)} \|\boldsymbol{k}_+^{(t)}\|_2^2 + \beta_{+,+}^{(t)} \|\boldsymbol{q}_+^{(t)}\|_2^2 +\text{\{lower order term\}}\\
    &\leq \frac{3\eta}{2M}  \|\boldsymbol{\mu}\|_2^2 V_{+}^{(t)}\Theta(\log{C})+\frac{3\eta}{2M}  \|\boldsymbol{\mu}\|_2^2 V_{+}^{(t)}\Theta(\log{C})
\end{aligned}
\end{equation}

Subsequently, we compute at the $t^{th}$ iteration the magnitude of the effect that the perturbation imposes on $\langle\boldsymbol{q}_+^{(t)},\boldsymbol{k}_+^{(t)}\rangle$.
\begin{equation}
\begin{aligned}
    &max_{\widetilde{X}^{(t)}\in B({X}^{(t)},\tau)}\langle \widetilde{\boldsymbol q}_+^{(t)}, \widetilde{\boldsymbol k}_+^{(t)} \rangle - \langle\boldsymbol{q}_+^{(t)},\boldsymbol{k}_+^{(t)} \rangle \\
    &=((1+\frac{\tau}{\|\boldsymbol{\mu}\|_2})^2-1)\langle\boldsymbol{q}_+^{(t)},\boldsymbol{k}_+^{(t)} \rangle\\
    &=((1+\frac{\tau}{\|\boldsymbol{\mu}\|_2})^2-1) \Theta(log{C})
\end{aligned}
\end{equation}

As $\tau$ and $\|\boldsymbol{\mu}\|$ are same order, and $\frac{3\eta}{2M}  \|\boldsymbol{\mu}\|_2^2 V_{+}^{(t)}=o(\frac{1}{N})$ by $V_{+}^{(t)}\leq2\log{\left(O(\frac{1}{\epsilon})\right)}$, $d_h = \widetilde{\Omega} \left( \max \{ \text{SNR}^4, \text{SNR}^{-4} \} N^2 \epsilon^{-2} \right)$ and $\eta\leq \widetilde{O}(\min \{ \|\boldsymbol{\mu}\|_2^{-2}, (\sigma_p^2 d)^{-1} \} \cdot d_h^{-\frac{1}{2}})$. Thus, we have 
\[\langle\boldsymbol{q}_+^{(t+1)},\boldsymbol{k}_+^{(t+1)} \rangle - \langle\boldsymbol{q}_+^{(t)},\boldsymbol{k}_+^{(t)} \rangle\leq max_{\widetilde{\boldsymbol X}^{(t)}\in B({\boldsymbol X}^{(t)},\tau)}\langle \widetilde{\boldsymbol q}_+^{(t)}, \widetilde{\boldsymbol k}_+^{(t)} \rangle - \langle\boldsymbol{q}_+^{(t)},\boldsymbol{k}_+^{(t)} \rangle\]

which indicates that the perturbation’s effect exceeds the one-step update on clean data. Therefore, $\langle\boldsymbol{q}_+^{(t)},\boldsymbol{k}_-^{(t)}\rangle$ stay $\Theta(log{C})$. By similar methods, we can bound the other $\langle\boldsymbol q^{(t)},\boldsymbol k^{(t)}\rangle$, which complete the proof.

\subsection{Proof of Claim 11}
We use similar methods in ~\ref{sec:cl10}. We first consider the increment of $\|\boldsymbol q\|_2^2$ and $\|\boldsymbol k\|_2^2$ at the $t$-th update when using the original clean data. We then show that, at this step, the effect introduced by the perturbation under adversarial samples dominates the increment learned from the clean data; hence $\|\boldsymbol q\|_2^2$ and $\|\boldsymbol k\|_2^2$ remains stable and bounded.

By the update rule of $\|\boldsymbol q\|_2^2$ and $\|\boldsymbol k\|_2^2$ we have
\begin{equation}
\begin{aligned}
    \|\boldsymbol{q}_+^{(t+1)}\|_2^2 - \|\boldsymbol{q}_+^{(t)}\|_2^2 &= 2 \langle \Delta\boldsymbol{q}_+^{(t)},\boldsymbol{q}_+^{(t)} \rangle + \langle \Delta\boldsymbol{q}_+^{(t)}, \Delta\boldsymbol{q}_+^{(t)} \rangle \\
    &= 2 \alpha_{+,+}^{(t)} \langle\boldsymbol{q}_+^{(t)},\boldsymbol{k}_+^{(t)} \rangle + 2 \sum_{n \in S_+} \sum_{i=2}^{M} \alpha_{n,+,i}^{(t)} \langle\boldsymbol{q}_+^{(t)},\boldsymbol{k}_{n,i}^{(t)} \rangle \\
    &\quad + \left( \alpha_{+,+}^{(t)}\boldsymbol{k}_+^{(t)} + \sum_{n \in S_+} \sum_{i=2}^{M} \alpha_{n,+,i}^{(t)}\boldsymbol{k}_{n,i}^{(t)} \right) \\
    &\quad \cdot \left( \alpha_{+,+}^{(t)}\boldsymbol{k}_+^{(t)\top} + \sum_{n \in S_+} \sum_{i=2}^{M} \alpha_{n,+,i}^{(t)}\boldsymbol{k}_{n,i}^{(t)\top} \right)\\
    &\leq 2 |\alpha_{+,+}^{(t)}| |\langle\boldsymbol{q}_+^{(t)},\boldsymbol{k}_+^{(t)} \rangle| + 2 \sum_{n \in S_+} \sum_{i=2}^{M} |\alpha_{n,+,i}^{(t)} ||\langle\boldsymbol{q}_+^{(t)},\boldsymbol{k}_{n,i}^{(t)} \rangle| +\text{\{lower order term\}}\\
    &\leq 2\frac{3\eta}{2M}  \|\boldsymbol{\mu}\|_2^2 \log{\left(O(\frac{1}{\epsilon})\right)}\Theta(\log{C})+2\sum_{n \in S_+} \sum_{i=2}^{M}\frac{3\eta}{2NM}  \|\boldsymbol{\mu}\|_2^2 \log{\left(O(\frac{1}{\epsilon})\right)}\Theta(\log{C})\\
    &\leq 12\eta  \|\boldsymbol{\mu}\|_2^2 \log{\left(O(\frac{1}{\epsilon})\right)}\Theta(\log{C})
\end{aligned}
\end{equation}
where the second inequality comes form the upper bounds for $\alpha$ and $\beta$ on clean data similar to (\ref{eq:103_2}).

Subsequently, we compute at the $t^{th}$ iteration the magnitude of the effect that the perturbation imposes on  $\|\boldsymbol{q}_+^{(t)}\|_2^2$.

\begin{equation}
\begin{aligned}
    &max_{\widetilde{X}^{(t)}\in B({X}^{(t)},\tau)}\|\widetilde{\boldsymbol q}_+^{(t)}\|_2^2 - \|\boldsymbol{q}_+^{(t)}\|_2^2\\
    &\leq ((1+\frac{\tau}{\|\boldsymbol{\mu}\|_2})^2-1)\|\boldsymbol{q}_+^{(t)}\|_2^2\\
    &\leq  ((1+\frac{\tau}{\|\boldsymbol{\mu}\|_2})^2-1)\Theta(\log{C})
\end{aligned}
\end{equation}

As $\tau$ and $\|\boldsymbol{\mu}\|_2$ are same order, and ${12\eta}  \|\boldsymbol{\mu}\|_2^2 \log{\left(O(\frac{1}{\epsilon})\right)}=o(\frac{1}{NM})$ by $d_h = \widetilde{\Omega} \left( \max \{ \text{SNR}^4, \text{SNR}^{-4} \} N^2 \epsilon^{-2} \right)$ and $\eta\leq \widetilde{O}(\min \{ \|\boldsymbol{\mu}\|_2^{-2}, (\sigma_p^2 d)^{-1} \} \cdot d_h^{-\frac{1}{2}})$. Thus, we have 
\[\|\boldsymbol{q}_+^{(t+1)}\|_2^2 - \|\boldsymbol{q}_+^{(t)}\|_2^2\leq max_{\widetilde{\boldsymbol X}^{(t)}\in B({\boldsymbol X}^{(t)},\tau)}\|\widetilde{\boldsymbol q}_+^{(t)}\|_2^2 - \|\boldsymbol{\boldsymbol q}_+^{(t)}\|_2^2\]

which indicates that the perturbation’s effect exceeds the one-step update on clean data. Therefore, $\|\boldsymbol{q}_+^{(t+1)}\|_2^2$ stay $\Theta(log{C})$.  By similar methods, we can bound the other $\|\boldsymbol q\|_2^2$ and $\|\boldsymbol k\|_2^2$,  which complete the proof.

The proof of convergence and test error is similar with Lemma~\ref{sec:convergence} and Section~\ref{sec:test_error}

\section{Proof of Theorem 3}
\begin{proof}
Fix an arbitrary \(\theta\). Consider the two classes separately. For the positive class (\(y=+\)):
\[
\delta_1^{(+)} = -\boldsymbol{\mu}_+, \quad \delta_2^{(+)} = 0
\]
is a valid perturbation since \(\|\delta_1^{(+)}\|_2 = \|\boldsymbol{\mu}_+\|_2 \le \tau\), \(\|\delta_2^{(+)}\|_2 = 0 \le \tau\).  
Then the adversarially perturbed point
\[
\widetilde x^{(+)} = [\,\boldsymbol{\mu}_+ -\boldsymbol{\mu}_+,\boldsymbol{\xi}_2,...,\boldsymbol{\xi}_M] = [0,\boldsymbol{\xi}_2,...,\boldsymbol{\xi}_M]
\]
lies in \(\mathcal{B}([\boldsymbol{\mu},\boldsymbol{\xi}_2,...,\boldsymbol{\xi}_M],\tau)\), so there exists a perturbation that can potentially flip the classifier's output for the positive class.

Similarly, for the negative class (\(y=-\)):
\[
\delta_1^{(-)} = -\boldsymbol{\mu}_-, \quad \delta_2^{(-)} = 0
\]
produces \(\widetilde x^{(-)} = [0,\boldsymbol{\xi}_2,...,\boldsymbol{\xi}_M] \in \mathcal{B}([\boldsymbol{\mu},\boldsymbol{\xi}_2,...,\boldsymbol{\xi}_M],\tau)\).

For each class independently, there exists a perturbation that can potentially flip its label.  

Thus, at least one class can be adversarially fooled with probability at least \(\min\{\Pr[y=+], \Pr[y=-]\}\).  
For uniform labels, this gives
\[
L_D^{\mathrm{rob}}(\theta) \ge \frac12 \cdot \frac12 = \frac14.
\]

This completes the proof.
\end{proof}
\section{Complete Calculation Process For Benign Overfitting}
\subsection{Calculations for $\alpha$ and $\beta$}\label{sec:cal_ab}
In this subsection, we give the calculactions for $\alpha$ and $\beta$ defined in Definition~\ref{def:vec_qk}.

\begin{equation*}
\begin{aligned}
\alpha_{+,+}^{(t)} &= \frac{\eta}{NM} \sum_{n \in S_+} -\widetilde{\ell}_n'(\theta) \langle \widetilde{\boldsymbol{\mu}}_+,\widetilde{\boldsymbol{\mu}}_+^{(t)} \rangle \\
&\quad \cdot \left( V_{+}^{(t)} \left( \frac{\exp(\langle \widetilde{\boldsymbol{q}}_{+}^{(t)}, \widetilde{\boldsymbol{k}}_{+}^{(t)} \rangle)}{\exp(\langle \widetilde{\boldsymbol{q}}_{+}^{(t)}, \widetilde{\boldsymbol{k}}_{+}^{(t)} \rangle) + \sum_{j=2}^M \exp(\langle \widetilde{\boldsymbol{q}}_{+}^{(t)}, \widetilde{\boldsymbol{k}}_{n,j}^{(t)} \rangle)} \right. \right. \\
&\quad \left. \left. - \left( \frac{\exp(\langle \widetilde{\boldsymbol{q}}_{+}^{(t)}, \widetilde{\boldsymbol{k}}_{+}^{(t)} \rangle)}{\exp(\langle \widetilde{\boldsymbol{q}}_{+}^{(t)}, \widetilde{\boldsymbol{k}}_{+}^{(t)} \rangle) + \sum_{j=2}^M \exp(\langle \widetilde{\boldsymbol{q}}_{+}^{(t)}, \widetilde{\boldsymbol{k}}_{n,j}^{(t)} \rangle)} \right)^2 \right) \right. \\
&\quad \left. - \sum_{i=2}^M \left( V_{n,i}^{(t)} \cdot \frac{\exp(\langle \widetilde{\boldsymbol{q}}_{+}^{(t)}, \widetilde{\boldsymbol{k}}_{+}^{(t)} \rangle)}{\exp(\langle \widetilde{\boldsymbol{q}}_{+}^{(t)}, \widetilde{\boldsymbol{k}}_{+}^{(t)} \rangle) + \sum_{j=2}^M \exp(\langle \widetilde{\boldsymbol{q}}_{+}^{(t)}, \widetilde{\boldsymbol{k}}_{n,j}^{(t)} \rangle)} \right. \right. \\
&\quad \left. \left. \cdot \frac{\exp(\langle \widetilde{\boldsymbol{q}}_{+}^{(t)}, \widetilde{\boldsymbol{k}}_{n,i}^{(t)} \rangle)}{\exp(\langle \widetilde{\boldsymbol{q}}_{+}^{(t)}, \widetilde{\boldsymbol{k}}_{+}^{(t)} \rangle) + \sum_{j=2}^M \exp(\langle \widetilde{\boldsymbol{q}}_{+}^{(t)}, \widetilde{\boldsymbol{k}}_{n,j}^{(t)} \rangle)} \right) \right)
\\&+\sum_{i=2}^M\langle\widetilde{\boldsymbol{\mu}}_+,\widetilde{\boldsymbol{\xi}}_{n,i}^{(t)}\rangle \cdot \left( V_+^{(t)} \left( \frac{\exp(\langle \widetilde{\boldsymbol{q}}_{n,i}^{(t)}, \widetilde{\boldsymbol{k}}_+^{(t)} \rangle)}{\exp(\langle \widetilde{\boldsymbol{q}}_{n,i}^{(t)}, \widetilde{\boldsymbol{k}}_+^{(t)} \rangle) + \sum_{j=2}^{M} \exp(\langle \widetilde{\boldsymbol{q}}_{n,i}^{(t)}, \widetilde{\boldsymbol{k}}_{n,j}^{(t)} \rangle)} \right. \right. \\
& \left. - \left( \frac{\exp(\langle \widetilde{\boldsymbol{q}}_{n,i}^{(t)}, \widetilde{\boldsymbol{k}}_+^{(t)} \rangle)}{\exp(\langle \widetilde{\boldsymbol{q}}_{n,i}^{(t)}, \widetilde{\boldsymbol{k}}_+^{(t)} \rangle) + \sum_{j=2}^{M} \exp(\langle \widetilde{\boldsymbol{q}}_{n,i}^{(t)}, \widetilde{\boldsymbol{k}}_{n,j}^{(t)} \rangle)} \right)^2 \right) \\
& - \sum_{k=2}^{M} \left( V_{n,i}^{(t)} \cdot \frac{\exp(\langle \widetilde{\boldsymbol{q}}_{n,i}^{(t)}, \widetilde{\boldsymbol{k}}_+^{(t)} \rangle)}{\exp(\langle \widetilde{\boldsymbol{q}}_{n,i}^{(t)}, \widetilde{\boldsymbol{k}}_+^{(t)} \rangle) + \sum_{j=2}^{M} \exp(\langle \widetilde{\boldsymbol{q}}_{n,i}^{(t)}, \widetilde{\boldsymbol{k}}_{n,j}^{(t)} \rangle)} \right. \\
& \left. \cdot \frac{\exp(\langle \widetilde{\boldsymbol{q}}_{n,i}^{(t)}, \widetilde{\boldsymbol{k}}_{n,k}^{(t)} \rangle)}{\exp(\langle \widetilde{\boldsymbol{q}}_{n,i}^{(t)}, \widetilde{\boldsymbol{k}}_+^{(t)} \rangle) + \sum_{j=2}^{M} \exp(\langle \widetilde{\boldsymbol{q}}_{n,i}^{(t)}, \widetilde{\boldsymbol{k}}_{n,j}^{(t)} \rangle)} \right)
\end{aligned}
\end{equation*}

\begin{equation*}
\begin{aligned}
\alpha_{n,+,i}^{(t)} &=  -\frac{\eta}{NM} \ell_n'^{(t)} \langle\widetilde{\boldsymbol{\mu}}_+, \widetilde{\boldsymbol{\mu}}_+^{(t)}\rangle\nonumber \\
&\quad \cdot \left( -V_{+}^{(t)} \cdot \frac{\exp(\langle \widetilde{\boldsymbol{q}}_{+}^{(t)}, \widetilde{\boldsymbol{k}}_{+}^{(t)} \rangle)}{\exp(\langle \widetilde{\boldsymbol{q}}_{+}^{(t)}, \widetilde{\boldsymbol{k}}_{+}^{(t)} \rangle) + \sum_{j=2}^M \exp(\langle \widetilde{\boldsymbol{q}}_{+}^{(t)}, \widetilde{\boldsymbol{k}}_{n,j}^{(t)} \rangle)} \right. \nonumber \\
&\qquad \left. + \frac{\exp(\langle \widetilde{\boldsymbol{q}}_{+}^{(t)}, \widetilde{\boldsymbol{k}}_{n,i}^{(t)} \rangle)}{\exp(\langle \widetilde{\boldsymbol{q}}_{+}^{(t)}, \widetilde{\boldsymbol{k}}_{+}^{(t)} \rangle) + \sum_{j=2}^M \exp(\langle \widetilde{\boldsymbol{q}}_{+}^{(t)}, \widetilde{\boldsymbol{k}}_{n,j}^{(t)} \rangle)} \right. \nonumber \\
&\quad + V_{n,i}^{(t)} \left( \frac{\exp(\langle \widetilde{\boldsymbol{q}}_{+}^{(t)}, \widetilde{\boldsymbol{k}}_{n,i}^{(t)} \rangle)}{\exp(\langle \widetilde{\boldsymbol{q}}_{+}^{(t)}, \widetilde{\boldsymbol{k}}_{+}^{(t)} \rangle) + \sum_{j=2}^M \exp(\langle \widetilde{\boldsymbol{q}}_{+}^{(t)}, \widetilde{\boldsymbol{k}}_{n,j}^{(t)} \rangle)} \right. \nonumber \\
&\qquad \left. - \left( \frac{\exp(\langle \widetilde{\boldsymbol{q}}_{+}^{(t)}, \widetilde{\boldsymbol{k}}_{n,i}^{(t)} \rangle)}{\exp(\langle \widetilde{\boldsymbol{q}}_{+}^{(t)}, \widetilde{\boldsymbol{k}}_{+}^{(t)} \rangle) + \sum_{j=2}^M \exp(\langle \widetilde{\boldsymbol{q}}_{+}^{(t)}, \widetilde{\boldsymbol{k}}_{n,j}^{(t)} \rangle)} \right)^2 \right) \nonumber \\
&\quad - \sum_{k \neq i} \left( V_{n,k}^{(t)} \cdot \frac{\exp(\langle \widetilde{\boldsymbol{q}}_{+}^{(t)}, \widetilde{\boldsymbol{k}}_{n,i}^{(t)} \rangle)}{\exp(\langle \widetilde{\boldsymbol{q}}_{+}^{(t)}, \widetilde{\boldsymbol{k}}_{+}^{(t)} \rangle) + \sum_{j=2}^M \exp(\langle \widetilde{\boldsymbol{q}}_{+}^{(t)}, \widetilde{\boldsymbol{k}}_{n,j}^{(t)} \rangle)} \right. \nonumber \\
&\qquad \left. \cdot \frac{\exp(\langle \widetilde{\boldsymbol{q}}_{+}^{(t)}, \widetilde{\boldsymbol{k}}_{n,k}^{(t)} \rangle)}{\exp(\langle \widetilde{\boldsymbol{q}}_{+}^{(t)}, \widetilde{\boldsymbol{k}}_{+}^{(t)} \rangle) + \sum_{j=2}^M \exp(\langle \widetilde{\boldsymbol{q}}_{+}^{(t)}, \widetilde{\boldsymbol{k}}_{n,j}^{(t)} \rangle)} \right) )\nonumber\\
& +\sum_{k=2}^M\langle\widetilde{\boldsymbol{\mu}}_+,\widetilde{\boldsymbol{\xi}}_{n,k}^{(t)}\rangle\cdot \left( -V_+^{(t)} \cdot \frac{\exp(\langle \widetilde{\boldsymbol{q}}_{n,k}^{(t)}, \widetilde{\boldsymbol{k}}_+^{(t)} \rangle)}{\exp(\langle \widetilde{\boldsymbol{q}}_{n,k}^{(t)}, \widetilde{\boldsymbol{k}}_+^{(t)} \rangle) + \sum_{j=2}^{M} \exp(\langle \widetilde{\boldsymbol{q}}_{n,k}^{(t)}, \widetilde{\boldsymbol{k}}_{n,j}^{(t)} \rangle)} \right. \nonumber\\
& \left. \cdot \frac{\exp(\langle \widetilde{\boldsymbol{q}}_{n,k}^{(t)}, \widetilde{\boldsymbol{k}}_{n,i}^{(t)} \rangle)}{\exp(\langle \widetilde{\boldsymbol{q}}_{n,k}^{(t)}, \widetilde{\boldsymbol{k}}_+^{(t)} \rangle) + \sum_{j=2}^{M} \exp(\langle \widetilde{\boldsymbol{q}}_{n,k}^{(t)}, \widetilde{\boldsymbol{k}}_{n,j}^{(t)} \rangle)} \right. \nonumber\\
& + V_{n,i}^{(t)} \left( \frac{\exp(\langle \widetilde{\boldsymbol{q}}_{n,k}^{(t)}, \widetilde{\boldsymbol{k}}_{n,i}^{(t)} \rangle)}{\exp(\langle \widetilde{\boldsymbol{q}}_{n,k}^{(t)}, \widetilde{\boldsymbol{k}}_+^{(t)} \rangle) + \sum_{j=2}^{M} \exp(\langle \widetilde{\boldsymbol{q}}_{n,k}^{(t)}, \widetilde{\boldsymbol{k}}_{n,j}^{(t)} \rangle)} \right. \nonumber\\
& \left. - \left( \frac{\exp(\langle \widetilde{\boldsymbol{q}}_{n,k}^{(t)}, \widetilde{\boldsymbol{k}}_{n,i}^{(t)} \rangle)}{\exp(\langle \widetilde{\boldsymbol{q}}_{n,k}^{(t)}, \widetilde{\boldsymbol{k}}_+^{(t)} \rangle) + \sum_{j=2}^{M} \exp(\langle \widetilde{\boldsymbol{q}}_{n,k}^{(t)}, \widetilde{\boldsymbol{k}}_{n,j}^{(t)} \rangle)} \right)^2 \right) \nonumber\\
& - \sum_{l \neq i} \left( V_{n,l}^{(t)} \cdot \frac{\exp(\langle \widetilde{\boldsymbol{q}}_{n,k}^{(t)}, \widetilde{\boldsymbol{k}}_{n,i}^{(t)} \rangle)}{\exp(\langle \widetilde{\boldsymbol{q}}_{n,k}^{(t)}, \widetilde{\boldsymbol{k}}_+^{(t)} \rangle) + \sum_{j=2}^{M} \exp(\langle \widetilde{\boldsymbol{q}}_{n,k}^{(t)}, \widetilde{\boldsymbol{k}}_{n,j}^{(t)} \rangle)} \right. \nonumber\\
& \left. \cdot \frac{\exp(\langle \widetilde{\boldsymbol{q}}_{n,k}^{(t)}, \widetilde{\boldsymbol{k}}_{n,l}^{(t)} \rangle)}{\exp(\langle \widetilde{\boldsymbol{q}}_{n,k}^{(t)}, \widetilde{\boldsymbol{k}}_+^{(t)} \rangle) + \sum_{j=2}^{M} \exp(\langle \widetilde{\boldsymbol{q}}_{n,k}^{(t)}, \widetilde{\boldsymbol{k}}_{n,j}^{(t)} \rangle)} \right)
\end{aligned}
\end{equation*}
We can also derive the calculations for other $\alpha$ and$\beta$, since they follow the same procedure as in Section F.1 of \citet{jiang2024unveil}.
\subsection{Proof of Lemma~\ref{lem:innerpro_s1}}\label{sec:pf_inner_pro1}

\subsection{Update Rules for Inner Products}\label{sec:update:innerp}
In this subsection, we give the update rules for the inner products of $\boldsymbol q$ and $\boldsymbol k$.

\begin{equation*}
\begin{aligned}
&\langle\boldsymbol{q}_+^{(t+1)},\boldsymbol{k}_+^{(t+1)} \rangle - \langle\boldsymbol{q}_+^{(t)},\boldsymbol{k}_+^{(t)} \rangle \\
&= \alpha_{+,+}^{(t)} \|\boldsymbol{k}_+^{(t)}\|_2^2 + \sum_{n \in S_+} \sum_{i=2}^{M} \alpha_{n,+,i}^{(t)} \langle\boldsymbol{k}_+^{(t)},\boldsymbol{k}_{n,i}^{(t)} \rangle \\
&\quad + \beta_{+,+}^{(t)} \|\boldsymbol{q}_+^{(t)}\|_2^2 + \sum_{n \in S_+} \sum_{i=2}^{M} \beta_{n,+,i}^{(t)} \langle\boldsymbol{q}_+^{(t)},\boldsymbol{q}_{n,i}^{(t)} \rangle \\
&\quad + \left( \alpha_{+,+}^{(t)}\boldsymbol{k}_+^{(t)} + \sum_{n \in S_+} \sum_{i=2}^{M} \alpha_{n,+,i}^{(t)}\boldsymbol{k}_{n,i}^{(t)} \right) \\
&\qquad \cdot \left( \beta_{+,+}^{(t)}\boldsymbol{q}_+^{(t)\top} + \sum_{n \in S_+} \sum_{i=2}^{M} \beta_{n,+,i}^{(t)}\boldsymbol{q}_{n,i}^{(t)\top} \right),
\end{aligned}
\end{equation*}

\begin{equation*}
\begin{aligned}
&\langle\boldsymbol{q}_-^{(t+1)},\boldsymbol{k}_-^{(t+1)} \rangle - \langle\boldsymbol{q}_-^{(t)},\boldsymbol{k}_-^{(t)} \rangle \\
&= \alpha_{-,-}^{(t)} \|\boldsymbol{k}_-^{(t)}\|_2^2 + \sum_{n \in S_-} \sum_{i=2}^{M} \alpha_{n,-,i}^{(t)} \langle\boldsymbol{k}_-^{(t)},\boldsymbol{k}_{n,i}^{(t)} \rangle \\
&\quad + \beta_{-,-}^{(t)} \|\boldsymbol{q}_-^{(t)}\|_2^2 + \sum_{n \in S_-} \sum_{i=2}^{M} \beta_{n,-,i}^{(t)} \langle\boldsymbol{q}_-^{(t)},\boldsymbol{q}_{n,i}^{(t)} \rangle \\
&\quad + \left( \alpha_{-,-}^{(t)}\boldsymbol{k}_-^{(t)} + \sum_{n \in S_-} \sum_{i=2}^{M} \alpha_{n,-,i}^{(t)}\boldsymbol{k}_{n,i}^{(t)} \right) \\
&\qquad \cdot \left( \beta_{-,-}^{(t)}\boldsymbol{q}_-^{(t)\top} + \sum_{n \in S_-} \sum_{i=2}^{M} \beta_{n,-,i}^{(t)}\boldsymbol{q}_{n,i}^{(t)\top} \right),
\end{aligned}
\end{equation*}
We can also derive the update rules for other $\boldsymbol q$ and $\boldsymbol k$, since they follow the same procedure as in Section F.2 of \citet{jiang2024unveil}.
\subsection{Proof of Lemma~\ref{lem:innerpro_s1}}\label{sec:pf_inner_pro1}
Let $T_0 = O \left( \frac{1}{\eta d_h^{\frac{1}{4}} (\|\boldsymbol{\mu}\|_2+\tau)^2 \|\boldsymbol{w}_O\|_2^2} \right)$. By Lemma~\ref{lem:up_v}, we have $|V_+^{(t)}|, |V_-^{(t)}|, |V_{n,i}^{(t)}| = O(d_h^{-\frac{1}{4}})$ for $t \in [0, T_0]$ by Lemma~\ref{lem:up_v}. Plugging this into the expression for $\alpha$ and $\beta$ gives
{
\begin{equation}
\begin{aligned}
|\alpha_{+,+}^{(t)}| &= \bigg|\frac{\eta}{NM} \sum_{n \in S_+} -\widetilde{\ell}_n'(\theta) \langle \widetilde{\boldsymbol{\mu}}_+,\widetilde{\boldsymbol{\mu}}_+^{(t)} \rangle \\
&\quad \cdot \left( V_{+}^{(t)} \left( \frac{\exp(\langle\boldsymbol{q}_{+}^{(t)},\boldsymbol{k}_{+}^{(t)} \rangle)}{\exp(\langle\boldsymbol{q}_{+}^{(t)},\boldsymbol{k}_{+}^{(t)} \rangle) + \sum_{j=2}^M \exp(\langle\boldsymbol{q}_{+}^{(t)},\boldsymbol{k}_{n,j}^{(t)} \rangle)} \right. \right. \\
&\quad \left. \left. - \left( \frac{\exp(\langle\boldsymbol{q}_{+}^{(t)},\boldsymbol{k}_{+}^{(t)} \rangle)}{\exp(\langle\boldsymbol{q}_{+}^{(t)},\boldsymbol{k}_{+}^{(t)} \rangle) + \sum_{j=2}^M \exp(\langle\boldsymbol{q}_{+}^{(t)},\boldsymbol{k}_{n,j}^{(t)} \rangle)} \right)^2 \right) \right. \\
&\quad \left. - \sum_{i=2}^M \left( V_{n,i}^{(t)} \cdot \frac{\exp(\langle\boldsymbol{q}_{+}^{(t)},\boldsymbol{k}_{+}^{(t)} \rangle)}{\exp(\langle\boldsymbol{q}_{+}^{(t)},\boldsymbol{k}_{+}^{(t)} \rangle) + \sum_{j=2}^M \exp(\langle\boldsymbol{q}_{+}^{(t)},\boldsymbol{k}_{n,j}^{(t)} \rangle)} \right. \right. \\
&\quad \left. \left. \cdot \frac{\exp(\langle\boldsymbol{q}_{+}^{(t)},\boldsymbol{k}_{n,i}^{(t)} \rangle)}{\exp(\langle\boldsymbol{q}_{+}^{(t)},\boldsymbol{k}_{+}^{(t)} \rangle) + \sum_{j=2}^M \exp(\langle\boldsymbol{q}_{+}^{(t)},\boldsymbol{k}_{n,j}^{(t)} \rangle)} \right) \right)
\\&+\sum_{i=2}^M\langle\widetilde{\boldsymbol{\mu}}_+,\widetilde{\boldsymbol{\xi}}_{n,i}^{(t)}\rangle \cdot \left( V_+^{(t)} \left( \frac{\exp(\langle\boldsymbol{q}_{n,i}^{(t)},\boldsymbol{k}_+^{(t)} \rangle)}{\exp(\langle\boldsymbol{q}_{n,i}^{(t)},\boldsymbol{k}_+^{(t)} \rangle) + \sum_{j=2}^{M} \exp(\langle\boldsymbol{q}_{n,i}^{(t)},\boldsymbol{k}_{n,j}^{(t)} \rangle)} \right. \right. \\
& \left. - \left( \frac{\exp(\langle\boldsymbol{q}_{n,i}^{(t)},\boldsymbol{k}_+^{(t)} \rangle)}{\exp(\langle\boldsymbol{q}_{n,i}^{(t)},\boldsymbol{k}_+^{(t)} \rangle) + \sum_{j=2}^{M} \exp(\langle\boldsymbol{q}_{n,i}^{(t)},\boldsymbol{k}_{n,j}^{(t)} \rangle)} \right)^2 \right) \\
& - \sum_{k=2}^{M} \left( V_{n,i}^{(t)} \cdot \frac{\exp(\langle\boldsymbol{q}_{n,i}^{(t)},\boldsymbol{k}_+^{(t)} \rangle)}{\exp(\langle\boldsymbol{q}_{n,i}^{(t)},\boldsymbol{k}_+^{(t)} \rangle) + \sum_{j=2}^{M} \exp(\langle\boldsymbol{q}_{n,i}^{(t)},\boldsymbol{k}_{n,j}^{(t)} \rangle)} \right. \\
& \left. \cdot \frac{\exp(\langle\boldsymbol{q}_{n,i}^{(t)},\boldsymbol{k}_{n,k}^{(t)} \rangle)}{\exp(\langle\boldsymbol{q}_{n,i}^{(t)},\boldsymbol{k}_+^{(t)} \rangle) + \sum_{j=2}^{M} \exp(\langle\boldsymbol{q}_{n,i}^{(t)},\boldsymbol{k}_{n,j}^{(t)} \rangle)} \right)\bigg|
\\
&\leq \frac{\eta (\|\boldsymbol{\mu}\|_2+\tau)^2}{NM} \cdot \frac{3NM}{4} \cdot O(d_h^{-\frac{1}{4}})+\frac{\eta (\|\boldsymbol{\mu}\|\tau+2\sigma_p\tau\sqrt{2\log(4NM/\delta)}+\tau^2)}{NM} \cdot \frac{3NM(M-1)}{4} \cdot O(d_h^{-\frac{1}{4}})
\\
&= O \left( \frac{\eta (\|\boldsymbol{\mu}\|_2+\tau)^2}{d_h^{\frac{1}{4}}} \right)
\end{aligned}
\end{equation}
}

where the inequality is by \( -\widetilde{\ell}_n^{\prime(t)} \leq 1 \) and the property that attention is smaller than 1
\(( \text{e.g. }\frac{\exp(\langle\boldsymbol{q}_+^{(t)},\boldsymbol{k}_+^{(t)} \rangle)}{\exp(\langle\boldsymbol{q}_+^{(t)},\boldsymbol{k}_+^{(t)} \rangle) + \sum_{j=2}^{M} \exp(\langle\boldsymbol{q}_+^{(t)},\boldsymbol{k}_{n,j}^{(t)} \rangle)} \leq 1 \) ). We also have
\begin{equation*}
\begin{aligned}
|\alpha_{n,+,i}^{(t)}| &= \left| -\frac{\eta}{NM} \ell_n'^{(t)} \langle\widetilde{\boldsymbol{\mu}}_+, \widetilde{\boldsymbol{\mu}}_+^{(t)}\rangle\nonumber \right.\\
&\quad \cdot \left( -V_{+}^{(t)} \cdot \frac{\exp(\langle\boldsymbol{q}_{+}^{(t)},\boldsymbol{k}_{+}^{(t)} \rangle)}{\exp(\langle\boldsymbol{q}_{+}^{(t)},\boldsymbol{k}_{+}^{(t)} \rangle) + \sum_{j=2}^M \exp(\langle\boldsymbol{q}_{+}^{(t)},\boldsymbol{k}_{n,j}^{(t)} \rangle)} \right. \nonumber \\
&\qquad \left. + \frac{\exp(\langle\boldsymbol{q}_{+}^{(t)},\boldsymbol{k}_{n,i}^{(t)} \rangle)}{\exp(\langle\boldsymbol{q}_{+}^{(t)},\boldsymbol{k}_{+}^{(t)} \rangle) + \sum_{j=2}^M \exp(\langle\boldsymbol{q}_{+}^{(t)},\boldsymbol{k}_{n,j}^{(t)} \rangle)} \right. \nonumber \\
&\quad + V_{n,i}^{(t)} \left( \frac{\exp(\langle\boldsymbol{q}_{+}^{(t)},\boldsymbol{k}_{n,i}^{(t)} \rangle)}{\exp(\langle\boldsymbol{q}_{+}^{(t)},\boldsymbol{k}_{+}^{(t)} \rangle) + \sum_{j=2}^M \exp(\langle\boldsymbol{q}_{+}^{(t)},\boldsymbol{k}_{n,j}^{(t)} \rangle)} \right. \nonumber \\
&\qquad \left. - \left( \frac{\exp(\langle\boldsymbol{q}_{+}^{(t)},\boldsymbol{k}_{n,i}^{(t)} \rangle)}{\exp(\langle\boldsymbol{q}_{+}^{(t)},\boldsymbol{k}_{+}^{(t)} \rangle) + \sum_{j=2}^M \exp(\langle\boldsymbol{q}_{+}^{(t)},\boldsymbol{k}_{n,j}^{(t)} \rangle)} \right)^2 \right) \nonumber \\
&\quad - \sum_{k \neq i} \left( V_{n,k}^{(t)} \cdot \frac{\exp(\langle\boldsymbol{q}_{+}^{(t)},\boldsymbol{k}_{n,i}^{(t)} \rangle)}{\exp(\langle\boldsymbol{q}_{+}^{(t)},\boldsymbol{k}_{+}^{(t)} \rangle) + \sum_{j=2}^M \exp(\langle\boldsymbol{q}_{+}^{(t)},\boldsymbol{k}_{n,j}^{(t)} \rangle)} \right. \nonumber \\
&\qquad \left. \cdot \frac{\exp(\langle\boldsymbol{q}_{+}^{(t)},\boldsymbol{k}_{n,k}^{(t)} \rangle)}{\exp(\langle\boldsymbol{q}_{+}^{(t)},\boldsymbol{k}_{+}^{(t)} \rangle) + \sum_{j=2}^M \exp(\langle\boldsymbol{q}_{+}^{(t)},\boldsymbol{k}_{n,j}^{(t)} \rangle)} \right) )\nonumber\\
& +\sum_{k=2}^M\langle\widetilde{\boldsymbol{\mu}}_+,\widetilde{\boldsymbol{\xi}}_{n,k}^{(t)}\rangle\cdot \left( -V_+^{(t)} \cdot \frac{\exp(\langle\boldsymbol{q}_{n,k}^{(t)},\boldsymbol{k}_+^{(t)} \rangle)}{\exp(\langle\boldsymbol{q}_{n,k}^{(t)},\boldsymbol{k}_+^{(t)} \rangle) + \sum_{j=2}^{M} \exp(\langle\boldsymbol{q}_{n,k}^{(t)},\boldsymbol{k}_{n,j}^{(t)} \rangle)} \right. \nonumber\\
& \left. \cdot \frac{\exp(\langle\boldsymbol{q}_{n,k}^{(t)},\boldsymbol{k}_{n,i}^{(t)} \rangle)}{\exp(\langle\boldsymbol{q}_{n,k}^{(t)},\boldsymbol{k}_+^{(t)} \rangle) + \sum_{j=2}^{M} \exp(\langle\boldsymbol{q}_{n,k}^{(t)},\boldsymbol{k}_{n,j}^{(t)} \rangle)} \right. \nonumber\\
& + V_{n,i}^{(t)} \left( \frac{\exp(\langle\boldsymbol{q}_{n,k}^{(t)},\boldsymbol{k}_{n,i}^{(t)} \rangle)}{\exp(\langle\boldsymbol{q}_{n,k}^{(t)},\boldsymbol{k}_+^{(t)} \rangle) + \sum_{j=2}^{M} \exp(\langle\boldsymbol{q}_{n,k}^{(t)},\boldsymbol{k}_{n,j}^{(t)} \rangle)} \right. \nonumber\\
& \left. - \left( \frac{\exp(\langle\boldsymbol{q}_{n,k}^{(t)},\boldsymbol{k}_{n,i}^{(t)} \rangle)}{\exp(\langle\boldsymbol{q}_{n,k}^{(t)},\boldsymbol{k}_+^{(t)} \rangle) + \sum_{j=2}^{M} \exp(\langle\boldsymbol{q}_{n,k}^{(t)},\boldsymbol{k}_{n,j}^{(t)} \rangle)} \right)^2 \right) \nonumber\\
& - \sum_{l \neq i} \left( V_{n,l}^{(t)} \cdot \frac{\exp(\langle\boldsymbol{q}_{n,k}^{(t)},\boldsymbol{k}_{n,i}^{(t)} \rangle)}{\exp(\langle\boldsymbol{q}_{n,k}^{(t)},\boldsymbol{k}_+^{(t)} \rangle) + \sum_{j=2}^{M} \exp(\langle\boldsymbol{q}_{n,k}^{(t)},\boldsymbol{k}_{n,j}^{(t)} \rangle)} \right. \nonumber\\
& \left.\left. \cdot \frac{\exp(\langle\boldsymbol{q}_{n,k}^{(t)},\boldsymbol{k}_{n,l}^{(t)} \rangle)}{\exp(\langle\boldsymbol{q}_{n,k}^{(t)},\boldsymbol{k}_+^{(t)} \rangle) + \sum_{j=2}^{M} \exp(\langle\boldsymbol{q}_{n,k}^{(t)},\boldsymbol{k}_{n,j}^{(t)} \rangle)} \right)\right | \\
&\leq \frac{\eta (\|\boldsymbol{\mu}\|_2+\tau)^2}{NM} \cdot M \cdot O(d_h^{-\frac{1}{4}}) \\
&{=O \left( \frac{\eta (\|\boldsymbol{\mu}\|_2+\tau)^2}{d_h^{\frac{1}{4}}N} \right)},
\end{aligned}
\end{equation*}

where the inequality is by \( -\widetilde{\ell}_n^{\prime(t)} \leq 1\),  Lemma~\ref{lem:con_ineq} and the property that attention is smaller than 1. Similarly, we have
\begin{equation*}
|\alpha_{-,-}^{(t)}|, |\beta_{+,+}^{(t)}|, |\beta_{-,-}^{(t)}| = O \left( \frac{\eta (\|\boldsymbol{\mu}\|_2+\tau)^2}{d_h^{\frac{1}{4}}} \right),
\end{equation*}
\begin{equation*}
|\alpha_{n,-,l}^{(t)}|, |\beta_{n,+,l}^{(t)}|, |\beta_{n,-,l}^{(t)}| = O \left( \frac{\eta (\|\boldsymbol{\mu}\|_2+\tau)^2}{d_h^{\frac{1}{4}} N} \right),
\end{equation*}
\begin{equation*}
|\alpha_{n,l,-}^{(t)}|, |\beta_{n,l,+}^{(t)}|, |\beta_{n,l,-}^{(t)}|, |\alpha_{n,l,n',l'}^{(t)}|, |\beta_{n,l,n',l'}^{(t)}| {=O \left( \frac{\eta (\|\boldsymbol{\mu}\|\tau+\sigma_p^2d)}{d_h^{\frac{1}{4}}N} \right)}
\end{equation*}
for  \(t \in [0, T_0]\).
Next we use induction to show that the following proposition $\mathcal{A}(t)$ holds for $t \in [0, T_0]$.
\textbf{$\mathcal{A}(t)$:}
\begin{equation*}
\begin{aligned}
&|\langle q_\pm^{(t)}, k_\pm^{(t)} \rangle|, |\langle\boldsymbol{q}_{n,i}^{(t)}, k_\pm^{(t)} \rangle|, |\langle q_\pm^{(t)},\boldsymbol{k}_{n,j}^{(t)} \rangle|, |\langle\boldsymbol{q}_{n,i}^{(t)},\boldsymbol{k}_{n',j}^{(t)} \rangle| \\
&= O \left( \max \{ \|\boldsymbol{\mu}\|_2^2, \sigma_p^2 d \} \cdot \sigma_h^2 \cdot \sqrt{d_h \log(6N^2M^2/\delta)} \right),
\end{aligned}
\end{equation*}
\begin{equation*}
\begin{aligned}
&|\langle q_\pm^{(t)}, q_\mp^{(t)} \rangle|, |\langle\boldsymbol{q}_{n,i}^{(t)}, q_\pm^{(t)} \rangle|, |\langle\boldsymbol{q}_{n,i}^{(t)},\boldsymbol{q}_{n',j}^{(t)} \rangle| \\
&= O \left( \max \{ \|\boldsymbol{\mu}\|_2^2, \sigma_p^2 d \} \cdot \sigma_h^2 \cdot \sqrt{d_h \log(6N^2M^2/\delta)} \right),
\end{aligned}
\end{equation*}
\begin{equation*}
\begin{aligned}
&|\langle k_\pm^{(t)}, k_\mp^{(t)} \rangle|, |\langle\boldsymbol{k}_{n,i}^{(t)}, k_\pm^{(t)} \rangle|, |\langle\boldsymbol{k}_{n,i}^{(t)},\boldsymbol{k}_{n',j}^{(t)} \rangle| \\
&= O \left( \max \{ \|\boldsymbol{\mu}\|_2^2, \sigma_p^2 d \} \cdot \sigma_h^2 \cdot \sqrt{d_h \log(6N^2M^2/\delta)} \right),
\end{aligned}
\end{equation*}
\begin{equation*}
\|q_\pm^{(t)}\|_2^2, \|k_\pm^{(t)}\|_2^2 = \Theta(\|\boldsymbol{\mu}\|_2^2 \sigma_h^2 d_h)
\end{equation*}
\begin{equation*}
\|\boldsymbol{q}_{n,i}^{(t)}\|_2^2, \|\boldsymbol{k}_{n,i}^{(t)}\|_2^2 = \Theta(\sigma_p^2 \sigma_h^2 d d_h)
\end{equation*}
for  \(i, j \in [M] \backslash \{1\}, n, n' \in [N]\).

By Lemma~\ref{lem:in_qk} we know that $\mathcal{A}(0)$ is true. Now we assume $\mathcal{A}(0), \ldots, \mathcal{A}(T)$ is true, then we need to prove that $\mathcal{A}(T+1)$ is true. We first proof $|\langle\boldsymbol{q}_+^{(T+1)},\boldsymbol{k}_+^{(T+1)} \rangle| = O \left( \max \{ \|\boldsymbol{\mu}\|_2^2, \sigma_p^2 d \} \cdot \sigma_h^2 \cdot \sqrt{d_h \log(6N^2M^2/\delta)} \right)$, as an example.

{
\begin{equation*}
\begin{aligned}
|\langle\boldsymbol{q}_+^{(t+1)},\boldsymbol{k}_+^{(t+1)} \rangle - \langle\boldsymbol{q}_+^{(t)},\boldsymbol{k}_+^{(t)} \rangle| 
= & \left| \alpha_{+,+}^{(t)} \|\boldsymbol{k}_+^{(t)}\|_2^2 + \sum_{n \in S_+} \sum_{i=2}^{M} \alpha_{n,+,i}^{(t)} \langle\boldsymbol{k}_+^{(t)},\boldsymbol{k}_{n,i}^{(t)} \rangle \right. \\
& \left. + \beta_{+,+}^{(t)} \|\boldsymbol{q}_+^{(t)}\|_2^2 + \sum_{n \in S_+} \sum_{i=2}^{M} \beta_{n,+,i}^{(t)} \langle\boldsymbol{q}_+^{(t)},\boldsymbol{q}_{n,i}^{(t)} \rangle \right. \\
& \left. + \left( \alpha_{+,+}^{(t)}\boldsymbol{k}_+^{(t)} + \sum_{n \in S_+} \sum_{i=2}^{M} \alpha_{n,+,i}^{(t)}\boldsymbol{k}_{n,i}^{(t)} \right) \right. \\
& \left. \cdot \left( \beta_{+,+}^{(t)}\boldsymbol{q}_+^{(t)\top} + \sum_{n \top} \sum_{i=2}^{M} \beta_{n,+,i}^{(t)}\boldsymbol{q}_{n,i}^{(t)\top} \right) \right| \\
\leq & O \left( \frac{\eta (\|\boldsymbol{\mu}\|_2+\tau)^2}{d_h^{\frac{1}{4}}} \right) \cdot \Theta(\|\boldsymbol{\mu}\|_2^2 \sigma_h^2 d_h) \\
& + NM \cdot O \left( \frac{\eta(\|\boldsymbol{\mu}\|_2+\tau)^2}{d_h^{\frac{1}{4}} N} \right) \cdot O \left( \max \{ \|\boldsymbol{\mu}\|_2^2, \sigma_p^2 d \} \cdot \sigma_h^2 \cdot \sqrt{d_h \log(6N^2M^2/\delta)} \right) \\
& + \{\text{lower order term\}}
=  O \left( \eta \|\boldsymbol{\mu}\|_2^2 (\|\boldsymbol{\mu}\|_2+\tau)^2\sigma_h^2 d_h^{\frac{3}{4}} \right)
\end{aligned}
\end{equation*}
}
Taking a summation, we obtain that
{\begin{equation*}
\begin{aligned}
|\langle\boldsymbol{q}_+^{(T+1)},\boldsymbol{k}_+^{(T+1)} \rangle| \leq & |\langle\boldsymbol{q}_+^{(0)},\boldsymbol{k}_+^{(0)} \rangle| + \sum_{t=0}^{T} |\langle\boldsymbol{q}_+^{(t+1)},\boldsymbol{k}_+^{(t+1)} \rangle - \langle\boldsymbol{q}_+^{(t)},\boldsymbol{k}_+^{(t)} \rangle| \\
\leq & |\langle\boldsymbol{q}_+^{(0)},\boldsymbol{k}_+^{(0)} \rangle| + \sum_{t=0}^{T_0-1} |\langle\boldsymbol{q}_+^{(t+1)},\boldsymbol{k}_+^{(t+1)} \rangle - \langle\boldsymbol{q}_+^{(t)},\boldsymbol{k}_+^{(t)} \rangle|\\
\leq & O \left( \max \{ \|\boldsymbol{\mu}\|_2^2, \sigma_p^2 d \} \cdot \sigma_h^2 \cdot \sqrt{d_h \log(6N^2M^2/\delta)} \right)\\
&+ O \left( \frac{1}{\eta d_h^{\frac{1}{4}} (\|\boldsymbol{\mu}\|_2+\tau)^2 \|\boldsymbol{w}_O\|_2^2} \right) \cdot O(\eta \|\boldsymbol{\mu}\|_2^2 (\|\boldsymbol{\mu}\|_2+\tau)^2\sigma_h^2 d_h^{\frac{3}{4}}) \\
= & O \left( \max \{ \|\boldsymbol{\mu}\|_2^2, \sigma_p^2 d \} \cdot \sigma_h^2 \cdot \sqrt{d_h \log(6N^2M^2/\delta)} \right) + O \left( \|\boldsymbol{\mu}\|_2^2 \sigma_h^2 d_h^{\frac{1}{2}} \right)
\end{aligned}
\end{equation*} }
Similarly to $\langle\boldsymbol{q}_+^{(t)},\boldsymbol{k}_+^{(t)} \rangle$, it is easy to know that the inner product does not change by a magnitude more than the product of $\max \{\alpha, \beta\}$ and $\max \{ \langle q, q \rangle, \langle k, k \rangle \}$ in a single iteration, which can be expressed as follows
\begin{equation*}
\begin{aligned}
|\langle\boldsymbol  q^{(t+1)}&,\boldsymbol  k^{(t+1)} \rangle - \langle\boldsymbol q^{(t)},\boldsymbol k^{(t)} \rangle| \\
= & O \left( \max \left\{ \frac{\eta (\|\boldsymbol{\mu}\|_2+\tau)^2}{d_h^{\frac{1}{4}}}, \frac{\eta (\sigma_p^2 d+\|\boldsymbol{\mu}\|\tau)}{d_h^{\frac{1}{4}} N} \right\} \right) \cdot \Theta(\max \{ \|\boldsymbol{\mu}\|_2^2 \sigma_h^2 d_h, \sigma_p^2 \sigma_h^2 d d_h \}) \\
= & O \left( \frac{\eta (\|\boldsymbol{\mu}\|_2+\tau)^2}{d_h^{\frac{1}{4}}} \right) \cdot \Theta(\max \{ \|\boldsymbol{\mu}\|_2^2 \sigma_h^2 d_h, \sigma_p^2 \sigma_h^2 d d_h \}) \\
= & O \left( (\|\boldsymbol{\mu}\|_2+\tau)^2 \sigma_h^2 d_h^{\frac{3}{4}} \cdot \max \{ \|\boldsymbol{\mu}\|_2^2, \sigma_p^2 d \} \right)
\end{aligned}
\end{equation*} 

where the second equality is by the condition that \(N \cdot \text{SNR}^2 = \Omega(1)\).

Taking a summation, we obtain that
\begin{equation*}
\begin{aligned}
|\langle\boldsymbol q^{(T+1)},\boldsymbol k^{(T+1)} \rangle &- \langle\boldsymbol q^{(0)},\boldsymbol k^{(0)} \rangle| \leq \sum_{t=0}^{T-1} |\langle\boldsymbol q^{(t+1)},\boldsymbol k^{(t+1)} \rangle - \langle\boldsymbol q^{(t)},\boldsymbol k^{(t)} \rangle| \\
&\leq \sum_{t=0}^{T_0-1} O \left( \eta (\|\boldsymbol{\mu}\|_2+\tau)^2 \sigma_h^2 d_h^{\frac{3}{4}} \cdot \max \{ \|\boldsymbol{\mu}\|_2^2, \sigma_p^2 d \} \right) \\
&= O \left( \frac{1}{\eta d_h^{\frac{1}{4}} (\|\boldsymbol{\mu}\|_2+\tau)^2 \|\boldsymbol{w}_O\|_2^2} \right) \cdot O \left( \eta (\|\boldsymbol{\mu}\|_2+\tau)^2 \sigma_h^2 d_h^{\frac{3}{4}} \cdot \max \{ \|\boldsymbol{\mu}\|_2^2, \sigma_p^2 d \} \right) \\
&= O \left( \max \{ \|\boldsymbol{\mu}\|_2^2, \sigma_p^2 d \} \cdot \sigma_h^2 d_h^{\frac{1}{4}} \right).
\end{aligned}
\end{equation*} 

It is clear that the magnitude of $\langle\boldsymbol{q}^{(T+1)},\boldsymbol{k}^{(T+1)} \rangle - \langle\boldsymbol{q}^{(0)},\boldsymbol{k}^{(0)} \rangle$ is smaller than $\max \{ \|\boldsymbol{\mu}\|_2^2, \sigma_p^2 d \} \cdot \sigma_h^2 \cdot \sqrt{d_h \log(6N^2M^2/\delta)}$. Thus the magnitude of the bound for $\langle\boldsymbol{q}^{(T+1)},\boldsymbol{k}^{(T+1)} \rangle$ is the same as that of $\langle\boldsymbol{q}^{(T)},\boldsymbol{k}^{(T)} \rangle$. The proof for $\langle\boldsymbol{q}^{(T+1)},\boldsymbol{q}^{(T+1)} \rangle$ and $\langle\boldsymbol{k}^{(T+1)},\boldsymbol{k}^{(T+1)} \rangle$ is exactly the same, and we can conclude the proof by an induction.

\subsection{Lower Bounds of $\alpha$ and $\beta$}\label{sec:low_ab}
In this subsection, we present some bounds for $\alpha$ and $\beta$ which can be used in ~\ref{sec:stage2} and ~\ref{sec:stage3}. All the calculations in this subsection are based on the precise expression for $\alpha$ and $\beta$ in \ref{sec:cal_ab} and assume that $\mathcal{B}(T_1), \ldots, \mathcal{B}(s), \mathcal{D}(T_1), \ldots, \mathcal{D}(s-1)$ hold ($s \in [T_1, t]$). Then the following propositions hold:

\begin{align*}
V_+^{(s)} &\geq 3M \cdot |V_{n,i}^{(s)}|, \\
V_-^{(s)} &\leq -3M \cdot |V_{n,i}^{(s)}|, \\
\text{softmax}(\langle\boldsymbol q_\pm^{(s)},\boldsymbol k_\pm^{(s)} \rangle), &\text{softmax}(\langle\boldsymbol{q}_{n,i}^{(s)},\boldsymbol k_\pm^{(s)} \rangle) \geq \frac{1}{M} - o(1), \\
\text{softmax}(\langle\boldsymbol q_\pm^{(s)},\boldsymbol{k}_{n,j}^{(s)} \rangle), &\text{softmax}(\langle\boldsymbol{q}_{n,i}^{(s)},\boldsymbol{k}_{n,j}^{(s)} \rangle) \leq \frac{1}{M} + o(1).
\end{align*}

Now we give the bounds respectively for $\alpha_{+,+}^{(s)}, \alpha_{n,+,i}^{(s)}, \alpha_{-,-}^{(s)}, \alpha_{n,-,i}^{(s)}, \alpha_{n,i,+}^{(s)}, \alpha_{n,i,-}^{(s)}, \alpha_{n,i,n',i'}^{(s)}$,\\$ \beta_{+,+}^{(s)}, \beta_{n,+,i}^{(s)}, \beta_{-,-}^{(s)}, \beta_{n,-,i}^{(s)}, \beta_{n,i,+}^{(s)}, \beta_{n,i,-}^{(s)}, \beta_{n,i,n',i'}^{(s)}$.

\begin{align*}
\alpha_{+,+}^{(s)} &= \frac{\eta}{NM} \sum_{n \in S_+} -\widetilde{\ell}_n'(\theta) \langle \widetilde{\boldsymbol{\mu}}_+,\widetilde{\boldsymbol{\mu}}_+^{(s)} \rangle \\
&\quad \cdot \left( V_{+}^{(s)} \left( \frac{\exp(\langle\boldsymbol{q}_{+}^{(s)},\boldsymbol{k}_{+}^{(s)} \rangle)}{\exp(\langle\boldsymbol{q}_{+}^{(s)},\boldsymbol{k}_{+}^{(s)} \rangle) + \sum_{j=2}^M \exp(\langle\boldsymbol{q}_{+}^{(s)},\boldsymbol{k}_{n,j}^{(s)} \rangle)} \right. \right. \\
&\quad \left. \left. - \left( \frac{\exp(\langle\boldsymbol{q}_{+}^{(s)},\boldsymbol{k}_{+}^{(s)} \rangle)}{\exp(\langle\boldsymbol{q}_{+}^{(s)},\boldsymbol{k}_{+}^{(s)} \rangle) + \sum_{j=2}^M \exp(\langle\boldsymbol{q}_{+}^{(s)},\boldsymbol{k}_{n,j}^{(s)} \rangle)} \right)^2 \right) \right. \\
&\quad \left. - \sum_{i=2}^M \left( V_{n,i}^{(s)} \cdot \frac{\exp(\langle\boldsymbol{q}_{+}^{(s)},\boldsymbol{k}_{+}^{(s)} \rangle)}{\exp(\langle\boldsymbol{q}_{+}^{(s)},\boldsymbol{k}_{+}^{(s)} \rangle) + \sum_{j=2}^M \exp(\langle\boldsymbol{q}_{+}^{(s)},\boldsymbol{k}_{n,j}^{(s)} \rangle)} \right. \right. \\
&\quad \left. \left. \cdot \frac{\exp(\langle\boldsymbol{q}_{+}^{(s)},\boldsymbol{k}_{n,i}^{(s)} \rangle)}{\exp(\langle\boldsymbol{q}_{+}^{(s)},\boldsymbol{k}_{+}^{(s)} \rangle) + \sum_{j=2}^M \exp(\langle\boldsymbol{q}_{+}^{(s)},\boldsymbol{k}_{n,j}^{(s)} \rangle)} \right) \right) \\
&+\sum_{i=2}^M\langle\widetilde{\boldsymbol{\mu}}_+,\widetilde{\boldsymbol{\xi}}_{n,i}^{(s)}\rangle \cdot \left( V_+^{(s)} \left( \frac{\exp(\langle\boldsymbol{q}_{n,i}^{(s)},\boldsymbol{k}_+^{(s)} \rangle)}{\exp(\langle\boldsymbol{q}_{n,i}^{(s)},\boldsymbol{k}_+^{(s)} \rangle) + \sum_{j=2}^{M} \exp(\langle\boldsymbol{q}_{n,i}^{(s)},\boldsymbol{k}_{n,j}^{(s)} \rangle)} \right. \right. \\
& \left. - \left( \frac{\exp(\langle\boldsymbol{q}_{n,i}^{(s)},\boldsymbol{k}_+^{(s)} \rangle)}{\exp(\langle\boldsymbol{q}_{n,i}^{(s)},\boldsymbol{k}_+^{(s)} \rangle) + \sum_{j=2}^{M} \exp(\langle\boldsymbol{q}_{n,i}^{(s)},\boldsymbol{k}_{n,j}^{(s)} \rangle)} \right)^2 \right) \\
& - \sum_{k=2}^{M} \left( V_{n,i}^{(s)} \cdot \frac{\exp(\langle\boldsymbol{q}_{n,i}^{(s)},\boldsymbol{k}_+^{(s)} \rangle)}{\exp(\langle\boldsymbol{q}_{n,i}^{(s)},\boldsymbol{k}_+^{(s)} \rangle) + \sum_{j=2}^{M} \exp(\langle\boldsymbol{q}_{n,i}^{(s)},\boldsymbol{k}_{n,j}^{(s)} \rangle)} \right. \\
& \left. \cdot \frac{\exp(\langle\boldsymbol{q}_{n,i}^{(s)},\boldsymbol{k}_{n,k}^{(s)} \rangle)}{\exp(\langle\boldsymbol{q}_{n,i}^{(s)},\boldsymbol{k}_+^{(s)} \rangle) + \sum_{j=2}^{M} \exp(\langle\boldsymbol{q}_{n,i}^{(s)},\boldsymbol{k}_{n,j}^{(s)} \rangle)} \right)
 \\
&\geq \frac{\eta}{NM} \sum_{n \in S_+} -\widetilde{\ell}_n'^{(s)} \langle \widetilde{\boldsymbol{\mu}}_+, \widetilde{\boldsymbol{\mu}}_+^{(s)}\rangle\cdot \frac{\exp(\langle\boldsymbol{q}_+^{(s)},\boldsymbol{k}_+^{(s)} \rangle)}{\exp(\langle\boldsymbol{q}_+^{(s)},\boldsymbol{k}_+^{(s)} \rangle) + \sum_{j=2}^M \exp(\langle\boldsymbol{q}_+^{(s)},\boldsymbol{k}_{n,j}^{(s)} \rangle)} \\
&\cdot \left( V_+^{(s)} \left( 1 - \frac{\exp(\langle\boldsymbol{q}_+^{(s)},\boldsymbol{k}_+^{(s)} \rangle)}{\exp(\langle\boldsymbol{q}_+^{(s)},\boldsymbol{k}_+^{(s)} \rangle) + \sum_{j=2}^M \exp(\langle\boldsymbol{q}_+^{(s)},\boldsymbol{k}_{n,j}^{(s)} \rangle)} \right) \right. \\
&\left. - \frac{1}{2} \cdot V_+^{(s)} \sum_{i=2}^M \frac{\exp(\langle\boldsymbol{q}_+^{(s)},\boldsymbol{k}_{n,i}^{(s)} \rangle)}{\exp(\langle\boldsymbol{q}_+^{(s)},\boldsymbol{k}_+^{(s)} \rangle) + \sum_{j=2}^M \exp(\langle\boldsymbol{q}_+^{(s)},\boldsymbol{k}_{n,j}^{(s)} \rangle)} \right) \\
&+\frac{\eta}{NM} \sum_{n \in S_+} -\widetilde{\ell}_n'^{(s)} \langle \widetilde{\boldsymbol{\mu}}_+, \widetilde{\boldsymbol{\xi}}_{n,i}^{(s)}\rangle\cdot \frac{\exp(\langle\boldsymbol{q}_{n,i}^{(s)},\boldsymbol{k}_+^{(s)} \rangle)}{\exp(\langle\boldsymbol{q}_{n,i}^{(s)},\boldsymbol{k}_+^{(s)} \rangle) + \sum_{j=2}^M \exp(\langle\boldsymbol{q}_{n,i}^{(s)},\boldsymbol{k}_{n,j}^{(s)} \rangle)} \\
&\cdot \left( V_+^{(s)} \left( 1 - \frac{\exp(\langle\boldsymbol{q}_{n,i}^{(s)},\boldsymbol{k}_+^{(s)} \rangle)}{\exp(\langle\boldsymbol{q}_{n,i}^{(s)},\boldsymbol{k}_+^{(s)} \rangle) + \sum_{j=2}^M \exp(\langle\boldsymbol{q}_{n,i}^{(s)},\boldsymbol{k}_{n,j}^{(s)} \rangle)} \right) \right. \\
&\left. - \frac{1}{2} \cdot V_+^{(s)} \sum_{k=2}^M \frac{\exp(\langle\boldsymbol{q}_{n,i}^{(s)},\boldsymbol{k}_{n,k}^{(s)} \rangle)}{\exp(\langle\boldsymbol{q}_{n,i}^{(s)},\boldsymbol{k}_+^{(s)} \rangle) + \sum_{j=2}^M \exp(\langle\boldsymbol{q}_{n,i}^{(s)},\boldsymbol{k}_{n,j}^{(s)} \rangle)} \right)\\
&\geq \frac{\eta}{2NM} \sum_{n \in S_+} -\ell_n'^{(s)} (\|\boldsymbol{\mu}\|_2-\tau)^2 V_+^{(s)} \cdot \frac{\exp(\langle\boldsymbol{q}_+^{(s)},\boldsymbol{k}_+^{(s)} \rangle)}{\exp(\langle\boldsymbol{q}_+^{(s)},\boldsymbol{k}_+^{(s)} \rangle) + \sum_{j=2}^M \exp(\langle\boldsymbol{q}_+^{(s)},\boldsymbol{k}_{n,j}^{(s)} \rangle)}\\
&\cdot \left( 1 - \frac{\exp(\langle\boldsymbol{q}_+^{(s)},\boldsymbol{k}_+^{(s)} \rangle)}{\exp(\langle\boldsymbol{q}_+^{(s)},\boldsymbol{k}_+^{(s)} \rangle) + \sum_{j=2}^M \exp(\langle\boldsymbol{q}_+^{(s)},\boldsymbol{k}_{n,j}^{(s)} \rangle)} \right)\\
&+\sum_{i=2}^M(\|\boldsymbol{\mu}\|\tau+\sigma_p\tau\sqrt{2\log(4NM/\delta)}+\tau^2) \cdot \left( V_+^{(s)} \left( \frac{\exp(\langle\boldsymbol{q}_{n,i}^{(s)},\boldsymbol{k}_+^{(s)} \rangle)}{\exp(\langle\boldsymbol{q}_{n,i}^{(s)},\boldsymbol{k}_+^{(s)} \rangle) + \sum_{j=2}^{M} \exp(\langle\boldsymbol{q}_{n,i}^{(s)},\boldsymbol{k}_{n,j}^{(s)} \rangle)} \right. \right. \\
    & \left. - \left( \frac{\exp(\langle\boldsymbol{q}_{n,i}^{(s)},\boldsymbol{k}_+^{(s)} \rangle)}{\exp(\langle\boldsymbol{q}_{n,i}^{(s)},\boldsymbol{k}_+^{(s)} \rangle) + \sum_{j=2}^{M} \exp(\langle\boldsymbol{q}_{n,i}^{(s)},\boldsymbol{k}_{n,j}^{(s)} \rangle)} \right)^2 \right)
\end{align*}

where the first inequality is by $V_+^{(s)} \geq 3M \cdot |V_{n,i}^{(s)}|$, the second inequality is by the fact that the sum of attention equal to 1 and Lemma~\ref{lem:con_ineq}. Similarly, we have

\begin{align}
\beta_{+,+}^{(s)} \geq&
\frac{\eta}{2NM} \sum_{n \in S_+} -\widetilde{\ell}_n^{'(s)} (\|\boldsymbol{\mu}\|_2-\tau)^2 V_+^{(s)} \cdot \frac{\exp(\langle\boldsymbol{q}_+^{(s)},\boldsymbol{k}_+^{(s)} \rangle)}{\exp(\langle\boldsymbol{q}_+^{(s)},\boldsymbol{k}_+^{(s)} \rangle) + (M - 1) \exp(\max_j \langle\boldsymbol{q}_+^{(s)},\boldsymbol{k}_{n,j}^{(s)} \rangle)} \nonumber\\&
\cdot \frac{\exp(\max_j \langle\boldsymbol{q}_+^{(s)},\boldsymbol{k}_{n,j}^{(s)} \rangle)}{\exp(\langle\boldsymbol{q}_+^{(s)},\boldsymbol{k}_+^{(s)} \rangle) + (M - 1) \exp(\max_j \langle\boldsymbol{q}_+^{(s)},\boldsymbol{k}_{n,j}^{(s)} \rangle)} +\text{\{lower order term\}}
\end{align}

\begin{align}
\alpha_{-,-}^{(s)} \geq&
\frac{\eta}{2NM} \sum_{n \in S_-} \widetilde{\ell}_n^{'(s)} (\|\boldsymbol{\mu}\|_2-\tau)^2 V_-^{(s)} \cdot \frac{\exp(\langle\boldsymbol{q}_-^{(s)},\boldsymbol{k}_-^{(s)} \rangle)}{\exp(\langle\boldsymbol{q}_-^{(s)},\boldsymbol{k}_-^{(s)} \rangle) + (M - 1) \exp(\max_j \langle\boldsymbol{q}_-^{(s)},\boldsymbol{k}_{n,j}^{(s)} \rangle)} \nonumber\\&
\cdot \frac{\exp(\max_j \langle\boldsymbol{q}_-^{(s)},\boldsymbol{k}_{n,j}^{(s)} \rangle)}{\exp(\langle\boldsymbol{q}_-^{(s)},\boldsymbol{k}_-^{(s)} \rangle) + (M - 1) \exp(\max_j \langle\boldsymbol{q}_-^{(s)},\boldsymbol{k}_{n,j}^{(s)} \rangle)} +\text{\{lower order term\}}
\end{align}

Similarly, by applying the update rules in Section~\ref{sec:update:innerp}, we can derive the following bounds on $\alpha$ and $\beta$.

\begin{equation*}
\begin{aligned}
    &\alpha_{+,+}^{(s)}, \alpha_{-,-}^{(s)}, \beta_{+,+}^{(s)}, \beta_{-,-}^{(s)}, \alpha_{n,i,+}^{(s)}, \alpha_{n,i,-}^{(s)}, \beta_{n,+,i}^{(s)}, \beta_{n,-,i}^{(s)} \geq 0, \\
    &\alpha_{n,+,i}^{(s)}, \alpha_{n,-,i}^{(s)}, \alpha_{n,i,n,j}^{(s)}, \beta_{n,i,+}^{(s)}, \beta_{n,i,-}^{(s)}, \beta_{n,j,n,i}^{(s)} \leq 0.
\end{aligned}
\end{equation*}

\subsection{Lower Bounds of $\langle\mathbf{q},\mathbf{k}\rangle$}\label{sec:low_qk}
In order to give the lower bounds for $\langle q, k \rangle$, we need to rewrite the bounds of $\alpha$ and $\beta$ in a more concise form. We first expand the equations in~\ref{sec:low_ab} under the assumption that $\mathcal{B}(s)$ {and $\mathcal{E}(s)$} holds for $s \in [T_1, t]$.

\begin{align*}
    \alpha_{+,+}^{(s)} &\geq \frac{\eta}{2NM} \sum_{n \in S_+} -\widetilde{\ell}_n^{'(s)} (\|\boldsymbol{\mu}\|_2-\tau)^2 V_+^{(s)} \cdot \frac{\exp(\langle\boldsymbol{q}_+^{(s)},\boldsymbol{k}_+^{(s)} \rangle)}{\exp(\langle\boldsymbol{q}_+^{(s)},\boldsymbol{k}_+^{(s)} \rangle) + \sum_{j'=2}^{M} \exp(\langle\boldsymbol{q}_+^{(s)},\boldsymbol{k}_{n,j'}^{(s)} \rangle)} \nonumber \\
    &\quad \cdot \left(1 - \frac{\exp(\langle\boldsymbol{q}_+^{(s)},\boldsymbol{k}_+^{(s)} \rangle)}{\exp(\langle\boldsymbol{q}_+^{(s)},\boldsymbol{k}_+^{(s)} \rangle) + \sum_{j'=2}^{M} \exp(\langle\boldsymbol{q}_+^{(s)},\boldsymbol{k}_{n,j'}^{(s)} \rangle)}\right) \nonumber \\  
    &+\sum_{i=2}^M(\|\boldsymbol{\mu}\|_2\tau+\sigma_p\tau\sqrt{2\log(4NM/\delta)}+\tau^2) \cdot \left( V_+^{(s)} \left( \frac{\exp(\langle\boldsymbol{q}_{n,i}^{(s)},\boldsymbol{k}_+^{(s)} \rangle)}{\exp(\langle\boldsymbol{q}_{n,i}^{(s)},\boldsymbol{k}_+^{(s)} \rangle) + \sum_{j=2}^{M} \exp(\langle\boldsymbol{q}_{n,i}^{(s)},\boldsymbol{k}_{n,j}^{(s)} \rangle)} \right. \right. \\
    & \left. - \left( \frac{\exp(\langle\boldsymbol{q}_{n,i}^{(s)},\boldsymbol{k}_+^{(s)} \rangle)}{\exp(\langle\boldsymbol{q}_{n,i}^{(s)},\boldsymbol{k}_+^{(s)} \rangle) + \sum_{j=2}^{M} \exp(\langle\boldsymbol{q}_{n,i}^{(s)},\boldsymbol{k}_{n,j}^{(s)} \rangle)} \right)^2 \right) \\
    &= \frac{\eta}{2NM} \sum_{n \in S_+} -\widetilde{\ell}_n^{'(s)} (\|\boldsymbol{\mu}\|_2-\tau)^2 V_+^{(s)} \cdot \frac{\exp(\langle\boldsymbol{q}_+^{(s)},\boldsymbol{k}_+^{(s)} \rangle)}{\exp(\langle\boldsymbol{q}_+^{(s)},\boldsymbol{k}_+^{(s)} \rangle) + \sum_{j'=2}^{M} \exp(\langle\boldsymbol{q}_+^{(s)},\boldsymbol{k}_{n,j'}^{(s)} \rangle)} \nonumber \\
    &\quad \cdot \frac{\sum_{j=2}^{M} \exp(\langle\boldsymbol{q}_+^{(s)},\boldsymbol{k}_{n,j}^{(s)} \rangle)}{\exp(\langle\boldsymbol{q}_+^{(s)},\boldsymbol{k}_+^{(s)} \rangle) + \sum_{j'=2}^{M} \exp(\langle\boldsymbol{q}_+^{(s)},\boldsymbol{k}_{n,j'}^{(s)} \rangle)}+\{\text{lower term}\} \nonumber \\
    &\geq \frac{\eta}{2NM} - \widetilde{\ell}_n^{'(s)} (\|\boldsymbol{\mu}\|_2-\tau)^2 V_+^{(s)} \cdot \frac{\exp(\langle\boldsymbol{q}_+^{(s)},\boldsymbol{k}_+^{(s)} \rangle)}{\exp(\langle\boldsymbol{q}_+^{(s)},\boldsymbol{k}_+^{(s)} \rangle) + \sum_{j'=2}^{M} \exp(\langle\boldsymbol{q}_+^{(s)},\boldsymbol{k}_{n,j'}^{(s)} \rangle)} \nonumber \\
    &\quad \cdot \frac{\exp(\langle\boldsymbol{q}_+^{(s)},\boldsymbol{k}_{n,j}^{(s)} \rangle)}{\exp(\langle\boldsymbol{q}_+^{(s)},\boldsymbol{k}_+^{(s)} \rangle) + \sum_{j'=2}^{M} \exp(\langle\boldsymbol{q}_+^{(s)},\boldsymbol{k}_{n,j'}^{(s)} \rangle)}+\{\text{lower term}\} \nonumber \\
    &\geq \frac{\eta}{2NM} - \widetilde{\ell}_n^{'(s)} (\|\boldsymbol{\mu}\|_2-\tau)^2 V_+^{(s)} \cdot \left(\frac{1}{M} - o(1)\right) \cdot \frac{\exp(\langle\boldsymbol{q}_+^{(s)},\boldsymbol{k}_{n,j}^{(s)} \rangle)}{\exp(\langle\boldsymbol{q}_+^{(s)},\boldsymbol{k}_+^{(s)} \rangle) + \sum_{j'=2}^{M} \exp(\langle\boldsymbol{q}_+^{(s)},\boldsymbol{k}_{n,j'}^{(s)} \rangle)} \nonumber \\
    &\geq \frac{\eta}{2NM} - \widetilde{\ell}_n^{'(s)} (\|\boldsymbol{\mu}\|_2-\tau)^2 V_+^{(s)} \cdot \left(\frac{1}{M} - o(1)\right) \cdot \frac{\exp(\langle\boldsymbol{q}_+^{(s)},\boldsymbol{k}_{n,j}^{(s)} \rangle)}{C\exp(\langle\boldsymbol{q}_+^{(s)},\boldsymbol{k}_+^{(s)} \rangle) }+\{\text{lower term}\} \nonumber\\
    &\geq \frac{\eta^2 C_5 (\|\boldsymbol{\mu}\|_2-\tau)^4 \|{\boldsymbol{w}_O}\|_2^2 (s - T_1)}{N} \cdot \frac{1}{\exp(\Lambda_{n,+,j}^{(s)})},
\end{align*}

where the first inequality is by Lemma~\ref{lemma:rsoftmax}(as $\mathcal{E}(s)$ holds) and $\text{softmax}(\langle\boldsymbol{q}_+^{(s)},\boldsymbol{k}_+^{(s)} \rangle) \geq \left(\frac{1}{M} - o(1)\right)$. In the fourth inequality, by $\langle\boldsymbol q^{(T_1)},\boldsymbol k^{(T_1)} \rangle = o(1)$ and the monotonicity of $\langle\boldsymbol{q}_+^{(s)},\boldsymbol{k}_+^{(s)} \rangle$ ($\langle\boldsymbol{q}_+^{(s)},\boldsymbol{k}_+^{(s)} \rangle$ is increasing and $\langle\boldsymbol{q}_+^{(s)},\boldsymbol{k}_{n,j}^{(s)} \rangle$ is decreasing), there exist a constant $C$ such that $C \exp(\langle\boldsymbol{q}_+^{(s)},\boldsymbol{k}_+^{(s)} \rangle) \geq \exp(\langle\boldsymbol{q}_+^{(s)},\boldsymbol{k}_+^{(s)} \rangle) + \sum_{j'=2}^{M} \exp(\langle\boldsymbol{q}_+^{(s)},\boldsymbol{k}_{n,j'}^{(s)} \rangle)$. In the last inequality, we plugging the lower bounds of $V_+^{(s)}$ and $-\widetilde{\ell}_n^{'(s)}$ and then absorb all the constant factors. Similarly, we have

\begin{equation}
    \beta_{+,+}^{(s)} \geq \frac{\eta^2 C_5 (\|\boldsymbol{\mu}\|_2-\tau)^4 \|{\boldsymbol{w}_O}\|_2^2 (s - T_1)}{N} \cdot \frac{1}{\exp(\Lambda_{n,+,j}^{(s)})},
\end{equation}

\begin{equation}
    \alpha_{-,-}^{(s)} \geq \frac{\eta^2 C_5 (\|\boldsymbol{\mu}\|_2-\tau)^4 \|{\boldsymbol{w}_O}\|_2^2 (s - T_1)}{N} \cdot \frac{1}{\exp(\Lambda_{n,-,j}^{(s)})},
\end{equation}

\begin{equation}
    \beta_{-,-}^{(s)} \geq \frac{\eta^2 C_5 (\|\boldsymbol{\mu}\|_2-\tau)^4 \|{\boldsymbol{w}_O}\|_2^2 (s - T_1)}{N} \cdot \frac{1}{\exp(\Lambda_{n,-,j}^{(s)})}.
\end{equation}
With the concise lower bounds for $\alpha$ and $\beta$ above and proposition $\mathcal{C}(s)$, we will give the lower bounds for the dynamics of $\langle\boldsymbol q,\boldsymbol k\rangle$.

\begin{align}
    \langle\boldsymbol{q}_+^{(s+1)},\boldsymbol{k}_+^{(s+1)} \rangle - \langle\boldsymbol{q}_+^{(s)},\boldsymbol{k}_+^{(s)} \rangle &= \alpha_{+,+}^{(s)} \|\boldsymbol{k}_+^{(s)}\|_2^2 + \sum_{n \in S_+} \sum_{i=2}^{M} \alpha_{n,+,i}^{(s)} \langle\boldsymbol{k}_+^{(s)},\boldsymbol{k}_{n,i}^{(s)} \rangle \nonumber \\
    &\quad + \beta_{+,+}^{(s)} \|\boldsymbol{q}_+^{(s)}\|_2^2 + \sum_{n \in S_+} \sum_{i=2}^{M} \beta_{n,+,i}^{(s)} \langle\boldsymbol{q}_+^{(s)},\boldsymbol{q}_{n,i}^{(s)} \rangle \nonumber \\
    &\quad + \left( \alpha_{+,+}^{(s)}\boldsymbol{k}_+^{(s)} + \sum_{n \in S_+} \sum_{i=2}^{M} \alpha_{n,+,i}^{(s)}\boldsymbol{k}_{n,i}^{(s)} \right) \cdot \left( \beta_{+,+}^{(s)}\boldsymbol{q}_+^{(s)} + \sum_{n \in S_+} \sum_{i=2}^{M} \beta_{n,+,i}^{(s)}\boldsymbol{q}_{n,i}^{(s)} \right) \nonumber \\
    &= \alpha_{+,+}^{(s)} \|\boldsymbol{k}_+^{(s)}\|_2^2 + \beta_{+,+}^{(s)} \|\boldsymbol{q}_+^{(s)}\|_2^2 + \{\text{lower order term}\} \nonumber \\
    &\geq \frac{2 \eta^2 C_5 (\|\boldsymbol{\mu}\|_2-\tau)^4 \|{\boldsymbol{w}_O}\|_2^2 (s - T_1)}{N} \cdot \frac{1}{\exp(\Lambda_{n,+,j}^{(s)})} \cdot \Theta(\|\boldsymbol{\mu}\|_2^2 \sigma_h^2 d_h) \nonumber \\
    &\quad + \{\text{lower order term}\} \nonumber \\
    &\geq \frac{\eta^2 C_6 (\|\boldsymbol{\mu}\|_2-\tau)_2^4 \|\boldsymbol{\mu}\|_2^2 \|{\boldsymbol{w}_O}\|_2^2 \sigma_h^2 d_h (s - T_1)}{N} \cdot \frac{1}{\exp(\Lambda_{n,+,j}^{(s)})},
\end{align}
Similarly, we have the lower bounds for the dynamics of other $\langle\boldsymbol q,\boldsymbol k\rangle$.
\subsection{Upper Bounds of $\langle q, k \rangle$}\label{sec:f6}

In order to give the upper bounds of $\langle q, k \rangle$ in stage II, we need to give the upper bounds of $\alpha$ and $\beta$ based on the equations in Section~\ref{sec:cal_ab} under the assumption that $\mathcal{D}(T_1), \ldots, \mathcal{D}(s-1)$ hold for $s \in [T_1, t]$.

\begin{equation}
\begin{aligned}
    \alpha_{+,+}^{(s)} &\leq \frac{\eta}{NM} \sum_{n \in S_+} -\widetilde{\ell}_n'(\theta) \langle \widetilde{\boldsymbol{\mu}}_+,\widetilde{\boldsymbol{\mu}}_+^{(t)} \rangle \\
&\quad \cdot \left( V_{+}^{(t)} \left( \frac{\exp(\langle\boldsymbol{q}_{+}^{(t)},\boldsymbol{k}_{+}^{(t)} \rangle)}{\exp(\langle\boldsymbol{q}_{+}^{(t)},\boldsymbol{k}_{+}^{(t)} \rangle) + \sum_{j=2}^M \exp(\langle\boldsymbol{q}_{+}^{(t)},\boldsymbol{k}_{n,j}^{(t)} \rangle)} \right. \right. \\
&\quad \left. \left. - \left( \frac{\exp(\langle\boldsymbol{q}_{+}^{(t)},\boldsymbol{k}_{+}^{(t)} \rangle)}{\exp(\langle\boldsymbol{q}_{+}^{(t)},\boldsymbol{k}_{+}^{(t)} \rangle) + \sum_{j=2}^M \exp(\langle\boldsymbol{q}_{+}^{(t)},\boldsymbol{k}_{n,j}^{(t)} \rangle)} \right)^2 \right) \right. \\
&\quad \left. - \sum_{i=2}^M \left( V_{n,i}^{(t)} \cdot \frac{\exp(\langle\boldsymbol{q}_{+}^{(t)},\boldsymbol{k}_{+}^{(t)} \rangle)}{\exp(\langle\boldsymbol{q}_{+}^{(t)},\boldsymbol{k}_{+}^{(t)} \rangle) + \sum_{j=2}^M \exp(\langle\boldsymbol{q}_{+}^{(t)},\boldsymbol{k}_{n,j}^{(t)} \rangle)} \right. \right. \\
&\quad \left. \left. \cdot \frac{\exp(\langle\boldsymbol{q}_{+}^{(t)},\boldsymbol{k}_{n,i}^{(t)} \rangle)}{\exp(\langle\boldsymbol{q}_{+}^{(t)},\boldsymbol{k}_{+}^{(t)} \rangle) + \sum_{j=2}^M \exp(\langle\boldsymbol{q}_{+}^{(t)},\boldsymbol{k}_{n,j}^{(t)} \rangle)} \right) \right)
\\&+\sum_{i=2}^M\langle\widetilde{\boldsymbol{\mu}}_+,\widetilde{\boldsymbol{\xi}}_{n,i}^{(t)}\rangle \cdot \left( V_+^{(t)} \left( \frac{\exp(\langle\boldsymbol{q}_{n,i}^{(t)},\boldsymbol{k}_+^{(t)} \rangle)}{\exp(\langle\boldsymbol{q}_{n,i}^{(t)},\boldsymbol{k}_+^{(t)} \rangle) + \sum_{j=2}^{M} \exp(\langle\boldsymbol{q}_{n,i}^{(t)},\boldsymbol{k}_{n,j}^{(t)} \rangle)} \right. \right. \\
& \left. - \left( \frac{\exp(\langle\boldsymbol{q}_{n,i}^{(t)},\boldsymbol{k}_+^{(t)} \rangle)}{\exp(\langle\boldsymbol{q}_{n,i}^{(t)},\boldsymbol{k}_+^{(t)} \rangle) + \sum_{j=2}^{M} \exp(\langle\boldsymbol{q}_{n,i}^{(t)},\boldsymbol{k}_{n,j}^{(t)} \rangle)} \right)^2 \right) \\
& - \sum_{k=2}^{M} \left( V_{n,i}^{(t)} \cdot \frac{\exp(\langle\boldsymbol{q}_{n,i}^{(t)},\boldsymbol{k}_+^{(t)} \rangle)}{\exp(\langle\boldsymbol{q}_{n,i}^{(t)},\boldsymbol{k}_+^{(t)} \rangle) + \sum_{j=2}^{M} \exp(\langle\boldsymbol{q}_{n,i}^{(t)},\boldsymbol{k}_{n,j}^{(t)} \rangle)} \right. \\
& \left. \cdot \frac{\exp(\langle\boldsymbol{q}_{n,i}^{(t)},\boldsymbol{k}_{n,k}^{(t)} \rangle)}{\exp(\langle\boldsymbol{q}_{n,i}^{(t)},\boldsymbol{k}_+^{(t)} \rangle) + \sum_{j=2}^{M} \exp(\langle\boldsymbol{q}_{n,i}^{(t)},\boldsymbol{k}_{n,j}^{(t)} \rangle)} \right) \\
    &\leq \frac{\eta}{NM} \sum_{n \in S_+}  (\|\boldsymbol{\mu}\|_2+\tau)^2 (V_+^{(s)} 
    \cdot \frac{\exp(\langle\boldsymbol{q}_+^{(s)},\boldsymbol{k}_+^{(s)} \rangle)}{\exp(\langle\boldsymbol{q}_+^{(s)},\boldsymbol{k}_+^{(s)} \rangle) + \sum_{j=2}^M \exp(\langle\boldsymbol{q}_+^{(s)},\boldsymbol{k}_{n,j}^{(s)} \rangle)} \\
    &\quad + \max_{i=2} |V_{n,i}^{(s)}| \cdot \frac{\sum_{j=2}^{M} \exp(\langle\boldsymbol{q}_+^{(s)},\boldsymbol{k}_{n,j}^{(s)} \rangle)}{\exp(\langle\boldsymbol{q}_+^{(s)},\boldsymbol{k}_+^{(s)} \rangle) + \sum_{j=2}^{M} \exp(\langle\boldsymbol{q}_+^{(s)},\boldsymbol{k}_{n,j}^{(s)} \rangle)}) \\
    &\quad+\sum_{i=2}^M  (\|\boldsymbol{\mu}\|\tau+\sigma_p\tau\sqrt{2\log(4NM/\delta)}+\tau^2) (V_+^{(s)} 
    \cdot \frac{\exp(\langle\boldsymbol{q}_{n,i}^{(s)},\boldsymbol{k}_+^{(s)} \rangle)}{\exp(\langle\boldsymbol{q}_{n,i}^{(s)},\boldsymbol{k}_+^{(s)} \rangle) + \sum_{j=2}^M \exp(\langle\boldsymbol{q}_{n,i}^{(s)},\boldsymbol{k}_{n,j}^{(s)} \rangle)} \\
    &\quad + \max_{i=2} |V_{n,i}^{(s)}| \cdot \frac{\sum_{j=2}^{M} \exp(\langle\boldsymbol{q}_{n,i}^{(s)},\boldsymbol{k}_{n,j}^{(s)} \rangle)}{\exp(\langle\boldsymbol{q}_{n,i}^{(s)},\boldsymbol{k}_+^{(s)} \rangle) + \sum_{j=2}^{M} \exp(\langle\boldsymbol{q}_{n,i}^{(s)},\boldsymbol{k}_{n,j}^{(s)} \rangle)}) \\
    &\leq \frac{\eta}{NM} \cdot \frac{3N}{4} \cdot (\|\boldsymbol{\mu}\|_2+\tau)^2 \cdot \left( V_+^{(s)} \cdot \frac{C}{\exp(\langle\boldsymbol{q}_+^{(s)},\boldsymbol{k}_+^{(s)} \rangle)} + \max_i |V_{n,i}^{(s)}| \cdot \frac{C}{\exp(\langle\boldsymbol{q}_+^{(s)},\boldsymbol{k}_+^{(s)} \rangle)} \right) \\
    &+\sum_{i=2}^M(\|\boldsymbol{\mu}\|\tau+\sigma_p\tau\sqrt{2\log(4NM/\delta)}+\tau^2)\cdot \left( V_+^{(s)} \cdot \frac{C}{\exp(\langle\boldsymbol{q}_{n,i}^{(s)},\boldsymbol{k}_+^{(s)} \rangle)} + \max_i |V_{n,i}^{(s)}| \cdot \frac{C}{\exp(\langle\boldsymbol{q}_{n,i}^{(s)},\boldsymbol{k}_+^{(s)} \rangle)} \right)\\
    &\leq \frac{\eta^2 C_9 (\|\boldsymbol{\mu}\|_2+\tau)^2}{\exp(\langle\boldsymbol{q}_+^{(s)},\boldsymbol{k}_+^{(s)} \rangle)}+\sum_{i=2}^M\frac{\eta^2 C_9 (\|\boldsymbol{\mu}\|_2\tau+\sigma_p\tau\sqrt{2\log(4NM/\delta)}+\tau^2)}{\exp(\langle\boldsymbol{q}_{n,i}^{(s)},\boldsymbol{k}_+^{(s)} \rangle)}
\end{aligned}
\end{equation}

where the first inequality is by $-\widetilde{\ell}_n^{'(s)} \leq 1$ and $\text{softmax}(\langle\boldsymbol{q}_+^{(s)},\boldsymbol{k}_+^{(s)} \rangle) \leq 1$. For the second inequality, we first consider
\(
\frac{\sum_{j=2}^{M} \exp(\langle\boldsymbol{q}_+^{(s)},\boldsymbol{k}_{n,j}^{(s)} \rangle)}{\exp(\langle\boldsymbol{q}_+^{(s)},\boldsymbol{k}_+^{(s)} \rangle) + \sum_{j=2}^{M} \exp(\langle\boldsymbol{q}_+^{(s)},\boldsymbol{k}_{n,j}^{(s)} \rangle)} 
\leq \frac{\sum_{j=2}^{M} \exp(\langle\boldsymbol{q}_+^{(s)},\boldsymbol{k}_{n,j}^{(s)} \rangle)}{\exp(\langle\boldsymbol{q}_+^{(s)},\boldsymbol{k}_+^{(s)} \rangle)}.
\)

Then by the monotonicity of $\langle\boldsymbol{q}_+^{(s)},\boldsymbol{k}_{n,j}^{(s)} \rangle$ and $\langle\boldsymbol{q}_+^{(T_1)},\boldsymbol{k}_{n,j}^{(T_1)} \rangle = o(1)$ we have 
\(
\sum_{j=2}^{M} \exp(\langle\boldsymbol{q}_+^{(s)},\boldsymbol{k}_{n,j}^{(s)} \rangle) \leq C
\)
for $s \in [T_1, t]$. The last inequality is by $V_+^{(s)}, V_{n,i}^{(s)} = o(1)$ for $s \in [T_1, t]$ and absorbing the constant factors. Similarly, we have

Similar to Section~\ref{sec:low_qk}, we apply the bounds of $\alpha$ and $\beta$ above to give the upper bounds for the dynamics $\langle\boldsymbol q,\boldsymbol k \rangle$.

\begin{equation}
\begin{aligned}
    \langle &\boldsymbol{q}_+^{(s+1)},\boldsymbol{k}_+^{(s+1)} \rangle - \langle\boldsymbol{q}_+^{(s)},\boldsymbol{k}_+^{(s)} \rangle 
    = \alpha_{+,+}^{(s)} \|\boldsymbol{k}_+^{(s)} \|_2^2 
    + \sum_{n \in S_+} \sum_{i=2}^{M} \alpha_{n,+,i}^{(s)} \langle\boldsymbol{k}_+^{(s)},\boldsymbol{k}_{n,i}^{(s)} \rangle \\
    &\quad + \beta_{+,+}^{(s)} \|\boldsymbol{q}_+^{(s)} \|_2^2 
    + \sum_{n \in S_+} \sum_{i=2}^{M} \beta_{n,+,i}^{(s)} \langle\boldsymbol{q}_+^{(s)},\boldsymbol{q}_{n,i}^{(s)} \rangle \\
    &\quad +\left( 
    \alpha_{+,+}^{(s)}\boldsymbol{k}_+^{(s)} 
    + \sum_{n \in S_+} \sum_{i=2}^{M} \alpha_{n,+,i}^{(s)}\boldsymbol{k}_{n,i}^{(s)}
\right)^\top
\\&\quad\cdot
\left(
    \beta_{+,+}^{(s)}\boldsymbol{q}_+^{(s)} 
    + \sum_{n \in S_+} \sum_{i=2}^{M} \beta_{n,+,i}^{(s)}\boldsymbol{q}_{n,i}^{(s)}
\right) \\
    &\quad + \sum_{n \in S_+} \sum_{i=2}^{M} \alpha_{n,+,i}^{(s)} \|\boldsymbol{k}_{n,i}^{(s)}\|_2^2 \\
    &= \alpha_{+,+}^{(s)} \|\boldsymbol{k}_+^{(s)} \|_2 
    + \beta_{+,+}^{(s)} \|\boldsymbol{q}_+^{(s)} \|_2 
    + \text{\{lower order terms\}} \\
    &\leq (\frac{\eta^2 C_9 (\|\boldsymbol{\mu}\|_2+\tau)^2}{\exp(\langle\boldsymbol{q}_+^{(s)},\boldsymbol{k}_+^{(s)} \rangle)}+\sum_{i=2}^M\frac{\eta^2 C_9 (\|\boldsymbol{\mu}\|_2\tau+\sigma_p\tau\sqrt{2\log(4NM/\delta)}+\tau^2)}{\exp(\langle\boldsymbol{q}_{n,i}^{(s)},\boldsymbol{k}_+^{(s)} \rangle)}) 
    \cdot \Theta(\|\boldsymbol{\mu}\|_2^2 \sigma_h^2 d_h) 
    + \text{\{lower order terms\}} \\
    &\leq \frac{\eta C_{10} \|\boldsymbol{\mu}\|_2^2(\|\boldsymbol{\mu}\|_2+\tau)^2 \sigma_h^2 d_h}{\exp(\langle\boldsymbol{q}_+^{(s)},\boldsymbol{k}_+^{(s)} \rangle)}+\sum_{i=2}^M\frac{\eta C_{10} \|\boldsymbol{\mu}\|_2^2(\|\boldsymbol{\mu}\|_2\tau+\sigma_p\tau\sqrt{2\log(4NM/\delta)}+\tau^2) \sigma_h^2 d_h}{\exp(\langle\boldsymbol{q}_{n,i}^{(s)},\boldsymbol{k}_+^{(s)} \rangle)}.
\end{aligned}
\end{equation}
Similarly, we have the upper bounds for the dynamics of other $\langle\boldsymbol q,\boldsymbol k \rangle$.

\subsection{Bounds for the Sum of $\alpha$ and $\beta$}\label{sec:f7}

The gradients of the inner products of $q$ and $k$ contain a lot of coefficients $\alpha$ and $\beta$, and in order to conveniently give the upper bounds of some lower order inner products, we will give upper bounds for the summation of $\alpha$ and $\beta$ (e.g. $\sum_{s=T_1}^{t} |\alpha_{+,+}^{(s)}|$).

Note that in the Jacobi matrix of the Softmax function, the elements on the diagonal are $\text{softmax}(a_i) \cdot (1 - \text{softmax}(a_i))$ and the elements on the off-diagonal are $\text{softmax}(a_i) \cdot \text{softmax}(a_j)$. In Stage II, the attentions on signals $\boldsymbol{\mu}_{\pm}$ increase and the attentions on noises $\boldsymbol{\xi}$ decrease, then we can consider the following cases:

\begin{itemize}
    \item if $a_i = \langle\boldsymbol{q}_{+},\boldsymbol{k}_{+} \rangle$ or $a_i = \langle\boldsymbol q_i,\boldsymbol{k}_{+} \rangle$, $\text{softmax}(a_i)$ has a constant upper bound 1, $\left(1 - \text{softmax}(a_i)\right)$ decreases as $\text{softmax}(a_i)$ increases. So the upper bound of $\text{softmax}(a_i) \cdot \left(1 - \text{softmax}(a_i)\right)$ decreases as $\text{softmax}(a_i)$ increases.
    \item if $a_i = \langle\boldsymbol{q}_{+},\boldsymbol k_j \rangle$ or $a_i = \langle\boldsymbol q_i,\boldsymbol k_j \rangle$, $\left(1 - \text{softmax}(a_i)\right)$ has a constant upper bound 1. So the upper bound of $\text{softmax}(a_i) \cdot \left(1 - \text{softmax}(a_i)\right)$ decreases as $\text{softmax}(a_i)$ decreases.
    \item if $a_j = \langle\boldsymbol{q}_{+},\boldsymbol k_j \rangle$ or $a_j = \langle\boldsymbol q_i,\boldsymbol k_j \rangle$, $\text{softmax}(a_i)$ has a constant upper bound 1. So the upper bound of $\text{softmax}(a_i) \cdot \text{softmax}(a_j)$ decreases as $\text{softmax}(a_j)$ decreases.
\end{itemize}

Based on the above cases, we first study the bounds of the following terms
\begin{itemize}
    \item $1 - \text{softmax}(\langle\boldsymbol{q}_+^{(s)},\boldsymbol{k}_+^{(s)} \rangle)$
    \item $1 - \text{softmax}(\langle\boldsymbol{q}_{n,i}^{(s)},\boldsymbol{k}_+^{(s)} \rangle)$
    \item $\text{softmax}(\langle\boldsymbol{q}_+^{(s)},\boldsymbol{k}_{n,j}^{(s)} \rangle)$
    \item $\text{softmax}(\langle\boldsymbol{q}_{n,i}^{(s)},\boldsymbol{k}_{n,j}^{(s)} \rangle)$
\end{itemize}

Note that $1 - \text{softmax}(\langle\boldsymbol{q}_+^{(s)},\boldsymbol{k}_+^{(s)} \rangle) = \sum_j \text{softmax}(\langle\boldsymbol{q}_+^{(s)},\boldsymbol{k}_{n,j}^{(s)} \rangle)$ and $1 - \text{softmax}(\langle\boldsymbol{q}_{n,i}^{(s)},\boldsymbol{k}_+^{(s)} \rangle) = \sum_j \text{softmax}(\langle\boldsymbol{q}_{n,i}^{(s)},\boldsymbol{k}_{n,j}^{(s)} \rangle)$, we only need to give the upper bounds for $\text{softmax}(\langle\boldsymbol{q}_+^{(s)},\boldsymbol{k}_{n,j}^{(s)} \rangle)$ and $\text{softmax}(\langle\boldsymbol{q}_{n,i}^{(s)},\boldsymbol{k}_{n,j}^{(s)} \rangle)$.

Assume that the propositions $\mathcal{B}(T_1), \ldots, \mathcal{B}(s), \mathcal{D}(T_1), \ldots, \mathcal{D}(s-1)$ hold $(s \in [T_1, t])$, we have
\begin{equation}
    |V_\pm^{(s)}|, |V_{n,i}^{(s)}| \leq O(d_h^{-\frac{1}{4}}) + \eta C_4 (\|\boldsymbol{\mu}\|_2+\tau)^2 \|\boldsymbol{w}_O\|_2^2 (s - T_1),
\end{equation}
\begin{equation}\label{eq:251}
    \Lambda_{n,\pm,j}^{(s)} \geq \log \left( \exp(\Lambda_{n,\pm,j}^{(T_1)}) + \frac{\eta^2 C_8  (\|\boldsymbol{\mu}\|_2-\tau)^2\|\boldsymbol{\mu}\|_2^2\|\boldsymbol{w}_O\|_2^2 d_h^{\frac{1}{2}}}{N \left( \log(6N^2M^2/\delta) \right)^2} \cdot (s - T_1)(s - T_1 - 1) \right),
\end{equation}
{\small\begin{equation}
    \Lambda_{n,i,\pm,j}^{(s)} \geq \log \left( \exp(\Lambda_{n,i,\pm,j}^{(T_1)}) + \frac{\eta^2 C_8 (\sigma_P^2 d +\sigma_p\tau\sqrt{2\log(4NM/\delta)}+\tau^2)\|\boldsymbol{\mu}\|_2^2 \|\boldsymbol{w}_O\|_2^2 d_h^{\frac{1}{2}}}{N \left( \log(6N^2M^2/\delta) \right)^2} \cdot (s - T_1)(s - T_1 - 1) \right),
\end{equation}}
for $i, j \in [M] \setminus \{1\}, n \in [N], s \in [T_1, t]$.

Then we have
\begin{equation}
\begin{aligned}
    &\frac{\exp(\langle\boldsymbol q_\pm^{(s)},\boldsymbol{k}_{n,j}^{(s)} \rangle)}
    {\exp(\langle\boldsymbol q_\pm^{(s)},\boldsymbol k_\pm^{(s)} \rangle) + \sum_{j'=2}^M \exp(\langle\boldsymbol q_\pm^{(s)},\boldsymbol{k}_{n,j'}^{(s)} \rangle)} \\
    &\quad \leq \frac{\exp(\langle\boldsymbol q_\pm^{(s)},\boldsymbol{k}_{n,j}^{(s)} \rangle)}
    {C \exp(\langle\boldsymbol q_\pm^{(s)},\boldsymbol k_\pm^{(s)} \rangle)} \\
    &\quad = \frac{1}{C \exp(\Lambda_{n,\pm,j}^{(s)})} \\
    &\quad \leq \frac{1}{C \exp(\Lambda_{n,\pm,j}^{(T_1)}) 
        + \dfrac{\eta^2 C_8 C (\|\boldsymbol{\mu}\|_2-\tau)^2\|\boldsymbol{\mu}\|_2^2 \|\boldsymbol{w}_O\|_2^2 d_h^{\tfrac{1}{2}}}
        {N \big( \log(6N^2M^2/\delta) \big)^2} \cdot (s - T_1)(s - T_1 - 1)} \\
    &\quad \leq \frac{1}{C_{13} 
        + \dfrac{\eta^2 C_{13}(\|\boldsymbol{\mu}\|_2-\tau)^2 \|\boldsymbol{\mu}\|_2^2 \|\boldsymbol{w}_O\|_2^2 d_h^{\tfrac{1}{2}}}
        {N \big( \log(6N^2M^2/\delta) \big)^2} \cdot (s - T_1)(s - T_1 - 1)}.
\end{aligned}
\end{equation}
For the first inequality, by $\langle\boldsymbol{q}^{(T_1)},\boldsymbol{k}^{(T_1)} \rangle = o(1)$ and the monotonicity of $\langle\boldsymbol{q}^{(s)},\boldsymbol{k}^{(s)} \rangle$ ($\langle\boldsymbol{q}^{(s)},\boldsymbol{k}^{(s)} \rangle$ is increasing and $\langle\boldsymbol{q}_\pm^{(s)},\boldsymbol{k}_{n,j}^{(s)} \rangle$ is decreasing), there exists a constant $C$ such that 
\(
C \exp(\langle\boldsymbol{q}_\pm^{(s)},\boldsymbol{k}_\pm^{(s)} \rangle) \geq \exp(\langle\boldsymbol{q}_\pm^{(s)},\boldsymbol{k}_\pm^{(s)} \rangle) + \sum_{j'=2}^M \exp(\langle\boldsymbol{q}_\pm^{(s)},\boldsymbol{k}_{n,j'}^{(s)} \rangle).
\). The second inequality is by plugging (\ref{eq:251}). For the last inequality, by $\Lambda_{n,\pm,j}^{(T_1)} = o(1)$, there exist a constant $C_{13}$ such that $C_{13} \leq C \exp(\Lambda_{n,\pm,j}^{(T_1)})$ and $C_{13} \leq C_8 C$. Similarly, we have
\begin{equation}
\begin{aligned}
    & \frac{\exp(\langle\boldsymbol{q}_{n,i}^{(s)},\boldsymbol{k}_{n,j}^{(s)} \rangle)}{\exp(\langle\boldsymbol{q}_{n,i}^{(s)},\boldsymbol{k}_+^{(s)} \rangle) + \sum_{j'=2}^M \exp(\langle\boldsymbol{q}_{n,i}^{(s)},\boldsymbol{k}_{n,j'}^{(s)} \rangle)} \\
    & \leq \frac{1}{C \exp(\Lambda_{n,i,+,j}^{(s)})} \\
    & \leq \frac{1}{C_{13} + \frac{\eta^2 C_{13} (\sigma_P^2 d+\sigma_p\tau\sqrt{2\log(4NM/\delta)}+\tau^2) \|\boldsymbol{\mu}\|_2^2 \|\boldsymbol{w}_O\|_2^2 d_h^{1/2}}{N \left( \log(6N^2M^2/\delta) \right)^2} \cdot (s - T_1)(s - T_1 - 1)}.
\end{aligned}
\end{equation}

Plugging above equations into the expressions of $\alpha, \beta$ we have

\begin{equation}\label{eq:255}
\begin{aligned}
    |\alpha_{+,+}^{(s)}| &\leq \Bigg|\frac{\eta}{NM} \sum_{n \in S_+} -\widetilde{\ell}_n'(\theta) \langle \widetilde{\boldsymbol{\mu}}_+,\widetilde{\boldsymbol{\mu}}_+^{(s)} \rangle \\
&\quad \cdot \left( V_{+}^{(s)} \left( \frac{\exp(\langle\boldsymbol{q}_{+}^{(s)},\boldsymbol{k}_{+}^{(s)} \rangle)}{\exp(\langle\boldsymbol{q}_{+}^{(s)},\boldsymbol{k}_{+}^{(s)} \rangle) + \sum_{j=2}^M \exp(\langle\boldsymbol{q}_{+}^{(s)},\boldsymbol{k}_{n,j}^{(s)} \rangle)} \right. \right. \\
&\quad \left. \left. - \left( \frac{\exp(\langle\boldsymbol{q}_{+}^{(s)},\boldsymbol{k}_{+}^{(s)} \rangle)}{\exp(\langle\boldsymbol{q}_{+}^{(s)},\boldsymbol{k}_{+}^{(s)} \rangle) + \sum_{j=2}^M \exp(\langle\boldsymbol{q}_{+}^{(s)},\boldsymbol{k}_{n,j}^{(s)} \rangle)} \right)^2 \right) \right. \\
&\quad \left. - \sum_{i=2}^M \left( V_{n,i}^{(s)} \cdot \frac{\exp(\langle\boldsymbol{q}_{+}^{(s)},\boldsymbol{k}_{+}^{(s)} \rangle)}{\exp(\langle\boldsymbol{q}_{+}^{(s)},\boldsymbol{k}_{+}^{(s)} \rangle) + \sum_{j=2}^M \exp(\langle\boldsymbol{q}_{+}^{(s)},\boldsymbol{k}_{n,j}^{(s)} \rangle)} \right. \right. \\
&\quad \left. \left. \cdot \frac{\exp(\langle\boldsymbol{q}_{+}^{(s)},\boldsymbol{k}_{n,i}^{(s)} \rangle)}{\exp(\langle\boldsymbol{q}_{+}^{(s)},\boldsymbol{k}_{+}^{(s)} \rangle) + \sum_{j=2}^M \exp(\langle\boldsymbol{q}_{+}^{(s)},\boldsymbol{k}_{n,j}^{(s)} \rangle)} \right) \right) \\
&+\sum_{i=2}^M\langle\widetilde{\boldsymbol{\mu}}_+,\widetilde{\boldsymbol{\xi}}_{n,i}^{(s)}\rangle \cdot \left( V_+^{(s)} \left( \frac{\exp(\langle\boldsymbol{q}_{n,i}^{(s)},\boldsymbol{k}_+^{(s)} \rangle)}{\exp(\langle\boldsymbol{q}_{n,i}^{(s)},\boldsymbol{k}_+^{(s)} \rangle) + \sum_{j=2}^{M} \exp(\langle\boldsymbol{q}_{n,i}^{(s)},\boldsymbol{k}_{n,j}^{(s)} \rangle)} \right. \right. \\
& \left. - \left( \frac{\exp(\langle\boldsymbol{q}_{n,i}^{(s)},\boldsymbol{k}_+^{(s)} \rangle)}{\exp(\langle\boldsymbol{q}_{n,i}^{(s)},\boldsymbol{k}_+^{(s)} \rangle) + \sum_{j=2}^{M} \exp(\langle\boldsymbol{q}_{n,i}^{(s)},\boldsymbol{k}_{n,j}^{(s)} \rangle)} \right)^2 \right) \\
& - \sum_{k=2}^{M} \left( V_{n,i}^{(s)} \cdot \frac{\exp(\langle\boldsymbol{q}_{n,i}^{(s)},\boldsymbol{k}_+^{(s)} \rangle)}{\exp(\langle\boldsymbol{q}_{n,i}^{(s)},\boldsymbol{k}_+^{(s)} \rangle) + \sum_{j=2}^{M} \exp(\langle\boldsymbol{q}_{n,i}^{(s)},\boldsymbol{k}_{n,j}^{(s)} \rangle)} \right. \\
& \left. \cdot \frac{\exp(\langle\boldsymbol{q}_{n,i}^{(s)},\boldsymbol{k}_{n,k}^{(s)} \rangle)}{\exp(\langle\boldsymbol{q}_{n,i}^{(s)},\boldsymbol{k}_+^{(s)} \rangle) + \sum_{j=2}^{M} \exp(\langle\boldsymbol{q}_{n,i}^{(s)},\boldsymbol{k}_{n,j}^{(s)} \rangle)} \right)
\Bigg| \\
    &\leq \frac{\eta (\|\boldsymbol{\mu}\|_2+\tau)^2 \cdot 3N}{NM} \cdot \left( O(d_h^{-\frac{1}{4}}) + \eta C_4 (\|\boldsymbol{\mu}\|_2+\tau)^2 \|\boldsymbol{w}_O\|_2^2 (s - T_1) \right) \\
    &\quad \cdot O\left( \frac{1}{C_{13} + \frac{\eta^2 C_{13} ((\|\boldsymbol{\mu}\|_2+\tau)^2-\tau)^2\|\boldsymbol{\mu}\|_2^2 \|\boldsymbol{w}_O\|_2^2 d_h^{\frac{1}{2}}}{N \left( \log(6N^2M^2/\delta) \right)^2} \cdot (s - T_1)(s - T_1 - 1)} \right) \\
    &= O\left( \frac{\eta \|\boldsymbol{\mu}\|_2^2 d_h^{-\frac{1}{4}}}{C_{13} + \frac{\eta^2 C_{13} (\|\boldsymbol{\mu}\|_2-\tau)^2\|\boldsymbol{\mu}\|_2^2 \|\boldsymbol{w}_O\|_2^2 d_h^{\frac{1}{2}}}{N \left( \log(6N^2M^2/\delta) \right)^2} \cdot (s - T_1)(s - T_1 - 1)} \right) \\
    &\quad + O\left( \frac{\eta^2 (\|\boldsymbol{\mu}\|_2+\tau)^4 \|\boldsymbol{w}_O\|_2^2 (s - T_1)}{C_{13} + \frac{\eta^2 C_{13} (\|\boldsymbol{\mu}\|_2-\tau)^2\|\boldsymbol{\mu}\|_2^2 \|\boldsymbol{w}_O\|_2^2 d_h^{\frac{1}{2}}}{N \left( \log(6N^2M^2/\delta) \right)^2} \cdot (s - T_1)(s - T_1 - 1)} \right) \\
    &= O\left( \eta (\|\boldsymbol{\mu}\|_2+\tau)^2 d_h^{-\frac{1}{4}} \right) + O\left( \frac{\eta^2 (\|\boldsymbol{\mu}\|_2+\tau)^4 \|\boldsymbol{w}_O\|_2^2 (s - T_1)}{C_{13} + \frac{\eta^2 C_{13} (\|\boldsymbol{\mu}\|_2-\tau)^2\|\boldsymbol{\mu}\|_2^2 \|\boldsymbol{w}_O\|_2^2 d_h^{\frac{1}{2}}}{N \left( \log(6N^2M^2/\delta) \right)^2} \cdot (s - T_1)(s - T_1 - 1)} \right).
\end{aligned}
\end{equation}

where the third equality is by
$
    \frac{\eta^2 C_{13} (\|\boldsymbol{\mu}\|_2-\tau)^2\|\boldsymbol{\mu}\|_2^2 \|\boldsymbol{w}_O\|_2^2 d_h^{\frac{1}{2}}}{N \left( \log(6N^2M^2/\delta) \right)^2} \cdot (s - T_1)(s - T_1 - 1) \geq 0 \text{ for } s \in [T_1, t].
$
Next, we give an upper bound for
$
    \frac{\eta^2 (\|\boldsymbol{\mu}\|_2+\tau)^4 \|\boldsymbol{w}_O\|_2^2 (s - T_1)}{C_{13} + \frac{\eta^2 C_{13} (\|\boldsymbol{\mu}\|_2-\tau)^2\|\boldsymbol{\mu}\|_2^2 \|\boldsymbol{w}_O\|_2^2 d_h^{\frac{1}{2}}}{N \left( \log(6N^2M^2/\delta) \right)^2} \cdot (s - T_1)(s - T_1 - 1)}
$
as follows:
\begin{equation}
\begin{aligned}
    & \frac{\eta^2 (\|\boldsymbol{\mu}\|_2+\tau)^4 \|\boldsymbol{w}_O\|_2^2 (s - T_1)}{C_{13} + \frac{\eta^2 C_{13} (\|\boldsymbol{\mu}\|_2-\tau)^2\|\boldsymbol{\mu}\|_2 \|\boldsymbol{w}_O\|_2^2 d_h^{\frac{1}{2}}}{N \left( \log(6N^2M^2/\delta) \right)^2} \cdot (s - T_1)(s - T_1 - 1)} \\
    & = \frac{\eta^2 (\|\boldsymbol{\mu}\|_2+\tau)^4 \|\boldsymbol{w}_O\|_2^2}{\frac{C_{13}}{(s - T_1)} + \frac{\eta^2 C_{13} (\|\boldsymbol{\mu}\|_2-\tau)^2\|\boldsymbol{\mu}\|_2^2 \|\boldsymbol{w}_O\|_2^2 d_h^{\frac{1}{2}}}{N \left( \log(6N^2M^2/\delta) \right)^2} \cdot (s - T_1) - \frac{\eta^2 C_{13} (\|\boldsymbol{\mu}\|_2-\tau)^2\|\boldsymbol{\mu}\|_2^2 \|\boldsymbol{w}_O\|_2^2 d_h^{\frac{1}{2}}}{N \left( \log(6N^2M^2/\delta) \right)^2}} \\
    & \leq \frac{\eta^2 (\|\boldsymbol{\mu}\|_2+\tau)^4 \|\boldsymbol{w}_O\|_2^2}{2 \sqrt{\frac{\eta^2 C_{13}^2 (\|\boldsymbol{\mu}\|_2-\tau)^2\|\boldsymbol{\mu}\|_2^2 \|\boldsymbol{w}_O\|_2^2 d_h^{\frac{1}{2}}}{N \left( \log(6N^2M^2/\delta) \right)^2}} - \frac{\eta^2 C_{13} (\|\boldsymbol{\mu}\|_2-\tau)^2\|\boldsymbol{\mu}\|_2^2 \|\boldsymbol{w}_O\|_2^2 d_h^{\frac{1}{2}}}{N \left( \log(6N^2M^2/\delta) \right)^2}} \\
    & = \frac{\eta^2 (\|\boldsymbol{\mu}\|_2+\tau)^4 \|\boldsymbol{w}_O\|_2^2}{\frac{2 \eta C_{13} (\|\boldsymbol{\mu}\|_2-\tau)\|\boldsymbol{\mu}\|_2 \|\boldsymbol{w}_O\|_2^2 d_h^{\frac{1}{4}}}{N^{\frac{1}{2}} \left( \log(6N^2M^2/\delta) \right)} - \frac{\eta^2 C_{13} (\|\boldsymbol{\mu}\|_2-\tau)^2\|\boldsymbol{\mu}\|_2^2 \|\boldsymbol{w}_O\|_2^2 d_h^{\frac{1}{2}}}{N \left( \log(6N^2M^2/\delta) \right)^2}} \\
    & = \frac{\eta^2 (\|\boldsymbol{\mu}\|_2+\tau)^4 \|\boldsymbol{w}_O\|_2^2}{\Theta \left( \frac{\eta (\|\boldsymbol{\mu}\|_2-\tau)\|\boldsymbol{\mu}\|_2 \|\boldsymbol{w}_O\|_2 d_h^{\frac{1}{4}}}{N^{\frac{1}{2}} \log(6N^2M^2/\delta)} \right)} \\
    & = O \left( \eta (\|\boldsymbol{\mu}\|_2+\tau)^2 N^{\frac{1}{2}} d_h^{-\frac{1}{4}} \log(6N^2M^2/\delta) \right),
\end{aligned}
\end{equation}
where the inequality is by $ax + \frac{b}{x} \geq 2\sqrt{ab}$ for $x > 0$, the third equality is by absorbing the lower order term $\frac{\eta^2 C_{13} (\|\boldsymbol{\mu}\|_2+\tau)^4 \|\boldsymbol{w}_O\|_2^2 d_h^{\frac{1}{2}}}{N \left( \log(6N^2M^2/\delta) \right)^2}$, the last equality is by $\|\boldsymbol{w}_O\|_2 = \Theta(1)$. Plugging this into (\ref{eq:255}) and get

\begin{equation}
\begin{aligned}
    |\alpha_{+,+}^{(s)}| &= O\left(\eta (\|\boldsymbol{\mu}\|_2+\tau)^2 d_h^{-\frac{1}{4}}\right) + O\left(\eta (\|\boldsymbol{\mu}\|_2+\tau)^2 N^{\frac{1}{2}} d_h^{-\frac{1}{4}} \log(6N^2M^2/\delta)\right) \\
    &= O\left(\eta (\|\boldsymbol{\mu}\|_2+\tau)^2 N^{\frac{1}{2}} d_h^{-\frac{1}{4}} \log(6N^2M^2/\delta)\right).
\end{aligned}
\end{equation}

Similarly, we have
\begin{equation}
\begin{aligned}
    &|\alpha_{-,-}^{(s)}|, |\beta_{+,+}^{(s)}|, |\beta_{-,-}^{(s)}| = O\left(\eta (\|\boldsymbol{\mu}\|_2+\tau)^2 N^{\frac{1}{2}} d_h^{-\frac{1}{4}} \log(6N^2M^2/\delta)\right), \\
    &|\alpha_{n,+,i}^{(s)}|, |\alpha_{n,-,i}^{(s)}| = O\left(\eta (\|\boldsymbol{\mu}\|_2+\tau)^2 N^{-\frac{1}{2}} d_h^{-\frac{1}{4}} \log(6N^2M^2/\delta)\right), \\
    &|\beta_{n,+,i}^{(s)}|, |\beta_{n,-,i}^{(s)}| = O\left(\frac{\eta (\|\boldsymbol{\mu}\|_2+\tau)^2\|\boldsymbol{\mu}\|_2 \log(6N^2M^2/\delta)}{\sigma_p d^{\frac{1}{2}} N^{\frac{1}{2}} d_h^{\frac{1}{4}}}\right) =\\
    &\qquad\qquad \qquad O\left(\eta (\|\boldsymbol{\mu}\|_2+\tau)^2 \cdot \text{SNR} \cdot N^{-\frac{1}{2}} d_h^{-\frac{1}{4}} \log(6N^2M^2/\delta)\right),
\end{aligned}
\end{equation}
for $i \in [M] \setminus \{1\}, n \in S_+$.

\begin{equation}
\begin{aligned}
    |\alpha_{n,i,+}^{(s)}|, |\alpha_{n,i,-}^{(s)}| &= O\left(\frac{\eta \|\boldsymbol{\mu}\|_2 \sigma_p d^{\frac{1}{2}} \log(6N^2M^2/\delta)}{N^{\frac{1}{2}} d_h^{\frac{1}{4}}}\right) = O\left(\eta (\|\boldsymbol{\mu}\|_2+\tau)^2 d_h^{-\frac{1}{4}} \log(6N^2M^2/\delta)\right),
\end{aligned}
\end{equation}
for $i \in [M] \setminus \{1\}, n \in S_+$, the last equality is by $N \cdot \text{SNR}^2 \geq \Omega(1)$.

\begin{equation}
\begin{aligned}
    |\beta_{n,i,+}^{(s)}|, |\beta_{n,i,-}^{(s)}| &= O\left(\frac{\eta \sigma_p^2 d \log(6N^2M^2/\delta)}{N^{\frac{1}{2}} d_h^{\frac{1}{4}}}\right) = O\left(\eta (\|\boldsymbol{\mu}\|_2+\tau)^2 N^{\frac{1}{2}} d_h^{-\frac{1}{4}} \log(6N^2M^2/\delta)\right),
\end{aligned}
\end{equation}

for $i \in [M] \setminus \{1\}, n \in S_+$, the last equality is by $N \cdot \text{SNR}^2 \geq \Omega(1)$.
\begin{equation}
\begin{aligned}
    |\alpha_{n,i,n,j}^{(s)}|, |\beta_{n,j,n,i}^{(s)}| &= O\left(\frac{\eta (\|\boldsymbol{\mu}\|_2+\tau) \sigma_p d^{\frac{1}{2}} \log(6N^2M^2/\delta)}{N^{\frac{1}{2}} d_h^{\frac{1}{4}}}\right) \\
    &= O\left(\eta (\|\boldsymbol{\mu}\|_2+\tau)^2 d_h^{-\frac{1}{4}} \log(6N^2M^2/\delta)\right),
\end{aligned}
\end{equation}

for $i, j \in [M] \setminus \{1\}, n \in [N]$, the last equality is by $N \cdot \text{SNR}^2 \geq \Omega(1)$.
\begin{equation}
\begin{aligned}
    |\alpha_{n,i,n',j}^{(s)}|, |\beta_{n,j,n',i}^{(s)}| &= O\left(\frac{\eta (\|\boldsymbol{\mu}\|_2+\tau) \sigma_p \log(6N^2M^2/\delta) \log(6N^2M^2/\delta)}{N^{\frac{1}{2}} d_h^{\frac{1}{4}}}\right) \\
    &= O\left(\eta (\|\boldsymbol{\mu}\|_2+\tau)^2 d^{-\frac{1}{2}} d_h^{-\frac{1}{4}} (\log(6N^2M^2/\delta))^2\right),
\end{aligned}
\end{equation}

for $i, j \in [M] \setminus \{1\}, n, n' \in [N], n \neq n'$, the last equality is by $N \cdot \text{SNR}^2 \geq \Omega(1)$. Taking a summation we obtain that
\begin{equation}
\begin{aligned}
    \sum_{s=T_1}^t |\alpha_{+,+}^{(s)}| &= O\left(\eta (\|\boldsymbol{\mu}\|_2+\tau)^2 \|\boldsymbol{w}_O\|_2^2 \log(6N^2M^2/\delta)\right) \cdot O\left(\eta (\|\boldsymbol{\mu}\|_2+\tau)^2 N^{\frac{1}{2}} d_h^{-\frac{1}{4}} \log(6N^2M^2/\delta)\right) \\
    &= O\left(N^{\frac{1}{2}} d_h^{-\frac{1}{4}}\right),
\end{aligned}
\end{equation}
where the last equality is by $\|\boldsymbol{w}_O\| = \Theta(1)$. Similarly, we have
\begin{equation}
\begin{aligned}
    \sum_{s=T_1}^t |\alpha_{-,-}^{(s)}|, \sum_{s=T_1}^t |\beta_{+,+}^{(s)}|, \sum_{s=T_1}^t |\beta_{-,-}^{(s)}|, \sum_{s=T_1}^t |\beta_{n,i,+}^{(s)}|, \sum_{s=T_1}^t |\beta_{n,i,-}^{(s)}| &= O\left(N^{\frac{1}{2}} d_h^{-\frac{1}{4}}\right),
\end{aligned}
\end{equation}
for $i \in [M] \setminus \{1\}, n \in S_+$.
\begin{equation}
\begin{aligned}
    \sum_{s=T_1}^t |\alpha_{n,+,i}^{(s)}|, \sum_{s=T_1}^t |\alpha_{n,-,i}^{(s)}| &= O\left(N^{-\frac{1}{2}} d_h^{-\frac{1}{4}}\right),
\end{aligned}
\end{equation}

for $i \in [M] \setminus \{1\}, n \in S_+$.
\begin{equation}
\begin{aligned}
    \sum_{s=T_1}^{t} |\beta_{n,+,i}^{(s)}|, \sum_{s=T_1}^{t} |\beta_{n,-,i}^{(s)}| &= O\left(\text{SNR} \cdot N^{-\frac{1}{2}} d_h^{-\frac{1}{4}}\right)
\end{aligned}
\end{equation}

for $i \in [M] \setminus \{1\}, n \in S_+$.
\begin{equation}
\begin{aligned}
    \sum_{s=T_1}^{t} |\alpha_{n,i,+}^{(s)}|, \sum_{s=T_1}^{t} |\alpha_{n,i,-}^{(s)}|, \sum_{s=T_1}^{t} |\alpha_{n,i,n,j}^{(s)}|, \sum_{s=T_1}^{t} |\beta_{n,j,n,i}^{(s)}| &= O\left(d_h^{-\frac{1}{4}}\right)
\end{aligned}
\end{equation}

for $i, j \in [M] \setminus \{1\}, n \in S_+$.
\begin{equation}
\begin{aligned}
    \sum_{s=T_1}^{t} |\alpha_{n,i,n',j}^{(s)}|, \sum_{s=T_1}^{t} |\beta_{n,j,n',i}^{(s)}| &= O\left(d^{-\frac{1}{2}} d_h^{-\frac{1}{4}} \log(6N^2M^2/\delta)\right)
\end{aligned}
\end{equation}

for $i, j \in [M] \setminus \{1\}, n, n' \in [N], n \neq n'$.

With these sums of $\alpha$ and $\beta$ above, we can easily prove Claim 3 and Claim 4.

\subsection{Proof of Claim 3}\label{sec:f8}

In this subsection, we assume that $\mathcal{E}(T_1), \ldots, \mathcal{E}(t)$ hold, and then proof that $\mathcal{C}(t+1)$ is true with the result of \ref{sec:f7}.
\begin{equation}
\begin{aligned}
    \left| \|\boldsymbol{q}_+^{(t+1)}\|_2^2 - \|\boldsymbol{q}_+^{(T_1)}\|_2^2 \right| &\leq \sum_{s=T_1}^{t} \left| \|\boldsymbol{q}_+^{(s+1)}\|_2^2 - \|\boldsymbol{q}_+^{(s)}\|_2^2 \right|
\end{aligned}
\end{equation}

\begin{equation}
\begin{aligned}
    &\leq \sum_{s=T_1}^{t} \left| 2\alpha_{+,+}^{(s)} \langle\boldsymbol{q}_+^{(s)},\boldsymbol{k}_+^{(s)} \rangle + 2 \sum_{n \in S_+} \sum_{i=2}^{M} \alpha_{n,+,i}^{(s)} \langle\boldsymbol{q}_+^{(s)},\boldsymbol{k}_{n,i}^{(s)} \rangle \right. \\
    &\quad + \left. ( \alpha_{+,+}^{(s)}\boldsymbol{k}_+^{(s)} + \sum_{n \in S_+} \sum_{i=2}^{M} \alpha_{n,+,i}^{(s)}\boldsymbol{k}_{n,i}^{(s)} ) \cdot ( \alpha_{+,+}^{(s)}\boldsymbol{k}_+^{(s)\top} + \sum_{n \in S_+} \sum_{i=2}^{M} \alpha_{n,+,i}^{(s)}\boldsymbol{k}_{n,i}^{(s)\top} ) \right| \\
    &\leq 2 \sum_{s=T_1}^{t} |\alpha_{+,+}^{(s)}| |\langle\boldsymbol{q}_+^{(s)},\boldsymbol{k}_+^{(s)} \rangle| + 2 \sum_{n \in S_+} \sum_{i=2}^{M} \sum_{s=T_1}^{t} |\alpha_{n,+,i}^{(s)}| |\langle\boldsymbol{q}_+^{(s)},\boldsymbol{k}_{n,i}^{(s)} \rangle| \\
    &\quad + \{\text{lower order term\}} \\
    &= O\left(N^{\frac{1}{2}} d_h^{-\frac{1}{4}}\right) \cdot O(\log(d_h^{\frac{1}{2}})) + N \cdot M \cdot O\left(N^{-\frac{1}{2}} d_h^{-\frac{1}{4}}\right) \cdot O(\log(d_h^{\frac{1}{2}})) \\
    &= O\left(N^{\frac{1}{2}} d_h^{-\frac{1}{4}} \log(d_h^{\frac{1}{2}})\right)
\end{aligned}
\end{equation}

where the first inequality is by triangle inequality. Since $\sigma_h^2 \geq (\max\{\sigma_p^2 d, \|\boldsymbol{\mu}\|_2^2\})^{-1} \cdot d_h^{-\frac{1}{2}} (\log(6N^2M^2/\delta))^{-2}$ and $d_h = \widetilde{\Omega}(\max\{\text{SNR}^4, \text{SNR}^{-4}\}N^2 \epsilon^{-2})$, we have $N^{\frac{1}{2}} d_h^{-\frac{1}{4}} \log(d_h^{\frac{1}{2}}) = o(\|\boldsymbol{\mu}\|_2^2 \sigma_h^2 d_h)$, so $\|\boldsymbol{q}_+^{(t+1)}\|_2^2 = \|\boldsymbol{q}_+^{(T_1)}\|_2^2 + o(\|\boldsymbol{\mu}\|_2^2 \sigma_h^2 d_h) = \Theta(\|\boldsymbol{\mu}\|_2^2 \sigma_h^2 d_h)$. Similarly, we have
\begin{equation}
\begin{aligned}
    \left|\|\boldsymbol{q}_-^{(t+1)}\|_2^2 - \|\boldsymbol{q}_-^{(T_1)}\|_2^2\right| &= O\left(N^{\frac{1}{2}} d_h^{-\frac{1}{4}} \log(d_h^{\frac{1}{2}})\right) = o(\|\boldsymbol{\mu}\|_2^2 \sigma_h^2 d_h), \\
    \left|\|\boldsymbol{k}_+^{(t+1)}\|_2^2 - \|\boldsymbol{k}_+^{(T_1)}\|_2^2\right| &= O\left((1 + \text{SNR}) N^{\frac{1}{2}} d_h^{-\frac{1}{4}} \log(d_h^{\frac{1}{2}})\right) = o(\|\boldsymbol{\mu}\|_2^2 \sigma_h^2 d_h), \\
    \left|\|\boldsymbol{q}_{n,i}^{(t+1)}\|_2^2 - \|\boldsymbol{q}_{n,i}^{(T_1)}\|_2^2 \right|&= O\left(d_h^{-\frac{1}{4}} \log(d_h^{\frac{1}{2}})\right) = o(\sigma_p^2 \sigma_h^2 d_h d_n), \\
    \left|\|\boldsymbol{k}_{n,i}^{(t+1)}\|_2^2 - \|\boldsymbol{k}_{n,i}^{(T_1)}\|_2^2 \right|&= O\left(N^{\frac{1}{2}} d_h^{-\frac{1}{4}} \log(d_h^{\frac{1}{2}})\right) = o(\sigma_p^2 \sigma_h^2 d_h d_n),
\end{aligned}
\end{equation}

so we have
\begin{equation}
\begin{aligned}
    \|q_\pm^{(t+1)}\|_2^2, \|k_\pm^{(t+1)}\|_2^2 &= \Theta(\|\boldsymbol{\mu}\|_2^2 \sigma_h^2 d_h), \\
    \|\boldsymbol{q}_{n,i}^{(t+1)}\|_2^2, \|\boldsymbol{k}_{n,i}^{(t+1)}\|_2^2 &= \Theta(\sigma_p^2 \sigma_h^2 d d_h)
\end{aligned}
\end{equation}
for $i \in [M] \setminus \{1\}, n \in [N]$.

\begin{equation}
\begin{aligned}
    |\langle\boldsymbol{q}_+^{(t+1)},\boldsymbol{q}_-^{(t+1)} \rangle| &\leq |\langle\boldsymbol{q}_+^{(T_1)},\boldsymbol{q}_-^{(T_1)} \rangle| + \sum_{s=T_1}^{t} |\langle\boldsymbol{q}_+^{(s+1)},\boldsymbol{q}_-^{(s+1)} \rangle - \langle\boldsymbol{q}_+^{(s)},\boldsymbol{q}_-^{(s)} \rangle| \\
    &\leq |\langle\boldsymbol{q}_+^{(T_1)},\boldsymbol{q}_-^{(T_1)} \rangle| \\
    &\quad + \sum_{s=T_1}^{t} \left| \alpha_{+,+}^{(s)} \langle\boldsymbol{q}_+^{(s)},\boldsymbol{k}_+^{(s)} \rangle + \sum_{n \in S_+} \sum_{i=2}^{M} \alpha_{n,+,i}^{(s)} \langle\boldsymbol{q}_+^{(s)},\boldsymbol{k}_{n,i}^{(s)} \rangle \right. \\
    &\quad + \left. \alpha_{-,-}^{(s)} \langle\boldsymbol{q}_-^{(s)},\boldsymbol{k}_-^{(s)} \rangle + \sum_{n \in S_-} \sum_{i=2}^{M} \alpha_{n,-,i}^{(s)} \langle\boldsymbol{q}_+^{(s)},\boldsymbol{k}_{n,i}^{(s)} \rangle \right. \\
    &\quad + \left( \alpha_{+,+}^{(s)}\boldsymbol{k}_+^{(s)} + \sum_{n \in S_+} \sum_{i=2}^{M} \alpha_{n,+,i}^{(s)}\boldsymbol{k}_{n,i}^{(s)} \right) \\
    &\quad \cdot \left. 
   \left( \alpha_{-,-}^{(s)}\boldsymbol{k}_-^{(s)} 
   + \sum_{n \in S_-} \sum_{i=2}^{M} \alpha_{n,-,i}^{(s)}\boldsymbol{k}_{n,i}^{(s)} \right)^\top 
   \right|
\end{aligned}
\end{equation}

\begin{equation}
\begin{aligned}
    &\leq |\langle\boldsymbol{q}_+^{(T_1)},\boldsymbol{q}_-^{(T_1)} \rangle| \\
    &\quad + \sum_{s=T_1}^{t} | \alpha_{+,+}^{(s)} ||\langle\boldsymbol{q}_+^{(s)},\boldsymbol{k}_+^{(s)} \rangle| + \sum_{n \in S_+} \sum_{i=2}^{M} \sum_{s=T_1}^t|\alpha_{n,+,i}^{(s)} ||\langle\boldsymbol{q}_+^{(s)},\boldsymbol{k}_{n,i}^{(s)} \rangle|  \\
    &\quad + \sum_{s=T_1}^t | \alpha_{-,-}^{(s)} ||\langle\boldsymbol{q}_-^{(s)},\boldsymbol{k}_-^{(s)} \rangle| + \sum_{n \in S_-} \sum_{i=2}^{M}\sum_{s=T_1}^t |\alpha_{n,-,i}^{(s)} |\langle\boldsymbol{q}_+^{(s)},\boldsymbol{k}_{n,i}^{(s)} \rangle|  \\
    &\quad + \{\text{lower order term\}} \\
    &\leq |\langle\boldsymbol{q}_+^{(T_1)},\boldsymbol{q}_-^{(T_1)} \rangle| \\
    &\quad + O\left(N^{\frac{1}{2}} d_h^{-\frac{1}{4}} \cdot o(1) + N \cdot M \cdot O\left(N^{-\frac{1}{2}} d_h^{-\frac{1}{4}}\right) \cdot \log(d_h^{\frac{1}{2}})\right) \\
    &= |\langle\boldsymbol{q}_+^{(T_1)},\boldsymbol{q}_-^{(T_1)} \rangle| + O\left(N^{\frac{1}{2}} d_h^{-\frac{1}{4}} \log(d_h^{\frac{1}{2}})\right) \\
    &= o(1),
\end{aligned}
\end{equation}
where the first inequality is triangle inequality,  the last equality is by $d_h = \widetilde{\Omega}(\max\{\text{SNR}^4, \text{SNR}^{-4}\}N^2 \epsilon^{-2})$. Similarly, we can prove:
\begin{align*}
\|\boldsymbol{q}_\pm^{(t)}\|_2^2, \|\boldsymbol{k}_\pm^{(t)}\|_2^2 &= \Theta \left( \|\mu\|_2^2 \sigma_h^2 d_h \right), \\
\|\boldsymbol{q}_{n,i}^{(t)}\|_2^2, \|\boldsymbol{k}_{n,i}^{(t)}\|_2^2 &= \Theta \left( \sigma_p^2 \sigma_h^2 d d_h \right), \\
|\langle\boldsymbol{q}_+^{(t)},\boldsymbol{q}_-^{(t)} \rangle|, |\langle\boldsymbol{q}_\pm^{(t)},\boldsymbol{q}_{n,i}^{(t)} \rangle|&, |\langle\boldsymbol{q}_{n,i}^{(t)},\boldsymbol{q}_{n',j}^{(t)} \rangle| = o(1), \\
|\langle\boldsymbol{k}_+^{(t)},\boldsymbol{k}_-^{(t)} \rangle|, |\langle\boldsymbol{k}_\pm^{(t)},\boldsymbol{k}_{n,i}^{(t)} \rangle|&, |\langle\boldsymbol{k}_{n,i}^{(t)},\boldsymbol{k}_{n',j}^{(t)} \rangle| = o(1),
\end{align*}
{ for } $i, j \in [M]\backslash\{1\}, n, n' \in [N], i \neq j \text{ or } n \neq n'$.

\subsection{Upper Bounds of $\langle q, k \rangle$}\label{sec:up_qk9}

In order to give the upper bounds for $\langle\boldsymbol q,\boldsymbol  k \rangle$ in stage III, we need to give the upper bounds of $\alpha$ and $\beta$ based on the equations in ~\ref{sec:cal_ab}. The main difference between this subsection and \ref{sec:f6} is that the bounds of $|V_\pm|, |V_{n,i}|$ is $\log \left( O\left( \frac{1}{\epsilon} \right) \right)$ in this subsection, while the bounds of $|V_\pm|, |V_{n,i}|$ is $\log \left( O\left( \frac{1}{\epsilon} \right) \right)$ in ~\ref{sec:f6}, resulting in different bounds for $\alpha$ and $\beta$. Now we take $\alpha_{+,+}^{(s)}$ as an example
\begin{equation}
\begin{aligned}
\alpha_{+,+}^{(s)} &\leq \frac{\eta}{NM} \sum_{n \in S_+} -\widetilde{\ell}_n'(\theta) \langle \widetilde{\boldsymbol{\mu}}_+,\widetilde{\boldsymbol{\mu}}_+^{(s)} \rangle \\
&\quad \cdot \left( V_{+}^{(s)} \left( \frac{\exp(\langle\boldsymbol{q}_{+}^{(s)},\boldsymbol{k}_{+}^{(s)} \rangle)}{\exp(\langle\boldsymbol{q}_{+}^{(s)},\boldsymbol{k}_{+}^{(s)} \rangle) + \sum_{j=2}^M \exp(\langle\boldsymbol{q}_{+}^{(s)},\boldsymbol{k}_{n,j}^{(s)} \rangle)} \right. \right. \\
&\quad \left. \left. - \left( \frac{\exp(\langle\boldsymbol{q}_{+}^{(s)},\boldsymbol{k}_{+}^{(s)} \rangle)}{\exp(\langle\boldsymbol{q}_{+}^{(s)},\boldsymbol{k}_{+}^{(s)} \rangle) + \sum_{j=2}^M \exp(\langle\boldsymbol{q}_{+}^{(s)},\boldsymbol{k}_{n,j}^{(s)} \rangle)} \right)^2 \right) \right. \\
&\quad \left. - \sum_{i=2}^M \left( V_{n,i}^{(s)} \cdot \frac{\exp(\langle\boldsymbol{q}_{+}^{(s)},\boldsymbol{k}_{+}^{(s)} \rangle)}{\exp(\langle\boldsymbol{q}_{+}^{(s)},\boldsymbol{k}_{+}^{(s)} \rangle) + \sum_{j=2}^M \exp(\langle\boldsymbol{q}_{+}^{(s)},\boldsymbol{k}_{n,j}^{(s)} \rangle)} \right. \right. \\
&\quad \left. \left. \cdot \frac{\exp(\langle\boldsymbol{q}_{+}^{(s)},\boldsymbol{k}_{n,i}^{(s)} \rangle)}{\exp(\langle\boldsymbol{q}_{+}^{(s)},\boldsymbol{k}_{+}^{(s)} \rangle) + \sum_{j=2}^M \exp(\langle\boldsymbol{q}_{+}^{(s)},\boldsymbol{k}_{n,j}^{(s)} \rangle)} \right) \right) \\
&+\sum_{i=2}^M\langle\widetilde{\boldsymbol{\mu}}_+,\widetilde{\boldsymbol{\xi}}_{n,i}^{(s)}\rangle \cdot \left( V_+^{(s)} \left( \frac{\exp(\langle\boldsymbol{q}_{n,i}^{(s)},\boldsymbol{k}_+^{(s)} \rangle)}{\exp(\langle\boldsymbol{q}_{n,i}^{(s)},\boldsymbol{k}_+^{(s)} \rangle) + \sum_{j=2}^{M} \exp(\langle\boldsymbol{q}_{n,i}^{(s)},\boldsymbol{k}_{n,j}^{(s)} \rangle)} \right. \right. \\
& \left. - \left( \frac{\exp(\langle\boldsymbol{q}_{n,i}^{(s)},\boldsymbol{k}_+^{(s)} \rangle)}{\exp(\langle\boldsymbol{q}_{n,i}^{(s)},\boldsymbol{k}_+^{(s)} \rangle) + \sum_{j=2}^{M} \exp(\langle\boldsymbol{q}_{n,i}^{(s)},\boldsymbol{k}_{n,j}^{(s)} \rangle)} \right)^2 \right) \\
& - \sum_{k=2}^{M} \left( V_{n,i}^{(s)} \cdot \frac{\exp(\langle\boldsymbol{q}_{n,i}^{(s)},\boldsymbol{k}_+^{(s)} \rangle)}{\exp(\langle\boldsymbol{q}_{n,i}^{(s)},\boldsymbol{k}_+^{(s)} \rangle) + \sum_{j=2}^{M} \exp(\langle\boldsymbol{q}_{n,i}^{(s)},\boldsymbol{k}_{n,j}^{(s)} \rangle)} \right. \\
& \left. \cdot \frac{\exp(\langle\boldsymbol{q}_{n,i}^{(s)},\boldsymbol{k}_{n,k}^{(s)} \rangle)}{\exp(\langle\boldsymbol{q}_{n,i}^{(s)},\boldsymbol{k}_+^{(s)} \rangle) + \sum_{j=2}^{M} \exp(\langle\boldsymbol{q}_{n,i}^{(s)},\boldsymbol{k}_{n,j}^{(s)} \rangle)} \right)
 \\
    &\leq \frac{\eta}{NM} \sum_{n \in S_+}  (\|\boldsymbol{\mu}\|_2+\tau)^2 (V_+^{(s)} 
    \cdot \frac{\exp(\langle\boldsymbol{q}_+^{(s)},\boldsymbol{k}_+^{(s)} \rangle)}{\exp(\langle\boldsymbol{q}_+^{(s)},\boldsymbol{k}_+^{(s)} \rangle) + \sum_{j=2}^M \exp(\langle\boldsymbol{q}_+^{(s)},\boldsymbol{k}_{n,j}^{(s)} \rangle)} \\
    &\quad + \max_{i=2} |V_{n,i}^{(s)}| \cdot \frac{\sum_{j=2}^{M} \exp(\langle\boldsymbol{q}_+^{(s)},\boldsymbol{k}_{n,j}^{(s)} \rangle)}{\exp(\langle\boldsymbol{q}_+^{(s)},\boldsymbol{k}_+^{(s)} \rangle) + \sum_{j=2}^{M} \exp(\langle\boldsymbol{q}_+^{(s)},\boldsymbol{k}_{n,j}^{(s)} \rangle)}) \\
    &\quad+\sum_{i=2}^M  (\|\boldsymbol{\mu}\|\tau+\sigma_p\tau\sqrt{2\log(4NM/\delta)}+\tau^2) (V_+^{(s)} 
    \cdot \frac{\exp(\langle\boldsymbol{q}_{n,i}^{(s)},\boldsymbol{k}_+^{(s)} \rangle)}{\exp(\langle\boldsymbol{q}_{n,i}^{(s)},\boldsymbol{k}_+^{(s)} \rangle) + \sum_{j=2}^M \exp(\langle\boldsymbol{q}_{n,i}^{(s)},\boldsymbol{k}_{n,j}^{(s)} \rangle)} \\
    &\quad + \max_{i=2} |V_{n,i}^{(s)}| \cdot \frac{\sum_{j=2}^{M} \exp(\langle\boldsymbol{q}_{n,i}^{(s)},\boldsymbol{k}_{n,j}^{(s)} \rangle)}{\exp(\langle\boldsymbol{q}_{n,i}^{(s)},\boldsymbol{k}_+^{(s)} \rangle) + \sum_{j=2}^{M} \exp(\langle\boldsymbol{q}_{n,i}^{(s)},\boldsymbol{k}_{n,j}^{(s)} \rangle)}) \\
    &\leq \frac{\eta}{NM} \cdot \frac{3N}{4} \cdot (\|\boldsymbol{\mu}\|_2+\tau)^2 \cdot \left( V_+^{(s)} \cdot \frac{C}{\exp(\langle\boldsymbol{q}_+^{(s)},\boldsymbol{k}_+^{(s)} \rangle)} + \max_i |V_{n,i}^{(s)}| \cdot \frac{C}{\exp(\langle\boldsymbol{q}_+^{(s)},\boldsymbol{k}_+^{(s)} \rangle)} \right) \\
    &+\sum_{i=2}^M(\|\boldsymbol{\mu}\|\tau+\sigma_p\tau\sqrt{2\log(4NM/\delta)}+\tau^2)\cdot \left( V_+^{(s)} \cdot \frac{C}{\exp(\langle\boldsymbol{q}_{n,i}^{(s)},\boldsymbol{k}_+^{(s)} \rangle)} + \max_i |V_{n,i}^{(s)}| \cdot \frac{C}{\exp(\langle\boldsymbol{q}_{n,i}^{(s)},\boldsymbol{k}_+^{(s)} \rangle)} \right) \\
&\leq \frac{\eta C_9 (\|\boldsymbol{\mu}\|_2+\tau)^2 \log \left( O\left( \frac{1}{\epsilon} \right) \right)}{\exp(\langle\boldsymbol{q}_+^{(s)},\boldsymbol{k}_+^{(s)} \rangle)},
\end{aligned}
\end{equation}
where the first inequality is by $-\widetilde{\ell}_n^{\prime(s)} \leq 1$ and $\text{softmax}(\langle\boldsymbol{q}_+^{(s)},\boldsymbol{k}_+^{(s)} \rangle) \leq 1$. For the second inequality, we first consider
$\frac{\sum_{j=2}^{M} \exp(\langle\boldsymbol{q}_+^{(s)},\boldsymbol{k}_{n,j}^{(s)} \rangle)}{\exp(\langle\boldsymbol{q}_+^{(s)},\boldsymbol{k}_+^{(s)} \rangle) + \sum_{j=2}^{M} \exp(\langle\boldsymbol{q}_+^{(s)},\boldsymbol{k}_{n,j}^{(s)} \rangle)} \leq \frac{\sum_{j=2}^{M} \exp(\langle\boldsymbol{q}_+^{(s)},\boldsymbol{k}_{n,j}^{(s)} \rangle)}{\exp(\langle\boldsymbol{q}_+^{(s)},\boldsymbol{k}_+^{(s)} \rangle)}$,
then by the monotonicity of $\langle\boldsymbol{q}_+^{(s)},\boldsymbol{k}_{n,j}^{(s)} \rangle$ and $\langle\boldsymbol{q}_+^{(T_1)},\boldsymbol{k}_{n,j}^{(T_1)} \rangle = o(1)$ we have $\sum_{j=2}^{M} \exp(\langle\boldsymbol{q}_+^{(s)},\boldsymbol{k}_{n,j}^{(s)} \rangle) \leq C$ for $t \in [T_1, T_3]$. The last inequality is by $V_+^{(s)}, |V_{n,i}^{(s)}| \leq 2 \log \left( O\left( \frac{1}{\epsilon} \right) \right)$ for $t \in [T_2, T_3]$ and absorbing the constant factors. Similar to ~\ref{sec:f6}, we can give the bounds for the other $\alpha$ and $\beta$ as follows:
\begin{equation}
\begin{aligned}
\alpha_{-,-}^{(s)} &\leq \frac{\eta C_9 (\|\boldsymbol{\mu}\|_2+\tau)^2 \log \left( O\left( \frac{1}{\epsilon} \right) \right)}{\exp(\langle\boldsymbol{q}_-^{(s)},\boldsymbol{k}_-^{(s)} \rangle)}, \\
\beta_{+,+}^{(s)} &\leq \frac{\eta C_9 (\|\boldsymbol{\mu}\|_2-\tau)^2 \log \left( O\left(\frac{1}{\epsilon} \right) \right)}{\exp(\langle\boldsymbol{q}_+^{(s)},\boldsymbol{k}_+^{(s)} \rangle)},
\end{aligned}
\end{equation}

Similar to~\ref{sec:low_qk}, we apply the bounds of $\alpha$ and $\beta$ above to give the upper bounds for the dynamics $\langle q, k \rangle$.

\begin{equation}
\begin{aligned}
    &\langle\boldsymbol{q}_+^{(s+1)},\boldsymbol{k}_+^{(s+1)} \rangle - \langle\boldsymbol{q}_+^{(s)},\boldsymbol{k}_+^{(s)} \rangle \\
    &= \alpha_{+,+}^{(s)} \|\boldsymbol{k}_+^{(s)}\|_2^2 + \sum_{n \in S_+} \sum_{i=2}^{M} \alpha_{n,+,i}^{(s)} \langle\boldsymbol{k}_+^{(s)},\boldsymbol{k}_{n,i}^{(s)} \rangle \\
    &\quad + \beta_{+,+}^{(s)} \|\boldsymbol{q}_+^{(s)}\|_2^2 + \sum_{n \in S_+} \sum_{i=2}^{M} \beta_{n,+,i}^{(s)} \langle\boldsymbol{q}_+^{(s)},\boldsymbol{q}_{n,i}^{(s)} \rangle \\
    &\quad + \left( \alpha_{+,+}^{(s)}\boldsymbol{k}_+^{(s)} + \sum_{n \in S_+} \sum_{i=2}^{M} \alpha_{n,+,i}^{(s)}\boldsymbol{k}_{n,i}^{(s)} \right) \\
    &\quad \cdot \left( \beta_{+,+}^{(s)}\boldsymbol{q}_+^{(s)\top} + \sum_{n \in S_+} \sum_{i=2}^{M} \beta_{n,+,i}^{(s)}\boldsymbol{q}_{n,i}^{(s)\top} \right) \\
    &= \alpha_{+,+}^{(s)} \|\boldsymbol{k}_+^{(s)}\|_2^2 + \beta_{+,+}^{(s)} \|\boldsymbol{q}_+^{(s)}\|_2^2 + \{\text{lower order term}\} \\
    &\leq \frac{2 \eta C_9 (\|\boldsymbol{\mu}\|_2+\tau)^2 \log \left( O\left( \frac{1}{\epsilon} \right) \right)}{\exp(\langle\boldsymbol{q}_+^{(s)},\boldsymbol{k}_+^{(s)} \rangle)} \cdot \Theta(\|\boldsymbol{\mu}\|_2^2 \sigma_h^2 d_h) + \{\text{lower order term}\} \\
    &\leq \eta C_{10} \|\boldsymbol{\mu}\|_2^2(\|\boldsymbol{\mu}\|_2+\tau) \sigma_h^2 d_h \log \left( O\left( \frac{1}{\epsilon} \right) \right) \frac{1}{\exp(\langle\boldsymbol{q}_+^{(s)},\boldsymbol{k}_+^{(s)} \rangle)},
\end{aligned}
\end{equation}
Similarly, we have the upper bounds for the dynamics of other $\langle\boldsymbol q,\boldsymbol k \rangle$.

\subsection{Bounds for the Sum of $\alpha$ and $\beta$}\label{sec:f10}
Assume that the propositions $\mathcal{F}(T_2), \ldots, \mathcal{F}(s), \mathcal{H}(T_2), \ldots, \mathcal{H}(s-1)$ hold $(s \in [T_1, t])$, we have
\begin{equation}
    |V_\pm^{(s)}| \leq 2 \log \left( O\left( \frac{1}{\epsilon} \right) \right), 
\end{equation}
\begin{equation}
    |V_{n,i}^{(s)}| = O(1), \tag{279}
\end{equation}
\begin{equation}
    \Lambda_{n,\pm,j}^{(s)} \geq \Lambda_{n,\pm,j}^{(T_2)} \geq \log \left( \exp(\Lambda_{n,\pm,j}^{(T_1)}) + \Theta\left( \frac{d_h^{\frac{1}{2}}}{N (\log(6N^2M^2/\delta))^3} \right) \right), 
\end{equation}
\begin{equation}
    \Lambda_{n,i,\pm,j}^{(s)} \geq \Lambda_{n,i,\pm,j}^{(T_2)} \geq \log \left( \exp(\Lambda_{n,i,\pm,j}^{(T_1)}) + \Theta\left( \frac{\sigma_p^2 d d d_h^{\frac{1}{2}}}{N \|\boldsymbol{\mu}\|_2^2 (\log(6N^2M^2/\delta))^3} \right) \right) 
\end{equation}
for $i, j \in [M] \backslash \{1\}, n \in [N], s \in [T_2, t]$. Similar to (\ref{eq:64}) and (\ref{eq:65}), we have
\begin{equation}
    \frac{\exp(\langle q_\pm^{(s)},\boldsymbol{k}_{n,j}^{(s)} \rangle)}{\exp(\langle q_\pm^{(s)}, k_\pm^{(s)} \rangle) + \sum_{j'=2}^{M} \exp(\langle q_\pm^{(s)},\boldsymbol{k}_{n,j'}^{(s)} \rangle)} = O\left( \frac{N (\log(6N^2M^2/\delta))^3}{d_h^{\frac{1}{2}}} \right) 
\end{equation}
\begin{equation}
    \frac{\exp(\langle\boldsymbol{q}_{n,i}^{(s)},\boldsymbol{k}_{n,j}^{(s)} \rangle)}{\exp(\langle\boldsymbol{q}_{n,i}^{(s)},\boldsymbol{k}_+^{(s)} \rangle) + \sum_{j'=2}^{M} \exp(\langle\boldsymbol{q}_{n,i}^{(s)},\boldsymbol{k}_{n,j'}^{(s)} \rangle)} = O\left( \frac{N \|\boldsymbol{\mu}\|_2^2 (\log(6N^2M^2/\delta))^3}{\sigma_p^2 d d_h^{\frac{1}{2}}} \right) 
\end{equation}
Plugging above equations into the expressions of $\alpha, \beta$ and letting $O\left( \log \left( O\left( \frac{1}{\epsilon} \right) \right) \right)$ be the upper bound for $|V_\pm^{(s)}|, |V_{n,i}^{(s)}|$ we have
\begin{equation}
\begin{aligned}
\alpha_{+,+}^{(s)} &\leq \frac{\eta}{NM} \sum_{n \in S_+} -\widetilde{\ell}_n'(\theta) \langle \widetilde{\boldsymbol{\mu}}_+,\widetilde{\boldsymbol{\mu}}_+^{(s)} \rangle \\
&\quad \cdot \left( V_{+}^{(s)} \left( \frac{\exp(\langle\boldsymbol{q}_{+}^{(s)},\boldsymbol{k}_{+}^{(s)} \rangle)}{\exp(\langle\boldsymbol{q}_{+}^{(s)},\boldsymbol{k}_{+}^{(s)} \rangle) + \sum_{j=2}^M \exp(\langle\boldsymbol{q}_{+}^{(s)},\boldsymbol{k}_{n,j}^{(s)} \rangle)} \right. \right. \\
&\quad \left. \left. - \left( \frac{\exp(\langle\boldsymbol{q}_{+}^{(s)},\boldsymbol{k}_{+}^{(s)} \rangle)}{\exp(\langle\boldsymbol{q}_{+}^{(s)},\boldsymbol{k}_{+}^{(s)} \rangle) + \sum_{j=2}^M \exp(\langle\boldsymbol{q}_{+}^{(s)},\boldsymbol{k}_{n,j}^{(s)} \rangle)} \right)^2 \right) \right. \\
&\quad \left. - \sum_{i=2}^M \left( V_{n,i}^{(s)} \cdot \frac{\exp(\langle\boldsymbol{q}_{+}^{(s)},\boldsymbol{k}_{+}^{(s)} \rangle)}{\exp(\langle\boldsymbol{q}_{+}^{(s)},\boldsymbol{k}_{+}^{(s)} \rangle) + \sum_{j=2}^M \exp(\langle\boldsymbol{q}_{+}^{(s)},\boldsymbol{k}_{n,j}^{(s)} \rangle)} \right. \right. \\
&\quad \left. \left. \cdot \frac{\exp(\langle\boldsymbol{q}_{+}^{(s)},\boldsymbol{k}_{n,i}^{(s)} \rangle)}{\exp(\langle\boldsymbol{q}_{+}^{(s)},\boldsymbol{k}_{+}^{(s)} \rangle) + \sum_{j=2}^M \exp(\langle\boldsymbol{q}_{+}^{(s)},\boldsymbol{k}_{n,j}^{(s)} \rangle)} \right) \right) \\
&+\sum_{i=2}^M\langle\widetilde{\boldsymbol{\mu}}_+,\widetilde{\boldsymbol{\xi}}_{n,i}^{(s)}\rangle \cdot \left( V_+^{(s)} \left( \frac{\exp(\langle\boldsymbol{q}_{n,i}^{(s)},\boldsymbol{k}_+^{(s)} \rangle)}{\exp(\langle\boldsymbol{q}_{n,i}^{(s)},\boldsymbol{k}_+^{(s)} \rangle) + \sum_{j=2}^{M} \exp(\langle\boldsymbol{q}_{n,i}^{(s)},\boldsymbol{k}_{n,j}^{(s)} \rangle)} \right. \right. \\
& \left. - \left( \frac{\exp(\langle\boldsymbol{q}_{n,i}^{(s)},\boldsymbol{k}_+^{(s)} \rangle)}{\exp(\langle\boldsymbol{q}_{n,i}^{(s)},\boldsymbol{k}_+^{(s)} \rangle) + \sum_{j=2}^{M} \exp(\langle\boldsymbol{q}_{n,i}^{(s)},\boldsymbol{k}_{n,j}^{(s)} \rangle)} \right)^2 \right) \\
& - \sum_{k=2}^{M} \left( V_{n,i}^{(s)} \cdot \frac{\exp(\langle\boldsymbol{q}_{n,i}^{(s)},\boldsymbol{k}_+^{(s)} \rangle)}{\exp(\langle\boldsymbol{q}_{n,i}^{(s)},\boldsymbol{k}_+^{(s)} \rangle) + \sum_{j=2}^{M} \exp(\langle\boldsymbol{q}_{n,i}^{(s)},\boldsymbol{k}_{n,j}^{(s)} \rangle)} \right. \\
& \left. \cdot \frac{\exp(\langle\boldsymbol{q}_{n,i}^{(s)},\boldsymbol{k}_{n,k}^{(s)} \rangle)}{\exp(\langle\boldsymbol{q}_{n,i}^{(s)},\boldsymbol{k}_+^{(s)} \rangle) + \sum_{j=2}^{M} \exp(\langle\boldsymbol{q}_{n,i}^{(s)},\boldsymbol{k}_{n,j}^{(s)} \rangle)} \right)
 \\
&\leq \frac{\eta (\|\boldsymbol{\mu}\|_2+\tau)^2}{NM} \cdot \frac{3N}{4} \cdot O\left(\log \left(O\left(\frac{1}{\epsilon}\right)\right)\right) \cdot O\left(\frac{N \left(\log(6N^2M^2/\delta)\right)^3}{d_h^{\frac{3}{2}}}\right) \\
    &= O\left(\frac{\eta N (\|\boldsymbol{\mu}\|_2+\tau)^2 \left(\log(6N^2M^2/\delta)\right)^3 \log \left(O\left(\frac{1}{\epsilon}\right)\right)}{d_h^{\frac{3}{2}}}\right)
\end{aligned}
\end{equation}
Taking a summation we obtain that
\begin{equation}
\begin{aligned}
    \sum_{s=T_2}^{t} |\alpha_{+,+}^{(s)}| &= O\left(\frac{1}{\eta \epsilon \|\boldsymbol{\mu}\|_2^2 \|\boldsymbol{w}_O\|_2^2}\right) \cdot O\left(\eta \|\boldsymbol{\mu}\|_2^2 N (\log(6N^2M^2/\delta))^3 \log\left(O\left(\frac{1}{\epsilon}\right)\right) \frac{\log\left(O\left(\frac{1}{\epsilon}\right)\right)}{d_h^{\frac{3}{2}}}\right) \\
    &= O\left(\frac{N (\log(6N^2M^2/\delta))^3 \log\left(O\left(\frac{1}{\epsilon}\right)\right)}{\epsilon d_h^{\frac{3}{2}}}\right),
\end{aligned}
\end{equation}
where the last equality is by $\|\boldsymbol{w}_O\| = \Theta(1)$. Similarly, we have the bounds for other sum of $\alpha$ and $\beta$. Thus, we can easily prove Claim 7 and Claim 8.
\subsection{Proof of Claim 7}\label{sec:pf_cl7}

In this subsection, we assume that $\mathcal{I}(T_2), \ldots, \mathcal{I}(t)$ hold, and then proof that $\mathcal{G}(t+1)$ is true with the result of ~\ref{sec:f10}.
\begin{equation}
\begin{aligned}
    \left\|\boldsymbol{q}_+^{(t+1)} \right\|_2^2 - \left\|\boldsymbol{q}_+^{(t)} \right\|_2^2 &\leq \sum_{s=T_2}^{t} \left| \left\|\boldsymbol{q}_+^{(s+1)} \right\|_2^2 - \left\|\boldsymbol{q}_+^{(s)} \right\|_2^2 \right| \\
    &\leq \sum_{s=T_2}^{t} \left| 2\alpha_{+,+}^{(s)} \langle\boldsymbol{q}_+^{(s)},\boldsymbol{k}_+^{(s)} \rangle + 2 \sum_{n \in S_+} \sum_{i=2}^{M} \alpha_{n,+,i}^{(s)} \langle\boldsymbol{q}_+^{(s)},\boldsymbol{k}_{n,i}^{(s)} \rangle \right. \\
    &\quad + \left. \langle \alpha_{+,+}^{(s)}\boldsymbol{k}_+^{(s)} + \sum_{n \in S_+} \sum_{i=2}^{M} \alpha_{n,+,i}^{(s)}\boldsymbol{k}_{n,i}^{(s)} \rangle \cdot \left( \alpha_{+,+}^{(s)}\boldsymbol{k}_+^{(s)} + \sum_{n \in S_+} \sum_{i=2}^{M} \alpha_{n,+,i}^{(s)}\boldsymbol{k}_{n,i}^{(s)} \right)^\top \right| \\
    &\leq 2 \sum_{s=T_2}^{t} \left| \alpha_{+,+}^{(s)} \langle\boldsymbol{q}_+^{(s)},\boldsymbol{k}_+^{(s)} \rangle \right| + 2 \sum_{n \in S_+} \sum_{i=2}^{M} \sum_{s=T_2}^{t} \left| \alpha_{n,+,i}^{(s)} \langle\boldsymbol{q}_+^{(s)},\boldsymbol{k}_{n,i}^{(s)} \rangle \right| \\
    &\quad + \{\text{lower order term\}} \\
    &= O\left( \frac{N (\log(6N^2M^2/\delta))^3 \log\left(O\left(\frac{1}{\epsilon}\right)\right)}{\epsilon d_h^{\frac{3}{2}}} \right) \cdot \log\left(\epsilon^{-1} d_h^{\frac{1}{2}}\right) \\
    &\quad + N \cdot M \cdot O\left( \frac{(\log(6N^2M^2/\delta))^3 \log\left(O\left(\frac{1}{\epsilon}\right)\right)}{\epsilon d_h^{\frac{3}{2}}} \right) \cdot \log\left(\epsilon^{-1} d_h^{\frac{1}{2}}\right) \\
    &= O\left( \frac{N (\log(6N^2M^2/\delta))^3 \log\left(O\left(\frac{1}{\epsilon}\right)\right) \log\left(\epsilon^{-1} d_h^{\frac{1}{2}}\right)}{\epsilon d_h^{\frac{3}{2}}} \right)
\end{aligned}
\end{equation}
where the first inequality is by triangle inequality, the second inequality is by the update rules in~\ref{sec:uprule}, the third inequality is by $t \leq T_3$. Since $\sigma_h^2 \geq \left(\max\{\sigma_p^2 d, \|\boldsymbol{\mu}\|_2^2\}\right)^{-1} \cdot d_h^{-\frac{1}{2}} (\log(6N^2M^2/\delta))^{-2}$ and $d_h = \widetilde{\Omega}\left(\max\{\text{SNR}^4, \text{SNR}^{-4}\}N^2 \epsilon^{-2}\right)$, we have
\begin{equation}
\frac{N (\log(6N^2M^2/\delta))^3 \log\left(O\left(\frac{1}{\epsilon}\right)\right) \log(\epsilon^{-1} d_h^{\frac{1}{2}})}{\epsilon d_h^{\frac{1}{2}}} = o(\|\boldsymbol{\mu}\|_2^2 \sigma_h^2 d_h),
\end{equation}
so $\|\boldsymbol{q}_+^{(t+1)}\|_2^2 = \|\boldsymbol{q}_+^{(T_2)}\|_2^2 + o(\|\boldsymbol{\mu}\|_2^2 \sigma_h^2 d_h) = \Theta(\|\boldsymbol{\mu}\|_2^2 \sigma_h^2 d_h)$. Similarly, we have 
\begin{align*}
        \|\boldsymbol q_\pm^{(t)}\|_2^2, \|\boldsymbol k_\pm^{(t)}\|_2^2 &= \Theta(\|\boldsymbol{\mu}\|_2^2 \sigma_h^2 d_h), \\
        \|\boldsymbol{q}_{n,i}^{(t)}\|_2^2, \|\boldsymbol{k}_{n,i}^{(t)}\|_2^2 &= \Theta\left(\sigma_p^2 \sigma_h^2 d d_h\right), \\
        |\langle\boldsymbol{q}_+^{(t)},\boldsymbol{q}_-^{(t)} \rangle|, |\langle\boldsymbol q_\pm^{(t)},\boldsymbol{q}_{n,i}^{(t)} \rangle|, &|\langle\boldsymbol{q}_{n,i}^{(t)},\boldsymbol{q}_{n',j}^{(t)} \rangle| = o(1), \\
        |\langle\boldsymbol{k}_+^{(t)},\boldsymbol{k}_-^{(t)} \rangle|, |\langle\boldsymbol k_\pm^{(t)},\boldsymbol{k}_{n,i}^{(t)} \rangle|, &|\langle\boldsymbol{k}_{n,i}^{(t)},\boldsymbol{k}_{n',j}^{(t)} \rangle| = o(1)
    \end{align*}
    for $i, j \in [M]\backslash\{1\}, n, n' \in [N], i \neq j \text{ or } n \neq n'$.

\end{document}